\documentclass[lettersize,journal]{IEEEtran}
\usepackage{amsmath,amsfonts, mathrsfs}
\usepackage{algorithmic}
\usepackage{algorithm}
\usepackage{array}
\usepackage[caption=false,font=normalsize,labelfont=sf,textfont=sf]{subfig}
\usepackage{stfloats}
\usepackage{url}
\usepackage{verbatim,comment}
\usepackage{graphicx}
\usepackage{cite}
\usepackage{amssymb, amsthm}

\newtheorem{defn}{Definition}
\newtheorem{thm}[defn]{Theorem}
\newtheorem{lem}[defn]{Lemma}
\newtheorem{cor}[defn]{Corollary}
\newtheorem{pro}[defn]{Proposition}

\newtheorem{assum}[defn]{Assumption}
\newcommand{\bs}{\boldsymbol}
\newcommand{\huber}{\ell_{\textnormal{H}}}
\newcommand{\hinge}{\ell_{0}}
\newcommand{\whL}{\widehat L}
\newcommand{\whF}{\widehat F}
\newcommand{\x}{\mathbf{x}}
\newcommand{\R}{\mathbb{R}}

\newcommand{\E}{\mathbb{E}}
\newcommand\norm[1]{\left\lVert#1\right\rVert}

\newcommand{\mathbbm}[1]{\text{\usefont{U}{bbm}{m}{n}#1}}
\DeclareMathOperator*{\argmin}{arg\,min}

\begin{document}

\title{A Differentially Private Weighted Empirical Risk Minimization Procedure and its Application to Outcome Weighted Learning}

\author{Spencer Giddens,
Yiwang Zhou, 
Kevin R. Krull,
Tara M. Brinkman, 
Peter X. K. Song,
Fang Liu 
       
\thanks{Manuscript received xxx; revised xxxx. \emph{(Spencer Giddens and Yiwang Zhou are co-first authors; Corresponding author: Fang Liu (fliu2@nd.edu)).}}
\thanks{S.~Giddens and F.~Liu are at the University of Notre Dame, Notre Dame, IN 46556, USA. Y.~Zhou K.~R.~Krull, T.~M.~Brinkman ae at St. Jude Children's Research Hospital, Memphis, TN 38138, USA. P.~X.K.~ Song is at the University of Michigan, Ann Arbor, MI 48109, USA.
This work was supported by the Schmitt Fellowship, Lucy Graduate Scholarship and Technology Ethics Center Research Assistantship to Giddens at Notre Dame, and was partially supported by an NSF grant DMS2113564 to Song.
Zhou, Krull, and Brinkman were supported in their work at St. Jude Children’s Research Hospital by the American Lebanese Syrian Associated Charities. 
The authors acknowledge the helpful insights from several discussions with Drs. Jie Ding and Ganghua Wang and the suggestions of three reviewers and an Associate Editor that helped improve the quality of the manuscript.}
\thanks{This paper has supplementary downloadable material available at xxxx.
The material includes proofs of theoretical results, details on the hyperparameter tuning method, and additional experiment results.
Contact fliu2@nd.edu for further questions about this work.}
}

\markboth{IEEE Transactions on Information Forensics and Security,~Vol.~xx, No.~x, xx~20xx}%
{Shell \MakeLowercase{\textit{et al.}}: A Sample Article Using IEEEtran.cls for IEEE Journals}

% TODO: Include this info?

\maketitle

\begin{abstract}
Data used to train predictive models via empirical risk minimization (ERM) often contain sensitive personal information. While differential privacy (DP) provides mathematically provable bounds to protect such data, previous work has focused almost exclusively on unweighted ERM. We consider weighted ERM (wERM) -- an important generalization where individual contributions to the objective function vary. We propose the first DP algorithm for general wERM with formal privacy guarantees and derive both its empirical and population excess risk bounds. Crucially, this general wERM framework provides a pathway for deriving privacy-preserving learning methods for individualized treatment rules, including the popular outcome-weighted learning (OWL) approach.
We evaluate DP-wERM  applied to OWL in simulated and real data experiments.
Our empirical results demonstrate that training OWL models via wERM provides strong DP guarantees while maintaining robust performance, proving the method is practical for sensitive, real-world data.  
\end{abstract}\vspace{-6pt}

\section{Introduction}\label{sec:intro}
\IEEEPARstart{O}{utcome} weighted learning (OWL) \cite{Zhao2012} is an effective machine learning (ML) technique for delivering individualized medical or business decisions.
For example, in a two-armed randomized trial, where participants are randomly assigned to one of two treatments, a ``benefit'' metric is measured to assess the effectiveness of the treatment for each participant.
The OWL loss function represents the benefit-weighted average shortfall between the randomly assigned treatment and predicted labels, minimization of which will lead to a classifier that can be used to predict the optimal treatment with the highest expected benefit for an individual. 

The current literature on OWL assumes that data scientists have unrestricted access to training data.
This assumption proves impractical in real-world applications, where, for example, sensitive health information exists in data sources such as randomized clinical trials, electronic health records, and national surveys, necessitating the protection of data privacy for ethical and legal purposes when training OWL models. 

Although our work is primarily motivated by the OWL framework, we emphasize that the primary novelty in our contribution is a framework for obtaining DP guarantees in the more general weighted empirical risk minimization (wERM) setting, which we refer to as \emph{differentially private weighted empirical risk minimization (DP-wERM)}.
It can be shown that DP-OWL is a special case of DP-wERM.
By applying the DP-wERM framework to train DP-OWL models, we aim to provide a practical solution for training OWL models on sensitive data, facilitating information sharing with mathematical privacy guarantees.

\vspace{-6pt}\subsection{Background}\label{subsec:background}
Individualized treatment rules (ITRs) are fundamental in precision medicine \cite{national2011toward, collins2015new} -- an innovative framework for maximizing the clinical benefit by taking into account individual heterogeneity to tailor treatments for subgroups of patients \cite{Qian2011} instead of relying on a one-size-fits-all paradigm.
Besides precision medicine, ITRs are also applicable to personalized advertising \cite{Wang2015, Sun2015} and recommender systems \cite{Schnabel2016, Lada2019}. 

Various methods have been proposed for ITRs.
Among them, OWL is a common framework for constructing ITRs by directly optimizing population-level expected outcomes.
\cite{Zhao2012} showed that estimating the optimal ITR is equivalent to a classification problem of treatment groups, where participants in different groups are weighted proportionally to their observed clinical benefits. 
OWL has also been extended to search for individualized continuous doses \cite{chen2016personalized}, derive ITRs for multiple treatment options \cite{zhou2018outcome}, and establish dynamic treatment regimes \cite{eguchi2022outcome}. 
\cite{Zhou2017} proposed residual weighted learning (RWL) to improve the finite sample performance of OWL by weighting individual participants with the residuals of their benefits obtained by a regression model and solving the optimization problem using a difference of convex algorithm.
Matched learning (M-learning) \cite{Wu2020} further improves the OWL-based framework by utilizing matching instead of inverse probability weighting to balance participants in different treatments, making the algorithm more applicable for the analysis of observational studies (e.g., large-scale electronic health records). 
We focus on OWL with formal privacy guarantees in this work and will explore incorporating privacy guarantees in other types of ITR frameworks in future work. 

Privacy literature has shown that ad-hoc anonymization of sensitive datasets, even if done carefully, can be insufficient and attackers can still leverage anonymized data to infer masked or removed values \cite{Narayanan2008, Sweeney2015, Ahn2015}.
Aggregated statistics are also vulnerable to privacy attacks \cite{Desfontain2021}. Even black-box access to trained ML models permits inferring membership in the training data \cite{Shokri2017} and approximating attribute values \cite{Zhao2021}. 
Differential privacy (DP) \cite{Dwork2006} is a state-of-the-art privacy notion that provides a mathematically rigorous framework for ensuring the privacy of sensitive datasets without making any assumptions on the would-be attacker's methods or auxiliary knowledge.
In the DP framework, carefully calibrated random noise is injected through a randomized mechanism into statistics or function outputs derived from sensitive data, lowering the probability of learning personal information used to generate the outputs.

\vspace{-6pt}\subsection{Related Work}\label{subsec:related_work}
\cite{Chaudhuri2008} analyzed the application of DP to releasing logistic regression coefficients in the framework of ERM.
\cite{chaudhuri2011} generalized DP logistic regression and developed DP counterparts to binary classification problems in a general unweighted ERM framework.
Subsequent works have incrementally improved DP-ERM analysis, such as by tightening bounds on excess risk or making the methods more computationally efficient \cite{Bassily2014, Kasiviswanathan2016}.
\cite{Kifer2012} extended the framework in \cite{chaudhuri2011} to the approximate DP framework with a broader set of regularizers. 
 
In work completed concurrently to, but independently from ours, \cite{Spicker2024} developed an alternative method for differentially private OWL.
Their method is an extension of a DP support vector machine (SVM) approach \cite{Rubenstein2012}, which is less general than the DP-ERM approach \cite{chaudhuri2011} in the sense that it is specific to the SVM framework and not directly generalizable to other ERM problems. 
Comparatively, our DP-wERM approach directly considers DP in the general wERM framework, of which weighted SVM is a special case, and is therefore more general in this way.
To our knowledge, no other privacy-preserving counterparts except for \cite{Spicker2024} have been developed for the models constructing ITRs.
Although some recent works estimate the individual treatment effects themselves in a DP-satisfying manner \cite{Betlei2021, Niu2022}, they focus on outcome predictions while our goal is to learn a treatment assignment rule with DP guarantees.

%----------------------------------------
\vspace{-6pt} \subsection{Our Contributions} 
First, motivated by the need to train OWL models on sensitive data, we formulate a general wERM framework, of which OWL is a special case. wERM assigns distinct individual weights within the loss function, naturally generalizing unweighted ERM setups where all observations carry equal weight. To our knowledge, our procedure represents the first DP-wERM extension of the foundational unweighted DP-ERM frameworks in \cite{Chaudhuri2008, chaudhuri2011, Kifer2012}.

Second, we prove that our proposed wERM procedure satisfies DP under common regularity conditions, and we derive its excess empirical and population risks as well as the sample complexity for the latter case. We also provide an algorithm for hyperparameter tuning that can utilize either a portion of the sensitive training data or an independent proxy dataset. Our sensitivity analysis demonstrates the robustness of this tuning algorithm, even when the independent dataset deviates from the sensitive dataset in key statistical aspects.

Third, we apply  the proposed DP-wERM to OWL in experiments on synthetic and real data and show that the trained DP-OWL models can generate ITRs that are comparable to those produced by non-private OWL methods, with similar empirical treatment value estimates.

%-----------------------------------------
\vspace{-6pt}\section{Preliminaries}
\label{sec:preliminaries}

In this section, we formalize the notions of wERM, OWL, and DP.
By default, an unspecified norm is the $\ell_2$ norm (i.e., $\|\cdot\| = \|\cdot\|_2$) and boldfaced symbols are vector-valued.

\vspace{-6pt}\subsection{Weighted Empirical Risk Minimization (wERM)}
\label{subsec:wERM}
Let $(\mathbf{x}_i, y_i) \in (\mathcal{X},\mathcal{Y})$ represent the $p$ features in observation $i$ and the corresponding outcome, respectively, for $i=1,\ldots,n$.
Let $\ell: \mathcal{Y} \!\times\! \mathcal{Y} \rightarrow \mathbb{R}$ denote a non-negative loss function over the label space $\mathcal{Y}$.
The goal in ERM is to determine a predictor function $f$ that minimizes the loss function over the data. 
Linear regression, logistic regression, and SVM are all examples of ERM.
We assume $f$ is parameterized by $\boldsymbol{\theta}$. 
It is also common to employ a regularizer $R$ that penalizes the complexity of $f_{\boldsymbol{\theta}}$.
The solution to a regularized unweighted ERM is 
\begin{align}
\argmin_{f_{\boldsymbol{\theta}}\in\mathcal{F}}\frac{1}{n} \bigg\{\sum_{i=1}^n \ell(f_{\boldsymbol{\theta}}(\mathbf{x}_i), y_i) + \gamma R(\boldsymbol{\theta})\bigg\},\label{eqn:ERM}
%   =& \argmin_{f_{\boldsymbol{\theta}}\in\mathcal{F}}\big\{\textstyle \frac{1}{n}\sum_{i=1}^n \ell_i(\boldsymbol{\theta}) + \frac{\gamma}{n}R(\boldsymbol{\theta})\big\}, 
\end{align}
where $\gamma>0$ is a tunable regularization constant and $\mathcal{F}$ is an appropriate predictor function space that governs specific instances of ERM.

In some applications, each of the $n$ loss terms $\ell(f_{\boldsymbol{\theta}}(\mathbf{x}_i), y_i)$ in Eqn.~\eqref{eqn:ERM} may be weighted differently, leading to weighted ERM (wERM).
\begin{defn}[Weighted empirical risk minimization (wERM)]
\label{def:wERM}
Define $(\mathbf{x}_i, y_i)$, $\ell$, $R$, $\gamma$, and $\mathcal{F}$ as above. 
Let $w_i$ denote the weight associated with individual loss $\ell_i(\boldsymbol{\theta}) = \ell(f_{\boldsymbol{\theta}}(\mathbf{x}_i), y_i)$.
A wERM problem is defined as
\begin{align}
\argmin_{f_{\boldsymbol{\theta}}\in\mathcal{F}}\frac{1}{n}\bigg\{\sum_{i=1}^n w_i\ell(f_{\boldsymbol{\theta}}(\mathbf{x}_i), y_i) + {\gamma}R(\boldsymbol{\theta})\bigg\}\label{eqn:wERM}
%    =&\argmin_{f_{\boldsymbol{\theta}}\in\mathcal{F}}\big\{\textstyle\frac{1}{n}\sum_{i=1}^n w_i \ell_i(\boldsymbol{\theta}) + \frac{\gamma}{n}R(\boldsymbol{\theta})\big\}.
\end{align}
\end{defn}
\noindent An example of wERM is OWL, as introduced next.

%---------------------------------
\vspace{-6pt}\subsection{Outcome Weighted Learning (OWL)}
\label{subsec:OWL}
Consider a dataset $D=\{(\mathbf{x}_i, y_i, B_i)\,;\,i=1,\dots,n\}$ collected from either a randomized clinical trial or an observational study with a total of $n$ participants. $\mathbf{x}=(x_1,\dots,x_p)^\intercal\in\mathcal{X}\subseteq\mathbb{R}^p$ contains features used for ITR derivation.
Without loss of generality (WLOG), we examine a binary treatment $Y\in\mathcal{Y}=\{-1,1\}$, assigned at the beginning of the study with probability $P(Y|\mathbf{x})$.
$B$ is observed  benefit of the treatments.

In a randomized trial with treatment allocation ratio 1:1, $P(Y=1|\mathbf{x})\!=\!P(Y=-1|\mathbf{x})\!=\!0.5$. 
In an observational study, treatment assignment depends on the features $\mathbf{x}$, and $P(Y|\mathbf{x})$ can be estimated using methods such as logistic regression.
Note that the observed $Y$ for a subject does not necessarily represent the underlying optimal treatment assignment that is most beneficial for that subject in the ITR setting. 
In other words, 
the goal of ITR is not to estimate a treatment assignment function that predicts the treatment an individual actually receives, but rather to leverage the observed benefit $B$ to learn a function that will assign individuals to their optimal treatment assignment in the future.

OWL \cite{Zhao2012} is derived in a seminal work enabling the estimation of an optimal ITR $T^*$ that maximizes the expected clinical benefit $\mathbb{E}\left[\frac{B}{P(Y|\mathbf{x})}\mathbbm{1}(Y=T(\mathbf{x}))\right]$, where $\mathbbm{1}(\cdot)$ is the indicator function and $P(Y|\mathbf{x})$ is the propensity score -- the probability of receiving treatment $Y$ given $\mathbf{x}$.
WLOG, we assume that the larger the value of $B$, the greater the treatment benefit.
The maximization problem above is equivalent to  $T^*\in \argmin_{T\in\mathcal{T}} \mathbb{E}\left[\frac{B}{P(Y|\mathbf{x})}\mathbbm{1}(Y\neq T(\mathbf{x}))\right]$, where $\frac{B}{P(Y|\mathbf{x})}\mathbbm{1}(Y\neq T(\mathbf{x}))$ is a weighted classification error, implying that 
OWL is essentially weighted classification.
Let $D=\{(\mathbf{x}_i, y_i, B_i)\,;\,i=1,\dots,n\}$ denote a set of observations and let $T(\mathbf{x})=1$ if $f_{\boldsymbol{\theta}}(\mathbf{x})>0$ (for some predictor function $f_{\boldsymbol{\theta}}$ parameterized by $\boldsymbol{\theta}$) and $T(\mathbf{x})=-1$ otherwise.
Then the optimization problem becomes the ERM problem
\begin{equation*}
    \textstyle f^*\in \argmin_{f_{\boldsymbol{\theta}}\in\mathcal{F}} \frac{1}{n} \sum_{i=1}^n \left\{\frac{B_i}{P(y_i|\mathbf{x}_i)}\mathbbm{1}(y_i\neq \text{sign}(f_{\boldsymbol{\theta}}(\mathbf{x}_i)))\right\}.
\end{equation*}
Since the weighted 0-1 loss is non-convex and discontinuous, a common strategy is to replace it with a convex hinge loss \cite{chen2016personalized}. 
Together with an $\ell_2$ regularizer often used to penalize the complexity of  $f_{\boldsymbol{\theta}}$, the final OWL optimization problem is
\begin{align}
&\argmin_{f_{\boldsymbol{\theta}}\in\mathcal{F}}\frac{1}{n} \left\{\sum_{i=1}^n\!w_i\text{max}\big(0,1\!-\!y_i f_{\boldsymbol{\theta}}(\mathbf{x}_i)\big)+{\gamma}\norm{\boldsymbol{\theta}} \!\right\},\label{eqn:OWL}\\
&\mbox{where }w_i={B_i}/{P(y_i|\mathbf{x}_i)}. \label{eqn:w} 
\end{align}
It is straightforward to verify that Eqn.~\eqref{eqn:OWL} is a special case of wERM in Definition \ref{def:wERM} with 
$ \ell_i(\boldsymbol{\theta})=\text{max}(0,1-y_i f_{\boldsymbol{\theta}}(\mathbf{x}_i))$. 
%\end{align}

%-------------------------------------------------
\vspace{-9pt}\subsection{Differential Privacy (DP)}
\label{subsec:DP}
We use the DP framework to achieve privacy guarantees when releasing results from wERM.
The DP framework ensures that, regardless of the inclusion or exclusion of any single individual's data, the results of a DP-satisfying mechanism acting on a sensitive dataset are sufficiently similar.

\begin{defn}[Neighboring datasets \cite{Dwork2006}]
\label{def:neighbor} Two datasets $D, \tilde{D} \in \mathcal{D}$ are considered \textnormal{neighboring datasets}, denoted by  $d(D, \tilde{D})=1$, if $D$ can be obtained from $\tilde{D}$ by modifying a single individual's data.\footnote{either deletion/removal or substitution/replacement of an individual.}
\end{defn}
\begin{defn}[$(\epsilon, \delta)$-DP \cite{Dwork2006b}]
\label{def:DP} A randomized mechanism $\mathcal{M}$ satisfies \textnormal{$(\epsilon, \delta)$-DP} if for all $S \subset \textnormal{Range}(\mathcal{M})$ and $d(D, \tilde{D})=1$,
    \begin{equation}\label{eqn:DP}
        P(\mathcal{M}(D) \in S) \le e^\epsilon P(\mathcal{M}(\tilde{D}) \in S) + \delta,
    \end{equation}
    where $\epsilon>0$ and $\delta\in[0,1)$ are privacy loss or privacy budget parameters.
    When $\delta=0$, it becomes \textnormal{$\epsilon$-DP}.
\end{defn}
Per Eqn.~\eqref{eqn:DP}, the privacy of an individual in the dataset is protected since the output from the DP mechanism remains similar whether their data are included or not. $\epsilon$ and $\delta$ control the similarity of the output distributions on the neighboring datasets.  Smaller $\epsilon$ and $\delta$ imply more privacy.
Generally, $\epsilon\le1$ is considered to provide strong privacy guarantees, though it is not unusual to see larger values of $\epsilon$ used in practice.
$\delta$ can be interpreted  as the probability that $\epsilon$-DP fails.
The value $\delta$ is often set to be inversely proportional to a polynomial of sample size $n$.

DP possesses several properties that facilitate its implementation in practice.
The immunity to post-processing property, for example, ensures that the output of a DP mechanism cannot be further manipulated to weaken the DP guarantees.
Additionally, in  situations where multiple statistics or outputs are released from a sensitive dataset, the quantification of the overall privacy loss over the application of multiple DP mechanisms to the  data is known as privacy loss composition and has been well-studied \cite{Dwork2006b, mcsherry2007mechanism, Dwork2010, Dong2022}. 

A common way to achieve DP guarantees for information release is to add noise, appropriately calibrated to a desired output, to the output before release.
The scale of the added noise relates to the global sensitivity of the output.
Global sensitivity (GS) is originally defined using the $\ell_1$ norm by \cite{Dwork2006}.
In this paper, we use a more general definition.
\begin{defn}[$\ell_p$-GS \cite{Liu2019}]
    Let $\mathbf{s}$ be an output calculated from a dataset.
    The \textnormal{$\ell_p$-GS} of $\mathbf{s}$ is 
    \begin{equation}
        \Delta_{p,\mathbf{s}} = \max_{d(D,\tilde{D}) = 1}\big\|\mathbf{s}(D) - \mathbf{s}(\tilde{D})\big\|_p,
    \end{equation}
    where $\norm{\cdot}_p$ is the $\ell_p$ norm for $p>0$.
\end{defn}
In general, the larger the GS of an output, the larger the noise scale of a DP mechanism $\mathcal{M}$ is to achieve privacy guarantees at pre-specified $(\epsilon, \delta)$.
For example, the Laplace mechanism satisfies $\epsilon$-DP \cite{Dwork2006} by adding noise drawn from a Laplace distribution with mean 0 and scale $\Delta_{1, \mathbf{s}}/\epsilon$ to output $\mathbf{s}$.
As a trade-off, DP mechanisms generally incur a cost in data utility; and the   trade-off between privacy protection and data utility depends on $(\epsilon, \delta)$, among other factors such as sample size, the DP mechanism itself, etc.

%-----------------------------------------------------
\vspace{-3pt}\section{Differentially Private wERM} 
\label{sec:methods}
Prior to presenting the DP version for the wERM problem in Definition \ref{def:wERM}, we first define some regularity conditions in Assumption \ref{assump}.
Except for the newly introduced assumption that weight $w_i\in(0,W]\,\forall i$, the assumptions  are the same as for the classic unweighted DP-ERM \cite{chaudhuri2011}.
%-----------------------------------------------------
\vspace{-9pt}\subsection{Regularity Conditions}
\label{subsec:regularity_conditions}
\begin{assum}\label{assump}
The regularity conditions for DP-wERM are: 
\begin{enumerate}
\item[{A1.}] $\norm{\mathbf{x}_i}\le 1$ and $w_i\in(0,W]$ for all $i$; 
\item[{A2.}] The set of predictor functions are linear predictors\footnote{We focus on linear predictor function in this work. Nonlinear cases in unweighted  ERM problems (e.g., the radial kernel for SVM) can be approximated with linear predictor functions  (e.g., \cite{chaudhuri2011} approximates the radial kernel for SVM via random projections \cite{Rahimi2007, Rahimi2008} in the DP-ERM framework); the same applies to the wERM problems. 
A predicted label $\hat{y}$ for a given $\mathbf{x}$ can be obtained from the estimated $\hat{f}=f_{\hat{\boldsymbol{\theta}}}(\mathbf{x})= \mathbf{x}^\intercal\hat{\boldsymbol{\theta}}$, where $\hat{\boldsymbol{\theta}}$ is an optimal solution to the wERM problem, such as sign($\hat{f}$) for $y\in\{1,-1\}$.} $\mathcal{F} = \{f_{\boldsymbol{\theta}}\,:\,f_{\boldsymbol{\theta}}(\mathbf{x}) = \mathbf{x}^\intercal\boldsymbol{\theta}, \boldsymbol{\theta} \in \mathbb{R}^p\}$; 
        \item[{A3.}] Loss function $\ell$ is convex and everywhere first-order differentiable;
        \item[{A4.}] For any observed label $y\in\{-1,1\}$ and  predicted $\hat{f} = f_{\hat{\boldsymbol{\theta}}}(\mathbf{x})$, 
        the loss $\ell(\hat{f}, y)$ must be expressible as a function of the product $z=y\hat{f}$; in other words, $\ell(\hat{f}, y) = \tilde{\ell}(z)$ for some function $\tilde{\ell}$; additionally, $|\tilde{\ell}^\prime(z)| \le 1\;\forall\; z$;
        \item[{A5.}] The regularizer $R$ is $1$-strongly 
        convex\footnote{Technically, $R$ could be $\Lambda$-strongly convex. Let $g: \mathbb{R}^p \rightarrow \mathbb{R}$. If, for all $\alpha \in (0, 1)$ and $\mathbf{x}, \mathbf{y}\in\mathbb{R}^p$,
$ \textstyle g(\alpha\mathbf{x} + (1-\alpha)\mathbf{y})\!\le \alpha g(\mathbf{x}) + (1-\alpha)g(\mathbf{y}) - \frac{1}{2}\Lambda\alpha(1-\alpha)\norm{\mathbf{x}\!-\!\mathbf{y}}^2\!,$
then $g$ is \textnormal{$\Lambda$-strongly convex} for $\Lambda>0$. This distinction between $\Lambda$ vs. 1-strong convexity does not impact our results; we assume 1-strong convexity WLOG to simplify the proof of Theorem \ref{thm:wERM_sens}.} and everywhere first-order differentiable.
\end{enumerate}
\end{assum}
Some of the conditions (e.g., $\norm{\mathbf{x}_i}\le1$) can be satisfied by minor data processing.
While we focus on predictor functions $f_{\hat{\boldsymbol{\theta}}}(\mathbf{x})$ without an explicit bias (intercept) term, an implicit bias term can still be used by augmenting $\mathbf{x}$ with a constant term.
A3 to A5 can be achieved by weighted versions of common classification methods with minor modifications (if any).
For example, the binary cross-entropy loss in logistic regression automatically satisfies A3 and A4;
the hinge loss function commonly used for SVM can be easily smoothed to an approximation that also satisfies A3 and A4; 
the popular $\ell_2$-norm regularizer satisfies A5.

\vspace{-6pt}\subsection{DP-wERM Algorithm} \label{subsec:DP-wERM}
Under Assumption~\ref{assump}, we adapt the DP-ERM output perturbation framework from \cite{Chaudhuri2008, chaudhuri2011} to implement the DP-wERM procedure with $\epsilon$-DP guarantees and list the steps in Algorithm~\ref{alg:DP-wERM}.  The output predictor $f^*$ in Line 8 satisfies $\epsilon$-DP by sanitizing the parameter estimate $\hat{\boldsymbol{\theta}}$ with DP guarantees (specifically, the DP noise generated in Lines 5 to 7 is equivalent to drawing a random sample from the spherical Laplace distribution $f(\mathbf{z}^*)\propto \exp(-\frac{\epsilon}{\Delta}\norm{\mathbf{z}^*})$ \cite{Chaudhuri2008}).

Compared to the unweighted DP-ERM algorithm in \cite{Chaudhuri2008, chaudhuri2011}, Algorithm~\ref{alg:DP-wERM} shares similar steps, but differs in key technical details. Specifically, the objective function in Line 3 is updated to Eqn.~\eqref{eqn:wERM},  and the $\ell_2$-GS of the estimator,  $(C+2W)/\gamma$,  in Line 4 is substantially different  from that of the unweighted ERM (see Theorem \ref{thm:wERM_sens}).
Subsequently, the weighting in wERM alters the theoretical utility guarantees, which is carefully analyzed in Section \ref{subsec:utility_analysis}.

\begin{algorithm}[!htb]
\caption{The DP-wERM procedure}\label{alg:DP-wERM}
\begin{algorithmic}[1]
    \STATE \textbf{Input}: Privacy budget $\epsilon>0$; dataset $D = (\mathbf{x}_i, y_i, w_i) \in \mathbb{R}^p \!\times\! \{-1, 1\} \!\times\! \mathbb{R}$ of feature-label-weight triples with $\norm{\mathbf{x}_i}\le 1$ and weights $w_i\in(0,W]$ for all $i$; constant $C$ (see Theorem \ref{thm:wERM_sens});  $f_{\boldsymbol{\theta}}(\mathbf{x}) = \mathbf{x}^\intercal\boldsymbol{\theta}$; loss $\ell$;   regularizer $R$ with regularization constant $\gamma>0$. 
    \STATE \textbf{Output}: Privacy-preserving predictor $f^*$.
    \STATE Set $\hat{\boldsymbol{\theta}} \gets \argmin_{\boldsymbol{\theta} \in \mathbb{R}^p} \frac{1}{n}\{\sum_{i=1}^n w_i\ell(f_{\boldsymbol{\theta}}(\mathbf{x}_i), y_i) + \gamma R(\boldsymbol{\theta})\}$ \COMMENT{\texttt{\small{wERM without DP}}}
    \STATE Set $\Delta \gets (C+2W)/\gamma$ \COMMENT{\texttt{\small{$\ell_2$-GS of $\hat{\boldsymbol{\theta}}$}}}
    \STATE Sample $\zeta$ from Gamma(shape $=p$, rate$=\epsilon/\Delta$)
    \STATE Sample $\mathbf{z}$ from $\mathcal{N}(\mathbf{0},I_p)$
    \STATE Set $\hat{\boldsymbol{\theta}}^* \gets \hat{\boldsymbol{\theta}} + \zeta\frac{\mathbf{z}}{\norm{\mathbf{z}}}$ \COMMENT{\texttt{\small{$\epsilon$-DP mechanism}}}
    \STATE \textbf{return} $f^*=\mathbf{x}^\intercal\hat{\boldsymbol{\theta}}^*$ for  any $\mathbf{x}\in \mathbb{R}^p$ with $\norm{\mathbf{x}}\le 1$
\end{algorithmic}
\end{algorithm}

The dimension $p$  of $\boldsymbol{\theta}$ often increases with the number features in $\mathbf{x}$.
The larger $p$ is, the larger the scale of DP noise injected into $\hat{\boldsymbol{\theta}}$ must be for a pre-specified privacy budget in Algorithm \ref{alg:DP-wERM}, leading to noisier $\hat{\boldsymbol{\theta}}^*$ and $f^*$.
It is generally good practice to perform  feature engineering (e.g.~variable selection or dimensionality reduction) to reduce $p$. 
If this step is guided by prior or domain knowledge that is independent of the private data $D$, then it incurs no additional privacy cost. Otherwise, it should either be performed using a DP procedure on $D$ with a small portion of the overall privacy budget or divide $D$ into two non-overlapping portions and use one portion for feature engineering without incurring additional privacy loss \cite{Kifer2012, Thakurta2013, Steinke2017, Chaudhuri2013}. 
This step may be treated as an internal pre-processing step; its output needs not to be released along with $f^*$ from Algorithm~\ref{alg:DP-wERM} although  the data curator choose to disclose it to make the procedure less of a ``black box''.

Algorithm~\ref{alg:DP-wERM} involves a hyperparameter $\gamma$. While non-DP hyperparameters are typically tuned via cross-validation on the training data, this approach cannot be directly applied under DP without incurring additional privacy loss. The DP literature often treats hyperparameter selection as a secondary concern, implicitly assuming they  are known \textit{a priori} or derived from independent data sources \cite{chaudhuri2011, Kifer2012, abadi2016deep}.\footnote{\cite{chaudhuri2011} provide a DP hyperparameter tuning method but do not account for its privacy loss against the sensitive dataset, effectively assuming an independent proxy dataset is available; \cite{Kifer2012} does not address hyperparameter selection; and \cite{abadi2016deep} omits discussion of how deep neural network hyperparameters are selected with DP guarantees.} When prior knowledge is unavailable, the sensitive dataset itself must be utilized by allocating either (i) a portion of the total privacy budget or (ii) a portion of the dataset for tuning. The budget-allocation approach permits using the entire dataset for both tuning and training under smaller individual budgets, whereas the data-splitting approach maintains the target privacy budget but trains the model on a smaller data subset. The hyperparameter tuning method in \cite{chaudhuri2011} follows this latter data-splitting approach. 
Algorithm~2 in the supplemental materials (SM) describes an $(n, \epsilon)$-adaptive approach to tuning $\gamma$ in Algorithm~\ref{alg:DP-wERM} using either a portion of the sensitive dataset or an independent dataset (the latter is adopted in the experiments in Section \ref{sec:sim_studies}). 
A more in-depth discussion of our approach to hyperparameter tuning, including when it is realistic to use an independent dataset for hyperparameter tuning, can be found in SM Section \ref{subsec:hyperparameter_tuning}.

%--------------------------------------
\vspace{-6pt}\subsection{DP guarantees of  DP-wERM algorithm}
\label{sec:DP_guarantees}
The proof that Algorithm~\ref{alg:DP-wERM} satisfies $\epsilon$-DP depends on strong convexity. It is straightforward to show that: (i) if a function $g$ is $\Lambda$-strongly convex, then $ag$ is $a\Lambda$-strongly convex for any $a\in\mathbb{R}$; (ii) if $f$ is convex and $g$ is $\Lambda$-strongly convex, then $f+g$ is $\Lambda$-strongly convex. Based on these, we re-state a lemma originally presented in \cite{chaudhuri2011} for the sake of self-containment; readers may refer to its proof in \cite{chaudhuri2011}.
\begin{lem}[\cite{chaudhuri2011}]
    \label{lem:Chaudhuri_lemma}
    Let $h_1: \mathbb{R}^p \rightarrow \mathbb{R}$ and $ h_2: \mathbb{R}^p \rightarrow \mathbb{R}$ be everywhere first-order differentiable functions.
    Assume $h_1(\boldsymbol{\theta})$ and $h_1(\boldsymbol{\theta}) + h_2(\boldsymbol{\theta})$ are both $\Lambda$-strongly convex.
    If $\hat{\boldsymbol{\theta}}_1 = \argmin_{\boldsymbol{\theta}}h_1(\boldsymbol{\theta})$ and $\hat{\boldsymbol{\theta}}_2 = \argmin_{\boldsymbol{\theta}}\{h_1(\boldsymbol{\theta}) + h_2(\boldsymbol{\theta})\}$, then
    \begin{equation}
        \textstyle\norm{\hat{\boldsymbol{\theta}}_1 - \hat{\boldsymbol{\theta}}_2} \le \frac{1}{\Lambda}\max_{\boldsymbol{\theta}}\norm{\nabla h_2(\boldsymbol{\theta})}.
    \end{equation}
\end{lem}

Given Assumption~\ref{assump} and Lemma \ref{lem:Chaudhuri_lemma}, we are now equipped to derive the $\ell_2$-GS of $\hat{\boldsymbol{\theta}}$ in the wERM optimization and establish the main result (Theorem \ref{thm:wERM_sens}).
Compared to ERM  \cite{chaudhuri2011}, the establishment of DP guarantees for wERM is more complex as the weight $w_i$ in Eqn.~\eqref{eqn:w} affects the GS of the solution to the wERM problem, as presented in the proof of Theorem \ref{thm:wERM_sens} (see Section VI of the SM). 
\begin{thm}[Main result]\label{thm:wERM_sens}
Let $\mathcal{L}$ denote a wERM problem given data $D$ in Eqn.~\eqref{eqn:wERM} satisfying the regularity conditions in Assumption \ref{assump}. Then 

(i) the $\ell_2$-GS of estimated parameters $\hat{\boldsymbol{\theta}}$ in the predictor function $f_{\boldsymbol{\theta}}$ is $\Delta_{2, \boldsymbol{\theta}} = (C+2W)/\gamma$, where constant $C=$
\begin{align}\label{eqn:C}
\begin{cases}
    0 & \mbox{if $\mathbf{w}$ is observed or predicted via an }\\[-3pt]
    & \mbox{externally trained model;}\\
%    kr^{-\frac{1}{2}}\sigma_d  
    H/n
    & \mbox{if $\mathbf{w}$ is predicted via a model trained}\\[-3pt]
    & \mbox{on $D$ with a large $n$; constant $H$ is}\\[-3pt]
    &\mbox{defined in Eqn.~\eqref{eqn:H} below;}\\
    nW & \mbox{otherwise.}
\end{cases}
\end{align}\vspace{-6pt}
\begin{equation}\label{eqn:H}
H :=\nabla_{\bs{\beta}} w_i(\bs{\beta}_0)^\top\! G^{-1}\!\left(\!\psi(\tilde{d}_*,\bs{\beta}_0)\!-\! \psi(d_*,\bs{\beta}_0)\!\right),
\end{equation} 
where $\bs\beta_0$ contains the true parameters of the model used to estimate $P(Y|\mathbf{x})$, $\sum_{j=i}^n \psi(d_i,\bs\beta)\!=\!0$ is the estimating equation, $G\!=\!\mathbb E\left[\frac{\partial}{\partial\bs{\beta}}\psi(d,\bs{\beta}_0)\right]$ ($d$ is an arbitrary data point),  and $(d_*,\tilde{d}_*)$ is the  observation pair that differ between two neighboring datasets (* can be any $d_i$ for $i\in\{1,\ldots,n\}$).

(ii) Algorithm~\ref{alg:DP-wERM} satisfies $\epsilon$-DP.
\end{thm}
Among the three options for $C$ in Eqn.~\eqref{eqn:C}, the  bound of $nW$ is   the most conservative and unlikely to be actived in practice. For the second scenario, $H/n\rightarrow 0$ as $n\rightarrow\infty$. Therefore, when $n$ is sufficiently large, $H/n <\!\!< 2W$,  effectively equivalent to setting $C=0$, which is the first option. If one nevertheless wishes to use $H/n$ as is, one may estimate $H$ from $D$, which would involves a DP procedure with an additional privacy loss, or an data-independent upper bound can be used. 
Specifically, assume 
\begin{align}\label{eqn:Hbound}
&\|\nabla_{\bs\beta} w_i(\bs{\beta}_0)\|_2\le C_w,\,
\|G^{-1}\|_{\mathrm{op}}\le C_G, \notag\\
\mbox{and } & \|\psi(d,\bs{\beta}_0)\|_2\le C_\psi \mbox{ for all $d$}, \notag\\ 
\mbox{then }&H \le 2C_w C_G C_\psi.
\end{align}
The bounds $C_w,C_G$, and $C_\psi$ are problem- and model- specific given that they depend on the estimating function $\psi$. For logistic regression, a commonly used model for obtaining propensity scores,  the result on $H$ is stated in Proposition \ref{pro:logistic}; the detailed deviation is provided in Section VI of the SM. 
\begin{pro}[Instantiation of Eqn.~\eqref{eqn:Hbound} in Logistic regression]\label{pro:logistic}
In logistic regression with features $\mathbf{x}$, value $B$, and label $Y\in\{-1,1\}$, define
$p_{\bs\beta}(\mathbf{x}):=(1+e^{-\mathbf{x}^\top \bs\beta})^{-1}$, then
\begin{align*}
\psi(d,\bs\beta) &=\mathbf{x}\bigl((Y+1)/2-p_{\bs\beta}(\mathbf{x})\bigr);\,
\|\psi(d,\bs\beta)\|\le\|\mathbf{x}\|.\\
\partial_{\bs\beta} \psi(d,\bs\beta) & = -p_{\bs\beta}(\mathbf{x})\bigl(1-p_{\bs\beta}(\mathbf{x})\bigr)\mathbf{x}\mathbf{x}^\top,\\
G&=-\mathbb{E}\!\left[p_{\bs\beta_0}(\mathbf{x})\bigl(1-p_{\bs\beta_0}(\mathbf{x})\bigr)\mathbf{x}\mathbf{x}^\top\right].
\end{align*}
Assume $B\in[0,C_B]$, $\|\mathbf{x}\|\le C_x$, 
$\mathbb{E}[\mathbf{x}\mathbf{x}^\top]\succeq b_x I$ for some  $b_x>0$, $|\mathbf{x}^\top \bs\beta_0|\le M\,\forall\,\mathbf{x}$ thus $p_{\bs{\beta}_0}(\mathbf{x})\ge 1/(1+e^M)$ (positivity). 
Then $C_\psi = C_x, C_w= C_B C_xe^M, C_G= (1+e^M)^2/(b_xe^M)$ and 
\begin{align}
H = 2C_\psi C_w C_G &= 2C_B C^2_x(1+e^M)^2/b_x.\label{eqn:H12}
\end{align}
\end{pro}
In practice, $C_B$ is typically known \textit{a priori} or fixed per domain knowledge; $C_x$ is pre-specified during data pre-processing; similarly, $M$ from lower-bounding $P(Y=1|\mathbf{x})$ at a constant $c>0$, leading to $e^M=c^{-1}-1$. If we use $C_x\!=\!1$ per Assumption 5.A1, Eqn.~\eqref{eqn:H12} simplifies to $2C_B(1+e^M)^2/b_x$.

%----------------------------------------------
\vspace{-3pt}\subsection{Application of DP-wERM to OWL} \label{subsec:DP-OWL}
Algorithm \ref{alg:DP-wERM} can be applied to train a privacy-preserving OWL model given that OWL is a special case of wERM, achieving DP guarantees by operating on the primal weighted SVM problem.
This contrasts with the typical approach of optimizing the dual problem leveraging the popular ``kernel trick'' for a nonlinear basis (e.g., Gaussian or radial kernels and polynomial kernels). This is because the dual formulation presents privacy vulnerabilities by exposing individual training samples that serve as support vectors. Nevertheless, alternative strategies exist; for instance, \cite{Rubenstein2012} solves the dual problem directly with DP, bypassing the typical privacy limitations of the kernel trick.

Since the hinge loss function in Eqn.~\eqref{eqn:OWL} is not smooth, we approximate it with the Huber loss \cite{Chapelle2007} so that A3 and A4 in Assumption \ref{assump} are satisfied.  Given $h>0$,  the \textnormal{Huber loss} $\ell_{\textnormal{H}}: \mathbb{R} \rightarrow \mathbb{R}$ is 
    \begin{equation}\label{eqn:huber}
        \ell_{\textnormal{H}}(z) = \begin{cases} 
        0, & \textnormal{if } z > 1+h \\
        \frac{1}{4h}(1+h-z)^2, & \textnormal{if } |1 - z| \le h \\
        1 - z, & \textnormal{if } z < 1 - h \\
        \end{cases},
    \end{equation}
where $z=y\hat{f}$ in the wERM formulation. $\ell_{\textnormal{H}}$ is convex and everywhere first-order differentiable and $|\ell^\prime_{\textnormal{H}}(z)| \le 1$. 

Corollary \ref{cor:DP-OWL} states Algorithm \ref{alg:DP-wERM} satisfies $\epsilon$-DP for OWL with the Huber loss.
The result is a direct extension of Theorem \ref{thm:wERM_sens} as all regularity conditions in Assumption \ref{assump} are satisfied. 
\begin{cor}[\textbf{DP guarantees of OWL via the DP-wERM algorithm}]\label{cor:DP-OWL}
Consider the  OWL problem
\begin{equation}\label{eqn:approx_OWL}
    \argmin_{f_{\boldsymbol{\theta}}\in\mathcal{F}} \frac{1}{n}\left\{\sum_{i=1}^n\frac{B_i}{P(y_i|\mathbf{x}_i)} \ell_{\textnormal{H}}(y_if_{\boldsymbol{\theta}}(\mathbf{x}_i)) + \gamma\norm{\boldsymbol{\theta}}\right\},
\end{equation}
where $\norm{\mathbf{x}_i}\!\le\!1$, $\mathcal{F}\!= \!\{f_{\boldsymbol{\theta}}:f_{\boldsymbol{\theta}}(\mathbf{x}) \!=\!\mathbf{x}^\intercal\boldsymbol{\theta}, \boldsymbol{\theta} \in \mathbb{R}^p\}$, and $\frac{B_i}{P(y_i|\mathbf{x}_i)}=w_i\in(0,W]$.
The $\ell_2$-GS of estimate $\hat{\boldsymbol{\theta}}$ is $(C+2W)/\gamma$, where $C\ge 0$ is a constant as defined in Theorem \ref{thm:wERM_sens}, and solving the OWL problem via Algorithm \ref{alg:DP-wERM} satisfies $\epsilon$-DP.
\end{cor} \vspace{-3pt}
Eqn.~\eqref{eqn:C} in Theorem \ref{thm:wERM_sens} remains applicable for determining $C$ in Corollary \ref{cor:DP-OWL}. But since $B_i$ is often observed in OWL, $C$ is mainly driven by how $P(y_i|\mathbf{x}_i)$ is obtained.

\vspace{-9pt}\subsection{Theoretical Utility Analysis}\label{subsec:utility_analysis}
We derive theoretical upper bounds on empirical and population excess risks for the loss evaluated at DP estimate $\bs\theta^*$ via DP-wERM Algorithm \ref{alg:DP-wERM}. The proofs are provided in SM Sections VIII and IX.
Define the empirical and population risks corresponding to the unregularized hinge loss $\hinge(z)=(1-z)_+$ and Huber Loss $\huber$ in Eqn.~\eqref{eqn:huber} by \vspace{-3pt}
\begin{align}
\!\!\!\textstyle \whL_0(\bs{\theta})\!=\!\frac1n\sum_{i=1}^n w_i\hinge(z_i(\bs{\theta})),\;
&L_0(\bs{\theta})\!=\!\E\bigl[w\hinge(A\x^\top\bs{\theta})\bigr],\notag\\
\!\!\!\textstyle \whL_{\textnormal{H}}(\bs{\theta})\!=\!\frac1n\sum_{i=1}^n w_i\huber(z_i(\bs{\theta})),\;
& L_{\textnormal{H}}(\bs{\theta}) \!=\!\E\bigl[w\huber(A\x^\top\bs{\theta})\bigr],\notag
\end{align}
respectively. 

Denote the empirical wERM Huber loss and its minimizers by
$\whF_{\textnormal{H}}(\bs{\theta})
=\whL_{\textnormal{H}}(\bs{\theta})+\frac{\gamma}{n}\norm{\bs{\theta}}^2$  
$\widehat{\bs{\theta}}\in \arg\min_{\bs{\theta}\in\R^p}\whF_{\textnormal{H}}(\bs{\theta})$; and the DP minimizer via  Algorithm~\ref{alg:DP-wERM} is $\hat{\bs{\theta}}^*$.
\begin{thm}[Empirical excess risk bound]\label{thm:empirical-main}
Define $ a_p(\rho)\triangleq p+\sqrt{2p\log(1/\rho)}+\log(1/\rho)$. Fix any $\bs{\theta}_0\in\R^p$. With probability at least $1-\rho$ over the distribution of DP  noise,\vspace{-3pt}
\begin{align}
&\whL_0(\hat{\bs{\theta}}^*)\!-\!\whL_0(\bs{\theta}_0)\!
\le
\frac{\lambda}{2}\norm{\bs{\theta}_0}^2
\!+\!\mathscr{E}(n,\lambda,h,\epsilon,\rho)
\!+\!\frac{Wh}{4} \mbox{ with}\!\!\notag\\
&\lambda=\frac{2\gamma}{n};\; 
\mathcal{E}(n,\lambda,h,\epsilon,\rho)
=\frac{\lambda+W/(2h)}{2}\left(\!\frac{(C\!+\!2W)a_p(\rho)}{n\lambda\epsilon}\right)^{\!2}\!.\notag
\end{align}
\end{thm}
Theorem \ref{thm:empirical-main} shows that the empirical error bound  at  $\hat{\bs{\theta}}^*$ is decomposed into three terms: $T_1= \lambda\norm{\bs{\theta}_0}^2/2$ incurred by the $\ell_2$ regularization,  $T_2=\mathcal{E}(n,\lambda,h,\epsilon,\rho)$ due to DP perturbation, and $T_3=Wh/4$ the Huber-to-hinge approximation error. 

\begin{thm}[Population excess risk bound]
\label{thm:population-generic}
Fix $\bs{\theta}_0\in\R^p$. Suppose, with probability at least $1-\delta_0$ over the sample distribution and the DP noise distribution,
$\big|L_{\textnormal{H}}(\bs{\theta})-\whL_{\textnormal{H}}(\bs{\theta})\big| \le \mathcal G_{n,h}(\delta_0)$ 
for $\bs{\theta}\in\{\bs{\theta}_0,\hat{\bs{\theta}}^*\}$. Then, with probability at least $1-\delta_0-\rho$ over the sample and privacy noise distribution,
\vspace{-3pt}
\begin{align*}
L_0(\hat{\bs{\theta}}^*\!)\!-\!L_0(\bs{\theta}_0)
\le\!
\frac{\lambda}{2}\norm{\bs{\theta}_0}^2
\!\!+\!\mathscr{E}(n,\lambda,h,\epsilon,\rho)
\!+\!\frac{W\!h}{4}
\!+\!2\mathcal G_{n,h}(\delta_0).
\end{align*}
\end{thm}
Compared to Theorem~\ref{thm:empirical-main}, the population excess risk has an extra term $T_4=2\mathcal G_{n,h}(\delta_0)$,  the generalization error taking on a generic form.
Corollary \ref{cor:sc} provides an instantiation with a specific form for $G_{n,h}(\delta_0)$, along with the corresponding sample complexity, in a similar manner as in \cite{Chaudhuri2008}.
\begin{cor}[An instantiation of Theorem \ref{thm:population-generic}]\label{cor:sc}
Under the same strongly convex regularized-ERM generalization conditions in \cite{Chaudhuri2008, chaudhuri2011} \footnote{Assumption \ref{assump}, plus i.i.d. on $d_i$, $|\ell''_{\textnormal{H}}(z)|\le (2h)^{-1}$, and $\norm{\bs{\theta}_0}\le\infty$.}, with high probability\vspace{-3pt}
\begin{align*} 
L_0(\hat{\bs{\theta}}^*)\!-\!L_0(\bs{\theta}_0) 
\lesssim &\,
\frac{\lambda}{2}\norm{\bs{\theta}_0}^2
\!+\!\frac{(\lambda+\frac{W}{2h})(C\!+\!2W)^2p^2\log^2(\frac{p}{\delta})}{\lambda^2 n^2\epsilon^2}\\
&+\frac{Wh}{4} +\frac{W^2}{\lambda n}.
\end{align*}
For a given target excess risk $\alpha>0$, choose
$\lambda\asymp \frac{\alpha}{\norm{\bs{\theta}_0}^2}$
and  $h\le c\frac{\alpha}{W}$ 
for a sufficiently small constant $c>0$, then $L_0(\hat{\bs{\theta}}^*) \le L_0(\bs{\theta}_0)+\alpha$ with high probability if
\begin{align*}
 n\gtrsim
\max\Bigg\{&
\frac{W^2\norm{\bs{\theta}_0}^2}{\alpha^2},
\frac{(C+2W) p\log(p/\delta)\norm{\bs{\theta}_0}}{\epsilon\alpha},\\
&\frac{(C+2W) p\log(p/\delta)\sqrt{W/h}\,\norm{\bs{\theta}_0}^2}{\epsilon\alpha^{3/2}}
\Bigg\}.
\end{align*}
\end{cor}
The results in Theorems \ref{thm:empirical-main}, \ref{thm:population-generic} and Corollary \ref{cor:sc} also confirm our empirical findings that the same $\gamma=n\lambda/2$ not only has different regularization strength $T_1$ at different $n$ but also affects the DP error term $T_2$: smaller $\gamma$ decreases $T_1$ but increases $T_2$. In addition, $h$ affects $T_2$ and $T_3$ (but not $T_1$): smaller $h$ decrease $T_3$ but worsens $T_2$, indicating the tradeoffs among different terms in the error bound with different choices of $\gamma$ and $h$.  $h$ and $\gamma$ do not affect  $T_4$. A useful asymptotic regime is $\gamma\to\infty$ as $n\to\infty$ but $\gamma/n\to0$.  $\gamma\to\infty$ makes $T_2$ and $T_4$ shrink and $\gamma/n\to0$ prevents $T_1$ from dominating the excess risks. To make $T_3$ shrink as $n\to\infty$, let $h\to0$ but $h\gamma^2\to\infty$. In summary, vanishing of the excess risks in Theorems \ref{thm:empirical-main}, \ref{thm:population-generic}, and Corollary \ref{cor:sc} as $n\to\infty$ would need  $\gamma\to\infty$, $\gamma/n\to0$, $h\to0$, and $h\gamma^2\to\infty$.

%-----------------------------------------------------------
\vspace{-3pt}\section{Experiments}\label{sec:sim_studies}
We examine the prediction performance of DP-wERM in OWL and benchmark its performance against non-private OWL on both synthetic and real datasets. 

%--------------------------------------------
\vspace{-9pt}\subsection{Data}
The synthetic data replicate a typical clinical trial for treatment effect evaluation. 
We draw $\mathbf{x}_i = (x_{i,1}, x_{i,2}, x_{i,3}, x_{i,4})$ from a uniform distribution (i.e., $x_{i,j} \sim U[0, 1]$ independently for all $i=1,\ldots,n$ and $j=1,\ldots,4$). 
We examine  treatment $Y\in\{-1,1\}$ and each individual is randomly assigned with a probability of 0.5 to either treatment.
The propensity is thus $P(y_i=1|\mathbf{x}_i) = P(y_i=-1|\mathbf{x}_i) = 0.5$.
We assume the underlying optimal treatment is the sign of the function $f(\mathbf{x}_i) = 1 + x_{i, 1} + x_{i, 2} - 1.8x_{i, 3} - 2.2x_{i, 4}$.
Treatment benefit (the outcome) $B_i$ for $i=1,\ldots,n$ is drawn independently from $\mathcal{N}(\mu_i, \sigma^2)$ with $\mu_i = 0.01 + 0.02x_{i, 4} + 3y_if(\mathbf{x}_i)$ and $\sigma=0.5$.
If a simulated $B_i<0$, then it is shifted to $B_i + |\min\{B_i\}| + 0.001$ to ensure that $w_i = \frac{B_i}{P(y_i|\mathbf{x}_i)}> 0$. 
%In this experiment, an individual who receives the optimal treatment per  randomization is likely to benefit more from the treatment compared to those assigned to the non-optimal treatment. 
We examined a wide range of privacy budgets $\epsilon$ and training dataset sizes $n$ as listed below:
\begin{list}{}{}
    \item $\epsilon\in\{0.1, 0.5, 1, 2, 5, 20, 50, 150, 300, 500, 800, 1000\}$,
    \item $n\in\{200, 500, 800, 1000, 1500, 2000, 2500\}$.
\end{list} 
The large  $\epsilon$ scenarios help gain insights into the asymptotic performance of the DP-wERM algorithm, though they do not correspond to a meaningful privacy guarantee. 

The real data experiments use data from two randomized clinical trials (RCTs). The first is an RCT conducted at St. Jude Children's Research Hospital (SJCRH) that evaluated the efficacy of melatonin on insomnia and neurocognitive impairment in childhood cancer survivors \cite{lubas2022randomized}. Participants were randomized 1:1 to receive either 3 mg of time-released melatonin or a placebo; our final analysis includes 246 participants (120 melatonin, 126 placebo). The primary outcome is the difference in nonverbal reasoning assessment scores between baseline and month 6. Baseline covariates include age at diagnosis (years), age at enrollment (years), sex, race, and cancer diagnosis.
The second dataset is from the CYP-GUIDES (Cytochrome Psychotropic Genotyping Under Investigation for Decision Support) trial, which evaluated hospitalized patients with severe depressive disorders \cite{tortora2020clinical, ruano2020results}. Patients were randomized 1:2 to either standard psychotropic therapy or genetically guided therapy, where CYP2D6-based prescribing recommendations were integrated directly into the EHR. A total of 1,459 genotyped patients are included in our analysis (477 standard therapy, 982 genetically guided therapy). The primary outcome is hospital length of stay (LOS, in hours) that assesses whether genotype-informed prescribing improves inpatient psychiatric care and shortens hospitalization. Baseline features include age at enrollment (years), sex, race, and primary depressive disorder diagnosis.

%------------------------
\vspace{-6pt}\subsection{Experiment settings}
\label{subsec:sim_settings}
In the SJCRH study, since the population of interest is adult survivors of childhood cancer, we bound the enrollment age by [18, 60] years
and the age at diagnosis by [0, 20] years.  The other three features, sex, race, and diagnosis, are coded as 0 or 1. 
In the CYP-GUIDES study, we bound the enrollment age by [18, 89] years. Sex (female vs. male), race (white vs. other), and diagnosis (primary depressive disorder or not), are 0-1 coded.
In both the synthetic and real data experiments, to meet the regularity condition $\norm{\mathbf{x}_i}\le1$ for DP-wERM, we pre-process each observation $\mathbf{x}_i$ for $i=1,\ldots,n$ by dividing it by $\sqrt{p+1}$ ($p$ features plus an intercept term).

As for the (global) upper bound $W$ for weight $w_i=\frac{B_i}{P(y_i|\mathbf{x}_i)}$, we set an upper bound $V=15$ for $B_i$ (if $B_i>15$ in the data, it is clipped to 15),  leading to $W=15/0.5=30$  in the synthetic data experiment. 
In the SJCRH study, we convert $B_i$, which is the difference in the scores from  the nonverbal reasoning  assessment at month 6 and baseline, to an age-adjusted Z-score that follows a $N(0,1)$ distribution.
The shifted benefit $B'_i=B_i + |\min\{B_i\}| + 0.001$ is then used as the final benefit measure.
We set the upper bound  $W=4/0.5=8$, where $4$ is the upper bound for $B_i'$, representing $2$ standard deviations from its mean.
In the CYP-GUIDES study, the primary outcome was hospital LOS. We defined benefit through a transformed outcome given by $\log(24\!\times\!365)-\log(\text{LOS})$, where $24\!\times\!365$ represents the total number of hours in a year and serves as an upper reference bound. Per the RCT design with the 1:2 randomization ratio between the two types of therapy, the global upper bound for the weight was set to $W=\log(24\!\times\!365)/(1/3)$. 
Since all three experiments are randomized trials, then $C=0$ and $\Delta_{2, \boldsymbol{\theta}}=2W/\gamma$ per Theorem \ref{thm:wERM_sens}.

The Huber loss hyper-parameter $h$ in Eqn.~\eqref{eqn:huber} is set at  a ``typical value'' of $0.5$ in all experiments \cite{Chapelle2007}. 
Our approach for choosing the hyperparameter $\gamma$, which depends on $n$ and $\epsilon$, in the synthetic data experiment is described in Section~\ref{subsec:hyperparameter_tuning}.
The choice of $\gamma$ in the real data experiments ($\gamma=50$ in SJCRH and $75$ for GYP-GUIDES) is informed by the results of the synthetic data experiment for datasets with similar $n$ and the sensitivity study results in Section \ref{subsec:hyperparameter_tuning}. %\footnote{Given the $n$ in each real dataset ($n=246$ and $n=1459$, respectively and $\epsilon=5$, Figs.~\ref{fig:sens_study_theta} and \ref{fig:sens_study_dist} suggest that $\gamma\approx50$ and $\gamma\approx75$ could be good choices for  SJCRH and CYP-GUIDES, respectively. } 

To evaluate the utility of the privacy-preserving OWL models trained via DP-wERM, we calculate the empirical treatment value, an estimate of the expected clinical benefit $E[\frac{B}{P(Y|\mathbf{x})}\mathbbm{1}(Y=T(\mathbf{x}))]$, defined as
\begin{equation}\label{eqn:empirical_treatment_value}
    V(\hat{f})\!=\!\textstyle\left(\!\sum_{i=1}^n\!\frac{\mathbbm{1}(y_i=\hat{f}(\mathbf{x}_i))}{P(y_i|\mathbf{x}_i)}B_i\right)\!\left(\!\sum_{i=1}^n\! \frac{\mathbbm{1}(y_i=\hat{f}(\mathbf{x}_i))}{P(y_i|\mathbf{x}_i)}\right)^{-1}\!,
\end{equation}
where $\hat{f}$ is the estimated DP predictor function via Algorithm \ref{alg:DP-wERM}.
%When $P(y_i|\mathbf{x}_i)=0.5$, this reduces to $V(\hat{f})= \big(\sum_{i=1}^n\mathbbm{1}(y_i=\hat{f}(\mathbf{x}_i))B_i\big)\big(\sum_{i=1}^n \mathbbm{1}(y_i=\hat{f}(\mathbf{x}_i))\big)^{-1}$.
In addition, in the synthetic data experiment, since the true optimal treatment is known, we can evaluate the accuracy of the optimal treatment assigned by the DP-OWL models. For the synthetic data experiment, we ran the DP-wERM algorithm on 200 simulated datasets for each $n$ and $\epsilon$ combination and obtained the average (95\% confidence intervals/CIs) optimal assignment accuracy rates and empirical treatment values on a test dataset of size 5,000 in each repeat. In the real data experiments, we conducted 200 DP runs at each $\epsilon$ and $\gamma$ to examine the stability of the DP-wERM algorithm.

%-----------------------------------
\vspace{-9pt}\subsection{Results}\label{subec:results}\vspace{-3pt}
The results from the synthetic data experiment are presented in Fig.~\ref{fig:sim_results} (the numerical values in these plots for a subset of $n$ and $\epsilon$ values are found in Table II in the SM). 
As expected, the optimal treatment accuracy rate and empirical treatment value improve as $n$ or $\epsilon$ increases, approaching the non-private results.\footnote{Compared to the empirical results from the concurrently developed DP-OWL algorithm in \cite{Spicker2024}, the optimal treatment accuracy rates in our experiment are similar to theirs for many combinations of $n$ and $\epsilon$, but we caution against direct comparisons as the experimental settings differ between the two studies.
Specifically, the ground-truth treatment assignment functions used in their experiment with synthetic data differ from ours, as evidenced by the differing accuracy levels for non-private OWL between their study vs. ours.} 
We also investigated the sensitivity of the results when $B$ values are re-scaled or discretized, $\gamma$ increases, and  privacy amplification via subsampling \cite{balle2018privacy} is applied to the DP-OWL algorithm, especially at smaller $n$ or $\epsilon$, and did not observe notable result changes from Fig. \ref{fig:sim_results}, suggesting robustness to these modifications. 
\begin{figure}[!htb]
\vspace{-6pt}\centering
\includegraphics[width=0.8\columnwidth]{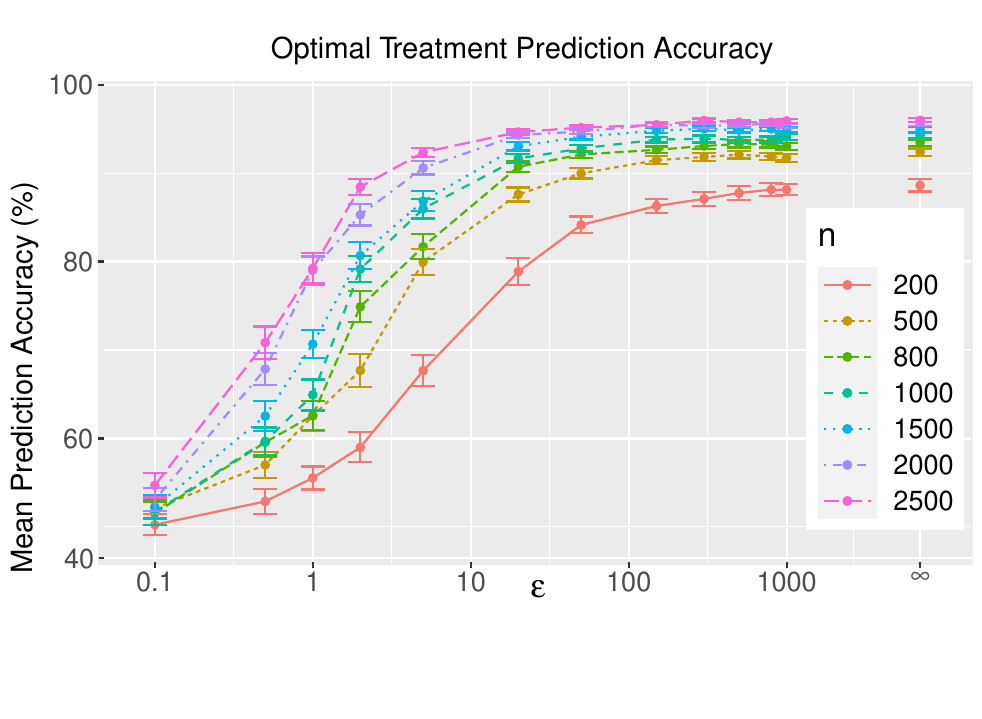}\\[6pt]
\includegraphics[width=0.8\columnwidth]{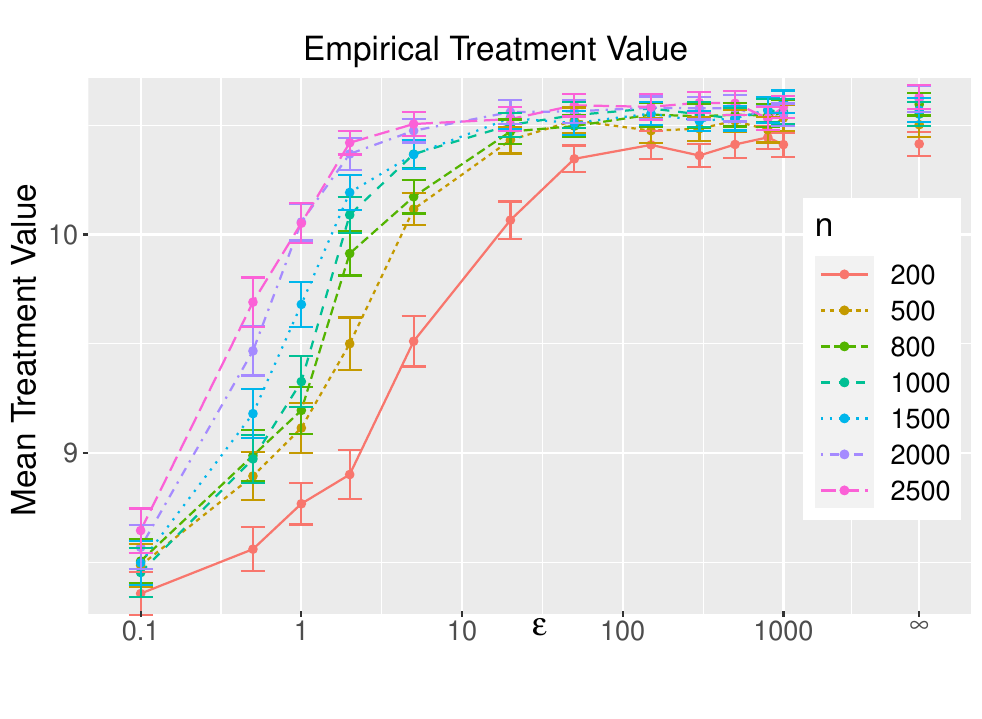}\vspace{-9pt}
\caption{Mean (95\% CI) DP-OWL optimal treatment assignment accuracy rate (top) and empirical treatment value (bottom) on the test set over 200 repeats in the synthetic data experiment.}
\label{fig:sim_results}\vspace{-9pt}
\end{figure}

The results in the two real data experiments are presented in Table~\ref{tab:etv50}. In summary, 
The ITRs derived by the non-private OWL and DP-wERM generate comparable empirical treatment values, suggesting that the utility of the privacy-preserving ITRs is well preserved in terms of estimated clinical benefit.
In the SJCRH study, the estimated value function increases monotonically as $\epsilon$ increases, approaching the non‑private OWL value as privacy constraints are relaxed. The estimated value function in the CYP‑GUIDES study is not strictly monotone for small $\epsilon$ values: when $\epsilon=0.1$ and 0.5, the estimated means fluctuate, accompanied by relatively large SDs. For moderate to large privacy budgets ($\epsilon\ge1$), a clear increasing trend in the value function is observed, consistent with improved utility under weaker privacy constraints.
We also examine the proportion of concordance pairs in the optimal treatment allocation obtained by the non-private and the DP ITRs and run the chi‑squared association test. Across both experiments, the concordance proportions are high and the $p$-values from the tests are close to zero, indicating that the privacy-preserving ITRs gave similar individualized treatment allocations to those obtained by OWL without DP. 
\begin{table}[!htb]
\centering\vspace{-6pt}
\caption{DP-OWL results in the two real-data experiments}\vspace{-6pt}
\label{tab:etv50}
\begin{tabular}{@{}c@{\hspace{3pt}}c@{\hspace{3pt}}c@{}c@{}c@{}c@{}}
    \hline
    $\epsilon$ &  Mean$^\#$  & SD$^\#$ & \multicolumn{1}{c}{\begin{tabular}[c]{@{}c@{}}optimal ITR\\allocation$^*$ \\ (melatonin:placebo)\end{tabular}} & \multicolumn{1}{c}{\begin{tabular}[c]{@{}c@{}}concordance \\pairs$^\ddagger$\\proportion\end{tabular}} & $p$-value$^\dagger$  \\ \hline
    \multicolumn{6}{c}{the SJCRH Study ($n=246$)}\\\hline
    0.1  & 2.093 & 0.091 & 156:90 & 0.703 & $1.37\!\times\! 10^{-6}$ \\
    0.5  & 2.094 & 0.101 & 172:74 & 0.793 & $4.96\!\times\!10^{-13}$ \\
    1  & 2.095 & 0.111 & 187:59 & 0.886 & $<2.2\!\times\!10^{-16}$ \\
    2  & 2.098 & 0.110 & 188:58 & 0.890 & $<2.2\!\times\!10^{-16}$ \\
    5  & 2.099 & 0.103 & 194:52 & 0.947 & $<2.2\!\times\!10^{-16}$ \\
     $\infty$  & 2.101 & N/A & 197:49 & N/A & N/A \\
    \hline
    \multicolumn{6}{c}{the CYP-GUIDES Study ($n=1459$)}\\
    \hline
    0.1 & 4.144 & 0.205 & 1164:295 & 0.801 & $6.10\!\times\!10^{-4}$ \\
    0.5 & 4.141 & 0.209 & 1251:208 & 0.862 & $5.49\!\times\!10^{-8}$ \\
    1   & 4.108 & 0.261 & 1424:35  & 0.981 & $<2.2\!\times\!10^{-16}$ \\
    2   & 4.123 & 0.162 & 1451:8   & 1.000 & $<2.2\!\times\!10^{-16}$ \\
    5   & 4.124 & 0.119 & 1451:8   & 1.000 & $<2.2\!\times\!10^{-16}$ \\
    $\infty$ & 4.151 & N/A & 1451:8 & N/A & N/A \\
    \hline
    \multicolumn{6}{l}{$^\#$ \footnotesize{empirical treatment value for the outcome over 200 repeats.}}\\
    \multicolumn{6}{l}{$^*$ \footnotesize{assigned if at least 50\% of repeats favor it;}}\\
    \multicolumn{6}{l}{$^\ddagger$ \footnotesize{optimal ITR allocation between DP-wERM and nonprivate OWL;}}\\
    \multicolumn{6}{l}{$^\dagger$ \footnotesize{chi-squared association test between DP and non-private ITRs;}}\\
    \multicolumn{6}{l}{\footnotesize{\hspace{5pt}  p-value $\le 0.05$ indicates statistically significant concordance.}}\\
    \hline
\end{tabular}\vspace{-9pt}
\end{table}

The privacy-preserving results in the SJCRH study are more closely aligned with the original nonprivate results, compared to the synthetic data experiments of similar $n$.
One possible explanation is that the SJCRH study presents a simpler ITR problem, with a stronger signal for the optimal treatment that is less likely to be obscured by the noise introduced by the DP mechanism. Specifically, the synthetic data experiment was designed to mimic real-world scenarios with two competing treatments, where different individuals can benefit from one or the other.
In the SJCRH study, placebo is expected to offer minimal benefit (with an expected $B$ close to 0), while melatonin is likely to be the optimal treatment for most subjects,
which is indeed the case as reflected by a melatonin:placebo ITR ratio of 197:49 and a mean empirical treatment value of 2.101 (close to the mean of $B$) in the original ITR results.
For the CYP‑GUIDES study, the ITR allocation is also unbalanced, with genetically guided therapy assigned to the majority of patients. Several explanations may account for this observation. First, genetically guided prescribing may indeed offer greater benefit for most patients in this inpatient psychiatric population; second, the limited number of baseline covariates included in the model may constrain its ability to capture treatment-effect heterogeneity. Consequently, the model may preferentially recommend a single treatment for most patients.
Nevertheless, in both real data experiments, the overall privacy-utility trend still holds -- smaller $\epsilon$ results in greater deviations from the original results.

To assesses the sensitivity of DP-wERM derived ITR to different $\gamma$ values, we examined more values of $\gamma$ in addition to $\gamma=50$ and 75 in Table \ref{tab:etv50}. Fig.~\ref{fig:real_results}, generated for various metrics, suggests that the learned ITRs are robust and remain stable across different $\gamma$ values given $\epsilon$.
\begin{figure}[!htb]
\centering
\includegraphics[width=0.95\columnwidth]{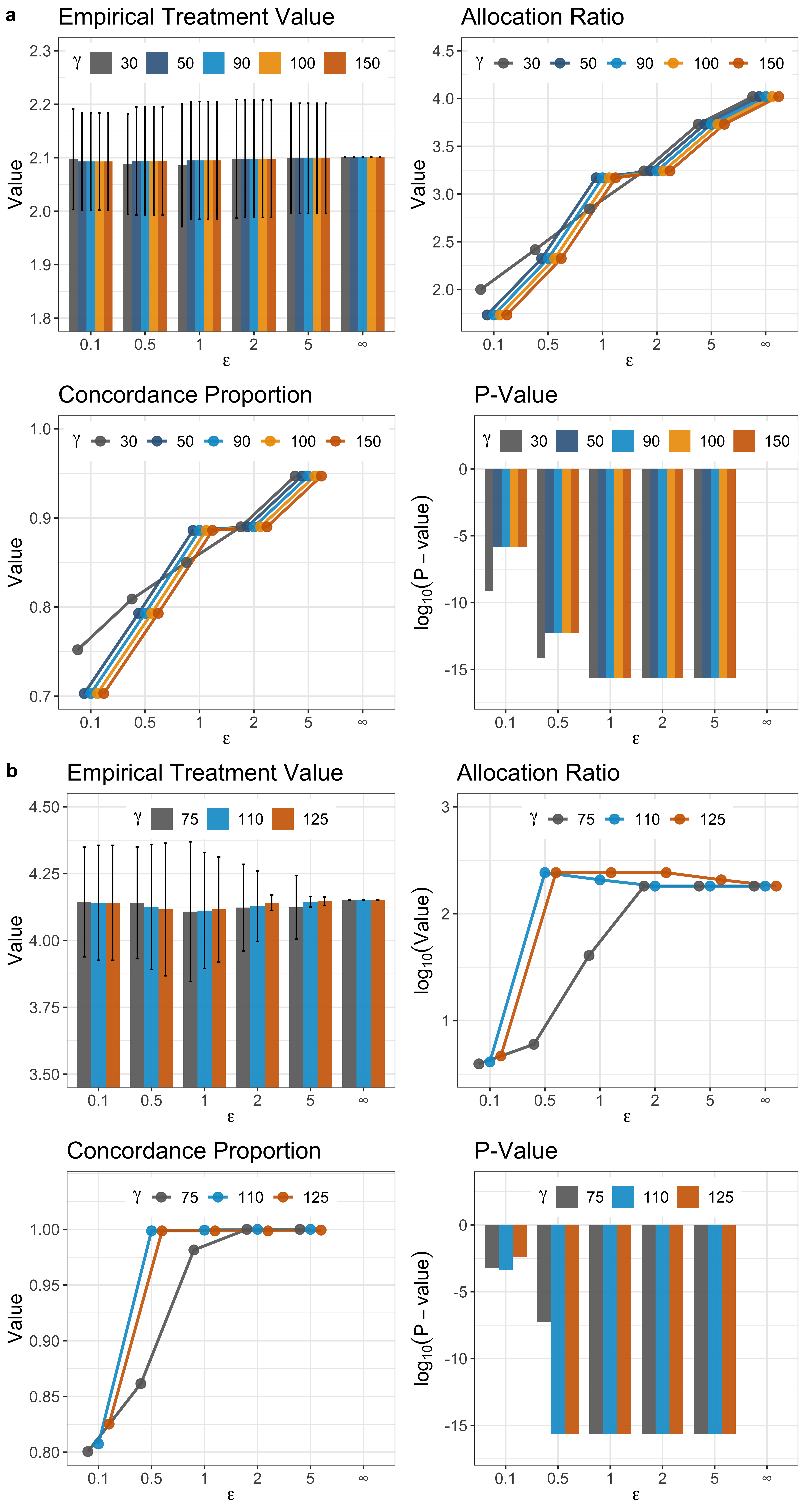}
\vspace{-9pt}
\caption{Sensitivity analysis of $\gamma$ in the two real data experiments:  the SJCRH study (a. top) and the CYP-GUIDES study (b. bottom). Upper left: Mean$\pm$SD of the empirical treatment value for the outcome over 200 repeats. Upper right: ITR allocation ratios (patients were assigned to melatonin or genetically guided therapy if it was recommended in more than 100 out of the 200 repeats). Bottom left: Concordance pair proportion in optimal ITR allocation between DP-wERM vs nonprivate OWL. Bottom right: $\log_{10}$ of the p-value obtained from the chi-squared association tests of the optimal allocation between private and nonprivate ITRs.} \label{fig:real_results}\vspace{-12pt}
\end{figure}

\vspace{-6pt}\subsection{Practical hyperparameter tuning \& sensitivity analysis} \label{subsec:hyperparameter_tuning}\vspace{-3pt}
Due to space limitation, we present our  hyperparameter tuning approach in  Algorithm 2 of the SM. In practice, to select $\gamma$ via Algorithm 2, one may first check whether there exists an independent dataset $D_0$ with distributional characteristics similar to those of the sensitive dataset $D$.  If such $D_0$ exists, even if it is relatively small, it can be used to tune $\gamma$ without incurring  any privacy loss on $D$.  
In the synthetic data experiment  in Section \ref{subsec:sim_settings}, we assumed such a $D_0$ existed and generated $D_0$ in the same way as but independently from $D$. With an optimal $\gamma$ obtained from $D_0$, we were able to use the entire $D$ and privacy budget to fit the DP-OWL model in the synthetic data experiment.\footnote{This was at least partially intentional so that we could focus on examining the privacy–utility trade-off attributable to the DP-wERM method while minimizing the influence of other aspects of the model-fitting process in the synthetic data  experiment.}  
That said, in many practical situations, having a $D_0$ that generally matches the distribution of $D$ is the exception rather than the rule. More commonly, one encounters the one of following two scenarios: (1) there is an independent $D_0$, but it differs from $D$ in certain statistical aspects; (2) $D_0$ is simply not available.  
We conduct experiments to investigate the robustness of Algorithm 2 yielding good choices on $\gamma$ when $D_0$ differ from $D$ in some aspects in scenario (1) and provide some practical guidance on hyperparameter tuning in scenario (2).

First, we note that since the optimal value for $\gamma$ in DP-wERM may vary with training size $n$ and  privacy budget $\epsilon$, when tuning $\gamma$ via Algorithm 2  after splitting $D_0$ into training and validation sets, one should ensure the training set is of size $n$ (sampling with replacement as necessary), for each candidate $\gamma$. 

Figs.~\ref{fig:sens_study_theta} and \ref{fig:sens_study_dist} show a subset of the results from the experiments when $D$ and $D_0$ differ in ground optimal treatment function and feature distribution, respectively. Complete sensitivity analysis results, including additional results for when $D$ and $D_0$ differ in number of features and for various sizes of $D_0$, are provided in the SM. In summary, across all combinations of $n$ and $\epsilon$, the optimal $\gamma$ chosen based on an ``imperfect'' $D_0$ (i.e., dissimilar to $D$ in some aspects) is consistently similar to the optimal $\gamma$ when $D_0$ is simulated the same way as $D$. This demonstrates the robustness of our hyperparameter tuning method, confirming that an independent proxy dataset can reliably guide hyperparameter selection even when it deviates from the sensitive data in some key statistical aspects.
The usefulness of the ITR results in the real-data experiments  further supports this observation, where $\gamma$ was informed by the independent $D_0$ of similar $n$ in the synthetic data experiment. %Our experimental results confirm these choices perform well; this demonstrates that despite distribution differences between the synthetic $D_0$ and real data $D$, the $D_0$-informed hyperparameter selection remains robust and effective.

When $D_0$ is unavailable, one can tune $\gamma$ by splitting $D$ into two disjoint parts: one portion is used to select $\gamma$ under $\epsilon$-DP  via a DP hyperparameter tuning approach (e.g., \cite{chaudhuri2011}); and
the other trains wERM with the selected hyperparameter under $\epsilon$-DP.
An alternative is to split $\epsilon$ between tuning $\epsilon_1$ and model fitting $\epsilon_2$.\footnote{It is well established that, for a single dataset, the sequential release of two DP outputs on the same dataset satisfying $\epsilon_1$-DP and $\epsilon_2$-DP, respectively, satisfies $(\epsilon_1+\epsilon_2)$-DP \cite{Dwork2006}.} While this approach utilizes all data for both tuning and training, the smaller individual budgets for each may present a comparable utility trade-off relative to the training on a subset of the full data.
\begin{figure}[!htb] 
    \centering \vspace{-12pt}
    % epsilon=0.5
    \subfloat{\includegraphics[width=0.285\linewidth]{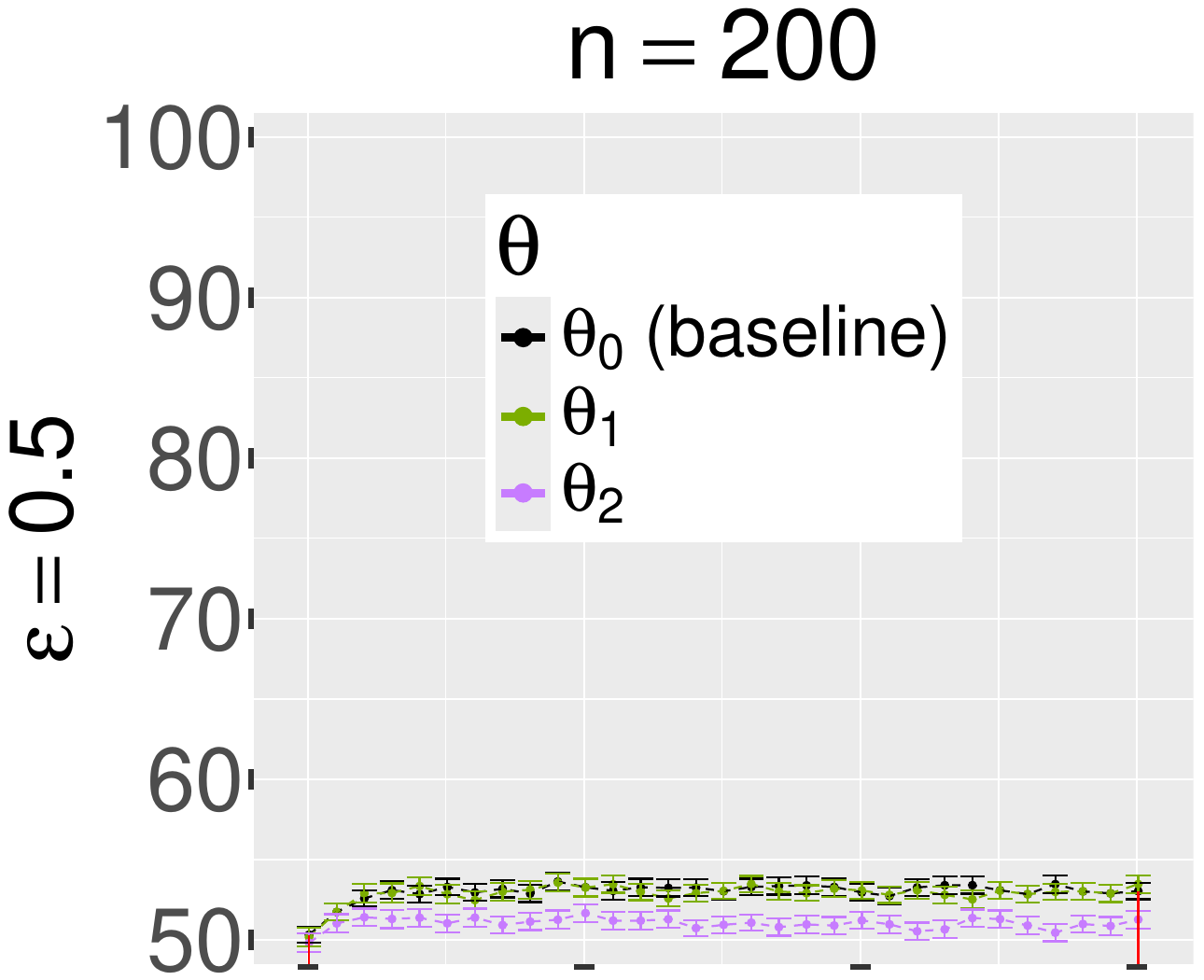}}
    \subfloat{\includegraphics[width=0.235\linewidth]{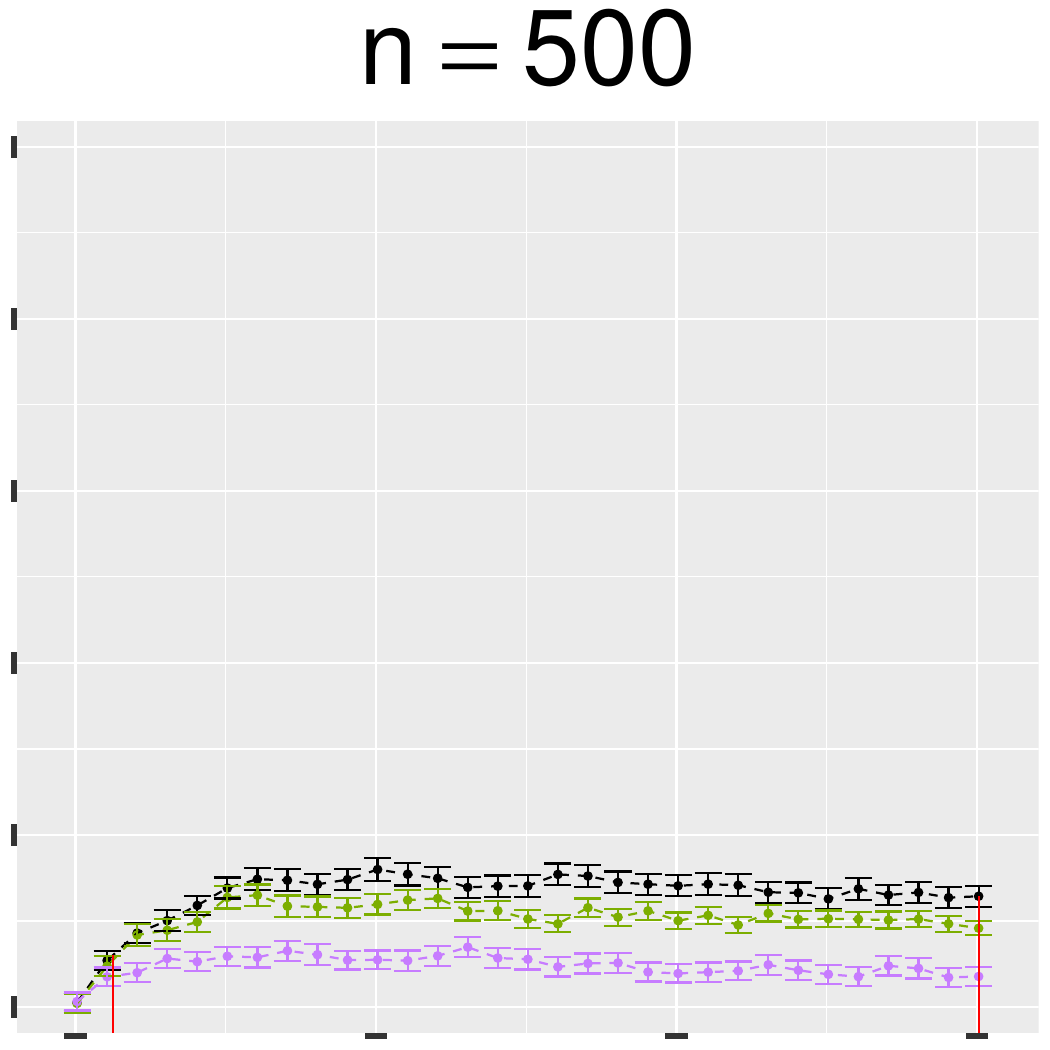}}
    \subfloat{\includegraphics[width=0.235\linewidth]{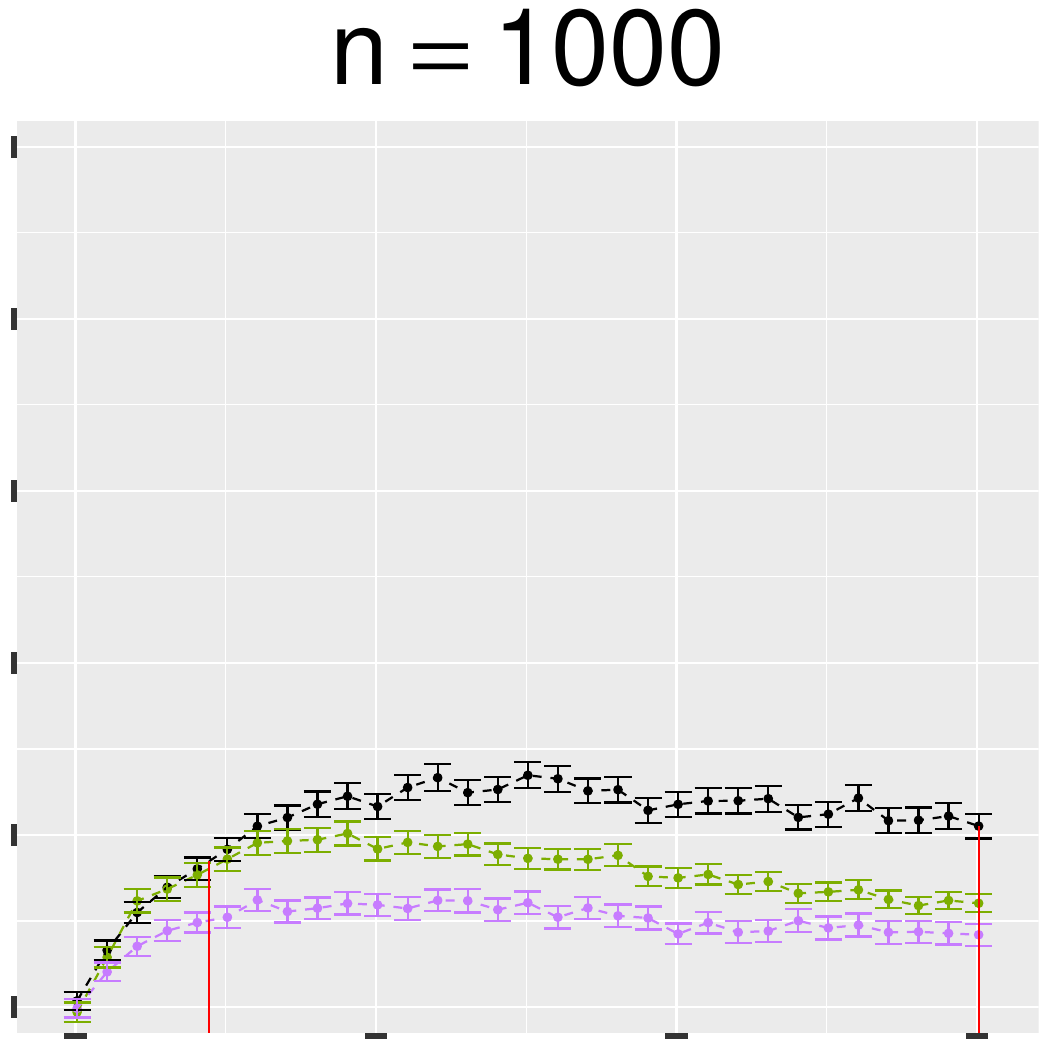}}
    \subfloat{\includegraphics[width=0.235\linewidth]{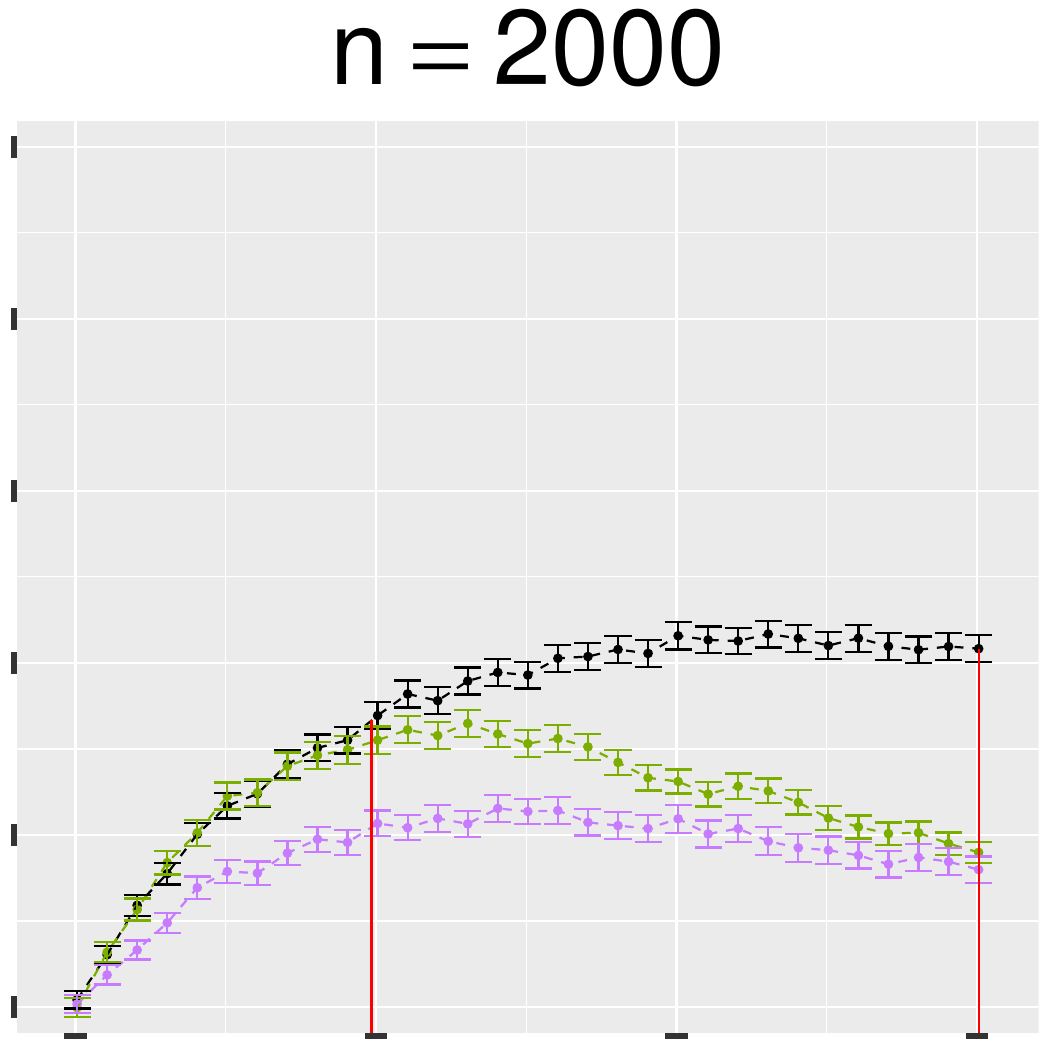}}
    \vspace{-8pt}

    % epsilon=1
    \subfloat{\includegraphics[width=0.285\linewidth]{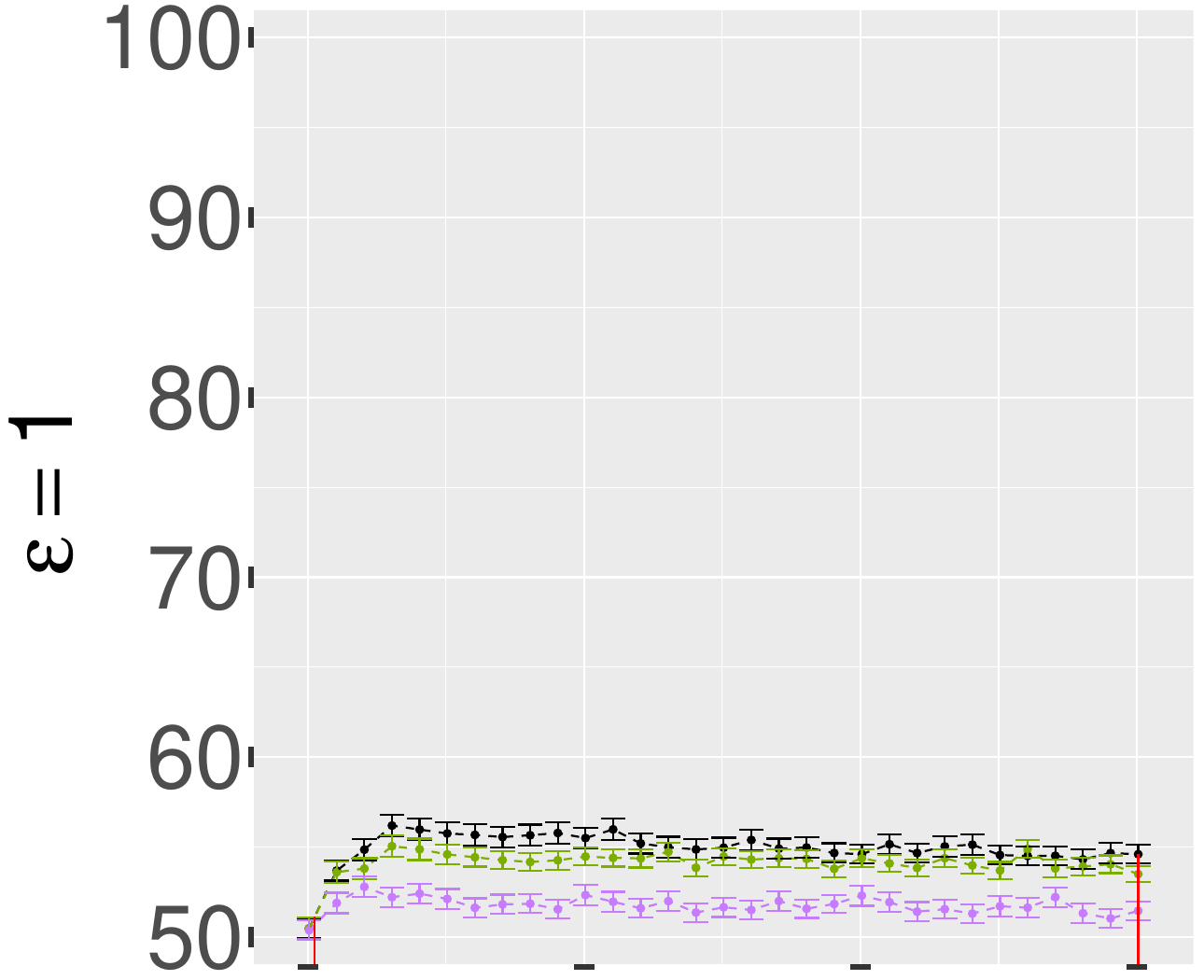}}
    \subfloat{\includegraphics[width=0.235\linewidth]{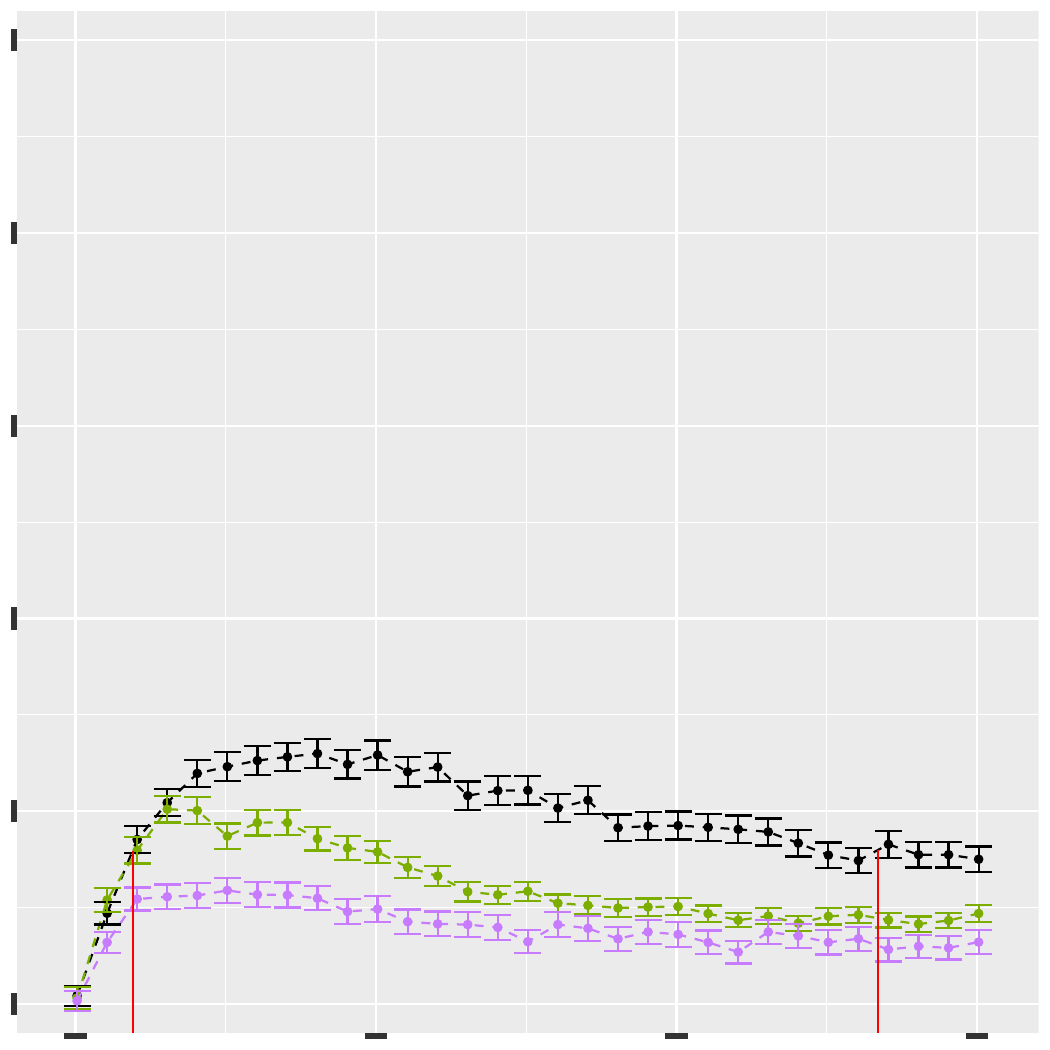}}
    \subfloat{\includegraphics[width=0.235\linewidth]{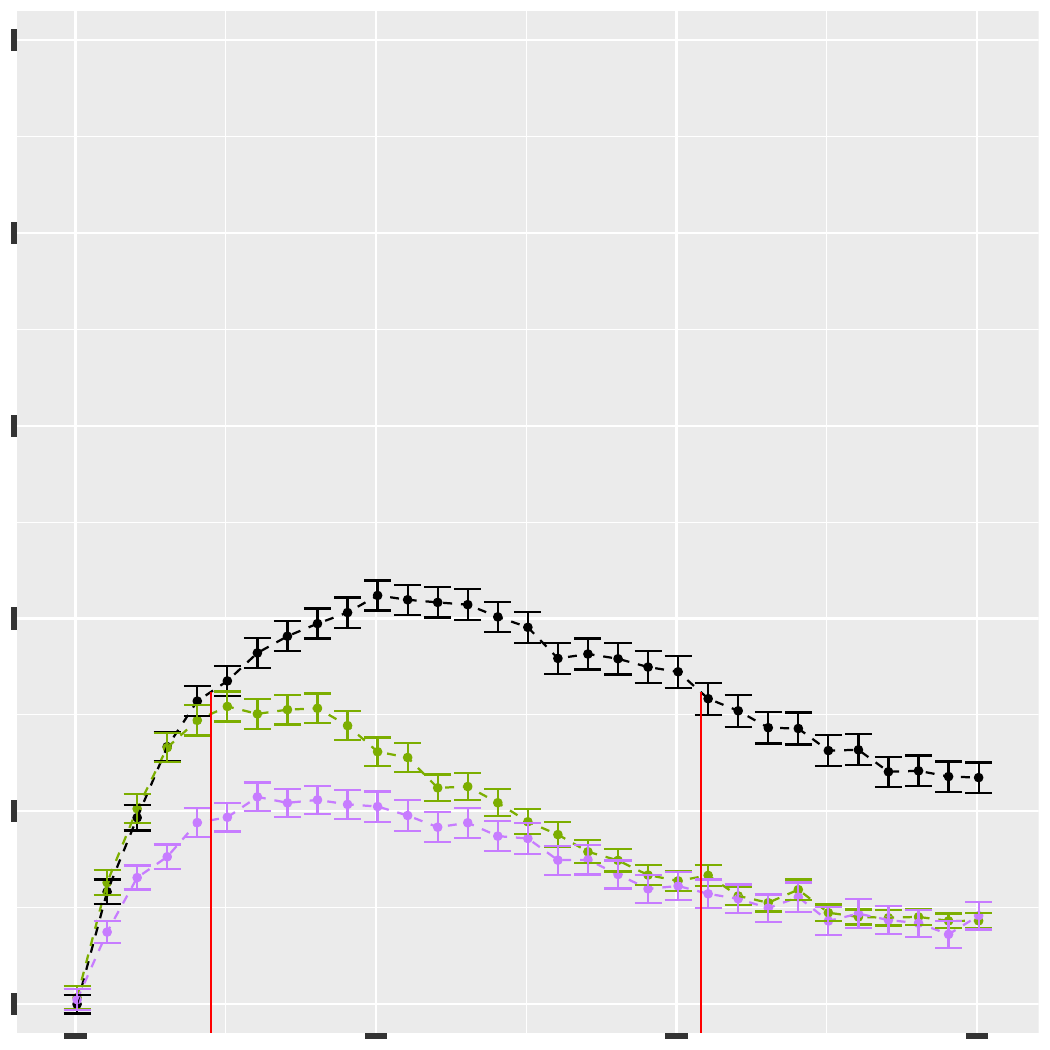}}
    \subfloat{\includegraphics[width=0.235\linewidth]{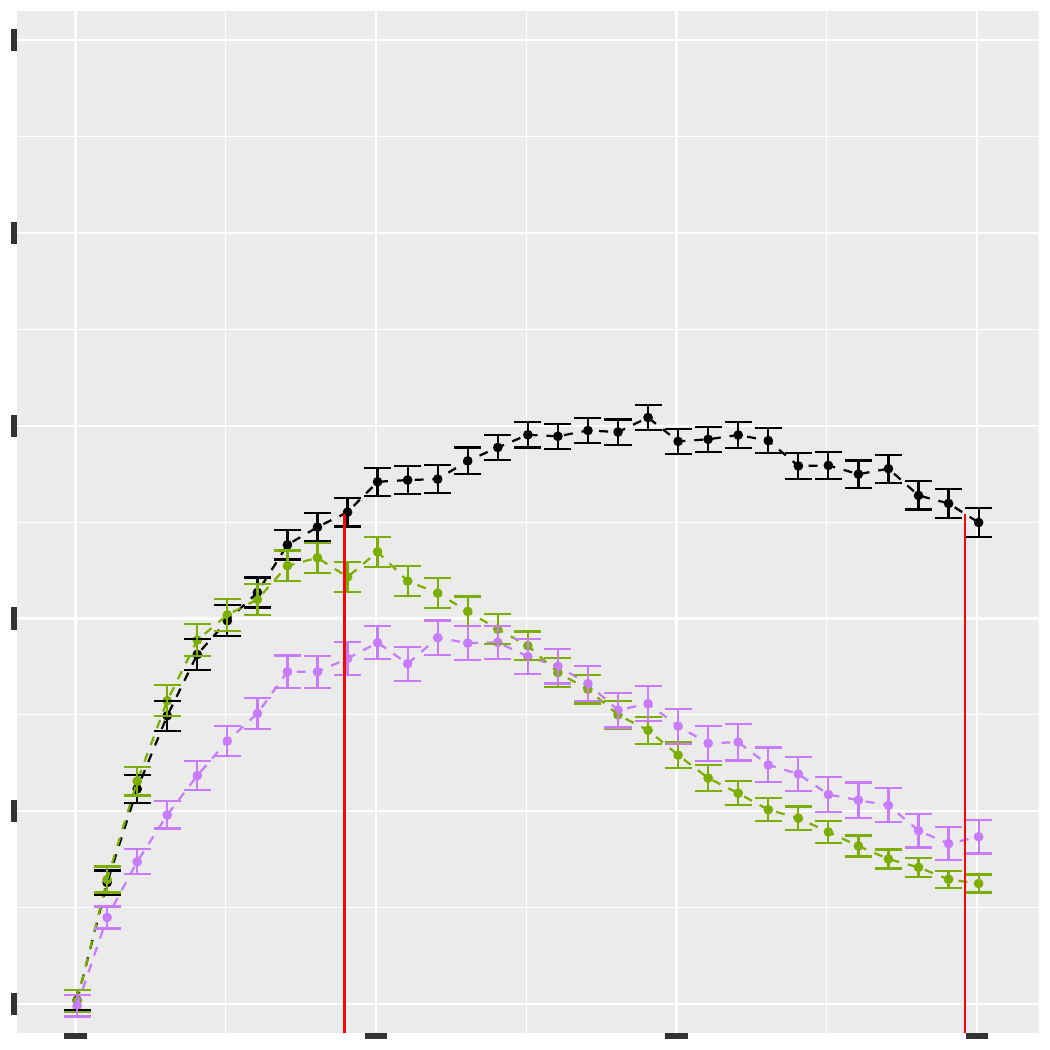}}
    \vspace{-8pt}

    % epsilon=5
    \subfloat{\includegraphics[width=0.285\linewidth]{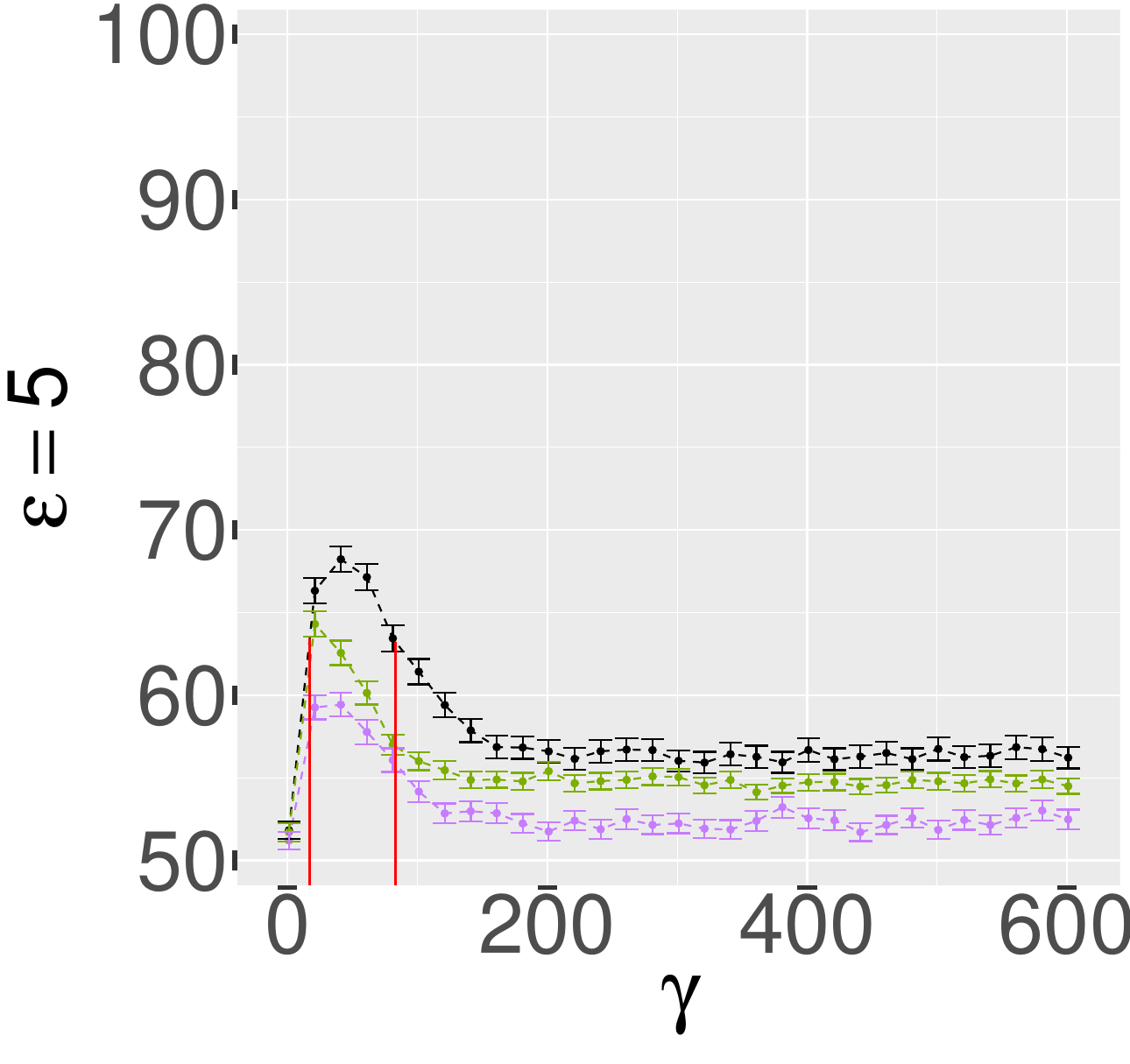}}
    \subfloat{\includegraphics[width=0.235\linewidth]{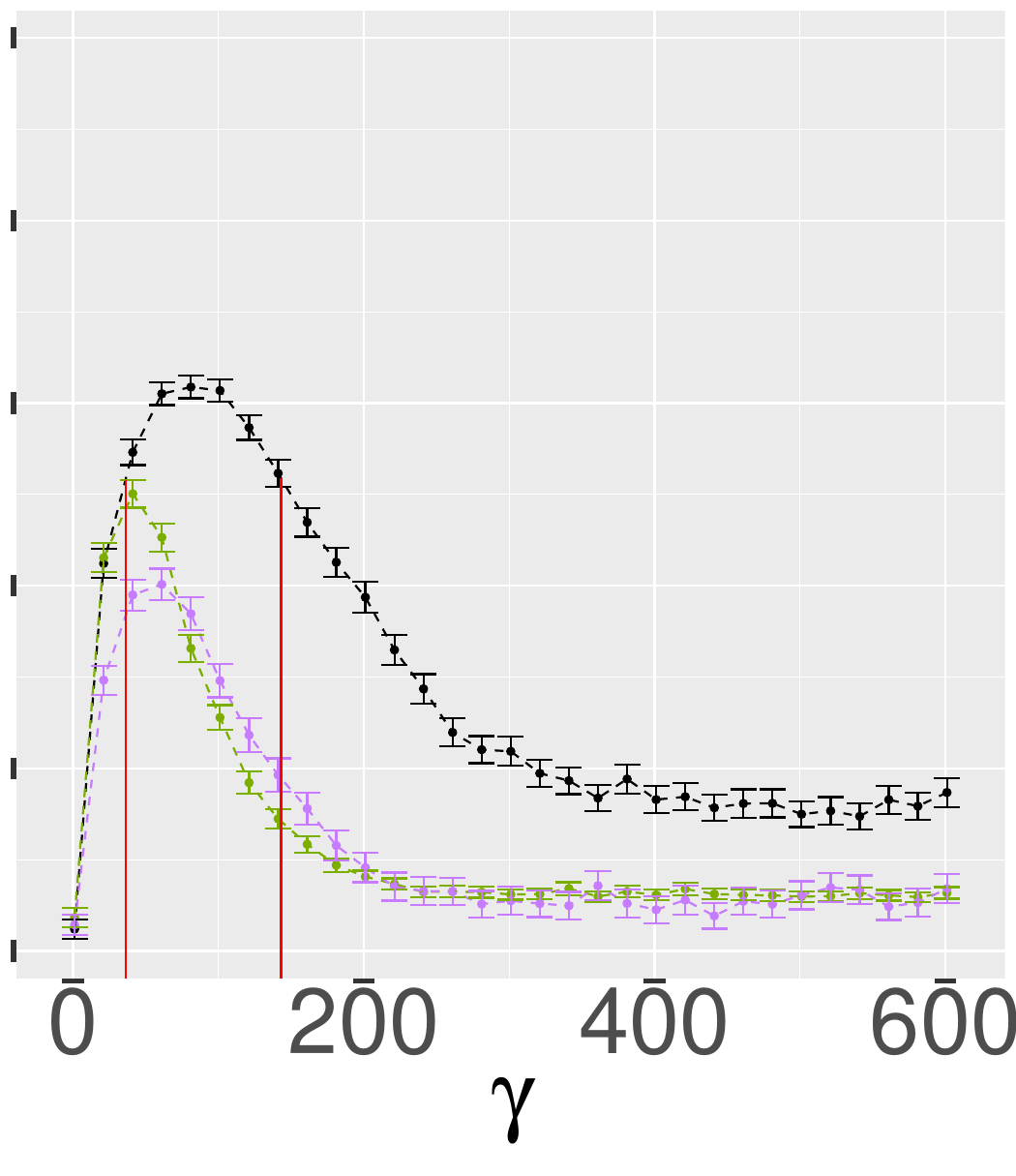}}
    \subfloat{\includegraphics[width=0.235\linewidth]{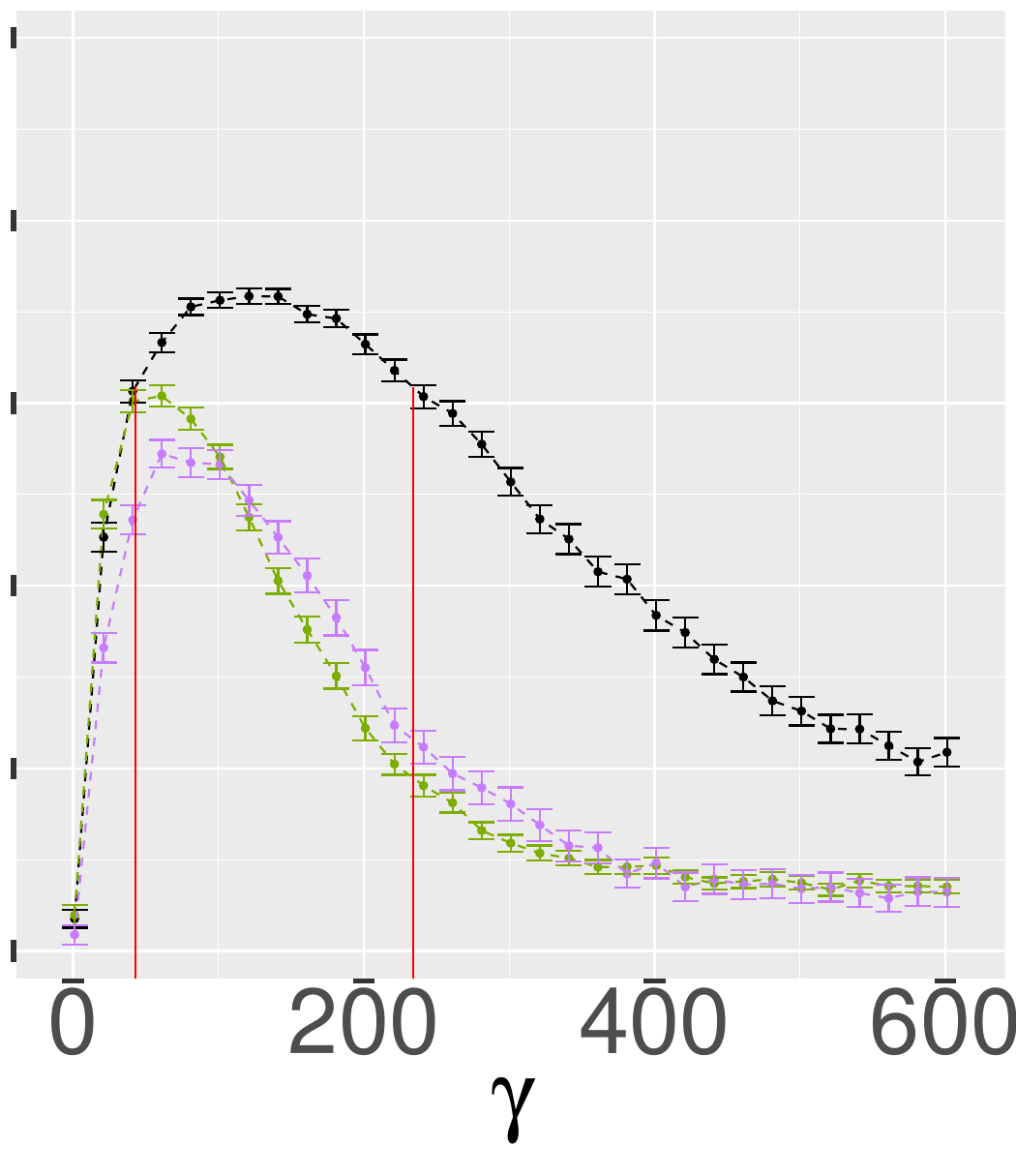}}
    \subfloat{\includegraphics[width=0.235\linewidth]{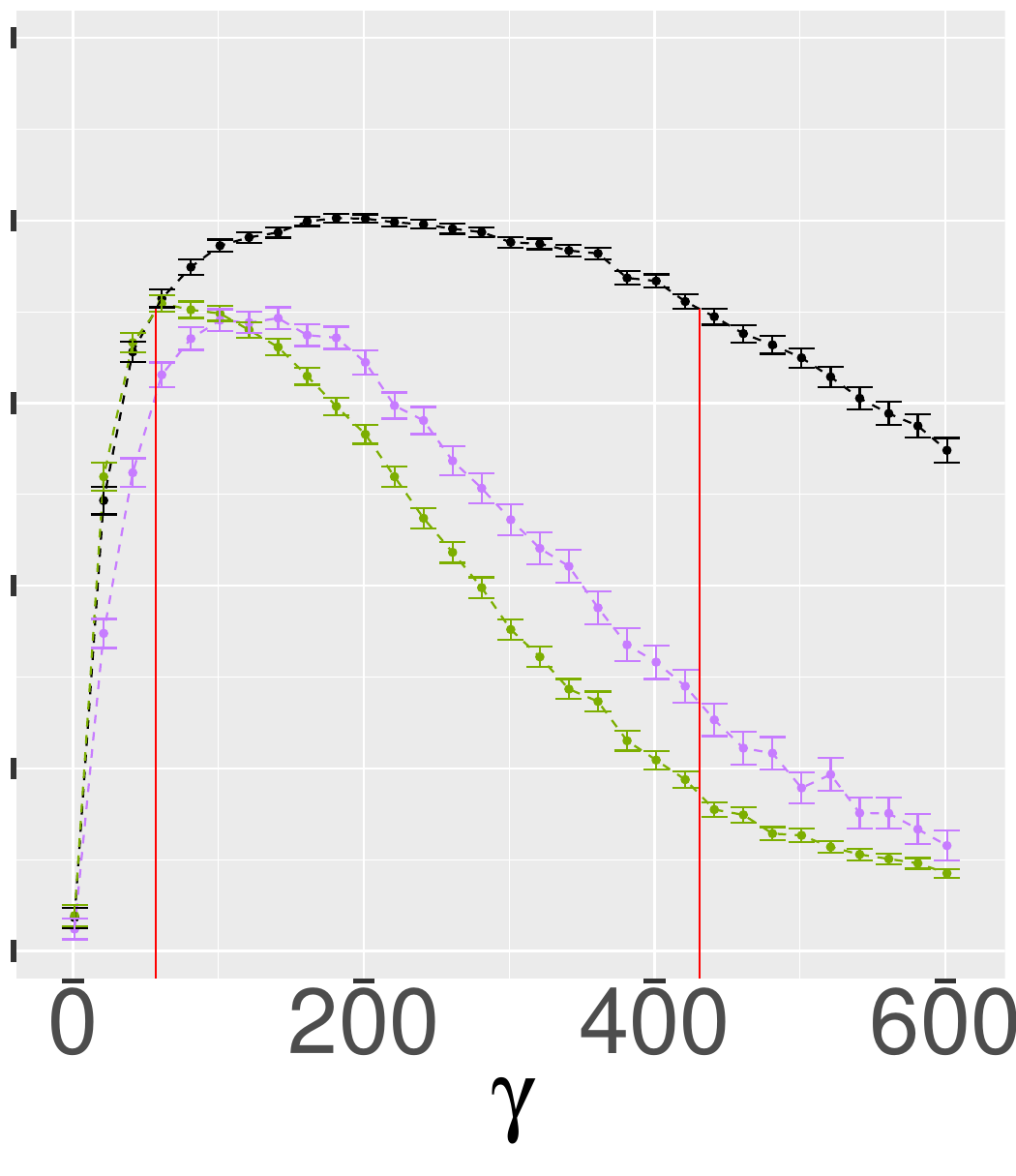}}
    \vspace{-10pt}
    \caption{Sensitivity Analysis of $\gamma$ Selection: average (95\% CI) validation set accuracy (\%) at different $\gamma$ for various combinations of $\epsilon$ and $n$.
    Different lines represent different $\boldsymbol{\theta}$ underlying the independent dataset $D_0$: 
    $\boldsymbol{\theta}_0 = [1, 1, 1, -1.8, -2.2]^\intercal$  (black) matches the sensitive dataset; $\boldsymbol{\theta}_1 = [1, -1, -1, 1.5, -1.5]^\intercal$ (green) and $\boldsymbol{\theta}_2 = [1, -0.5, 0.5, 1, 1.5, -2.5, -2]^\intercal$ (magenta) represent alternatives.
    The red vertical lines mark the region of $\gamma$ where the corresponding accuracy in the baseline case ($\boldsymbol{\theta}_0$) is within $5\%$ of the highest accuracy.}  \label{fig:sens_study_theta}

    \subfloat{\includegraphics[width=0.285\linewidth]{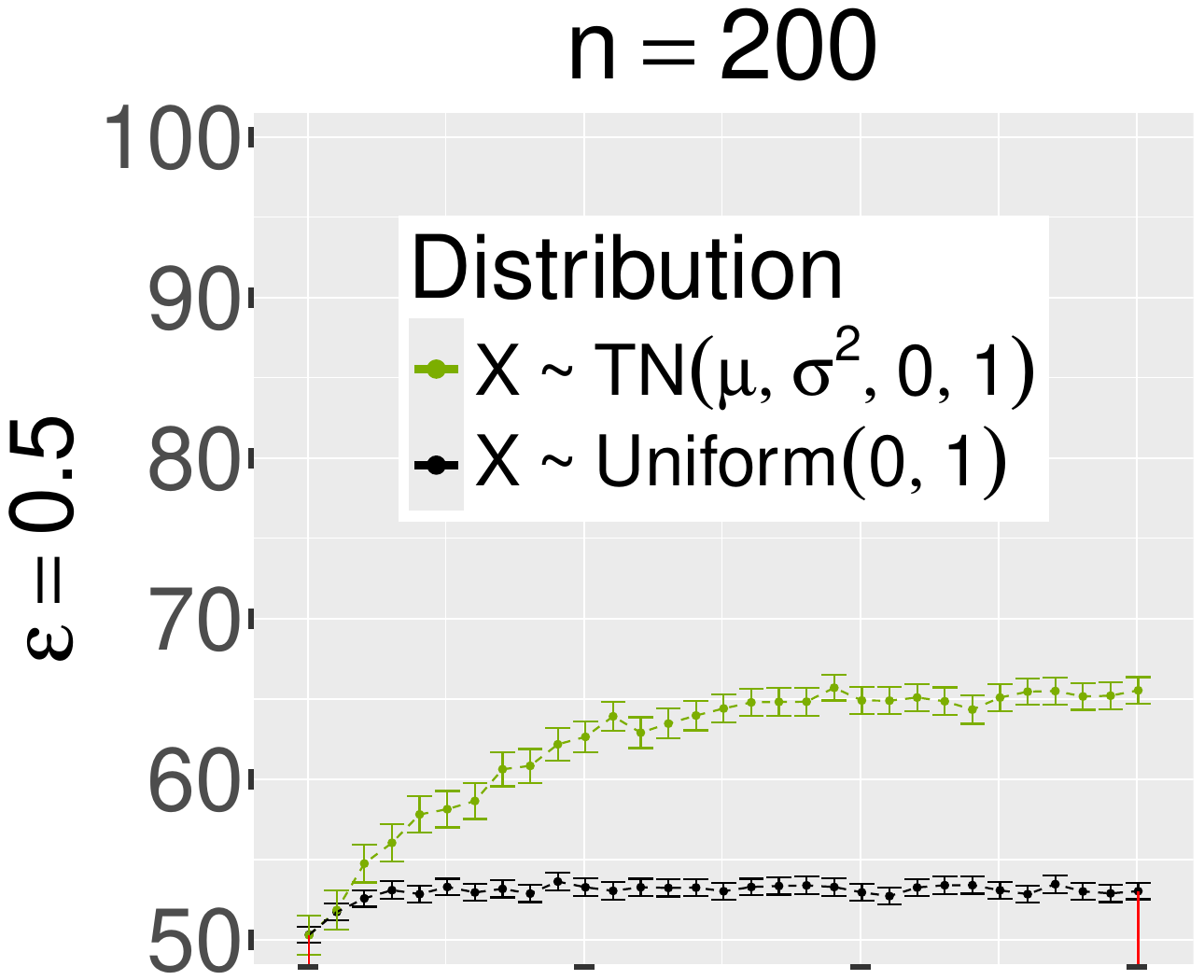}}
    \subfloat{\includegraphics[width=0.235\linewidth]{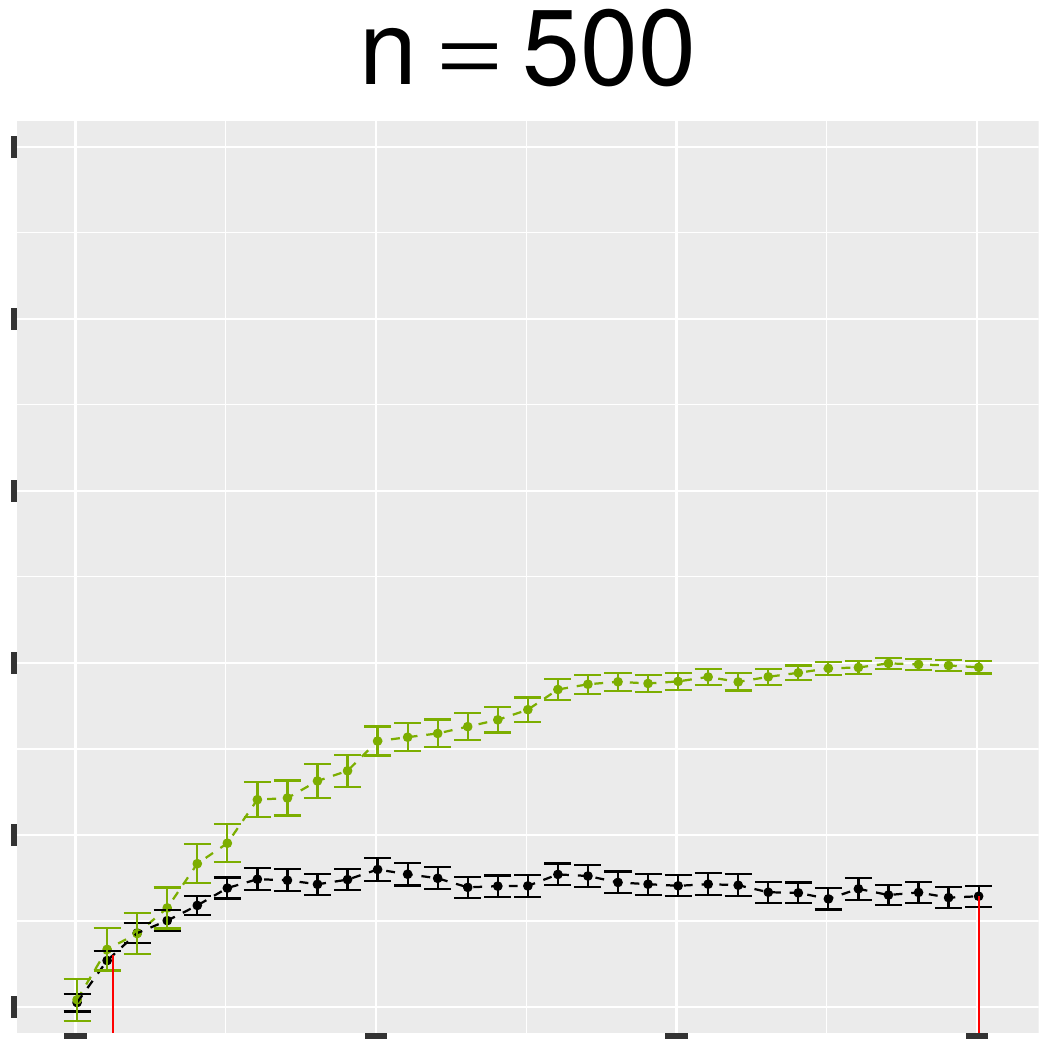}}
    \subfloat{\includegraphics[width=0.235\linewidth]{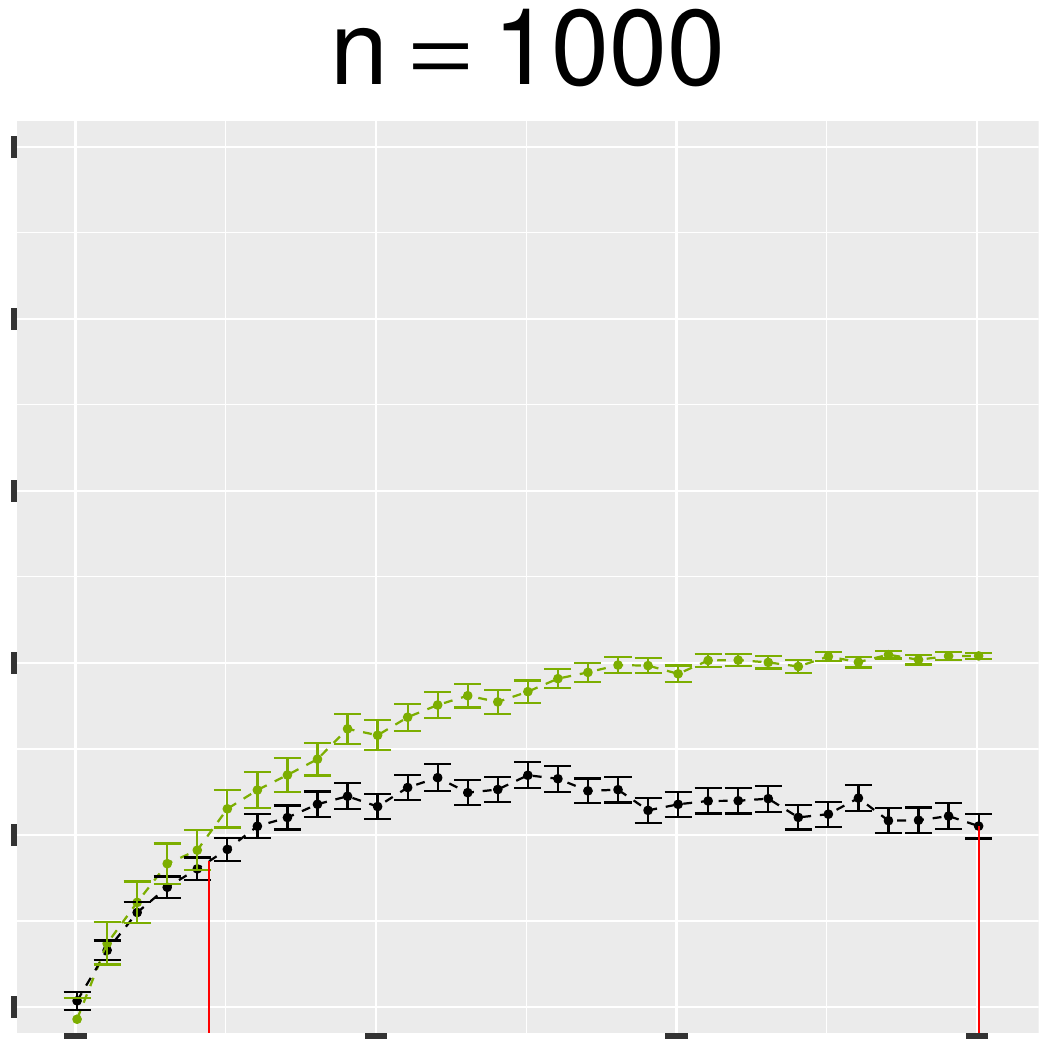}}
    \subfloat{\includegraphics[width=0.235\linewidth]{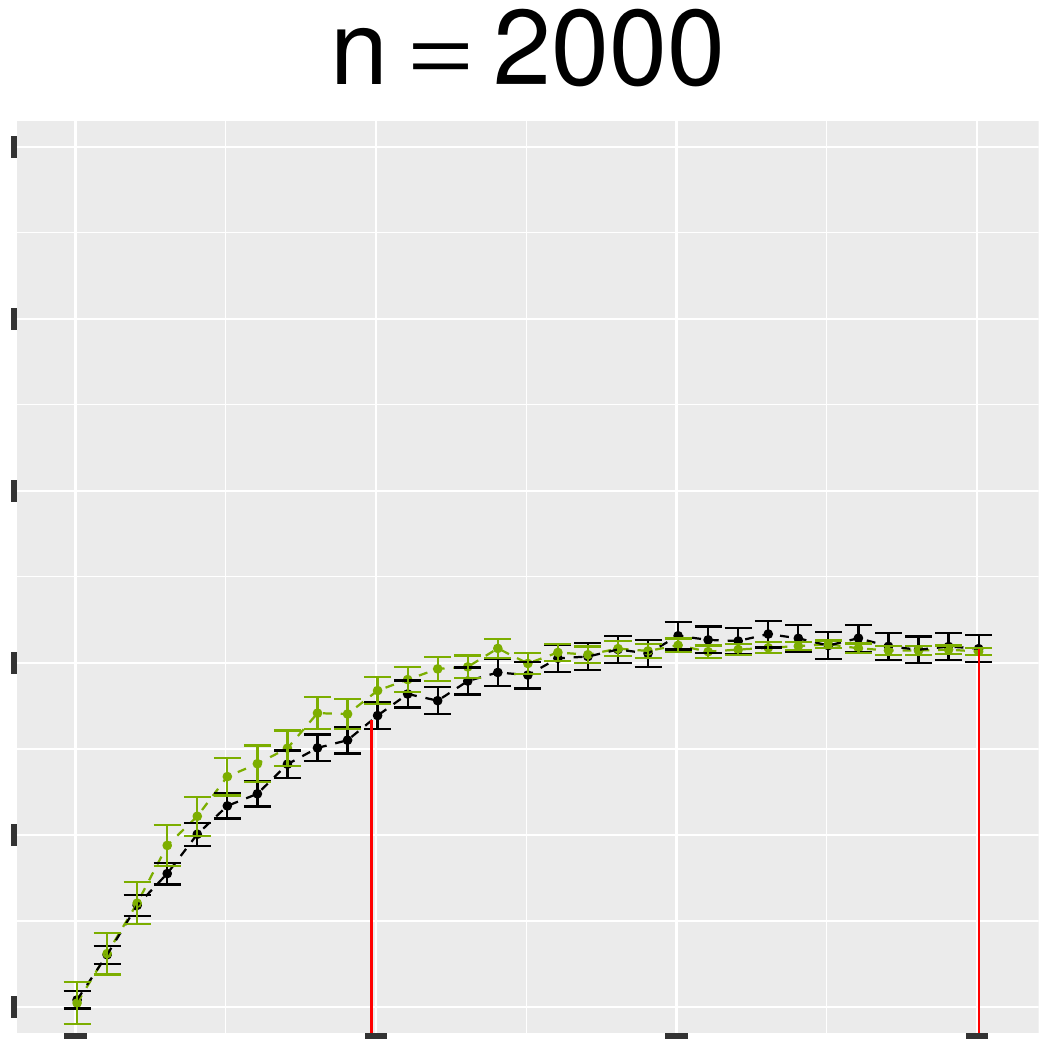}}
    \vspace{-8pt}

    % epsilon=1
    \subfloat{\includegraphics[width=0.285\linewidth]{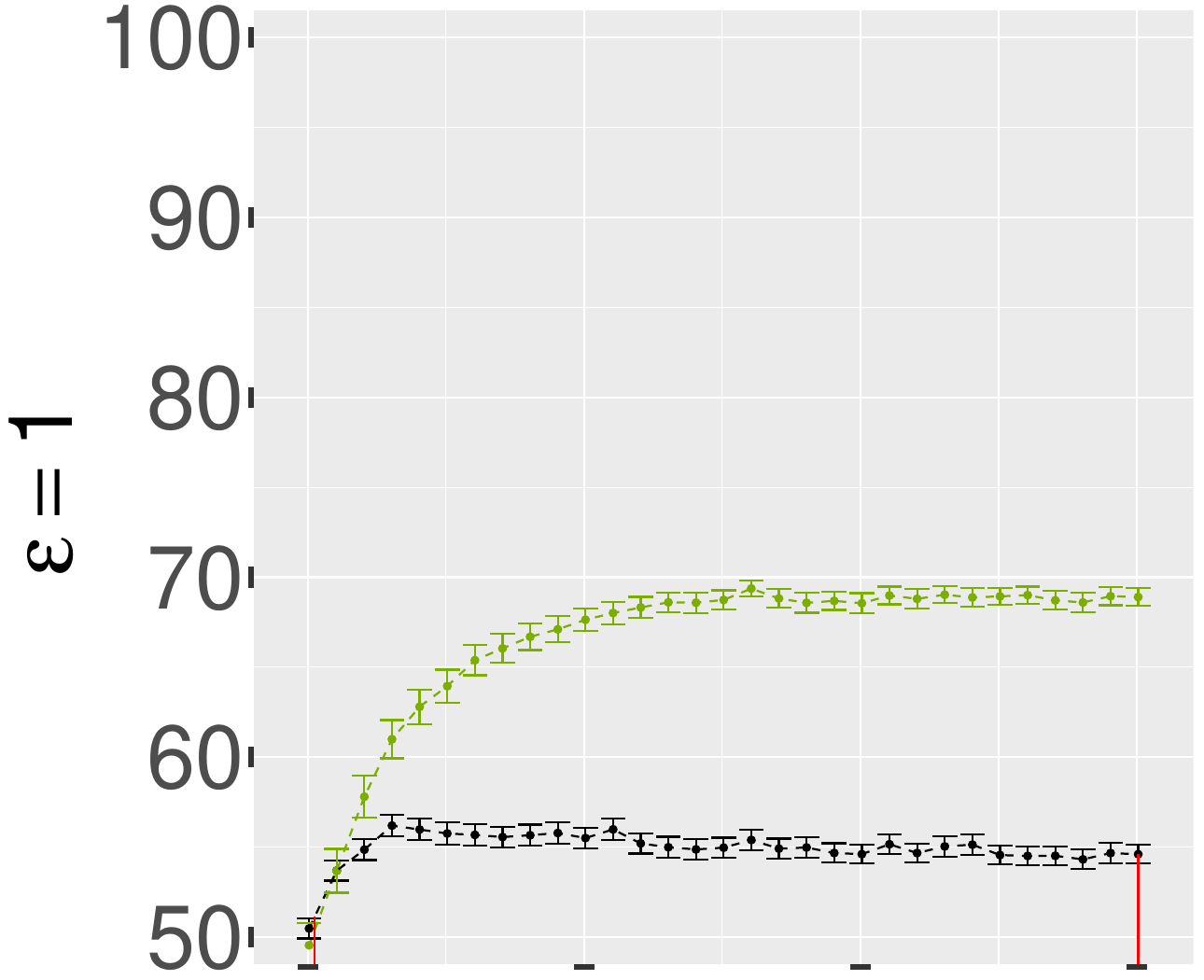}}
    \subfloat{\includegraphics[width=0.235\linewidth]{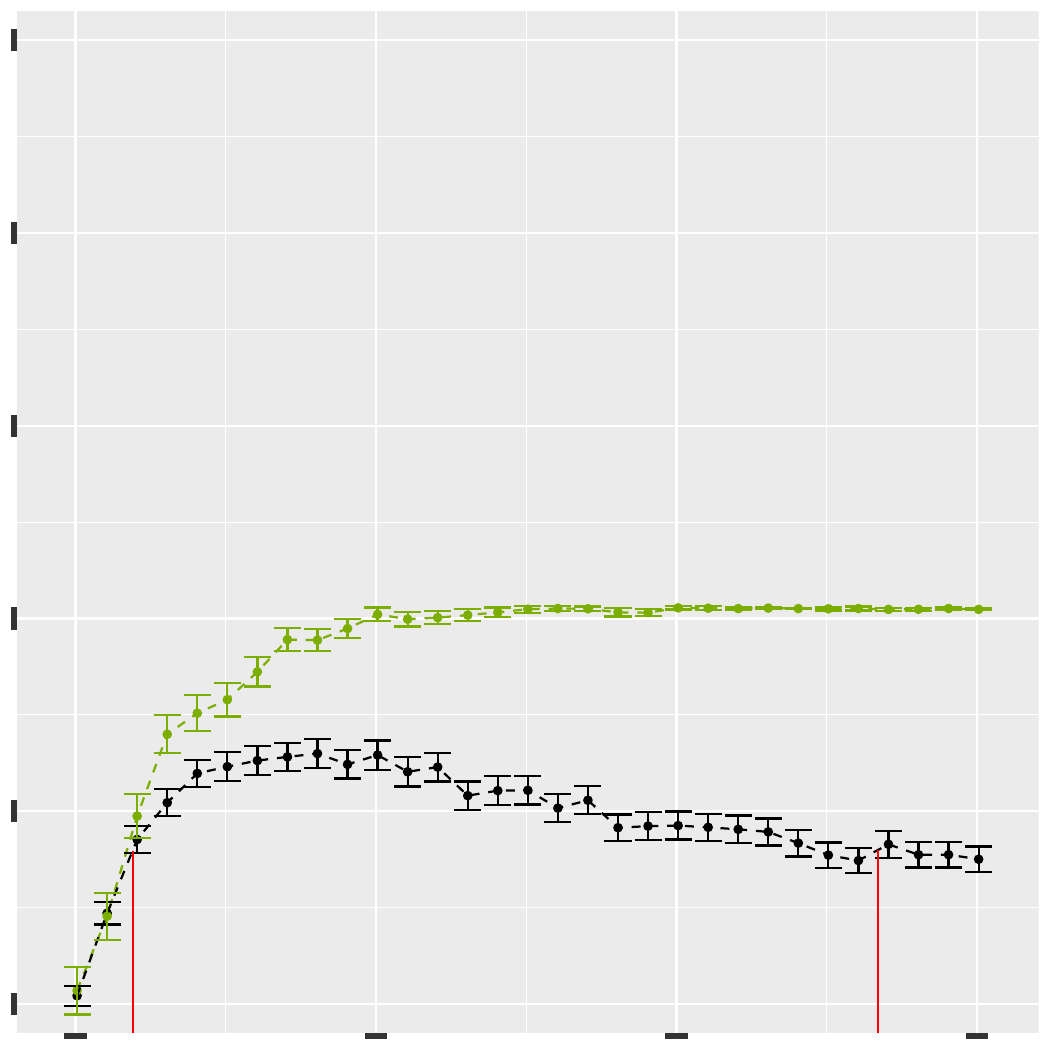}}
    \subfloat{\includegraphics[width=0.235\linewidth]{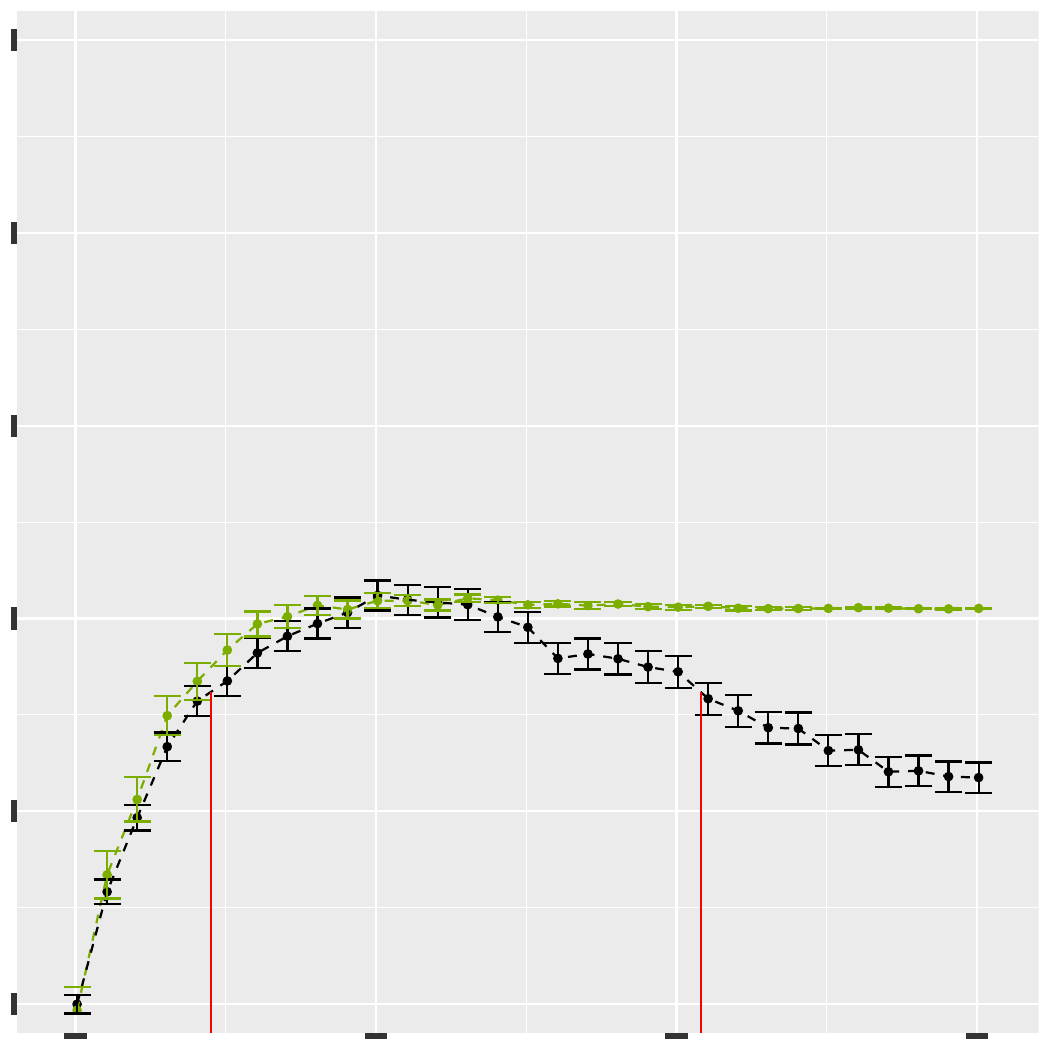}}
    \subfloat{\includegraphics[width=0.235\linewidth]{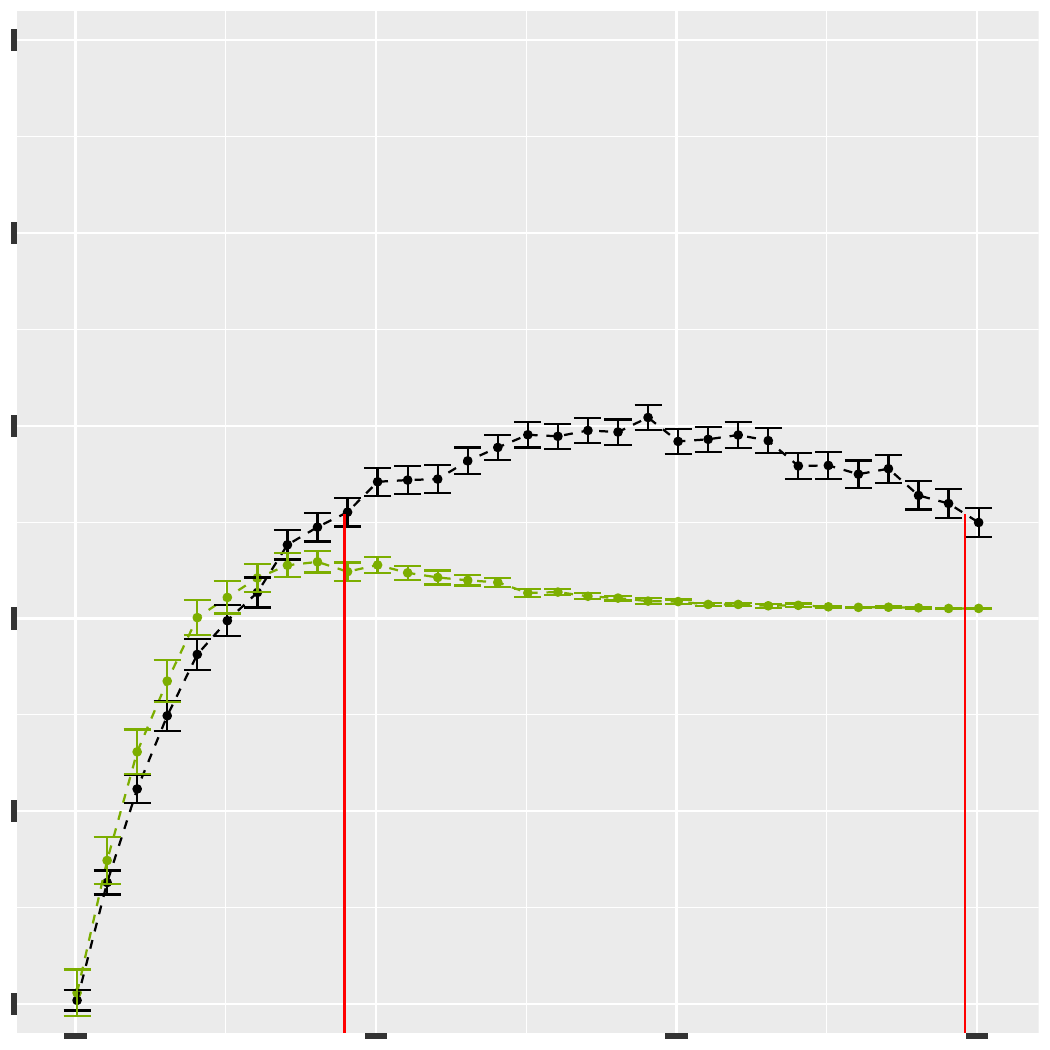}}
    \vspace{-8pt}

    % epsilon=5
    \subfloat{\includegraphics[width=0.285\linewidth]{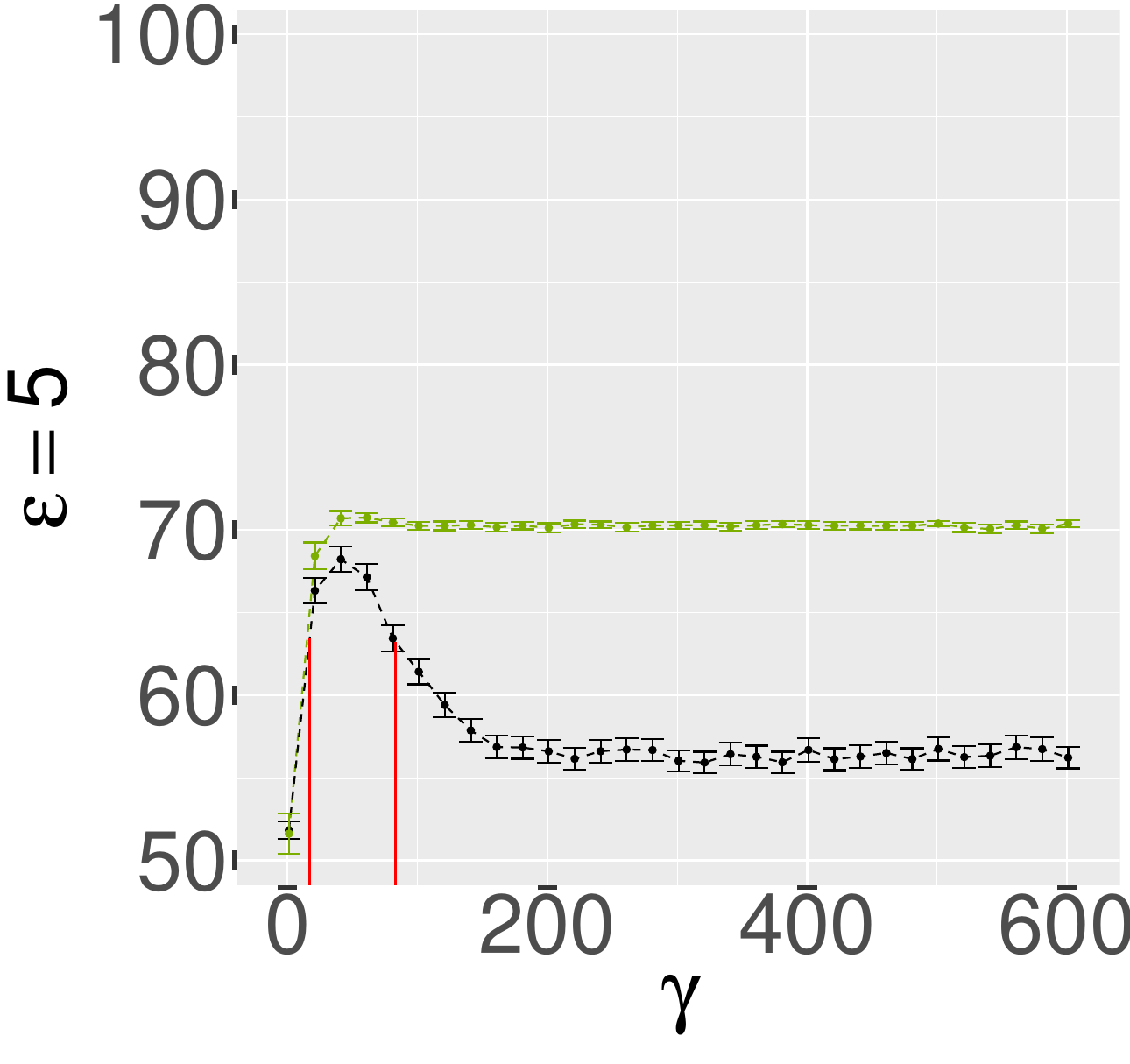}}
    \subfloat{\includegraphics[width=0.235\linewidth]{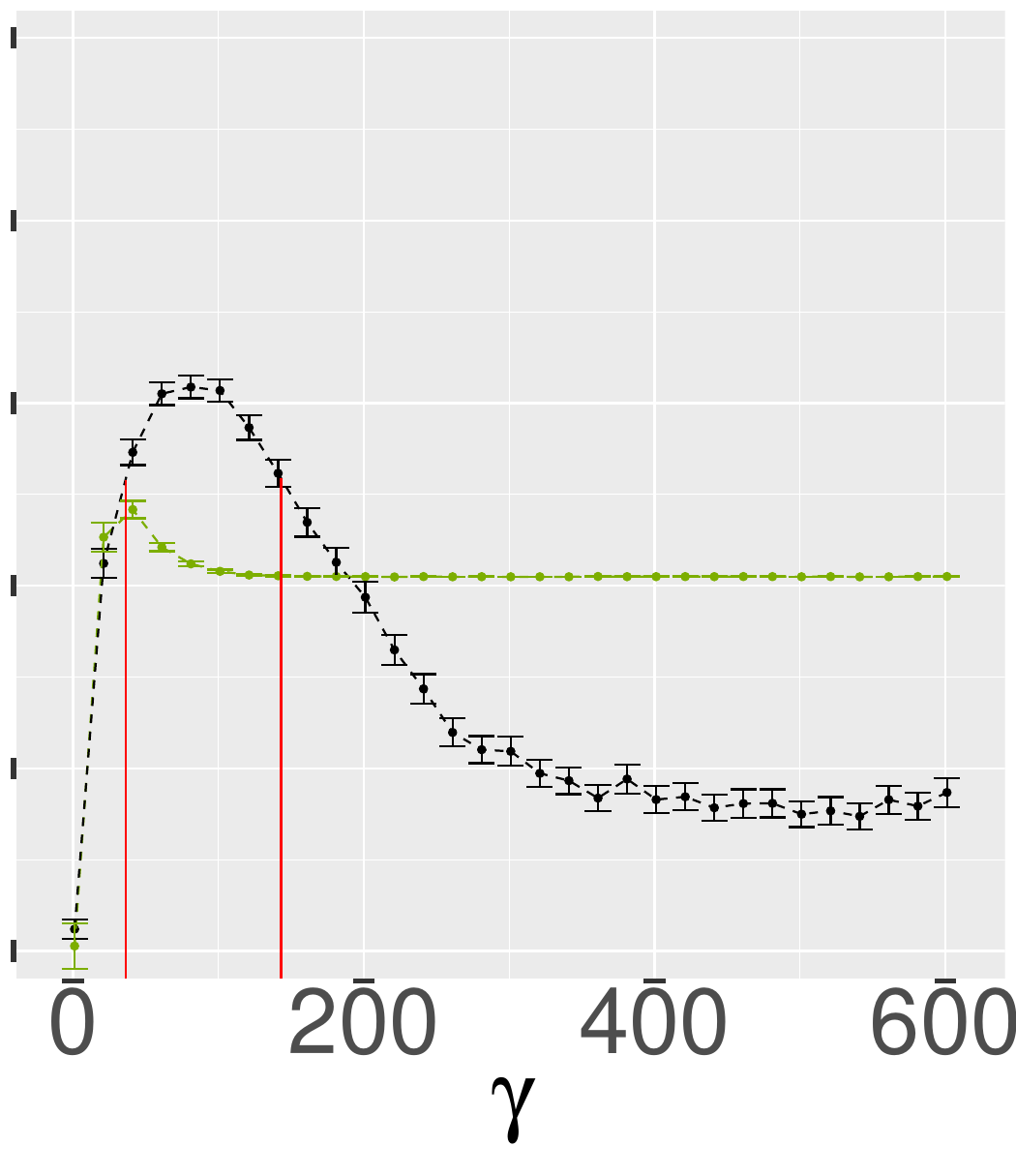}}
    \subfloat{\includegraphics[width=0.235\linewidth]{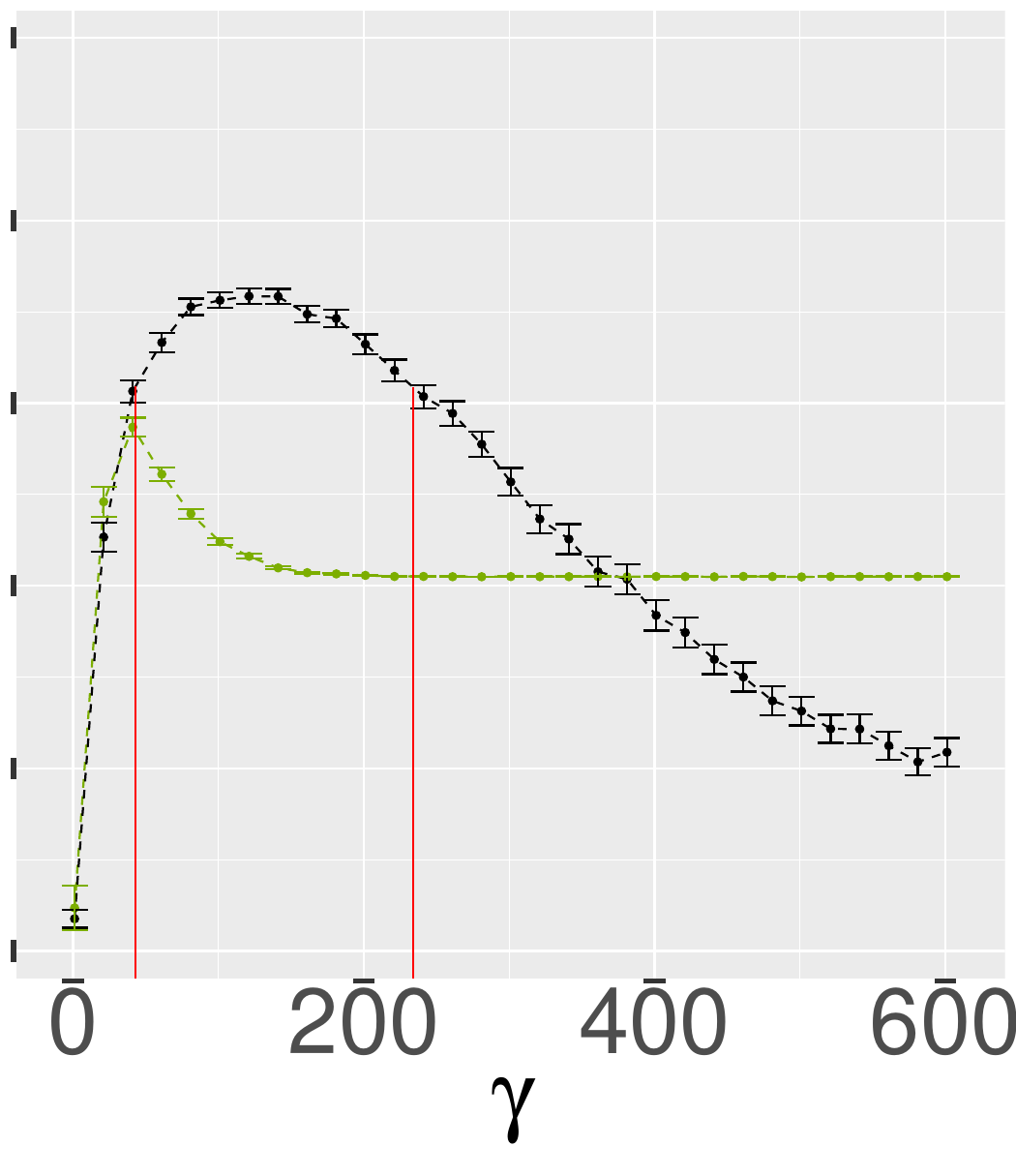}}
    \subfloat{\includegraphics[width=0.235\linewidth]{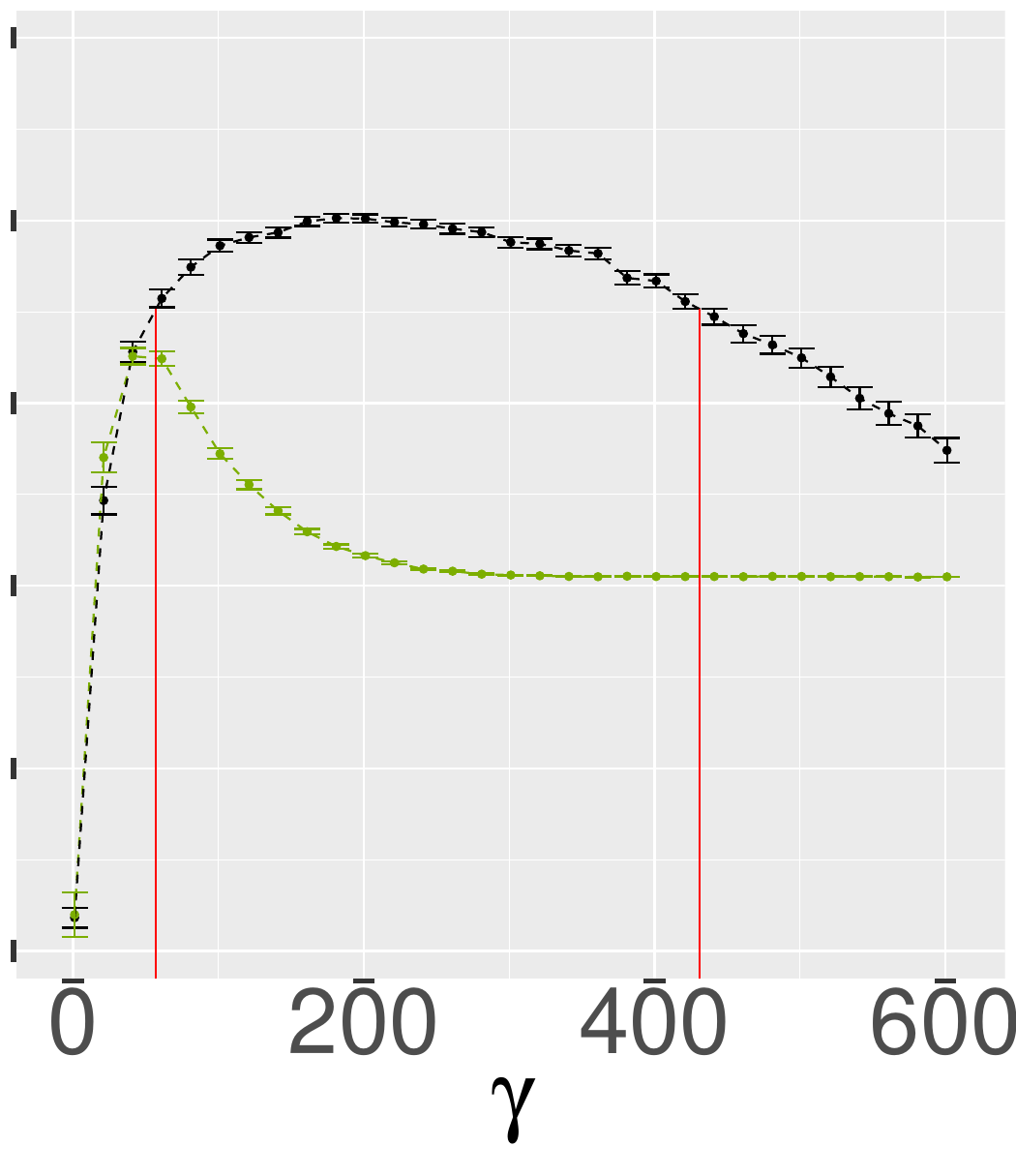}}
    \vspace{-10pt}
    \caption{Sensitivity Analysis of $\gamma$ Selection: average (95\% CI) validation set accuracy (\%) at different $\gamma$ for various combinations of $\epsilon$ and $n$ with varying distributions for features $\mathbf{x}$ in independent dataset $D_0$ (green: truncated normal (between 0 and 1) normal with $\sigma=0.3$ and $\mu_1\!=\!0.25, \mu_2\!=\!0.5, \mu_3\!=\!0.75, \mu_4\!=\!0.5$ for $x_1$ to $x_4$, respectively; black: $\mathbf{x}\sim$ Unif$(0,1)$, the baseline case). The red vertical lines mark the region of $\gamma$ where the corresponding accuracy in the baseline case is within $5\%$ of the highest accuracy.}
    \label{fig:sens_study_dist}\vspace{-12pt}
\end{figure}

%--------------------------------------
\vspace{-3pt}\section{Discussion}
\label{sec:discussion}
We have proposed a general DP-wERM algorithm to allow for weighted loss functions in ERM problems.  We applied the DP-wERM algorithm to OWL that is often used in medical studies. The experimental results demonstrate the feasibility of achieving satisfactory performance in learning optimal ITR  and empirical treatment value, while ensuring formal privacy guarantees for individuals at a reasonably small privacy budget in datasets of adequate sample sizes. This encouraging finding reassures study participants that their privacy is protected and demonstrate to researchers and investigators the practical utility of privacy-preserving ITRs trained via DP-wERM. 

This work represents an early effort in developing DP-ITR methods, applies the proposed DP-wERM to the static treatment OWL framework and focuses the evaluation in RCTs in the experiments. Future work should explore DP algorithms for alternative ITR frameworks like RWL and M-learning that can be more effective or general than OWL for ITR learning. As a a starting point, we derive the $\ell_2$ global sensitivity of the M-learning output and present the result in the SM. 
It is also of interest to extend the DP procedures to broader clinical applications, including dose-finding trials, multi-treatment settings, and multi-stage decision-making.  
Evaluating DP-OWL in larger observational regimes, where ITRs are commonly applied -- such as EHRs, personalized advertising, and recommender systems --  is another promising direction. Because these applications draw massive datasets from diverse sources (e.g., clinical networks, insurance, or e-commerce) compared to typical randomized trials, the larger sample sizes are expected to improve the DP-wERM algorithm's empirical utility.

Future work will also look into improving the DP-wERM  performance under limited sample sizes or small privacy budgets -- a challenging problem given the inherently competing objectives of DP and OWL:  DP bounds individual influence to protect privacy whereas OWL explicitly weights specific individual contributions to optimize treatment rules. This tension only intensifies  as $n$ or $\epsilon$ decreases. 
\cite{chaudhuri2011} shows that objective function perturbation for unweighted ERM problems typically yields higher utility than output perturbation at the same privacy loss, suggesting adapting objective perturbation to the wERM framework is a promising direction. Additionally, relaxing strict $\epsilon$-DP to $(\epsilon, \delta)$-DP or alternative DP notions could offer an advantageous utility trade-off.

\section*{Data and Code}
The general DP-wERM algorithm is coded in R package  \texttt{DPpack} \cite{DPpack, Giddens2023}. The code for this work is at \url{https://github.com/sgiddens/DP-OWL}. 
Due to protocol restrictions, the SJCRH data are available only upon St. Jude IRB approval and submission of a formal research proposal to Drs. Kevin Krull and Tara Brinkman. The data from the CYP-GUIDES study is available at Kaggle \cite{kaggle_cypguides}.

\bibliographystyle{IEEEtran}
\bibliography{references}

@article{Zhao2012,
  ISSN = {01621459},
  author = {Yingqi Zhao 
    and Donglin Zeng 
    and John Rush 
    and Michael Kosorok},
  journal = {J. Amer. Statist. Assoc.},
  number = {499},
  pages = {1106-1118},
  title = {Estimating Individualized Treatment Rules Using Outcome Weighted Learning},
  volume = {107},
  year = {2012}
}

@book{van2000asymptotic,
  title={Asymptotic Statistics},
  author={Van der Vaart, Aad W},
  volume={3},
  year={2000},
  publisher={Cambridge University Press}
}

@article{stefanski2002calculus,
  title={The Calculus of M-estimation},
  author={Stefanski, Leonard A and Boos, Dennis D},
  journal={The American Statistician},
  volume={56},
  number={1},
  pages={29-38},
  year={2002}
}

@article{ma2005robust,
  title={Robust Semiparametric M-estimation and the Weighted Bootstrap},
  author={Ma, Shuangge and Kosorok, Michael R},
  journal={Journal of Multivariate Analysis},
  volume={96},
  number={1},
  pages={190-217},
  year={2005}
}

@article{duchi2021statistics,
  title={Statistics of Robust Optimization: {A} Generalized Empirical Likelihood Approach},
  author={Duchi, John C and Glynn, Peter W and Namkoong, Hongseok},
  journal={Mathematics of Operations Research},
  volume={46},
  number={3},
  pages={946-969},
  year={2021}
}

@inproceedings{huber1967behavior,
  title={The Behavior of Maximum Likelihood Estimates Under Nonstandard Conditions},
  author={Huber, Peter J and others},
  booktitle={Proc. Fifth Berkeley Symposium on Mathematical Statistics and Probability},
  volume={1},
  pages={221-233},
  year={1967}
}

@article{Qian2011,
  author = {Min Qian and Susan A. Murphy},
  title = {Performance Guarantees for Individualized Treatment Rules},
  volume = {39},
  journal = {Ann. Statist.},
  number = {2},
  pages = {1180 - 1210},
  year = {2011}
}

@article{Sun2015, 
  title={Causal Inference via Sparse Additive Models with Application to Online Advertising}, 
  volume={29}, 
  number={1}, 
  journal={Proc.  AAAI Conf. Artif. Intell.}, 
author={Wei Sun and Pengyuan Wang and Dawei Yin and Jian Yang and Yi Chang}, 
  year={2015}
}

@inproceedings{Wang2015,
  author = {Pengyuan Wang and Wei Sun and Dawei Yin and Jian Yang and Yi Chang},
  title = {Robust Tree-Based Causal Inference for Complex Ad Effectiveness Analysis},
  year = {2015},
  isbn = {9781450333177},
   booktitle = {Proc. Eighth ACM Int. Conf. Web Search Data Mining (WSDM)},
  pages = {67-76}
}

@inproceedings{Schnabel2016,
  author = {Tobias Schnabel and Adith Swaminathan and Ashudeep Singh and Navin Chandak and Thorsten Joachims},
  title = {Recommendations as Treatments: {D}ebiasing Learning and Evaluation},
  year = {2016},
  booktitle = {Proc. 33rd Int. Conf. Mach. Learn.},
  pages = {1670–1679}
}

@inproceedings{Lada2019,
  author = {Akos Lada and Alexander Peysakhovich and Diego Aparicio and Michael Bailey},
  title = {Observational Data for Heterogeneous Treatment Effects with Application to Recommender Systems},
  year = {2019},
  booktitle = {Proc. Conf. Econ. Comput.},
  pages = {199-213}
}

@INPROCEEDINGS{Narayanan2008,
  author={Arvind Narayanan and Vitaly Shmatikov},
  booktitle={Proc. IEEE Symp. Secur. Privacy}, 
  title={Robust De-anonymization of Large Sparse Datasets}, 
  year={2008},
  pages={111-125}
}

@article{Sweeney2015,
  author={Latanya Sweeney},
  title={Only You, Your Doctor, and Many Others May Know},
  year={2015},
  month={September},
  day={28},
  url={https://techscience.org/a/2015092903/},
  journal={Technology Science}
}

@article{Ahn2015,
    title={Whose Genome Is It Anyway?: {R}e-identification and Privacy Protection in Public and Participatory Genomics},
    author={Sejin Ahn},
    journal={The San Diego law review},
    year={2015},
    volume={52},
    pages={751}
}

@misc{Desfontain2021,
    author = {Damien Desfontain},
    year = {2021},
    title = {Demystifying the {US Census Bureau's} reconstruction attack},
    howpublished = {\url{https://desfontain.es/blog/us-census-reconstruction-attack.html}},
    note = {Accessed: August 20, 2025}
}

@article{balle2018privacy,
  title={Privacy Amplification by Subsampling: {T}ight Analyses via Couplings and Divergences},
  author={Balle, Borja and Barthe, Gilles and Gaboardi, Marco},
  journal={NeuRIPS},
  volume={31},
  year={2018}
}

@INPROCEEDINGS{Shokri2017,
  author={Reza Shokri 
    and Marco Stronati 
    and Congzheng Song 
    and Vitaly Shmatikov},
  booktitle={2017 Proc. IEEE Symp. Secur. Privacy (SP)},
  title={Membership Inference Attacks Against Machine Learning Models},
  year={2017},
  volume={},
  number={},
  pages={3-18}
}

@INPROCEEDINGS{Zhao2021,
  author={Benjamin Zi and Hao Zhao 
    and Aviral Agrawal
    and Catisha Coburn
    and Hassan Jameel Asghar
    and Raghav Bhaskar 
    and Mohamed Ali Kaafar 
    and Darren Webb 
    and Peter Dickinson},
  booktitle={2021 IEEE European Symposium on Security and Privacy},
  title={On the (In)Feasibility of Attribute Inference Attacks on Machine Learning Models},
  year={2021},
  volume={},
  number={},
  pages={232-251}}

@InProceedings{Dwork2006,
  author={Cynthia Dwork
    and Frank McSherry
    and Kobbi Nissim
    and Adam Smith},
  title={Calibrating Noise to Sensitivity in Private Data Analysis},
  booktitle={Theory of Cryptography},
  year={2006},
  publisher={Springer Berlin Heidelberg},
  pages={265-284}
}

@inproceedings{Chaudhuri2008,
    author = {Chaudhuri, Kamalika and Monteleoni, Claire},
    title = {Privacy-preserving logistic regression},
    year = {2008},
    isbn = {9781605609492},
    booktitle = {Proc. NeurIPS},
    pages = {289-296}
}

@article{chaudhuri2011,
  author  = {Kamalika Chaudhuri
    and Claire Monteleoni
    and Anand D. Sarwate},
  title   = {Differentially Private Empirical Risk Minimization},
  journal = {J. Mach. Learn. Res.},
  year    = {2011},
  volume  = {12},
  number  = {29},
  pages   = {1069-1109}
}

@inproceedings{abadi2016deep,
  title={Deep Learning with Differential Privacy},
  author={Abadi, Martin and Chu, Andy and Goodfellow, Ian and McMahan, H Brendan and Mironov, Ilya and Talwar, Kunal and Zhang, Li},
  booktitle={Proc. ACM Conf. Comput. Commun. Secur. (CCS)},
  pages={308-318},
  year={2016}
}

@inproceedings{mcsherry2007mechanism,
  title={Mechanism Design via Differential Privacy},
  author={McSherry, Frank and Talwar, Kunal},
  booktitle={Proc. 48th FOCS},
  pages={94-103},
  year={2007}
}

@InProceedings{Kifer2012,
  title = 	 {Private Convex Empirical Risk Minimization and High-dimensional Regression},
  author = 	 {Daniel Kifer
    and Adam Smith
    and Abhradeep Thakurta},
  booktitle = {Proc. 25th Conf. Learn. Theory (COLT)},
  pages = 	 {25.1-25.40},
  year = 	 {2012},
  volume = 	 {23}
}

@article{Rubenstein2012, 
  title={Learning in a Large Function Space: {P}rivacy-Preserving Mechanisms for {SVM} Learning},
  volume={4},
  number={1},
  journal={Journal of Privacy and Confidentiality},
  author={Benjamin I. P. Rubinstein
    and Peter L. Bartlett 
    and Ling Huang 
    and Nina Taft},
  year={2012}
}

@article{Zhou2017,
  author = {Xin Zhou 
    and Nicole Mayer-Hamblett 
    and Umer Khan 
    and Michael Kosorok},
  title = {Residual Weighted Learning for Estimating Individualized Treatment Rules},
  journal = {J. Amer. Statist. Assoc.},
  volume = {112},
 number = {517},
  pages = {169-187},
  year  = {2017}}

@article{Wu2020,
  author = {Peng Wu 
    and Donglin Zeng 
    and Yuanjia Wang},
  title = {Matched Learning for Optimizing Individualized Treatment Strategies Using Electronic Health Records},
  journal = {J. Amer. Statist. Assoc.},
  volume = {115},
  number = {529},
  pages = {380-392},
  year  = {2020}
}

@InProceedings{Niu2022,
  title = {Differentially Private Estimation of Heterogeneous Causal Effects},
  author = {Fengshi Niu and Harsha Nori and Brian Quistorff and Rich Caruana and Donald Ngwe and Aadharsh Kannan},
  booktitle = {Proc. First Proc. Conf. Causal Learn. Reasoning (CLEAR)},
  pages = {618-633},
  year = {2022},
  volume = 	 {177}
}

@unpublished{Betlei2021,
  title = {Differentially Private Individual Treatment Effect Estimation from Aggregated Data},
  author = {Artem Betlei and Th{\'e}ophane Gregoir and Thibaud Rahier and Alo{\"i}s Bissuel and Eustache Diemert and Massih-Reza Amini},
  url = {https://hal.science/hal-03339723},
  year = {2021},
  month = {Sep},
  HAL_ID = {hal-03339723},
  HAL_VERSION = {v1}
}

@inproceedings{Rahimi2007,
 author = {Ali Rahimi and
    Benjamin Recht},
 booktitle = {NeuRIPS},
 pages = {},
 title = {Random Features for Large-Scale Kernel Machines},
 volume = {20},
 year = {2007}
}

@inproceedings{Rahimi2008,
 author = {Ali Rahimi and Benjamin Recht},
 booktitle = {NeuRIPS},
 pages = {},
 title = {Weighted Sums of Random Kitchen Sinks: {R}eplacing Minimization with Randomization in Learning},
 volume = {21},
 year = {2008}
}

@inproceedings{Bassily2014,
  author={Raef Bassily and Adam Smith and Abhradeep Thakurta},
  booktitle={Proc. 55th FOCS}, 
  title={Private Empirical Risk Minimization: {E}fficient Algorithms and Tight Error Bounds}, 
  year={2014},
  volume={},
  number={},
  pages={464-473}
}

@InProceedings{Kasiviswanathan2016,
  title = 	 {Efficient Private Empirical Risk Minimization for High-dimensional Learning},
  author = 	 {Shiva Prasad Kasiviswanathan and Hongxia Jin},
  booktitle = 	 {Proc. 33rd Int. Conf. Mach. Learn.},
  pages = 	 {488-497},
  year = 	 {2016},
  volume = 	 {48}
}

@InProceedings{Dwork2006b,
  author={Cynthia Dwork
    and Krishnaram Kenthapadi
    and Frank McSherry
    and Ilya Mironov
    and Moni Naor},
  title={Our Data, Ourselves: {P}rivacy Via Distributed Noise Generation},
  booktitle={Adv. Cryptol. (EUROCRYPT'06)},
  year={2006},
  publisher={Springer Berlin Heidelberg},
  pages={486-503}
}

@INPROCEEDINGS{Dwork2010,
  author={Cynthia Dwork
    and Guy N. Rothblum 
    and Salil Vadhan},
  booktitle={Proc. 51st FOCS}, 
  title={Boosting and Differential Privacy}, 
  year={2010},
  volume={},
  number={},
  pages={51-60}
}

@ARTICLE{Liu2019,
  author = {Fang Liu},
  journal={IEEE Trans.~ on Knowledge and Data Engineering},
  title={Generalized Gaussian Mechanism for Differential Privacy},
  year={2019},
  volume={31},
  number={4},
  pages={747-756}
}

@ARTICLE{Chapelle2007,
  author={Olivier Chapelle},
  journal={Neural Computation},
  title={Training a Support Vector Machine in the Primal},
  year={2007},
  volume={19},
  number={5},
  pages={1155-1178}
}

@article{Dong2022,
    author = {Dong, Jinshuo and Roth, Aaron and Su, Weijie J.},
    title = {Gaussian Differential Privacy},
    journal = {Journal of the Royal Statistical Society Series B: Statistical Methodology},
    volume = {84},
    number = {1},
    pages = {3-37},
    year = {2022},
    month = {02},
    issn = {1369-7412}
}

@InProceedings{Thakurta2013,
  title = 	 {Differentially Private Feature Selection via Stability Arguments, and the Robustness of the Lasso},
  author = 	 {Abhradeep Guha Thakurta and Adam Smith},
  booktitle = 	 {Proc. 26th Conf. Learn. Theory (COLT)},
  pages = 	 {819-850},
  year = 	 {2013},
  volume = 	 {30}
}

@INPROCEEDINGS{Steinke2017,
  author={Thomas Steinke and Jonathan Ullman},
  booktitle={Proc. 58th FOCS}, 
  title={Tight Lower Bounds for Differentially Private Selection}, 
  year={2017},
  volume={},
  number={},
  pages={552-563}
}

@article{Chaudhuri2013,
  title={A Near-Optimal Algorithm for Differentially-Private Principal Components.},
  author={Kamalika Chaudhuri and Anand D Sarwate and Kaushik Sinha},
  journal={J. Mach. Learn. Res.},
  volume={14},
  year={2013}
}

@Manual{DPpack,
  title = {{DPpack}: {D}ifferentially Private Statistical Analysis and Machine Learning},
  author = {Spencer Giddens and Fang Liu},
  note = {{R} package version 0.1.0},
  url = {https://cran.r-project.org/package=DPpack},
  year = {2023}
}

@article{Giddens2023,
  title={{DPpack}: An {R} Package for Differentially Private Statistical Analysis and Machine Learning}, 
  author={Spencer Giddens and Fang Liu},
  year={2023},
  journal={arXiv preprint arXiv:2309.10965},
}

@misc{national2011toward,
  title={National Research Council: Toward Precision Medicine: {B}uilding a Knowledge Network for Biomedical Research and a New Taxonomy of Disease},
  year={2011},
  publisher={National Academies Press}
}

@article{collins2015new,
  title={A New Initiative on Precision Medicine},
  author={Collins, Francis S and Varmus, Harold},
  journal={New England Journal of Medicine},
  volume={372},
  number={9},
  pages={793-795},
  year={2015}
}

@article{chen2016personalized,
  title={Personalized Dose Finding Using Outcome Weighted Learning},
  author={Chen, Guanhua and Zeng, Donglin and Kosorok, Michael R},
  journal={J. Amer. Statist. Assoc.},
  volume={111},
  number={516},
  pages={1509-1521},
  year={2016}
}

@inproceedings{zhou2018outcome,
  title={Outcome-Weighted Learning for Personalized Medicine with Multiple Treatment Options},
  author={Zhou, Xuan and Wang, Yuanjia and Zeng, Donglin},
  booktitle={2018 IEEE 5th Int. Conf. Data Sci. Adv. Anal. (DSAA)},
  pages={565-574},
  year={2018}
}

@Inbook{eguchi2022outcome,
    author={Eguchi, Shinto
        and Komori, Osamu},
    title={Outcome Weighted Learning in Dynamic Treatment Regimes},
    bookTitle={Minimum Divergence Methods in Statistical Machine Learning: From an Information Geometric Viewpoint},
    year={2022},
    publisher={Springer Japan},
    address={Tokyo},
    pages={197-216},
    isbn={978-4-431-56922-0}
}

@article{lubas2022randomized,
  title={A Randomized Double-blind Placebo-controlled Trial of the Effectiveness of Melatonin on Neurocognition and Sleep in Survivors of Childhood Cancer},
  author={Lubas, Margaret and Mandrell, Belinda and Greene, William L and Howell, Carrie and Christensen, Robbin and Kimberg, Cara and Li, Chenghong and Ness, Kirsten and Srivastava, Deo Kumar and Hudson, Melissa},
  journal={Pediatric Blood \& Cancer},
  volume={69},
  number={1},
  pages={e29393},
  year={2022}
}

@article{Spicker2024,
    author = {Spicker, Dylan and Moodie, Erica E. M. and Shortreed, Susan M.},
    title = {Differentially Private Outcome-Weighted Learning for Optimal Dynamic Treatment Regime Estimation},
    journal = {Stat},
    year = {2024},
    volume = {13},
    issue = {1},
    pages={e641}
}

@article{wang2024matching,
  title={A Matching-based Machine Learning Approach to Estimating Optimal Dynamic Treatment Regimes with Time-to-event Outcomes},
  author={Wang, Xuechen and Lee, Hyejung and Haaland, Benjamin and Kerrigan, Kathleen and Puri, Sonam and Akerley, Wallace and Shen, Jincheng},
  journal={Statistical Methods in Medical Research},
  volume={33},
  number={5},
  pages={794-806},
  year={2024}
}

@article{tortora2020clinical,
  title={Clinical database of the CYP-guides trial: an open data resource on psychiatric hospitalization for severe depression},
  author={Tortora, Joseph and Robinson, Saskia and Baker, Seth and Rua{\~n}o, Gualberto},
  journal={Data in Brief},
  volume={30},
  pages={105457},
  year={2020},
  publisher={Elsevier}
}

@article{ruano2020results,
  title={Results of the CYP-GUIDES randomized controlled trial: Total cohort and primary endpoints},
  author={Rua{\~n}o, Gualberto and Robinson, Saskia and Holford, Theodore and Mehendru, Raveen and Baker, Seth and Tortora, Joseph and Goethe, John W},
  journal={Contemporary clinical trials},
  volume={89},
  pages={105910},
  year={2020},
  publisher={Elsevier}
}

@misc{kaggle_cypguides,
    author = {{Kaggle}},
  title        = {Clinical Dataset of the CYP-GUIDES Trial},
  howpublished = {\url{https://www.kaggle.com/datasets/shashwatwork/clinical-dataset-of-the-cypguides-trial}}
}
\vspace{-24pt}

\begin{IEEEbiography}[{\includegraphics[width=1in,height=1.25in,clip,keepaspectratio]{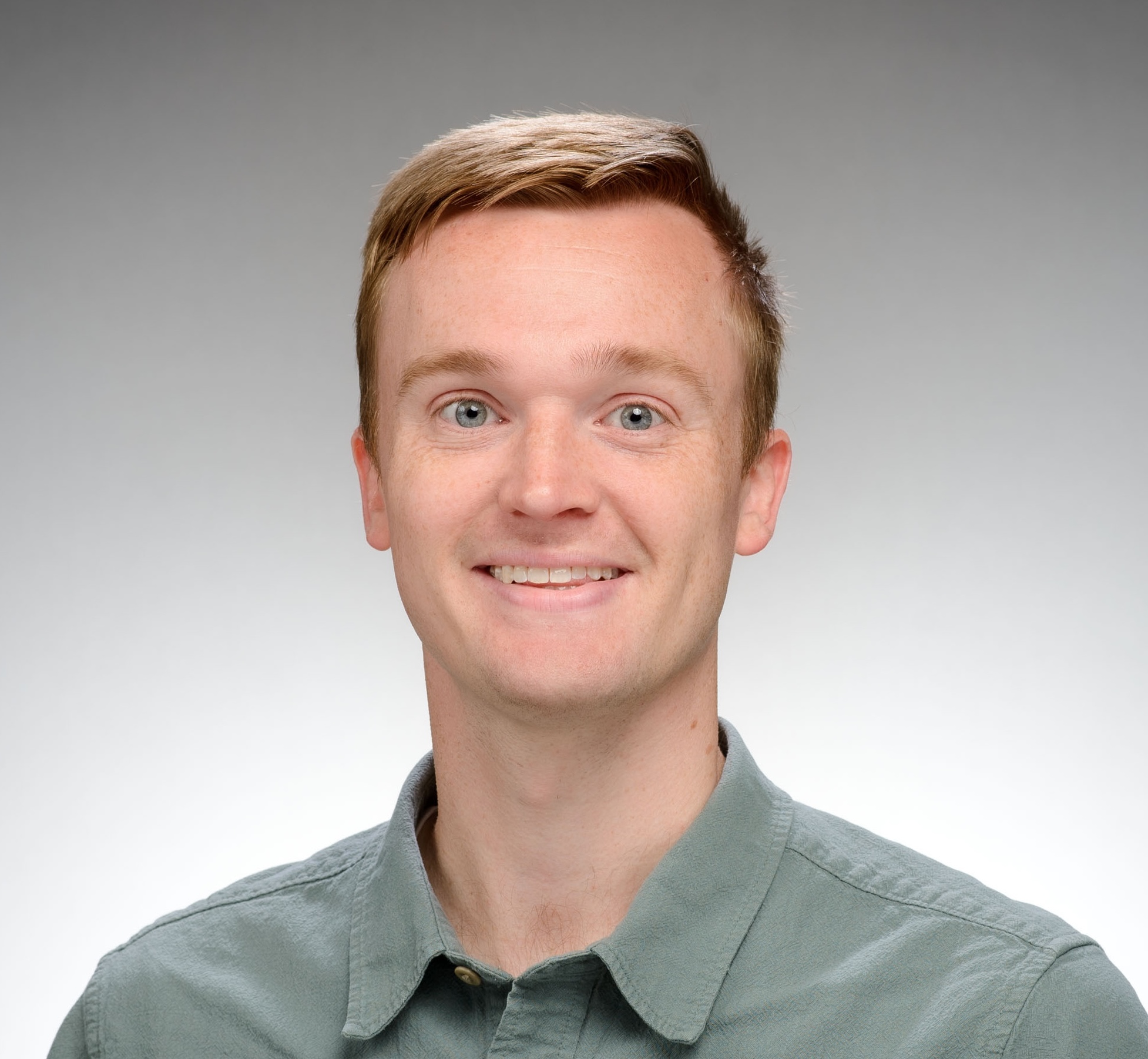}}]{Spencer Giddens}
is a Postdoctoral Researcher with the AI Trust and Reliability  Lab in the Lucy Family Institute for Data \& Society at the University of Notre Dame, Notre Dame, IN, USA.
He received his Ph.D. in Applied and Computational Mathematics and Statistics from Notre Dame in 2025. 
His research interests include differential privacy and its applications to statistical methods and machine learning, synthetic data, and AI fairness.
\end{IEEEbiography}

% \begin{IEEEbiographynophoto}{Spencer Giddens}
% is a Postdoctoral Researcher with the AI Trust and Reliability  Lab in the Lucy Family Institute for Data \& Society at the University of Notre Dame, Notre Dame, IN, USA.
% He received his Ph.D. in Applied and Computational Mathematics and Statistics from Notre Dame in 2025. 
% His research interests include differential privacy and its applications to statistical methods and machine learning, synthetic data, and AI fairness.
\vspace{-30pt}

\begin{IEEEbiography}[{\includegraphics[width=1in,height=1.25in,clip,keepaspectratio]{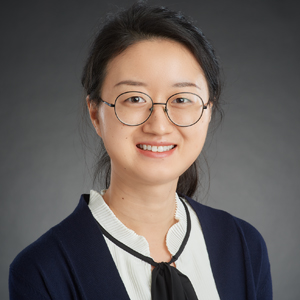}}]{Yiwang Zhou}
is an Assistant Member in the Department of Biostatistics at St. Jude Children's Research Hospital, Memphis, TN, USA. She received her Ph.D. in Biostatistics from the University of Michigan, Ann Arbor, USA. Her research interests include individualized treatment rules for precision medicine, development of machine learning models, and network analysis. 
\end{IEEEbiography}

% \begin{IEEEbiographynophoto}{Yiwang Zhou}
% is an Assistant Member in the Department of Biostatistics at St. Jude Children's Research Hospital, Memphis, TN, USA. She received her Ph.D. in Biostatistics from the University of Michigan, Ann Arbor, USA. Her research interests include individualized treatment rules for precision medicine, development of machine learning models, and network analysis. 
\vspace{-30pt}

\begin{IEEEbiography}[{\includegraphics[width=1in,height=1.25in,clip,keepaspectratio]{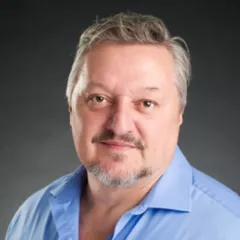}}]{Kevin R. Krull}
is a Member and Chair of the Department of Psychology and Biobehavioral Sciences and an Endowed Chair at St. Jude Children's Research Hospital, Memphis, TN, USA. He earned his Ph.D. in Clinical Psychology and Neuroscience from Florida State University, Tallahassee, Florida, USA. His research interests include neurocognitive outcomes of childhood cancer, late effects of cancer therapy, biologic and psychosocial moderators/mediators, and treatment of attention and executive function problems. 
\end{IEEEbiography}

% \begin{IEEEbiographynophoto}{Kevin R. Krull}
% is a Member and Chair of the Department of Psychology and Biobehavioral Sciences and an Endowed Chair at St. Jude Children's Research Hospital, Memphis, TN, USA. He earned his Ph.D. in Clinical Psychology and Neuroscience from Florida State University, Tallahassee, Florida, USA. His research interests include neurocognitive outcomes of childhood cancer, late effects of cancer therapy, biologic and psychosocial moderators/mediators, and treatment of attention and executive function problems.
\vspace{-30pt}

\begin{IEEEbiography}[{\includegraphics[width=1in,height=1.25in,clip,keepaspectratio]{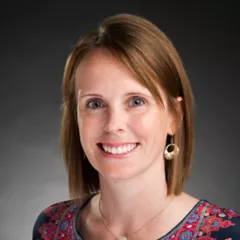}}]{Tara M. Brinkman}
is an Associate Member in the Department of Psychology and Biobehavioral Sciences and the Department of Epidemiology and Cancer Control at St. Jude Children's Research Hospital, Memphis, TN, USA. She earned her Ph.D. in Psychology from Michigan State University, East Lansing, MI, USA. Her research interests include neuro-behavioral late effects of childhood cancer, psychosocial impact of pediatric brain tumors, and psychological and behavioral interventions.
\end{IEEEbiography}

% \begin{IEEEbiographynophoto}{Tara M. Brinkman}
% is an Associate Member in the Department of Psychology and Biobehavioral Sciences and the Department of Epidemiology and Cancer Control at St. Jude Children's Research Hospital, Memphis, TN, USA. She earned her Ph.D. in Psychology from Michigan State University, East Lansing, MI, USA. Her research interests include neuro-behavioral late effects of childhood cancer, psychosocial impact of pediatric brain tumors, and psychological and behavioral interventions.
\vspace{-30pt}

\begin{IEEEbiography}[{\includegraphics[width=1in,height=1.25in,clip,keepaspectratio]{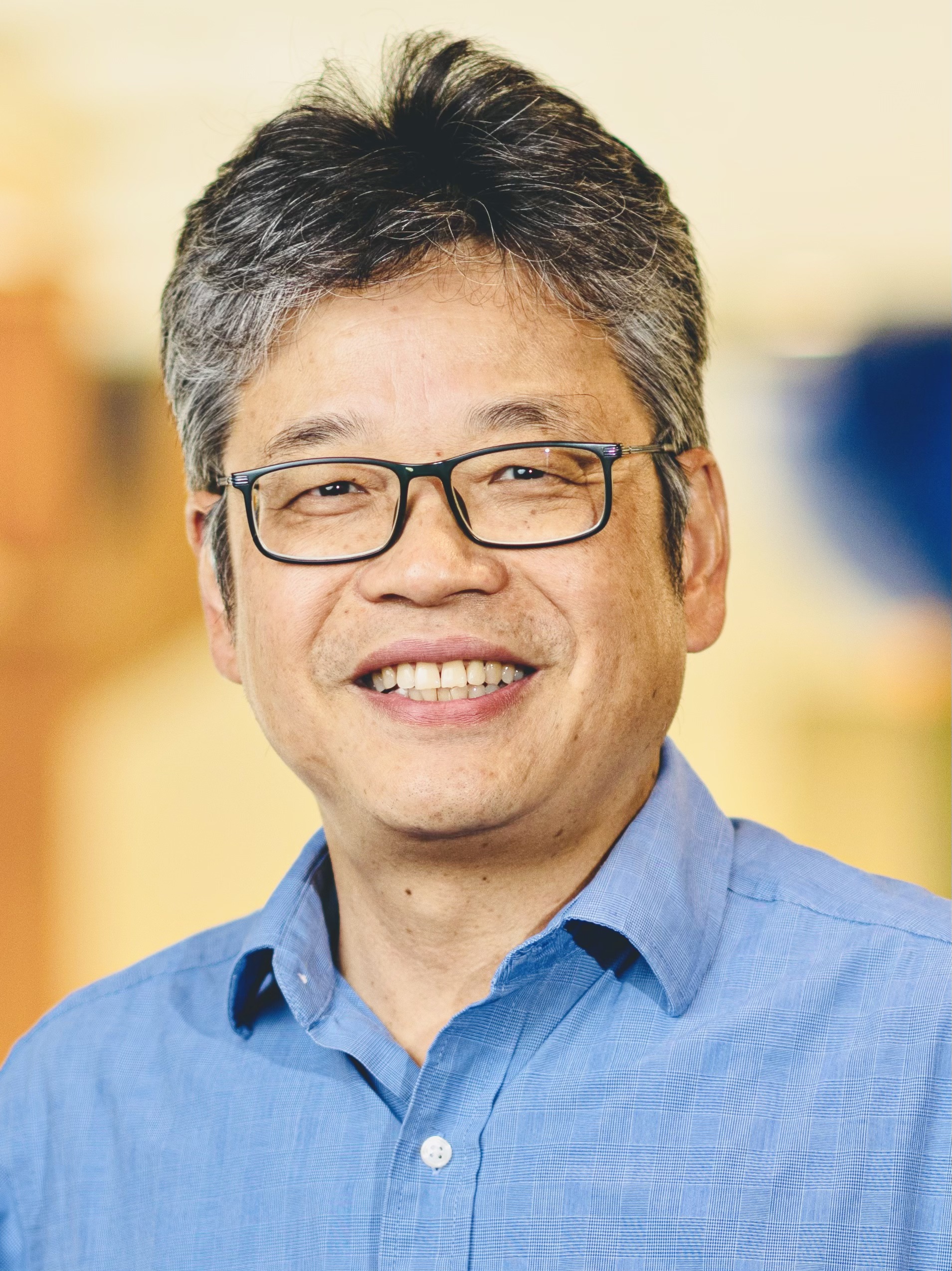}}]{Peter X. K. Song}
is a Professor at the Department of Biostatistics in the School of Public Health, University of Michigan, Ann Arbor, MI, USA. 
He received his Ph.D. in Statistics from the University of British Columbia, Canada. He is a Fellow of the American Statistical Association,  the Institute of Mathematical Statistics, and American Association for the Advancement of Science.
His research interests include data integration analytics, distributed inference, federated learning, mediation analysis and spatiotemporal modeling. 
\end{IEEEbiography}

% \begin{IEEEbiographynophoto}{Peter X. K. Song}
% is a Professor at the Department of Biostatistics in the School of Public Health, University of Michigan, Ann Arbor, MI, USA. 
% He received his Ph.D. in Statistics from the University of British Columbia, Canada. He is a Fellow of the American Statistical Association,  the Institute of Mathematical Statistics, and American Association for the Advancement of Science.
% His research interests include data integration analytics, distributed inference, federated learning, mediation analysis and spatiotemporal modeling. 
\vspace{-30pt}

\begin{IEEEbiography}[{\includegraphics[width=1in,height=1.25in,clip,keepaspectratio]{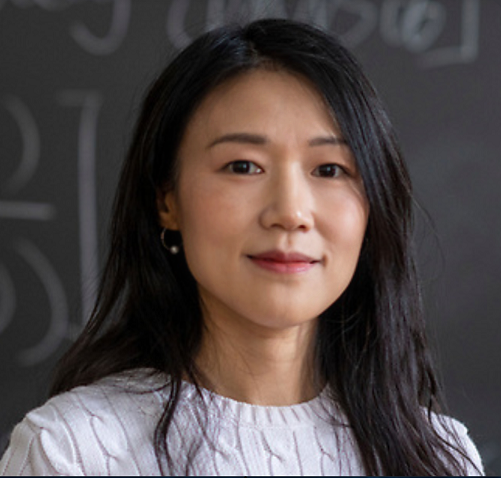}}]{Fang Liu}
is Notre Dame Collegiate Professor in the Department of Applied and Computational Mathematics and Statistics and Director of Health Data Exploration \& Analytics Lab at the Lucy Family Institute for Data \& Society at the University of Notre Dame, Notre Dame, IN, USA.  She obtained her Ph.D. from the University of Michigan, Ann Arbor, USA. She is a Fellow of the American Statistical Association. Her research interests include data privacy, synthetic data, trustworthy machine learning, and Bayesian statistics.
\end{IEEEbiography}

% \begin{IEEEbiographynophoto}{Fang Liu}
% is Notre Dame Collegiate Professor in the Department of Applied and Computational Mathematics and Statistics and Director of Health Data Exploration \& Analytics Lab at the Lucy Family Institute for Data \& Society at the University of Notre Dame, Notre Dame, IN, USA.  She obtained her Ph.D. from the University of Michigan, Ann Arbor, USA. She is a Fellow of the American Statistical Association. Her research interests include data privacy, synthetic data, trustworthy machine learning, and Bayesian statistics.

\vfill

\end{document}

%%%% Counter increments for figures/tables/equations
\addtocounter{equation}{13}
\addtocounter{figure}{4}
\addtocounter{table}{1}
\addtocounter{algorithm}{1}
\addtocounter{section}{5}

\maketitle

\vspace{-9pt}\section{Proofs of Theorem 7 and Proposition 8}

\underline{\textit{Proof of Theorem 7}}\vspace{-3pt}
\begin{proof}
Let $D$ and $\tilde{D}$ be neighboring datasets consisting of feature-label-weight triples. 
WLOG, assume they differ on individual $n$.
That is, $\mathbf{d}_i = (\mathbf{x}_i, y_i) = \tilde{\mathbf{d}}_i= (\tilde{\mathbf{x}}_i, \tilde{y}_i)$ for $i =1,2,\ldots, n-1$, and $\mathbf{d}_n  \ne  \tilde{\mathbf{d}}_n$, where $\mathbf{d}_n= (\mathbf{x}_n, y_n)$ and $\tilde{\mathbf{d}}_n=(\tilde{\mathbf{x}}_n, \tilde{y}_n)$. 
Define
\begin{align}
    g_1(\boldsymbol{\theta}) & = \textstyle\frac{1}{n}\sum_{i=1}^n w_i\ell(f_{\boldsymbol{\theta}}(\mathbf{x}_i), y_i) + \frac{\gamma}{n}R(\boldsymbol{\theta}), \notag\\
    \tilde{g}_1(\boldsymbol{\theta}) &= \textstyle\frac{1}{n}\sum_{i=1}^n \tilde{w}_i\ell(f_{\boldsymbol{\theta}}(\tilde{\mathbf{x}}_i), \tilde{y}_i) + \frac{\gamma}{n}R(\boldsymbol{\theta}), \notag\\
    g_2(\boldsymbol{\theta})
    &= \textstyle\tilde{g}_1(\boldsymbol{\theta}) - g_1(\boldsymbol{\theta}) \notag\\
    & = \textstyle\frac{1}{n}\sum_{i=1}^{n-1}(\tilde{w}_i-w_i)\ell(f_{\boldsymbol{\theta}}(\mathbf{x}_i),y_i)\label{eq:g2}\\
    &\;\;\textstyle+\frac{1}{n}\{\tilde{w}_n\ell(f_{\boldsymbol{\theta}}(\tilde{\mathbf{x}}_n), \tilde{y}_n) - w_n\ell(f_{\boldsymbol{\theta}}(\mathbf{x}_n), y_n)\} \notag
\end{align}
since $(\mathbf{x}_i, y_i) = (\tilde{\mathbf{x}}_i, \tilde{y}_i)$ for $i =1,2,\ldots, n-1$.
By regularity condition A3, $\ell$ is convex and everywhere first-order differentiable, and by regularity condition A5, $R$ is 1-strongly convex and everywhere differentiable.
Since $w_i>0$ and $\tilde{w}_i>0$, $g_1$ and $\tilde{g}_1 = g_1 + g_2$ are both $\frac{\lambda}{n}$-strongly convex and everywhere differentiable.
By regularity conditions A2 and A4, we can re-write $\ell(f_{\boldsymbol{\theta}}(\mathbf{x}_i), y_i)$ as $\tilde{\ell}(y_i\mathbf{x}_i^\intercal\boldsymbol{\theta})$, which implies
\begin{align*}
    &\nabla_{\boldsymbol{\theta}}g_2(\boldsymbol{\theta}) = \textstyle\frac{1}{n}\sum_{i=1}^{n-1}(\tilde{w}_i-w_i)\nabla_{\boldsymbol{\theta}}\ell(f_{\boldsymbol{\theta}}(\mathbf{x}_i), y_i)\\
    &\hspace{0.5cm}+\textstyle\frac{1}{n}\{\tilde{w}_n\nabla_{\boldsymbol{\theta}}\ell(f_{\boldsymbol{\theta}}(\tilde{\mathbf{x}}_n), \tilde{y}_n) - w_n\nabla_{\boldsymbol{\theta}}\ell(f_{\boldsymbol{\theta}}(\mathbf{x}_n), y_n)\} \\
    &=\textstyle\frac{1}{n}\sum_{i=1}^{n-1}(\tilde{w}_i-w_i)\nabla_{\boldsymbol{\theta}}\tilde{\ell}(y_i \mathbf{x}_i^\intercal\boldsymbol{\theta}) \\
    &\hspace{0.5cm}+\textstyle \frac{1}{n}\{\tilde{w}_n\nabla_{\boldsymbol{\theta}}\tilde{\ell}(\tilde{y}_n\tilde{\mathbf{x}}_n^\intercal\boldsymbol{\theta}) - w_n\nabla_{\boldsymbol{\theta}}\tilde{\ell}(y_n\mathbf{x}_n^\intercal\boldsymbol{\theta})\}\\
    &=\textstyle\frac{1}{n}\sum_{i=1}^{n-1}(\tilde{w}_i-w_i)y_i\mathbf{x}^\intercal_i\tilde{\ell}^\prime(y_i\mathbf{x}^\intercal_i\boldsymbol{\theta})\\
    &\hspace{0.5cm}+\textstyle \frac{1}{n}\{\tilde{w}_n\nabla_{\boldsymbol{\theta}}\tilde{\ell}(\tilde{y}_n\tilde{\mathbf{x}}_n^\intercal\boldsymbol{\theta}) - w_n\nabla_{\boldsymbol{\theta}}\tilde{\ell}(y_n\mathbf{x}_n^\intercal\boldsymbol{\theta})\}\\
\end{align*}\vspace{-36pt}
\begin{align}
    &\|\nabla_{\boldsymbol{\theta}}g_2(\boldsymbol{\theta})\|\leq \textstyle \|\frac{1}{n}\sum_{i=1}^{n-1}(\tilde{w}_i-w_i)\| \label{eq:delta_g_ineq1}\\
    &\hspace{0.5cm} \textstyle+
    \|\frac{1}{n}\{\tilde{w}_n\nabla_{\boldsymbol{\theta}}\tilde{\ell}(\tilde{y}_n\tilde{\mathbf{x}}_n^\intercal\boldsymbol{\theta}) - w_n\nabla_{\boldsymbol{\theta}}\tilde{\ell}(y_n\mathbf{x}_n^\intercal\boldsymbol{\theta})\}\|\notag\\
    &\leq \textstyle \|\frac{1}{n-1}\sum_{i=1}^{n-1}(\tilde{w}_i-w_i)\| \notag\\
    &\hspace{0.5cm}\textstyle+
    \|\frac{1}{n}\{\tilde{w}_n\nabla_{\boldsymbol{\theta}}\tilde{\ell}(\tilde{y}_n\tilde{\mathbf{x}}_n^\intercal\boldsymbol{\theta}) - w_n\nabla_{\boldsymbol{\theta}}\tilde{\ell}(y_n\mathbf{x}_n^\intercal\boldsymbol{\theta})\}\|\notag\\
    &\le\textstyle \|\frac{1}{n-1}\sum_{i=1}^{n-1}(\tilde{w}_i-w_i)\|+\notag\\
    &\hspace{0.5cm}\textstyle
    \frac{1}{n}\!\left(\|\tilde{w}_n \tilde{y}_n\tilde{\mathbf{x}}^\intercal_n\tilde{\ell}^\prime(\tilde{y}_n\tilde{\mathbf{x}}^\intercal_n\boldsymbol{\theta})\|\!+\!
\|w_ny_n\mathbf{x}^\intercal_n\tilde{\ell}^\prime(y_n\mathbf{x}^\intercal_n\boldsymbol{\theta})\|\right)  \notag\\\
    &= \textstyle\|\frac{1}{n-1}\sum_{i=1}^{n-1}(\tilde{w}_i-w_i)\|+\notag\\
    &\hspace{0.5cm}\textstyle+\frac{1}{n}\big(|\tilde{w}_n| |\tilde{y}_n| \norm{\tilde{\mathbf{x}}_n} |\tilde{\ell}^\prime(\tilde{y}_n\tilde{\mathbf{x}}^\intercal_n\boldsymbol{\theta})| +\notag\\
    &\hspace{1cm}|w_n| |y_n| \norm{\mathbf{x}^\intercal_n} |\tilde{\ell}^\prime(y_n\mathbf{x}^\intercal_n\boldsymbol{\theta})|\big) \notag \\
    &\le \textstyle\|\frac{1}{n-1}\sum_{i=1}^{n-1}(\tilde{w}_i - w_i)\| + \frac{1}{n}(|\tilde{w}_n| - |w_n|). \label{eq:delta_g_ineq2}
\end{align} 
The inequalities in Eqns.~\eqref{eq:delta_g_ineq1} and \eqref{eq:delta_g_ineq2} hold because of the assumption that $|\tilde{\ell}^\prime(z)| \le 1$, $|y_i|\le 1$, and $\norm{\mathbf{x}_i}\le 1$ for all $i$.
To derive an upper bound for Eqn.~\eqref{eq:delta_g_ineq2}, we examine two scenarios below.

Case 1: When $\mathbf{w}$ is observed (or predicted from a model trained on data independent of $D$ and $\tilde{D}$), $w_i=\tilde{w}_i$  (or $\hat{w}_i=\hat{\tilde{w}}_i$ if predicted) for $i=1,\ldots,n-1$, implying that the first term in Eqn.~\eqref{eq:delta_g_ineq2} is 0.
Combining this with assumptions A1 and A4, Eqn.~\eqref{eq:delta_g_ineq2} can be bounded by
\begin{align}\label{eqn:bound1}
    \textstyle\|\frac{1}{n-1}\sum_{i=1}^{n-1}(\tilde{w}_i - w_i)\| + \frac{1}{n}(|\tilde{w}_n| - |w_n|) \le \frac{2W}{n}.
\end{align}

Case 2: When $\mathbf{w}$ is not observed, but rather predicted from the sensitive dataset itself, the predictions $\hat{w}_i$ and $\hat{\tilde{w}}_i$ may differ for $i=1,\ldots,n$.
In this case, the first term in Eqn.~\eqref{eq:delta_g_ineq2} cannot be dropped, so
\begin{align}
    &\textstyle\|\frac{1}{n-1}\sum_{i=1}^{n-1}(\tilde{w}_i - w_i)\| + \frac{1}{n}(|\tilde{w}_n| - |w_n|)\notag \\
    \le &\textstyle\|\frac{1}{n-1}\sum_{i=1}^{n-1}(\tilde{w}_i - w_i)\| + \frac{2W}{n}.\label{eq:eq:delta_g_ineq3}
 \end{align}
A loose upper bound for Eqn.~\eqref{eq:eq:delta_g_ineq3} is 
\begin{align}
    \textstyle W+\frac{2W}{n},\label{eq:bound2}
\end{align}
\noindent since each difference $|\hat{\tilde{w}}_i-\hat{w}_i|$ for $i=1,\ldots,n-1$ is by assumption upper bounded by $W$. 
This bound is vastly conservative and counter-intuitive as the influence of a single data point on the estimates of $\mathbf{w}$ should diminish as $n\to\infty$, but Eqn.~\eqref{eq:bound2} approaches a non-zero constant $W$.
As $n$ increases, it is more reasonable to expect that the influence of a single data point (e.g., observation $n$ in $D$ and $\tilde{D}$) on the estimation of the true weight $w_i=\tilde{w}_i$ for $i=1,\ldots,n-1$ will shrink and therefore that $\hat{\tilde{w}}_i-\hat{w}_i\to 0$.

To tighten the bound, we assume $\mathbf{w}$ is predicted by an M estimation approach and leverage the established asymptotics of M estimators.  M estimator includes the most common estimation frameworks, such as maximum likelihood estimation, least squared estimation, robust estimators (e.g. Huber's estimator), and many common ERM problems.
Suppose $\hat{\bs\beta}$ is an M-estimator defined through the estimating equation $\frac{1}{n}\sum_{j=i}^n \psi(d_i,\bs\beta)=0$ in dataset $D$; similarly, $\hat{\tilde{\bs\beta}}$ is the corresponding estimator based on the neighboring dataset $\tilde D$. Under standard regularity conditions \cite{huber1967behavior, van2000asymptotic, stefanski2002calculus, ma2005robust, duchi2021statistics}, 
$\sqrt{n}\,(\hat{\bs\beta}-\bs{\beta}_0)\overset{d}{\to}\mathcal N(0,\Sigma_{\bs{\beta}})$, where $\bs{\beta}_0$ is the true parameter.
\begin{comment}
By the delta method since $w_i$ is a function of $\bs{\beta}_0$, then for each $i$,
$\sqrt{n}\,(\hat w_i-w_i)\overset{d}{\to}\mathcal N(0,\sigmy_i^2)$
and
$\sqrt{n}\,(\hat{\tilde w}_i-\tilde w_i)\overset{d}{\to}\mathcal N(0,\tilde\sigmy_i^2)$,
where $w_i=\tilde w_i$ for $i\ne n-1$. 
Thus $\hat w_i-w_i=O_p(n^{-1/2})$ and $\hat{\tilde w}_i-w_i=O_p(n^{-1/2}$). 
Now suppose, without loss of generality, that $D$ and $\tilde D$ differ only in the $n$th observation. 
\end{comment}
By the asymptotic linear representation of $M$-estimators,
\begin{align*}
\hat{\bs{\beta}}-\bs{\beta}_0 =&\textstyle
-(nG)^{-1}\sum_{i=1}^n \psi(d_i,\bs{\beta}_0)+o_p(n^{-1/2}),\\
\hat{\tilde{\bs{\beta}}}-\bs{\beta}_0 =&\textstyle
-(nG)^{-1}\sum_{i=1}^{n-1}\psi(d_i,\bs{\beta}_0)\\
&-(nG)^{-1}\psi(\tilde d_n,\bs{\beta}_0)+o_p(n^{-1/2}),   
\end{align*}
where $G\!=\!\mathbb E\left[\frac{\partial}{\partial\bs{\beta}}\psi(d,\bs{\beta}_0)\right]$.

\begin{comment}
Differencing the two equations yields
\begin{align}
\hat{\tilde{\bs{\beta}}}-\hat{\bs{\beta}}\!= &
-n^{-1}G^{-1}\!\left(\psi(\tilde d_n,\bs{\beta}_0)-\psi(d_n,\bs{\beta}_0)\!\right)\label{eqn:diff.beta}\\
&+o_p(n^{-\frac{1}{2}}).\notag
\end{align}    
%and therefore$$\hat{\tilde{\bs{\beta}}}-\hat{\bs{\beta}} = O_p(n^{-1}).$$
\end{comment}
\begin{comment}
Since $\hat{\tilde{\bs{\beta}}}$ and $\hat{\bs{\beta}}$ are based on datasets differing in only one observation, then
$\hat{\tilde{\bs{\beta}}}-\hat{\bs{\beta}}$
is much smaller, typically of order $O_p(n^{-1})$; that is,
$$
\hat{\tilde{\bs{\beta}}}-\hat{\bs{\beta}}
=
-\frac{1}{n}G^{-1}\Bigl(\psi(\tilde d_n,\bs{\beta}_0)-\psi(d_n,\bs{\beta}_0)\Bigr)
+ o_p(n^{-1}).
$$
The sharper $o_p(n^{-1})$ does not come from the usual asymptotic linear representation but from a more local argument that uses the fact that the two estimating equations differ by only one observation.
\end{comment}
Let $\textstyle \Psi_n(\bs\beta)=\frac{1}{n}\sum_{j=1}^n \psi(d_j,\bs\beta)$ and thus $\Psi_n(\hat{\bs{\beta}})=0$, and  the perturbed estimator $\hat{\tilde{\bs{\beta}}}$ solves
\begin{align*}
 \tilde{\Psi}_n(\bs\beta) & =\textstyle 
\frac{1}{n}\sum_{j=1}^{n-1}\psi(d_j,\bs\beta)
+\frac{1}{n}\psi(\tilde{d}_n,\bs\beta)\\
&=\textstyle  \Psi_n(\bs\beta) +
\frac{1}{n}\Bigl(\psi(\tilde{d}_n,\bs\beta)-\psi(d_n,\bs\beta)\Bigr)=0;
\end{align*}
that is, $\tilde{\Psi}_n(\hat{\tilde{\bs{\beta}}})=0$. Since $\Psi_n(\hat{\bs{\beta}})=0$,
$$ \tilde{\Psi}_n(\hat{\bs{\beta}})=
\frac{1}{n}\Bigl(\psi(\tilde{d}_n,\hat{\bs{\beta}})-\psi(d_n,\hat{\bs{\beta}})\Bigr).$$
Now expand $\tilde{\Psi}_n(\hat{\tilde{\bs{\beta}}})$ around $\hat{\bs{\beta}}$,
$$0= \tilde{\Psi}_n(\hat{\tilde{\bs{\beta}}})
= \tilde{\Psi}_n(\hat{\bs{\beta}}) +
\tilde{\Psi}'_n(\check{\beta})(\hat{\tilde{\bs{\beta}}}-\hat{\bs{\beta}})$$
for some $\check{\beta}$ between $\hat{\tilde{\bs{\beta}}}$ and $\hat{\bs{\beta}}$. All taken together
$$
\hat{\tilde{\bs{\beta}}}-\hat{\bs{\beta}}
=-\tilde{\Psi}'_n(\hat{\tilde{\bs{\beta}}})^{-1}
\frac{1}{n} \Bigl(\psi(\tilde{d}_n,\hat{\bs{\beta}})-\psi(d_n,\hat{\bs{\beta}})\Bigr).
$$
Since the inverse Jacobian is $O_p(1)$, provided that $A$ is non-singular, and 
$$
\sup_{\|\bs\beta-\bs{\beta}_0\|\le c n^{-1/2}}
\|\Psi'_n(\bs\beta)-A\| \xrightarrow{p} 0
$$
and $\psi(\tilde{d}_n,\hat{\bs{\beta}})-\psi(d_n,\hat{\bs{\beta}})$ is $O_p(1)$ (since $ \bs\beta-\bs\beta_0 \xrightarrow{p} 0$ and if  $\psi$ is continuous in $\bs\beta$ and locally bounded, or dominated by an integrable envelope), then
$$
\hat{\tilde{\bs{\beta}}}-\hat{\bs{\beta}}=O_p(n^{-1}),
$$
To sharpen the bound $O_p(n^{-1})$ into  $o_p(n^{-1})$, we note 
$$
\tilde{\Psi}'_n(\hat{\tilde{\bs{\beta}}})^{-1}
= G^{-1}+o_p(1);
$$
we then replace $\hat{\bs{\beta}}$ by $\bs{\beta}_0$ inside $\psi$:
$$
\psi(\tilde{d}_n,\hat{\bs{\beta}})-\psi(d_n,\hat{\bs{\beta}})
=
\psi(\tilde{d}_n,\bs{\beta}_0)-\psi(d_n,\bs{\beta}_0)+R_n,
$$
attaching a remainder $R_n$, where, by a mean-value expansion in $\bs\beta$,
$$
\|R_n\|\le
\Bigl(\|\partial_{\bs{\beta}} \psi(\tilde{d}_n,\bar{\bs\beta}_1)\|
+ \|\partial_{\bs{\beta}} \psi(d_n,\bar{\bs\beta}_2)\|
\Bigr)\|\hat{\bs{\beta}}-\bs{\beta}_0\|.
$$
If $\partial_{\bs{\beta}} \psi(d,\bs\beta)$ is locally bounded in probability and $\hat{\bs{\beta}}-\bs{\beta}_0=O_p(n^{-1/2})$, then
$R_n=O_p(n^{-1/2})$ and 
$$
n^{-1}R_n = O_p(n^{-3/2}) = o_p(n^{-1}).
$$
Taken together
$$
\hat{\tilde{\bs{\beta}}}-\hat{\bs{\beta}} =
-(nG)^{-1}\Bigl(\psi(\tilde{d}_n,\bs{\beta}_0)-\psi(d_n,\bs{\beta}_0)\Bigr)+o_p(n^{-1}).
$$

Weight $w_i$ is a function of $\bs{\beta}_0$. For notation simplicity, we denote that function as $w_i(\bs{\beta}_0)$. For any shared observation $i\!=\!1,\ldots,n-1$, a first-order Taylor expansion of  $\hat{\tilde{w}}_i-\hat{w}_i$ around $\bs{\beta}_0$ gives
\begin{align*}
&\hat{\tilde{w}}_i-\hat{w}_i = \nabla_{\bs{\beta}} w_i(\bs{\beta}_0)^\top(\hat{\tilde{\bs{\beta}}}-\hat{\bs{\beta}})+o_p(n^{-1})=\\
& -(nG)^{-1}
\! \nabla_{\bs{\beta}} w_i(\bs{\beta}_0)^\top\! \!\left(\!\psi(\tilde{d}_n,\bs{\beta}_0)\!-\! \psi(d_n,\bs{\beta}_0)\!\right)\!+\!o_p(n^{-1}),
\end{align*} 
suggesting that $\hat{\tilde{w}}_i-\hat{w}_i\xrightarrow[]{p}0$ as $n\rightarrow\infty$. In other words, one may set $C=0$ for a sufficiently large $n$, or estimate 
\begin{equation}\label{eqn:H1}
H :=\nabla_{\bs{\beta}} w_i(\bs{\beta}_0)^\top\! G^{-1}\!\left(\!\psi(\tilde{d}_n,\bs{\beta}_0)\!-\! \psi(d_n,\bs{\beta}_0)\!\right)   
\end{equation} 
from data, or upper bound it. While, in principle, one could estimate $H$ using, for example, a plug-in estimator that replaces $\bs{\beta}_0$ and $A$ with estimates derived from the data, this approach has several drawbacks. First, because the estimate of $H$ is itself data-dependent, releasing $H$ would an additional privacy loss. Second, the term $\psi(\tilde{d}_n,\bs{\beta}_0)\!-\! \psi(d_n,\bs{\beta}_0)$ in Eqn.~\eqref{eqn:H1} suggests that a bound would still be required, since it must be evaluated over all possible pairs of datasets. For these reasons, we choose to upper bound $H$, which is a more reasonable, convenient, and practical alternative. Specifically, per the Cauchy-Schwarz inequality, we have
\begin{align*}
&\left|\nabla_\beta w_i(\bs{\beta}_0)^\top G^{-1}\bigl(\psi(\check d_n,\bs{\beta}_0)-\psi(d_n,\bs{\beta}_0)\bigr)\right|\\
\le\; & \|\nabla_\beta w_i(\bs{\beta}_0)\|\,\|G^{-1}\|\,\|\psi(\check d_n,\bs{\beta}_0)-\psi(d_n,\bs{\beta}_0)\|.
\end{align*}
Suppose
$ \|\nabla_\beta w_i(\bs{\beta}_0)\|_2\le C_w$; 
$\|G^{-1}\|_{\mathrm{op}}\le C_G$; and $\|\psi(d,\bs{\beta}_0)\|_2\le C_\psi$ for all $d$, then $\|\psi(\check d_n,\bs{\beta}_0)-\psi(d_n,\bs{\beta}_0)\|\le 2C_\psi$ and
\begin{align}
&\left|\nabla_\beta w_i(\bs{\beta}_0)^\top G^{-1}\bigl(\psi(\check d_n,\bs{\beta}_0)-\psi(d_n,\bs{\beta}_0)\bigr) \right|
\le H,\notag\\
& \mbox{where } H:= 2C_w C_G C_\psi,\label{eqn:H2}
\end{align}
resulting in a data-free bound without incurring privacy loss. We may further express  $\nabla_\beta w_i(\bs{\beta}_0)$ in terms of $p_i$. Since  $w_i(\beta)=\frac{B_i}{p_i(\beta,\mathbf{x}_i)}$, where $B_i\in[0, C_B]$, then using the chain rule, we have
\begin{align*}
  \nabla_\beta w_i(\beta) & =
-\,B_i\,\frac{\nabla_\beta p_i(\beta,\mathbf{x}_i)}{p_i(\beta,\mathbf{x}_i)^2}\\
\|\nabla_\beta w_i(\bs{\beta}_0)\| &=
B_i\,\frac{\|\nabla_\beta  p_i(\bs{\beta}_0,\mathbf{x}_i)\|}{ p_i(\bs{\beta}_0,\mathbf{x}_i)^2}.
\end{align*}
Assume $ p_i(\bs{\beta}_0,\mathbf{x}_i)\ge p_{\min}>0$ (a positivity condition)  and 
$\|\nabla_\beta  p_i(\bs{\beta}_0,\mathbf{x}_i)\|\le L_p$
for all $i$, then
$$\|\nabla_\beta w_i(\bs{\beta}_0)\|
\le C_w = C_B L_pp_{\min}^{-2}.$$
\begin{comment}
If $p_i$ can be expressed as in $p_i(\beta,\mathbf{x}_i)=h(m_i(\beta,\mathbf{x}_i))$,
where $h$ is a scalar link function and $m_i$ is a smooth score function, then
\begin{align*}
\nabla_\beta p_i(\beta,\mathbf{x}_i) &=
h'(m_i(\beta,\mathbf{x}_i))\,\nabla_\beta m_i(\beta,\mathbf{x}_i),  \\
\|\nabla_\beta w_i(\bs{\beta}_0)\| &\le
C_B\frac{|h'(m_i(\bs{\beta}_0,\mathbf{x}_i))|\,\|\nabla_\beta m_i(\bs{\beta}_0,\mathbf{x}_i)\|}{p_i(\bs{\beta}_0,\mathbf{x}_i)^2}.
\end{align*}
If $|h'(m_i(\bs{\beta}_0,\mathbf{x}_i))|\le L_h$ and
$\|\nabla_\beta m_i(\bs{\beta}_0,\mathbf{x}_i)\|\le L_m$,
then 
$$C_w=\frac{C_B\,L_h\,L_m}{p_{\min}^2}.$$ 
\end{comment}
Each individual bound $C_w,C_G$, and $C_\psi$ will be problem- and model- specific given that they depend on the estimating function $\psi$. 
All taken together, along with Lemma 6, define $\boldsymbol{\hat{\theta}}_1 = \argmin_{\boldsymbol{\theta}} g_1(\boldsymbol{\theta}),\boldsymbol{\hat{\theta}}_2 = \argmin_{\boldsymbol{\theta}}\{g_1(\boldsymbol{\theta}) + g_2(\boldsymbol{\theta})\}$, we obtain 
\begin{align}
    \pushQED{\qed}
   & \textstyle\Delta_{2, \boldsymbol{\theta}} = \!\! \max\limits_{\substack{\{D, \tilde{D}\}, d(D,\tilde{D}) = 1}}\norm{\boldsymbol{\hat{\theta}}_1 \!-\! \boldsymbol{\hat{\theta}}_2}_2\!=\! \frac{n}{\gamma} \max\limits_{\boldsymbol{\theta}}\norm{\nabla_{\boldsymbol{\theta}}g_2(\boldsymbol{\theta})} \notag \\
    &\textstyle \leq  \left(\frac{n}{\gamma}\right) \left(\frac{C+2W}{n}\right) = \frac{C+2W}{\gamma}, \mbox{where }\label{eqn:GS}\\
    & C = 
    \begin{cases}
        0 & \mbox{if $\mathbf{w}$ is observed or predicted}\\[-3pt]
        & \mbox{from a model trained on data}\\[-3pt]
        & \mbox{independent from the}\\[-3pt]
        & \mbox{sensitive dataset;} \\
       % kr^{-1/2}\sigma_d  & \mbox{if $\mathbf{w}$ is predicted from the}\\[-3pt]
        H/n  & \mbox{if $\mathbf{w}$ is predicted from the}\\[-3pt]
        & \mbox{sensitive dataset for large $n$;} \\[-3pt]
        & \mbox{$H$ is defined in Eqn.~\eqref{eqn:H1} or \eqref{eqn:H2}}\\
        nW & \mbox{a vastly conservative option.}
    \end{cases}
\end{align}
After obtaining $\Delta_{2, \boldsymbol{\theta}}$, the proof that Algorithm 1 satisfies $\epsilon$-DP is identical to the proof of Theorem 6 in \cite{chaudhuri2011}, only replacing the sensitivity there with $\Delta_{2, \boldsymbol{\theta}}$ from Eqn.~\eqref{eqn:GS}.
\end{proof}

\underline{\emph{Proof of Proposition 8}} 

Logistic regression is a commonly used approach for propensity score calculation, where,
\begin{align*}
p_\beta(\mathbf{x})&=\big(1+e^{-\mathbf{x}^\top \bs\beta}\big)^{-1},\\
\psi(d,\beta)&=x\bigl(1(y=1)-p_\beta(\mathbf{x})\bigr),\\
\nabla_\beta p_i(\beta,\mathbf{x}_i) &=
p_\beta(\mathbf{x}_i)\bigl(1-p_\beta(\mathbf{x}_i)\bigr)\mathbf{x}_i,\\
\nabla_\beta w_i(\beta) & =
%=-B_i\frac{p_\beta(\mathbf{x}_i)(1-p_\beta(\mathbf{x}_i))\mathbf{x}_i}{p_\beta(\mathbf{x}_i)^2}
-B_i\frac{1-p_\beta(\mathbf{x}_i)}{p_\beta(\mathbf{x}_i)}\,\mathbf{x}_i.
\end{align*}
1) $C_\psi$: Since  $p_\beta(\mathbf{x})\!\in\![0,1]$, 
$|y-p_\beta(\mathbf{x})|\le 1$, then
$$\|\psi(d,\bs{\beta}_0)\|=
\|\mathbf{x}(1(y=1)-p_{\bs{\beta}_0}(\mathbf{x}))\|\le\|\mathbf{x}\|;$$  
assume
$\|\mathbf{x}\|\le C_x $ for all $\mathbf{x}$, then 
$$C_\psi = C_x$$
\begin{align*}
2)\; C_w\!\!:\; &\|\nabla_\beta w_i(\bs{\beta}_0)\|
= B_i\,\frac{1-p_{\bs{\beta}_0}(\mathbf{x}_i)}{p_{\bs{\beta}_0}(\mathbf{x}_i)}\,\|\mathbf{x}_i\|\qquad\qquad\qquad\\
& \le C_B\cdot C_x\frac{1-p_{\min}}{p_{\min}} =C_w.
\end{align*}
Equivalently, we may assume $|\mathbf{x}_i^\top \bs{\beta}_0|\le M$ for all $i$, then
\begin{align*}
\frac{1-p_{\bs{\beta}_0}(\mathbf{x}_i)}{p_{\bs{\beta}_0}(\mathbf{x}_i)}
=e^{-\mathbf{x}_i^\top \bs{\beta}_0}\le e^M  \\
\|\nabla_\beta w_i(\bs{\beta}_0)\|
\le  C_B\cdot C_x\cdot e^M = C_w.
\end{align*}
\begin{align*}
3)\; C_G:\;  &\partial_\beta \psi(d,\beta) = -p_\beta(\mathbf{x})\bigl(1-p_\beta(\mathbf{x})\bigr)\mathbf{x}\mathbf{x}^\top,\\
G=&\operatorname{E}\!\left[\partial_\beta \psi(d,\bs{\beta}_0)\right]\!=\!\operatorname{E}\!\left[-p_{\bs{\beta}_0}\!(\mathbf{x})\bigl(1\!-\!p_{\bs{\beta}_0}(\mathbf{x})\bigr)\mathbf{x}\mathbf{x}^\top\!\right].
\end{align*}
To bound $\|G^{-1}\|_{\mathrm{op}}$, assume
$\mathbb{E}[\mathbf{x}\mathbf{x}^\top]\succeq b_x I$ for some  $b_x>0$,
and $|\mathbf{x}^\top \bs{\beta}_0|\le M$ for all $\mathbf{x}$.
Then
\begin{align*}
&p_{\bs{\beta}_0}(x)\bigl(1-p_{\bs{\beta}_0}(x)\bigr)\ge \eta_M,\mbox{ where }\eta_M =\frac{e^M}{(1+e^M)^2}.
\end{align*}
Therefore,
\begin{align*}
&G\succeq \eta_M E[\mathbf{x}\mathbf{x}^\top] \succeq\eta_M b_x I,\\
&\lambda_{\min}(G)\ge \eta_M b_x,\\
&\|G^{-1}\| \le C_G=\frac{1}{b_x \eta_M } =
\frac{(1+e^M)^2}{b_xe^M}.
\end{align*}

\vspace{6pt}
\section{Global sensitivity in output perturbation of M-learning}
This section derives the $\ell_2$-GS of the output from M-learning  \cite{Wu2020}, which is introduced in Section I-A of the main text.
The result is an extension of Theorem 7.
The regularity conditions listed in Assumption 5 apply here. 

M-learning performs matching instead of inverse probability weighting as used in OWL for estimating ITRs to alleviate confounding and assess individuals' treatment responses to alternative treatments.
%Matching-based value functions are used to compare matched pairs.
Denote the matched set for subject $i$ as $\mathcal{M}_i$, which consists of subjects with opposite treatments but similar covariates as subject $i$, where similarity is defined under a suitable distance metric. The Huber loss with an $\ell_2$ regularization term on  the model parameters in M-learning  is
\begin{align*}
L=\sum_{i=1}^n \sum_{j\in\mathcal{M}_i}|\mathcal{M}_i|^{-1}g(|\tilde{B}_i-\tilde{B}_j|)\ell_{\text{Huber}}(z_{ij})+\gamma||\boldsymbol{\theta}||^2,
\end{align*}
where $g(.)$ is a user-specified monotonically increasing function (e.g. typical choices are $g(.)=1$ or the identity function), $z_{ij}= y_i\text{sign}(|\tilde{B}_i-\tilde{B}_j|)\boldsymbol{\theta}^T\mathbf{x}_i$.
$\tilde{B}_i= B_i -s(\mathbf{x}_i) $ is a surrogate residualized outcome for the observed outcome $B_i$, where $s()$ is a measurable function. It is shown that using $\tilde{B}_i$, instead of simply $B_i$ helps improve the performance of M-learning \cite{Wu2020, wang2024matching}. 

Let $d_g(\boldsymbol{\theta})$ denote the difference of two loss functions based on two neighboring datasets $(\mathbf{x}, \mathbf{Y},\mathbf{B})$ and $(\mathbf{x}', \mathbf{Y}',\mathbf{B}')$.
WLOG, we assume the $n$-th individuals in two datasets are not the same and the values of $\tilde{B}$ in a subset of individuals $\mathcal{S}$ of the dataset ($|\mathcal{S}|\ge1$ individuals) are affected from altering the $n$-th individual. Then 

\small
\begin{align*}
    d_g(\boldsymbol{\theta})= \!\sum_{i\in\mathcal{S}}\!\bigg\{\!\!&\sum_{j\in\mathcal{M}'_i}\!|\mathcal{M'}_i|^{-1}g(|\tilde{B}'_i-\tilde{B}'_j|)\ell_{\text{Huber}}(z'_{ij}))- \\
    \!\!\!&\textstyle\sum_{j\in\mathcal{M}_i}\!\!|\mathcal{M}_i|^{-1}g(|\tilde{B}_i-\tilde{B}_j|)\ell_{\text{Huber}}(z_{ij}))\bigg\},\\
    \nabla_{\boldsymbol{\theta}}d_g(\boldsymbol{\theta}) 
    =\textstyle\sum_{i\in\mathcal{S}}\bigg\{ &\textstyle\sum_{j\in\mathcal{M}'_i}|\mathcal{M'}_i|^{-1}d_g(|\tilde{B}'_i-\tilde{B}'_j|) y'_i\\
    &\cdot\text{sign}(|\tilde{B}'_i-\tilde{B}'_j|)\nabla_{\boldsymbol{\theta}}\ell_{\text{Huber}}(z'_{ij})\mathbf{x}'_i-\\
    & \textstyle\sum_{j\in\mathcal{M}_i}|\mathcal{M}_i|^{-1}g(|\tilde{B}_i-\tilde{B}_j|)y_i \\
    &\cdot\text{sign}(\tilde{B}_i-\tilde{B}_j)\nabla_{\boldsymbol{\theta}}\ell_{\text{Huber}}(z_{ij})\mathbf{x}_i\bigg\},
\end{align*}\vspace{-12pt}
\begin{align*}
    ||\nabla_{\boldsymbol{\theta}}d_g(\boldsymbol{\theta})|| = &\bigg\|\textstyle
    \sum_{i\in\mathcal{S}}\bigg\{\sum_{j\in\mathcal{M}'_i}|\mathcal{M'}_i|^{-1}g(|\tilde{B}'_i-\tilde{B}'_j|) y'_i\\
    &\cdot\text{sign}(|\tilde{B}'_i-\tilde{B}'_j|)\nabla_{\boldsymbol{\theta}}\ell_{\text{Huber}}(z'_{ij})\mathbf{x}'_i-\\
    & \textstyle \qquad\quad \sum_{j\in\mathcal{M}_i}|\mathcal{M}_i|^{-1}g(|\tilde{B}_i-\tilde{B}_j|)y_i\\
    &\cdot\text{sign}(|\tilde{B}_i-\tilde{B}_j|)\nabla_{\boldsymbol{\theta}}\ell_{\text{Huber}}(z_{ij})\mathbf{x}_i\bigg\}\bigg\|
    \end{align*}\vspace{-12pt}
\begin{align*}
    \le & \textstyle|\mathcal{S}|\bigg\|
    \sum_{j\in\mathcal{M}'}|\mathcal{M'}_i|^{-1}g(|\tilde{B}'_i-\tilde{B}'_j|) y'_i\\
    &\cdot\text{sign}(|\tilde{B}'_i-\tilde{B}'_j|)\nabla_{\boldsymbol{\theta}}\ell_{\text{Huber}}(z'_{ij})\mathbf{x}'_i\bigg\|+\\
    & \textstyle|\mathcal{S}| \bigg\|\sum_{j\in\mathcal{M}_i}|\mathcal{M}_i|^{-1}g(|\tilde{B}_i-\tilde{B}_j|)y_i\\
    &\cdot\text{sign}(|\tilde{B}_i-\tilde{B}_j|)\nabla_{\boldsymbol{\theta}}\ell_{\text{Huber}}(z_{ij})\mathbf{x}_i\bigg\|
\end{align*}\vspace{-12pt}
\begin{align*}
    \le &\textstyle |\mathcal{S}| \sum_{j\in\mathcal{M}'_i}|\mathcal{M'}_i|^{-1} \bigg\|
    g(|\tilde{B}'_i-\tilde{B}'_j|) y'_i\\
    &\cdot\text{sign}(|\tilde{B}'_i-\tilde{B}'_j|)\nabla_{\boldsymbol{\theta}}\ell_{\text{Huber}}(z'_{ij})\mathbf{x}'_i\bigg\|+\\
    &\textstyle |\mathcal{S}|\sum_{j\in\mathcal{M}_i}|\mathcal{M}_i|^{-1} \bigg\|g(|\tilde{B}_i-\tilde{B}_j|)y_i\\
    &\cdot\text{sign}(|\tilde{B}_i-\tilde{B}_j|)\nabla_{\boldsymbol{\theta}}\ell_{\text{Huber}}(z_{ij})\mathbf{x}_i\bigg\|.
\end{align*}
\normalsize
Given $y_i\in\{-1,1\}$, $\text{sign}(|\tilde{B}'_i-\tilde{B}'_j|)\in\{-1,1\}$ and the assumptions of $||\mathbf{x}_i||_2\le 1$ $\forall i$ and $\|\nabla_{\boldsymbol{\theta}}\ell_{\text{Huber}}\|\le1$ (Assumption 5), then

\small
\begin{align*}
    ||\nabla_{\boldsymbol{\theta}}d_g(\boldsymbol{\theta})|| &\le|\mathcal{S}|\sum_{j\in\mathcal{M}_i}\!|\mathcal{M}_i|^{-1} \bigg\|g(|\tilde{B}_i-\tilde{B}_j|)\bigg\| \\
    &\;\;\;+ |\mathcal{S}|\sum_{j\in\mathcal{M}'_i}|\mathcal{M}'_i|^{-1} \bigg\|g(|\tilde{B}'_i-\tilde{B}'_j|)\bigg\|\\
    &\le\!2|\mathcal{S}|\max_i\max_j\big|g(|\tilde{B}_i-\tilde{B}_j|)\big|.
\end{align*}

\normalsize
Per Lemma 6, the $\ell_2$ sensitivity of $\hat{\boldsymbol{\theta}}$ is
\begin{equation}\label{eqn:M.sens}
\Delta_{\hat{\boldsymbol{\theta}}}= \frac{2|\mathcal{S}|}{\gamma} \sup \big|g(|\tilde{B}_i-\tilde{B}_j|)\big|.
\end{equation}
As for the value of $\sup \big|g(|\tilde{B}_i-\tilde{B}_j|)\big|$, it depends on function $g$ and the model used to obtain $\tilde{B}$.
If $g()\!=\!1$, then $|\mathcal{S}|\!=\!1,\sup \big|g(|\tilde{B}_i-\tilde{B}_j|)\big|=1$ and $\Delta_{\hat{\boldsymbol{\theta}}}=2/\gamma$.
If linear model is used to estimate $\tilde{B}$ and $g$ is the identity function, then it is reasonable to set $\sup \big|g(|\tilde{B}_i-\tilde{B}_j|)\big|=6$. 

Finally, privacy-preserving estimate for $\boldsymbol{\theta}$ from an M-learning procedure with $\epsilon$-DP guarantees given the derived sensitivity $\Delta_{\hat{\boldsymbol{\theta}}}$ in Eqn.~\eqref{eqn:M.sens} is obtained by $\hat{\boldsymbol{\theta}}^*=\hat{\boldsymbol{\theta}}+\mathbf{e}$, where $f(\mathbf{e})\propto e^{-(\epsilon/\Delta_{\hat{\boldsymbol{\theta}}})\|e\|}=e^{-(\epsilon/\Delta_{\hat{\boldsymbol{\theta}}})\sqrt{e_i^2+ \cdots + e^2_p}}$. 

\vspace{6pt}
\section{Proof of Theorem 11}
Before proving the main theorem on the error bound, we first present Lemmas \ref{lem:huber-approximation}, \ref{lem:smoothness}, and \ref{lem:radial-tail} on which the main theorem is based.
\begin{lem}[Huberized hinge approximation]
\label{lem:huber-approximation}
For every $z\in\R$,
\begin{align}
0&\le \huber(z)-\hinge(z)\le \frac{h}{4} \label{eq:pointwise-huber-approx}\\
0&\le \whL_{\textnormal{H}}(\bs{\theta})-\whL_0(\bs{\theta})\le \frac{Wh}{4},\mbox{ for every } \bs{\theta}\in\R^p,\label{eq:empirical-huber-approx}\\
0&\le L_{\textnormal{H}}(\bs{\theta})-L_0(\bs{\theta})\le \frac{Wh}{4}.\label{eq:population-huber-approx}
\end{align}
\end{lem}

\begin{proof}
If $z>1+h$, both $\hinge$ and $\huber$  are zero. If $z<1-h$, then both losses equal $1-z$. As for $|1-z|\le h$, let $u=1-z$, so $u\in[-h,h]$. Then
$$ \huber(z)=\frac{(h+u)^2}{4h}. $$
If $u<0$, then $\hinge(z)=0$, and
$$
0\le \huber(z)=\frac{(h+u)^2}{4h}\le \frac{h}{4}.
$$
If $u\ge 0$, then $\hinge(z)=u$, and
$$
\huber(z)-\hinge(z)
=\frac{(h+u)^2}{4h}-u
=\frac{(h-u)^2}{4h}
\in\left[0,\frac{h}{4}\right],
$$
leading to Eq.~\eqref{eq:pointwise-huber-approx}. Multiplying Eq.~\eqref{eq:pointwise-huber-approx} by $w_i\le W$ and averaging gives Eq.~\eqref{eq:empirical-huber-approx}, further taking expectations yields Eq.~\eqref{eq:population-huber-approx}.
\end{proof}

\begin{lem}[Smoothness of DP-OWL loss]
\label{lem:smoothness}
The empirical Huber risk $\whL_{\textnormal{H}}$ is $(W/2h)$-smooth and 
$\whF_{\textnormal{H}\lambda}$ is $(\lambda+W/2h)$-smooth.
\end{lem}

\begin{proof}
Since $\huber''(z)=1/2h$ on the quadratic component of the Huber loss   and outside that region $\huber''(z)=0$ wherever the second derivative is defined, the derivative is continuous at the breakpoints, so $\huber'$ is $1/(2h)$-Lipschitz. For one observation,
$$
\nabla_{\bs{\theta}}\{w_i\huber(y_i x_i^\top\bs{\theta})\}
=w_i\huber'(y_i x_i^\top\bs{\theta})y_i x_i,
$$
and its Hessian, where defined, is
$$
w_i\huber''(y_i x_i^\top\bs{\theta})x_i x_i^\top.
$$
Since $w_i\le W$ and $\norm{x_i}\le 1$, the spectral norm of this Hessian is $\le W/(2h)$ and averaging preserves the same bounds. The regularizer $\lambda\norm{\bs{\theta}}^2/2$ adds smoothness $\lambda$, resulting in $(\lambda+W/2h)$-smoothness for $\whF_{\textnormal{H}\lambda}$.
\end{proof}

\begin{lem}[Radial Laplace tail]
\label{lem:radial-tail}
Define probability density function of $\eta \propto\exp(-\eps\norm{\eta}/\Delta_2)$ in $\R^p$. For $0<\rho<1$, define
\begin{equation}\label{eq:aprho}
a_p(\rho)\triangleq p+\sqrt{2p\log(1/\rho)}+\log(1/\rho).
\end{equation}
Then, with probability at least $1-\rho$,
\begin{equation}
\label{eq:eta-tail}
\norm{\eta}\le \frac{\Delta_2}{\eps}a_p(\rho)
=\frac{C+2W}{n\lambda\eps}a_p(\rho).
\end{equation}
\end{lem}

\begin{proof}
\begin{align}
\label{eq:radial-gamma}
&\hat{\bs{\theta}}^*=\widehat{\bs{\theta}}_{\textnormal{H}}+\eta,
\mbox{ where }
q(\eta)\propto \exp\left(-\frac{\eps\norm{\eta}}{\Delta_2}\right), \\
&\qquad\mbox{ or equivalently, }
\norm{\eta}\sim \mathrm{Gamma}\left(p,\mbox{rate}=\frac{\eps}{\Delta_2}\right).\notag
\end{align}
By Eq.~\eqref{eq:radial-gamma}, $\norm{\eta}$ has the same distribution as $(\Delta_2/\eps)G$, where $G\sim\mathrm{Gamma}(p,\mbox{rate}=1)$, equivalently, a sum of $p$ independent samples from distribution $\exp(1)$. The standard Gamma Chernoff bound gives
$$
\Prob\{G\ge p+\sqrt{2pt}+t\}\le e^{-t}.
$$
Setting $t=\log(1/\rho)$ yield Eq.~\eqref{eq:eta-tail}.
\end{proof}

Now we can proceed to prove Theorem 11.
\begin{proof}
Per Lemma \ref{lem:radial-tail},
$
\norm{\eta}\le \frac{C+2W}{n\lambda\eps}a_p(\rho).
$
Because $\widehat{\bs{\theta}}_{\textnormal{H}}$ minimizes the differentiable objective $\whF_{\textnormal{H}\lambda}$,
$\nabla \whF_{\textnormal{H}\lambda}(\widehat{\bs{\theta}}_{\textnormal{H}})=0$. By \ref{lem:smoothness},
\begin{align}
&\whF_{\textnormal{H}\lambda}(\hat{\bs{\theta}}^*)
=\whF_{\textnormal{H}\lambda}(\widehat{\bs{\theta}}_{\textnormal{H}}+\eta)\notag\\
\le &\whF_{\textnormal{H}\lambda}(\widehat{\bs{\theta}}_{\textnormal{H}})
+\nabla\whF_{\textnormal{H}\lambda}(\widehat{\bs{\theta}}_{\textnormal{H}})^\top\eta
+\frac{\lambda+W/2h}{2}\norm{\eta}^2\notag\\
=&\whF_{\textnormal{H}\lambda}(\widehat{\bs{\theta}}_{\textnormal{H}})
+\frac{\lambda+W/(2h)}{2}\norm{\eta}^2\notag\\
\le& \whF_{\textnormal{H}\lambda}(\widehat{\bs{\theta}}_{\textnormal{H}})
+\mathscr{E}(n,\lambda,h,\eps,\rho).
\label{eq:smooth-privacy-proof-step}
\end{align}
Since $\widehat{\bs{\theta}}_{\textnormal{H}}$ minimizes $\whF_{\textnormal{H}\lambda}$,
$
\whF_{\textnormal{H}\lambda}(\widehat{\bs{\theta}}_{\textnormal{H}})\le \whF_{\textnormal{H}\lambda}(\bs{\theta}_0)
=\whL_{\textnormal{H}}(\bs{\theta}_0)+\frac{\lambda}{2}\norm{\bs{\theta}_0}^2.
$ 
By \ref{lem:huber-approximation},
\begin{align*}
\whL_0(\hat{\bs{\theta}}^*)
\le \whL_{\textnormal{H}}(\hat{\bs{\theta}}^*)
\le \whF_{\textnormal{H}\lambda}(\hat{\bs{\theta}}^*)\\
\whL_{\textnormal{H}}(\bs{\theta}_0)\le \whL_0(\bs{\theta}_0)+\frac{Wh}{4}.   
\end{align*}
All taken together, we have the bound in Theorem 11.
\end{proof}

\section{Proof of Theorem 12}
\begin{proof}
Suppose that, with probability at least \(1-\delta_{\mathrm{gen}}\) over the sample,
\begin{equation}\label{eq:generic-generalization}
\left|L_h(\theta)-\whL_h(\theta)\right|
\le \mathfrak G_{n,h}(\delta_{\mathrm{gen}})
\end{equation}
By Eq.~\eqref{eq:generic-generalization},
$$
L_0(\hat{\bs{\theta}}^*)
\le L_{\textnormal{H}}(\hat{\bs{\theta}}^*)
\le \whL_{\textnormal{H}}(\hat{\bs{\theta}}^*)+\mathcal G_{n,h}(\delta_0).
$$
From the proof of Theorem 11, 
$$
\whL_{\textnormal{H}}(\hat{\bs{\theta}}^*)
\le \whF_{\textnormal{H}\lambda}(\hat{\bs{\theta}}^*)
\le \whF_{\textnormal{H}\lambda}(\bs{\theta}_0)+\mathscr{E}(n,\lambda,h,\eps,\rho).
$$
Therefore,
$$
L_0(\hat{\bs{\theta}}^*)
\le
\whL_{\textnormal{H}}(\bs{\theta}_0)+\frac{\lambda}{2}\norm{\bs{\theta}_0}^2
+\mathscr{E}(n,\lambda,h,\eps,\rho)
+\mathcal G_{n,h}(\delta_0).
$$
Applying Eq.~\eqref{eq:generic-generalization} at $\bs{\theta}_0$ and Lemma \ref{lem:huber-approximation}, we have
$$
\whL_{\textnormal{H}}(\bs{\theta}_0)
\le L_{\textnormal{H}}(\bs{\theta}_0)+\mathcal G_{n,h}(\delta_0)
\le L_0(\bs{\theta}_0)+\frac{Wh}{4}+\mathcal G_{n,h}(\delta_0),
$$
All taken together, we have the bound in Theorem 12.
\end{proof}

\section{Proof of Corollary 13}
\begin{proof}
We decompose the population risk into a generalization term, a regularization bias term, and a privacy perturbation term.  In wERM, the individual weighted loss is $\psi_i(\bs{\theta}) = w_i \ell_h(z_i)$  differentiate $\psi_i$ with respect to $\bs{\theta}$, we have
$$\nabla_{\bs{\theta}}\psi_i(\bs{\theta}) =
w_i \ell_h'(z_i) y_i x_i.$$
Since $\|x_i\|_2\le 1,|\ell_h'(z)|\le 1$, and
$w_i\le W$,
we get
$$\|\nabla_{\bs{\theta}}\psi_i(\bs{\theta})\|_2
\le w_i \left| \ell_h'(y_i x_i^\top \bs{\theta}) \right| \|x_i\|_2
\le W.$$
So therefore, the Lipschitz scale for the loss is $W$ and the generalization term has scale $W^2/(\lambda n)$. The regularization bias is $\lambda\norm{\bs{\theta}_0}^2/2$. By Theorem 11, the DP noise term is
$$
\frac{\lambda+W/(2h)}{2}
\left(\frac{(C+2W) a_p(\delta)}{n\lambda\eps}\right)^2,
$$
which is of the displayed order after replacing $a_p(\delta)$ by a constant multiple of $p\log(p/\delta)$. Finally, Lemma \ref{lem:huber-approximation} adds $Wh/4$ when translating the Huber-loss risk guarantee back to the original hinge-loss risk guarantee, proving the bound in Corollary 13. 

Set $\lambda\asymp \alpha/\norm{\bs{\theta}_0}^2$. The first term in the bound is at most a constant multiple of $\alpha$ when
$$n\gtrsim \frac{W^2\norm{\bs{\theta}_0}^2}{\alpha^2};$$
the sum of the second and third terms involving $\lambda$ is at most a constant multiple of $\alpha$ when
$$n\gtrsim \frac{(C+2W) p\log(p/\delta)\norm{\bs{\theta}_0}}{\eps\alpha};$$
and the fourth term is at most a constant multiple of $\alpha$ when
$$n\gtrsim
\frac{(C+2W) p\log(p/\delta)\sqrt{W/h}\,\norm{\bs{\theta}_0}^2}{\eps\alpha^{3/2}}.$$
The smoothing condition $h\le c\alpha/W$ ensures $Wh/4\le O(\alpha)$. Allocating the constants across the four terms yield the sample complexity. 

Finally, $\lesssim $ and $ \gtrsim$ are used in Corollary instead of $\le$ and $\ge$ because $\frac{W^2}{\lambda n}$ is actually $\mathcal{O}(\frac{W^2}{\lambda n})$ with constant and potentially logarithm factors hidden in the former.
\end{proof}

\vspace{6pt}
\section{Details on hyperparameter tuning}
This section presents an $(n, \epsilon)$-adaptive procedure in Algorithm \ref{alg:reg_const_selection} for tuning $\gamma$.
This procedure was used to tune $\gamma$ in the simulation studies.

\begin{algorithm}[!htp]
\caption{$(n,\epsilon)$-adaptive hyperparameter tuning using an independent dataset}\label{alg:reg_const_selection}
\begin{algorithmic}[1]
    \State \textbf{Input}: Training data size $n$; privacy budget $\epsilon$;  hyperparameter tuning dataset $D_0$ of size $n_0$; minimum validation set size $m$ ($m<n_0$); candidate regularization constant set $\{\gamma_1, \ldots, \gamma_k\}$; number of repeats $r$; evaluation metric $v(f^*, D_0)$ for fitted model $f^*$.
    \State \textbf{Output}: $(n,\epsilon)$-adaptive  hyperparameter $\gamma$
    \For{$i = 1, 2, \ldots, k$}
    \For{$j = 1, 2, \ldots, r$}
    \If{$n>n_0-m$}
        \State Set $D_{j,\textnormal{valid}}$ to be a random subset of $D_0$ of size $m$, sampled without replacement
        \State Set $D_{j,\textnormal{train}}$ to be a set of size $n_0$, sampled with replacement from $D_0\setminus D_{j,\textnormal{valid}}$
    \Else
        \State Set $D_{j,\textnormal{train}}$ to a set of size $n_0$,  sampled  without replacement from $D_0$
        \State Set $D_{j,\textnormal{valid}}=D\setminus D_{j,\textnormal{train}}$
    \EndIf
    \State Run Algorithm 1 
    given $\epsilon$, $D_{j,\textnormal{train}}$, and $\gamma_i$ to obtain $f^*_{j,i}$
    \State Compute and store $v(f^*_{j,i}, D_{j,\textnormal{valid}})$
    \EndFor
    \State Set $\bar{v}_i=r^{-1}\sum_{j=1}^r v(f^*_{j,i}, D_{j,\textnormal{valid}})$
    \EndFor
    \State Set $i^* = \argmax_i \{\bar{v}_1, \bar{v}_2, \ldots, \bar{v}_k\}$ \Comment{\footnotesize{\texttt{WLOG, larger metric for better model}}}
    \State \textbf{return} $\gamma=\gamma_{i^*}$
\end{algorithmic}
\end{algorithm}

Algorithm \ref{alg:reg_const_selection} can be used either when a subset of the sensitive dataset or an independent dataset $D_0$ from the sensitive dataset is available for tuning $\gamma$.
If $D_0$ is an independent data source from $D$, tuning $\gamma$ with $D_0$ via Algorithm \ref{alg:reg_const_selection} would not incur privacy loss on the sensitive dataset that is used for training the wERM algorithm.
When the size of this dataset $n_0$ is small, such as in the former case, or if $D_0$ is a small pilot study, as in the latter case, it is likely the steps in lines 6 and 7 will be used.
If $D_0$ is not ``perfect'' (e.g., its $\mathbf{X}_0$ might not completely overlap with $\mathbf{X}$ in the sensitive dataset or contain the latter as a subset, or $Y_0$ is a surrogate marker of $Y$ in the sensitive dataset) but is large in size ($n_0>n+m)$, then lines 9 and 10 Algorithm \ref{alg:reg_const_selection} will be run.
The evaluation metric function $v$ in Algorithm \ref{alg:reg_const_selection} should be defined to reflect an important performance metric for the original wERM problem, such as the the optimal treatment prediction accuracy or the empirical treatment value in OWL.

To tune $\gamma$ for the simulation study, we apply Algorithm \ref{alg:reg_const_selection} to $D_0$ for each combination of sample size $n$ and privacy budget $\epsilon$ used.
We set $m=500$, $n_0=1000$, and $r=1000$.
For the simulation study, we set $v$ to be the optimal treatment prediction accuracy, though setting $v$ to be the empirical treatment value was found to produce nearly identical results.

\subsection{Sensitivity studies of Algorithm \ref{alg:reg_const_selection}}
\label{SMsec:sensitivity}

In our simulation study, we tuned the regularization constant via Algorithm \ref{alg:reg_const_selection} on an independent dataset that came from the same underlying population and had the same data characteristics (type of features, optimal treatment function, etc.) as the sensitive dataset we wanted to protect.
In practice, it is unlikely that we would have access to an independent dataset that mimics the sensitive dataset so closely.
Thus, it is fair to question whether the selected regularization constant would perform similarly if the independent dataset differed from the sensitive dataset.
We perform a sensitivity analysis to answer this question by varying the characteristics of the independent dataset.
We show that the the regularization constant selected via Algorithm \ref{alg:reg_const_selection} under various scenarios is sufficiently close to the optimal choice that would be obtained if the independent dataset came from the same underlying population as the sensitive dataset -- that is, hyperparameter selection via Algorithm \ref{alg:reg_const_selection} exhibits robustness to characteristics of the independent dataset.

The sensitivity analysis settings are as follows.
We examine four different scenarios on how the independent datasets can be different from the sensitive dataset: 
\begin{enumerate}
    \item[] \hspace{-18pt}\textbf{Scenario 1:} varying the number of data points $n_0$ in the independent dataset,
    \item[]\hspace{-18pt}\textbf{Scenario 2:}  varying the coefficients $\boldsymbol{\theta}$ of the optimal treatment function,
    \item[]\hspace{-18pt}\textbf{Scenario 3:} varying the distribution of the features $\mathbf{X}$,
    \item[]\hspace{-18pt}\textbf{Scenario 4:} varying the size $p$ of the optimal treatment function coefficients $\boldsymbol{\theta}$.
\end{enumerate}
In each case, $n_0$ data points in the independent dataset  were split into a training set of size $0.8n_0$  and a validation set of size $0.2n_0$.
For each combination of sample size $n\in\{200,500,1000,2000\}$ and privacy budget $\epsilon\in\{0.5,1,5,20,150\}$, Algorithm \ref{alg:reg_const_selection} is run on the independent training set and used to fit a DP-OWL model with privacy budget $\epsilon$ at regularization constant $\gamma\in\{1,21,41,\ldots,601\}$.
The mean (with 95\% confidence interval) treatment assignment accuracy on the validation set over 1,000 repeats is shown in Figures \ref{fig:suppl_materials_theta} and \ref{fig:suppl_materials_P}.
The results for the empirical treatment value, though not plotted, yielded similar conclusions and insights as described below.

In each plot, the black line corresponds to the baseline case of an independent dataset of size $n_0=1,000$ that matches the sensitive dataset in distribution and treatment function parameter. %and what we would see from the sensitive dataset itself.
It is clear that the optimal value for $\gamma$ that yields the highest accuracy varies as a function of $n$ and $\epsilon$.
The region of $\gamma$ values in between the vertical red lines indicates the $\gamma$ values where the accuracy is measured to be within $5\%$ of highest accuracy.
The range is fairly large for most simulated scenarios, indicating that the regularization constant is relatively ``robust'' as strong performance does not require a strict precision in the parameter tuning.
We seek to determine whether the $\gamma$ value that would be chosen based on alternative independent datasets would likely fall within this range.

\paragraph{Robustness of Algorithm \ref{alg:reg_const_selection} in Scenario 1.}
Figure  \ref{fig:suppl_materials_theta}(a) plots the results of the regularization constant tuning process for independent datasets of sizes $n_0\in\{100, 500, 1000, 2500\}$.
We observe that the optimal $\gamma$ values for these cases are consistently similar, meaning that even a smaller independent dataset (e.g., $n_0=100$) compared to the size of the sensitive dataset ($n=200$ to 2000) can still be useful for tuning the regularization constant.

\paragraph{Robustness of Algorithm \ref{alg:reg_const_selection} in Scenario 2.}
The results in Figure \ref{fig:suppl_materials_theta}(b) correspond to different sets of coefficients (denoted by $\boldsymbol{\theta}$) for the underlying optimal treatment function from which the independent datasets ($n_0=1,000$)  were simulated. 
$\boldsymbol{\theta}_0$ is the baseline case that matches the sensitive dataset, and $\boldsymbol{\theta}_1$ and $\boldsymbol{\theta}_2$ represent two alternative optimal treatment functions. 
Despite the fact that a model fitted on these data would perform poorly in terms of prediction on the sensitive data (since the target function to learn is different), the optimal regularization constant values are still nearly the same for each corresponding pair of $\epsilon$ and sample size $n$.

\paragraph{Robustness of Algorithm \ref{alg:reg_const_selection} in Scenario 3.} 
Figure \ref{fig:suppl_materials_P}(a) examines the case where the distribution of features $\mathbf{X}$ in the independent dataset is not the same as that in the sensitive dataset.
Specifically, each element in $\mathbf{X}$ follows a uniform distribution in the sensitive dataset but a truncated normal distribution (truncated between 0 and 1 with standard deviation 0.3) with varying means (0.25, 0.5, 0.75, 0.5) for each of the four elements in the independent dataset.
While the shape of plot looks different and shows longer plateaus in the trajectory of the regularization constant in the independent dataset, choosing one at the lower end of the plateau still consistently lands within the ``good window'' for the sensitive dataset.

\paragraph{Robustness of Algorithm \ref{alg:reg_const_selection} in Scenario 4.}
Figure \ref{fig:suppl_materials_P}(b) visualizes the case where $\boldsymbol{\theta}$  for the optimal treatment assignment function contains a different number of non-zero elements.
In the baseline case, there are $p=4$ non-zero coefficients, the same as the sensitive dataset.
Here, we also examine the cases of $p\in\{2,6,8\}$ non-zero coefficients and observe that the optimal hyperparameter choice for the cases with different values of $p$ is consistently close to the optimal for the baseline case that matches the sensitive dataset.

\begin{figure*}[!htb]\vspace{-12pt}
    \centering
    % epsilon=0.5
    \subfloat{\includegraphics[width=0.18\linewidth]{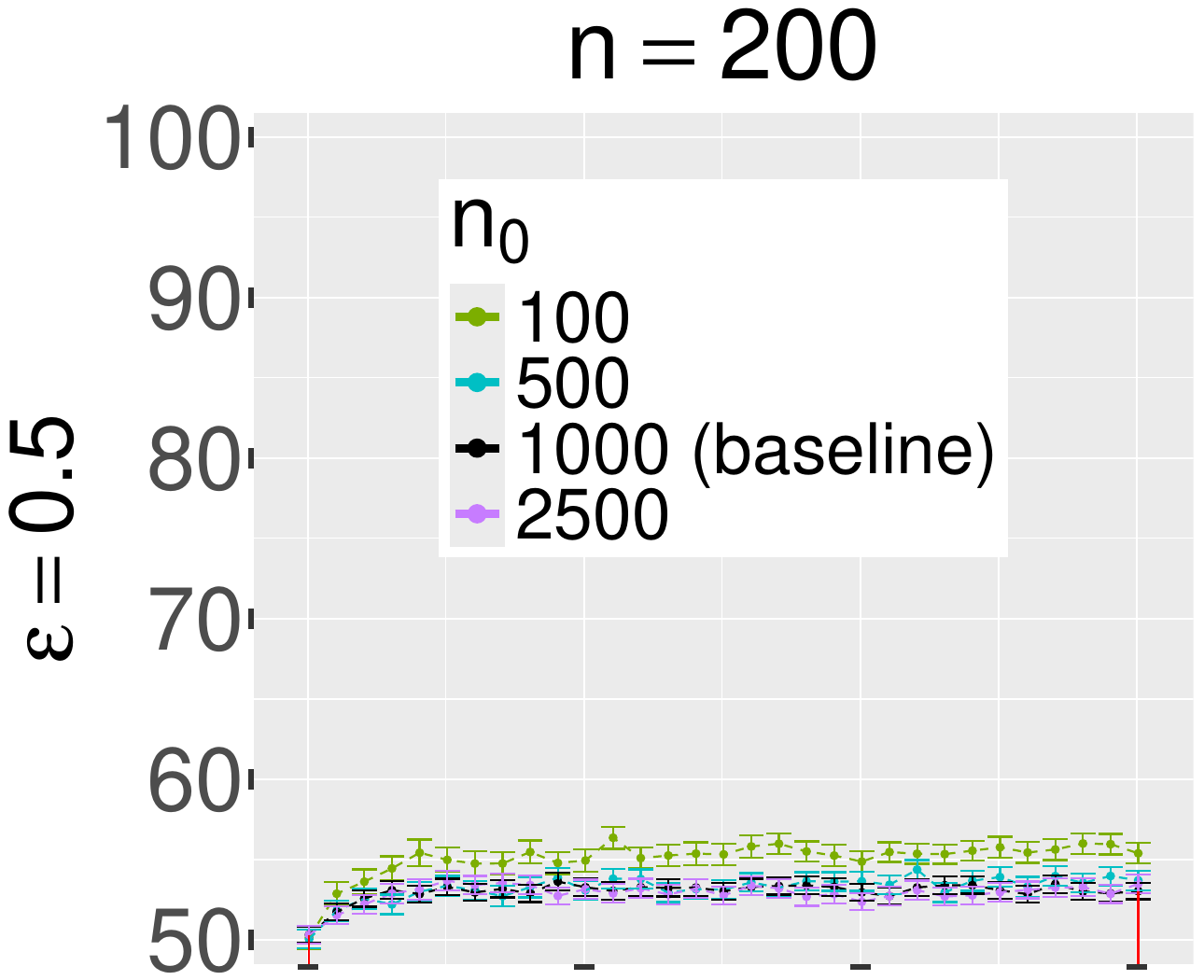}}
    \subfloat{\includegraphics[width=0.145\linewidth]{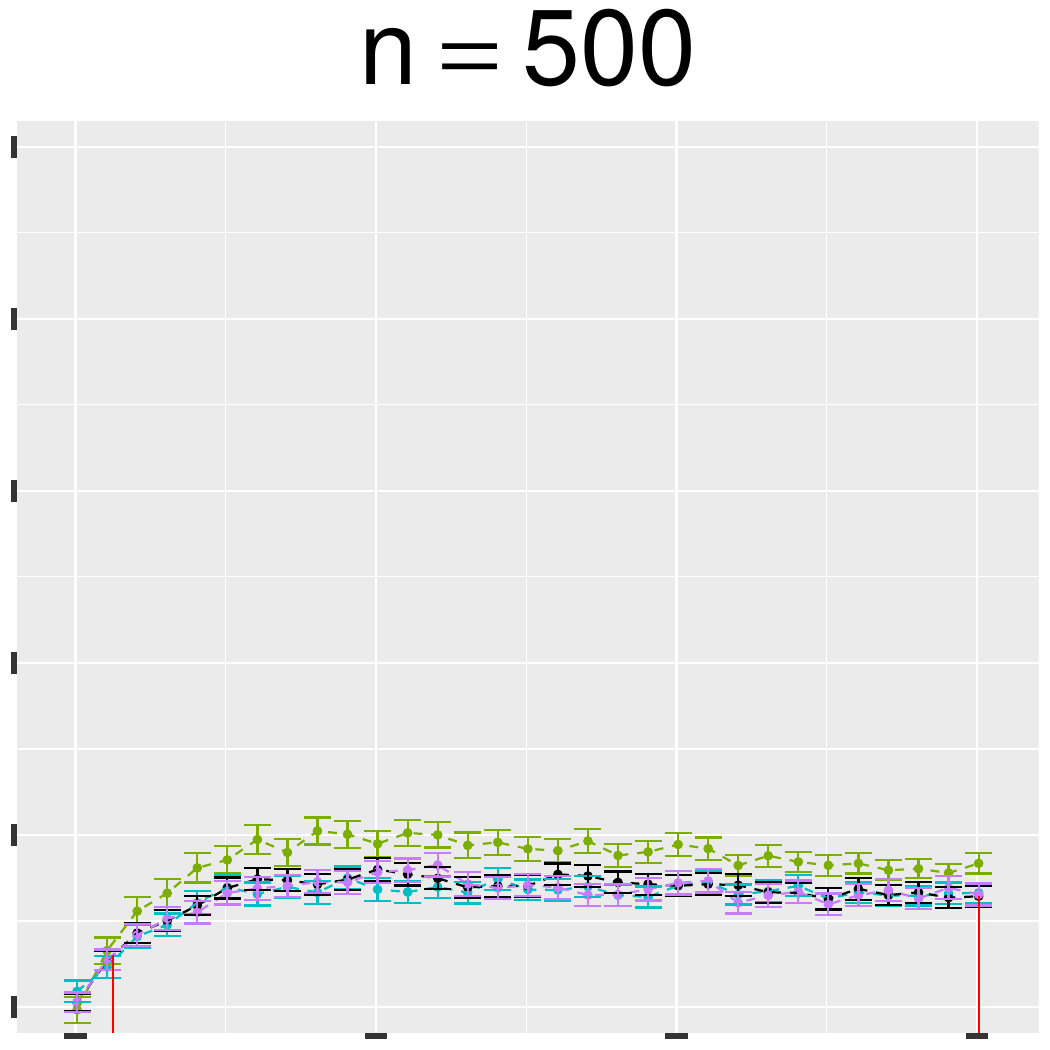}}
    \subfloat{\includegraphics[width=0.145\linewidth]{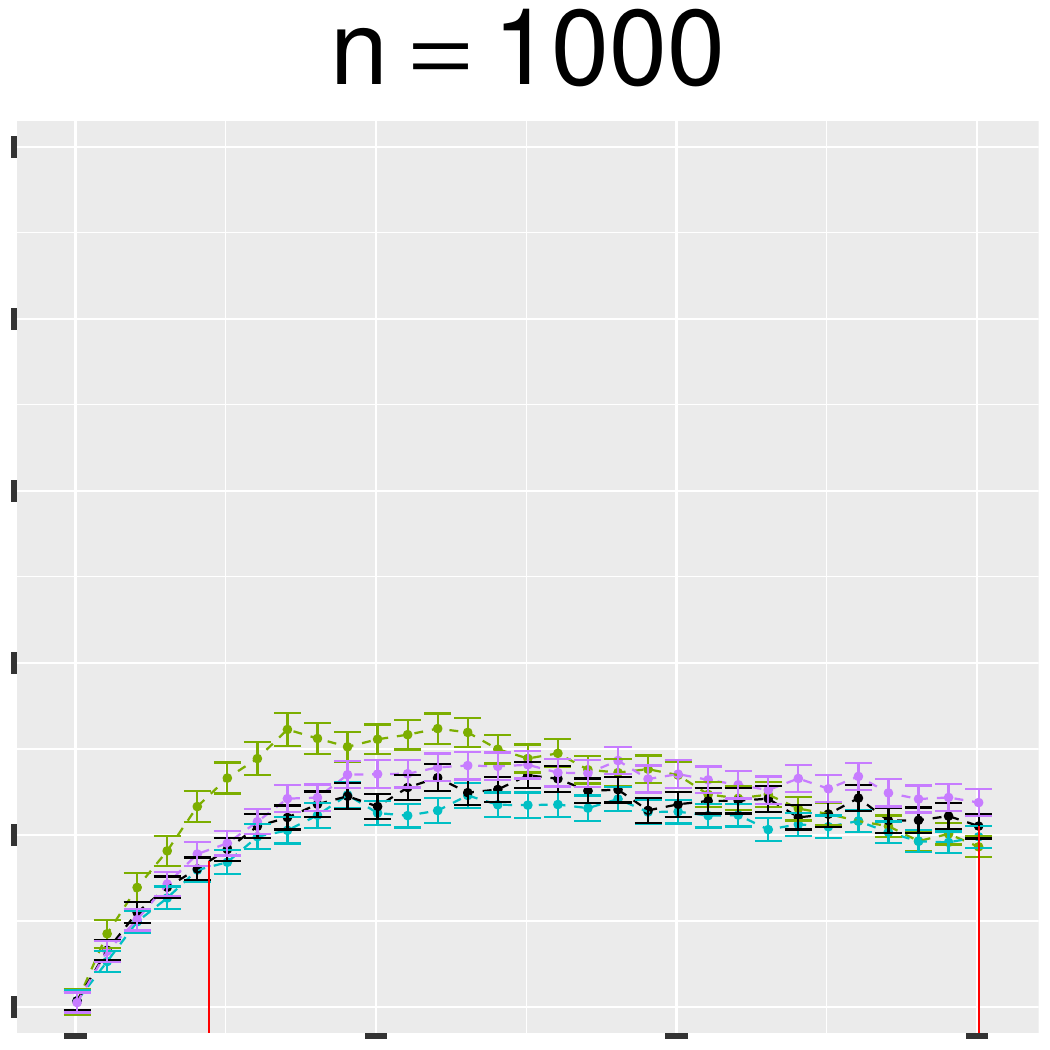}}
    \subfloat{\includegraphics[width=0.145\linewidth]{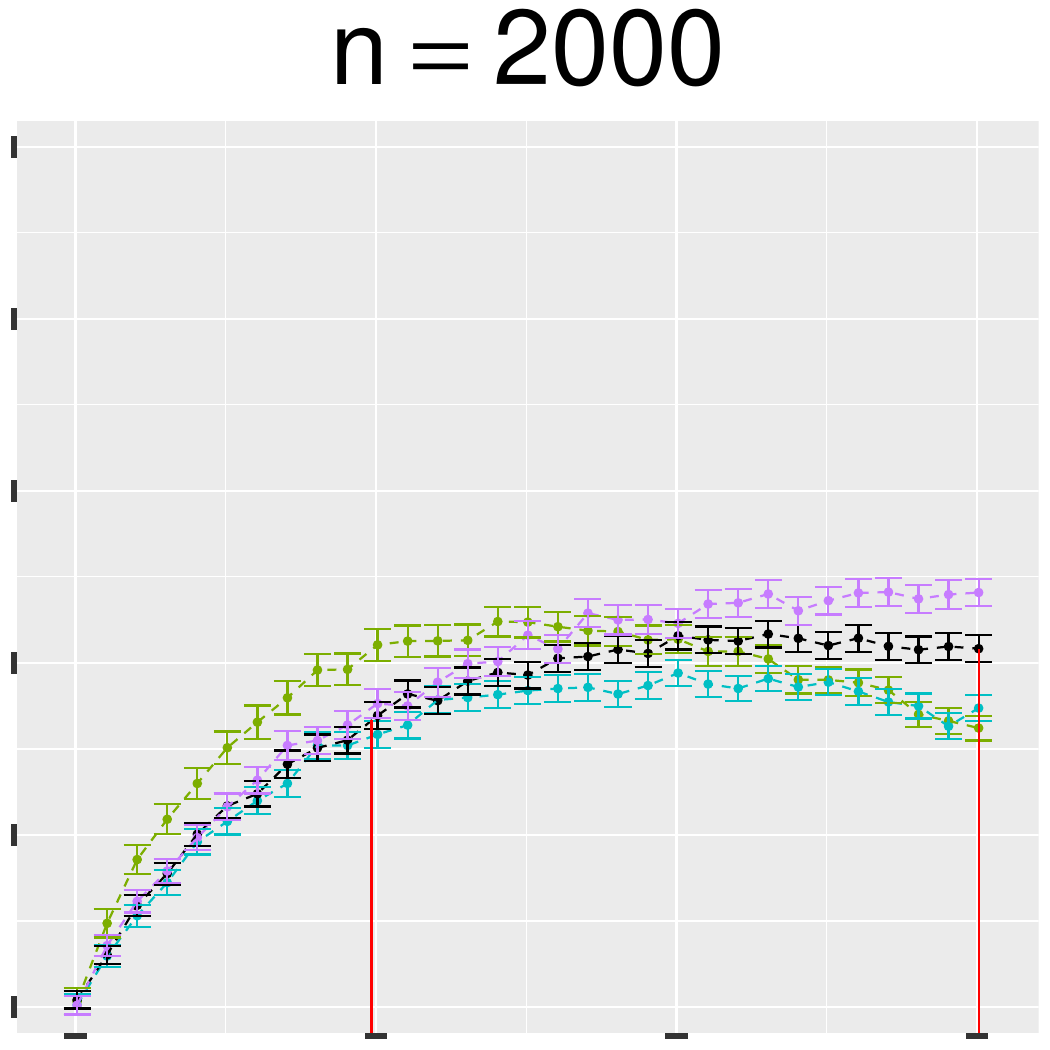}}\\\vspace{-6.5pt}
    % epsilon=1
    \subfloat{\includegraphics[width=0.18\linewidth]{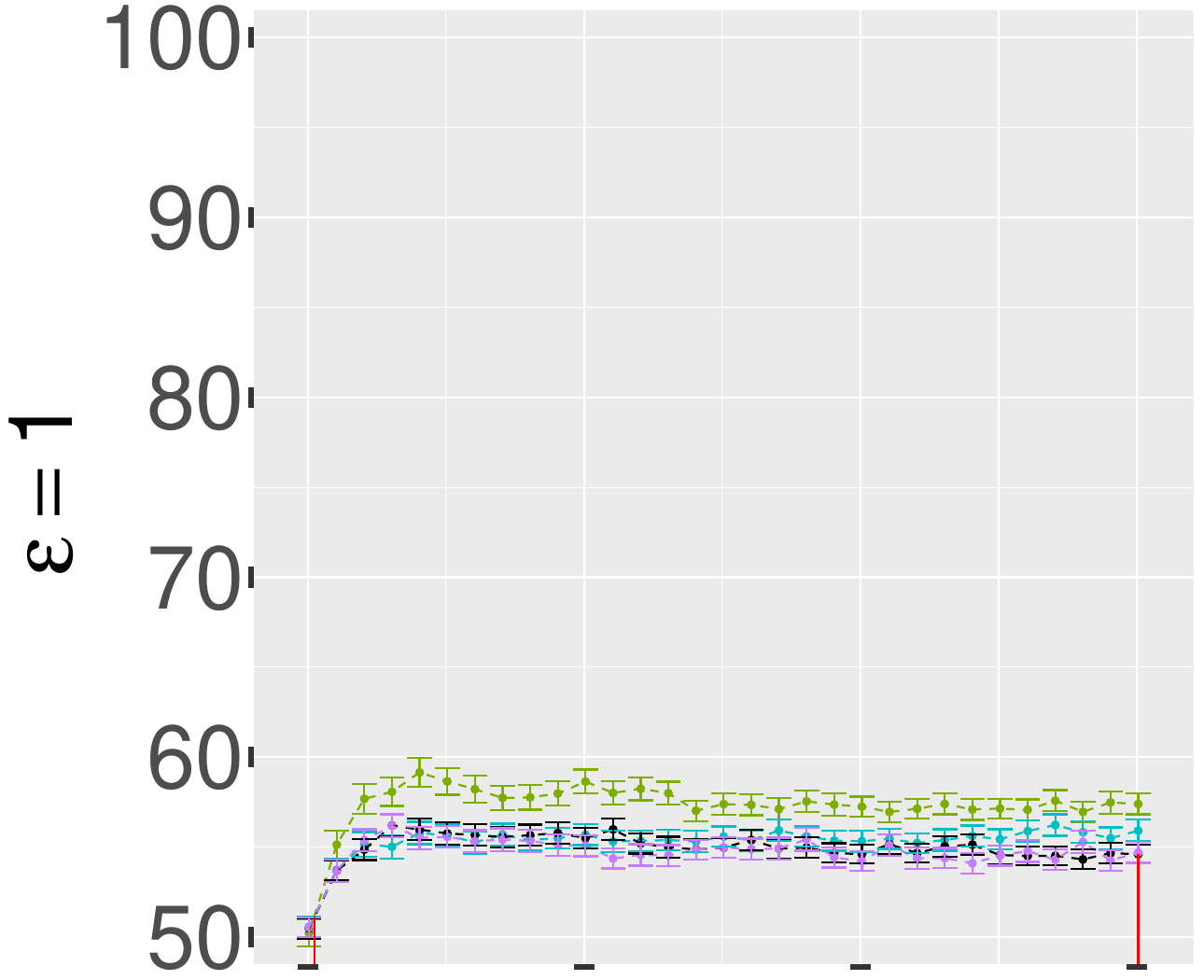}}
    \subfloat{\includegraphics[width=0.145\linewidth]{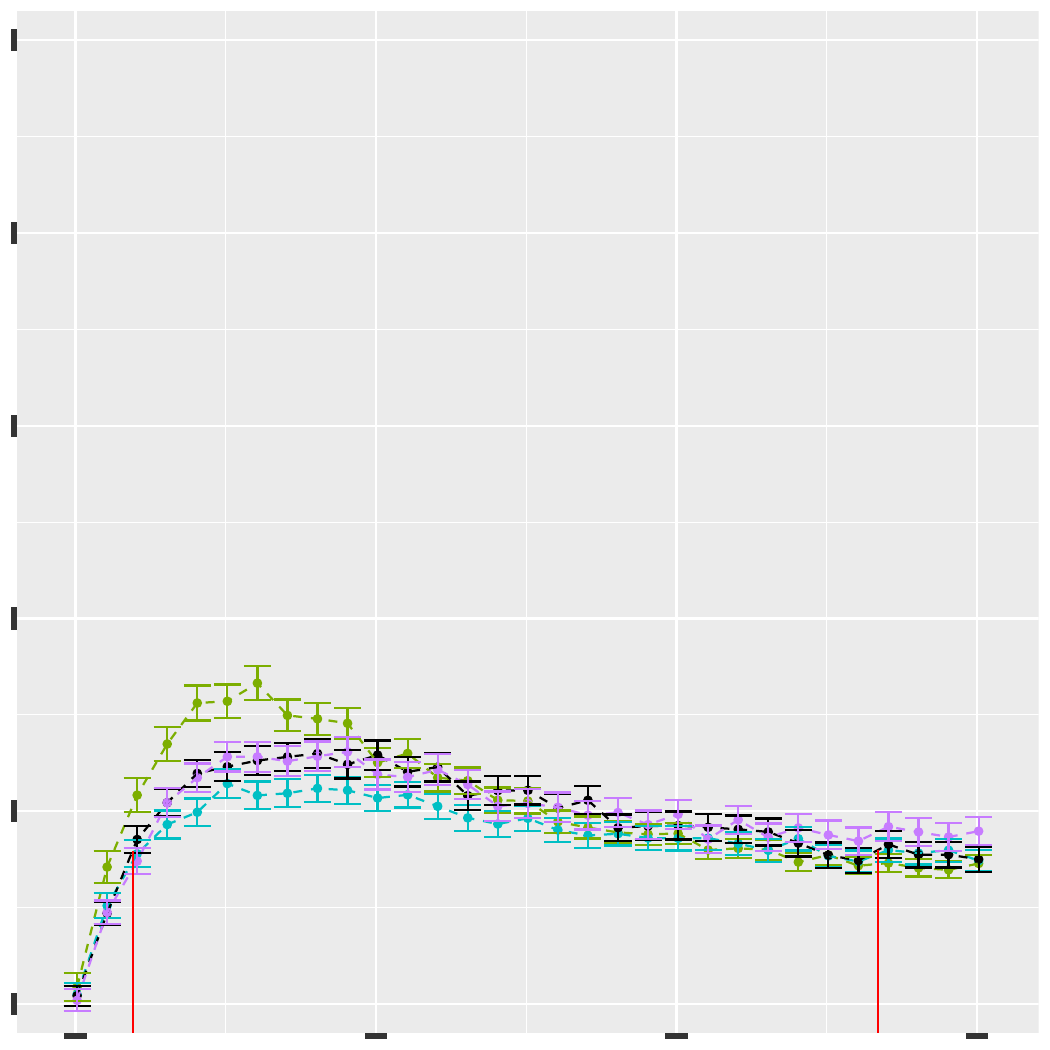}}
    \subfloat{\includegraphics[width=0.145\linewidth]{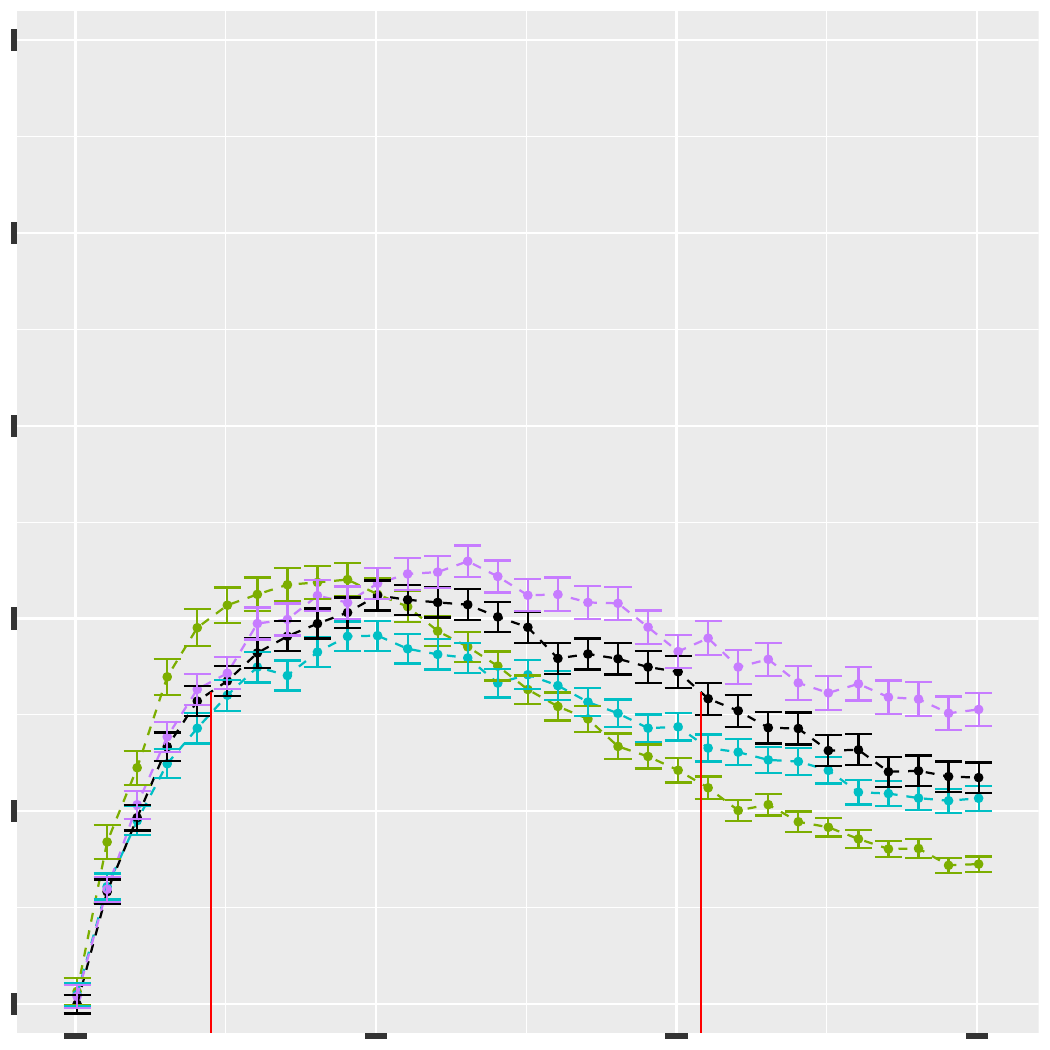}}
    \subfloat{\includegraphics[width=0.145\linewidth]{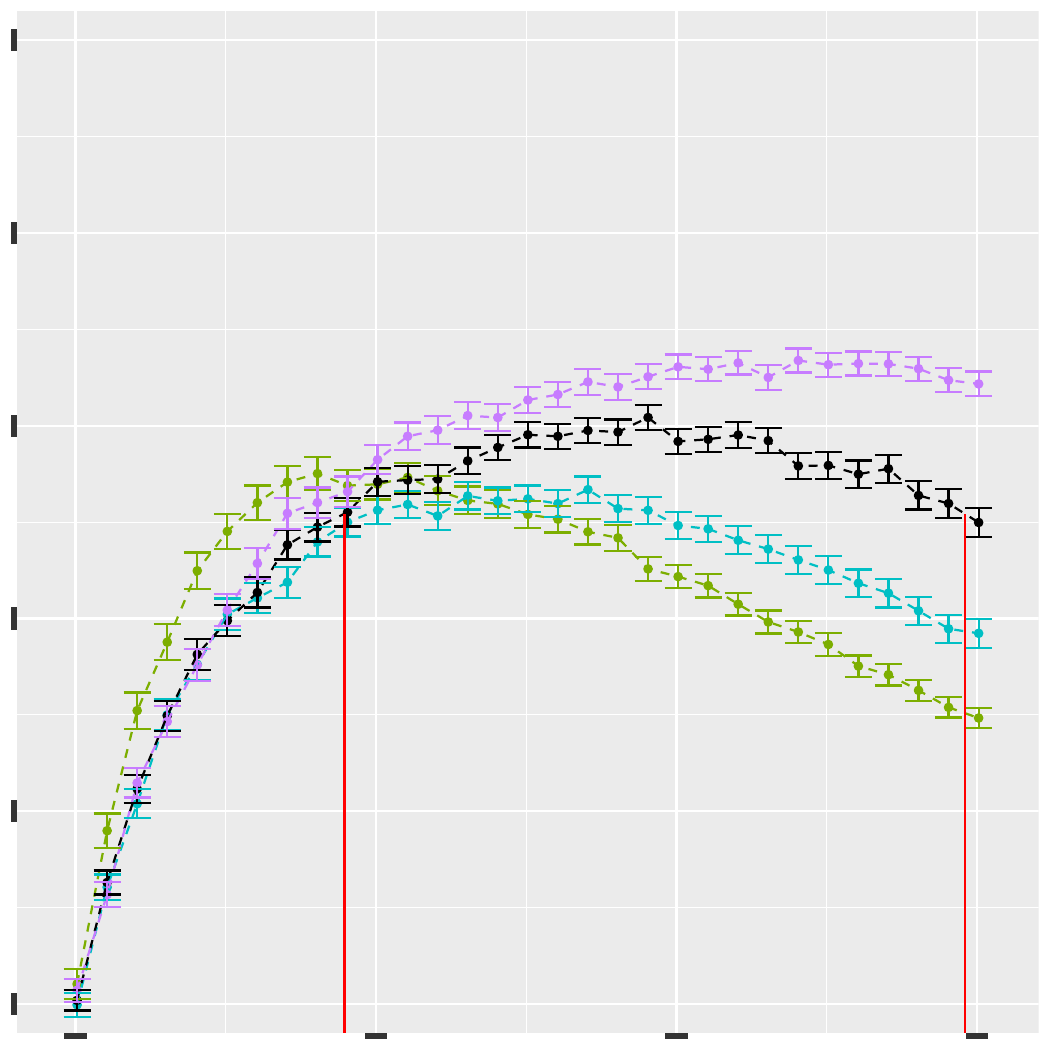}}\\\vspace{-6.5pt}
    % epsilon=5
    \subfloat{\includegraphics[width=0.18\linewidth]{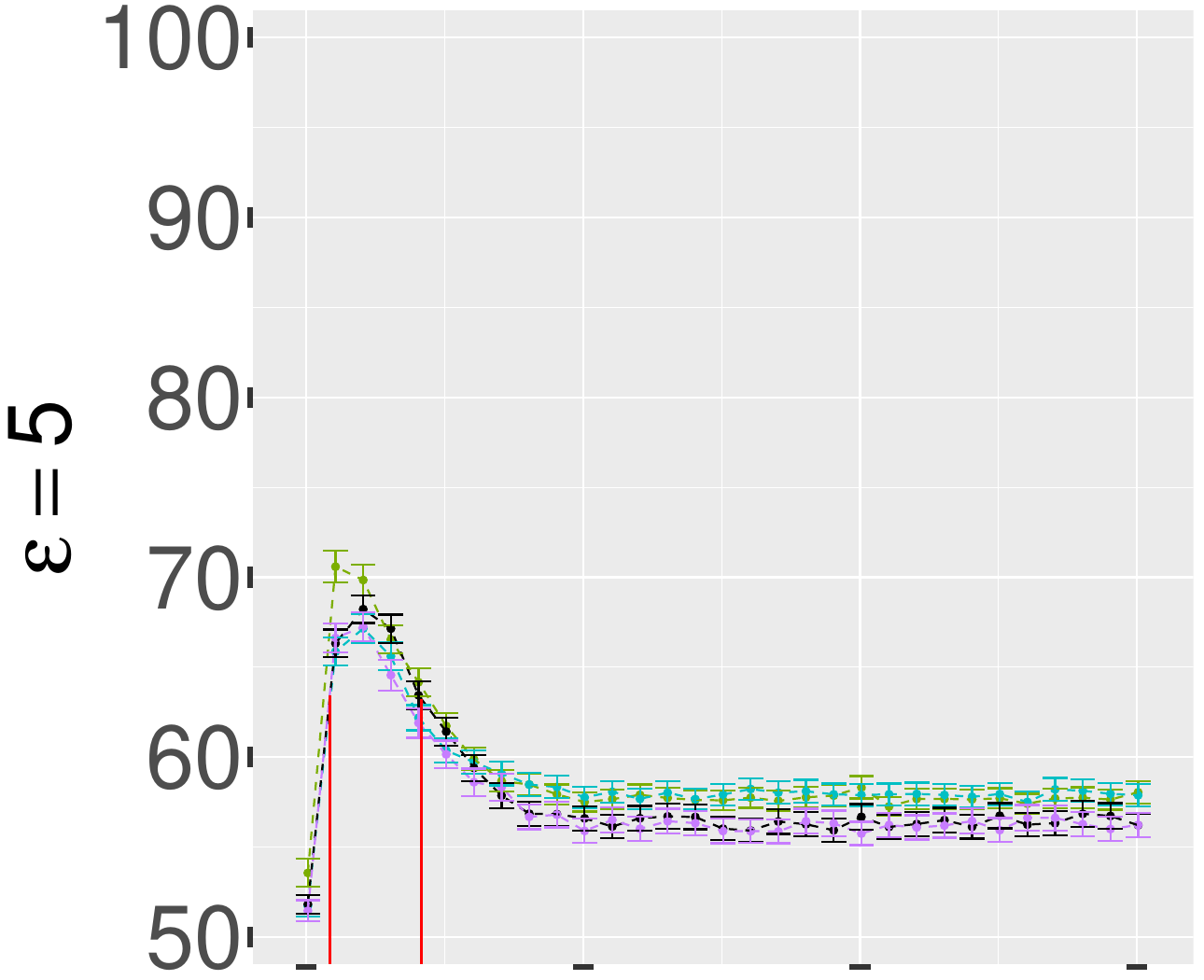}}
    \subfloat{\includegraphics[width=0.145\linewidth]{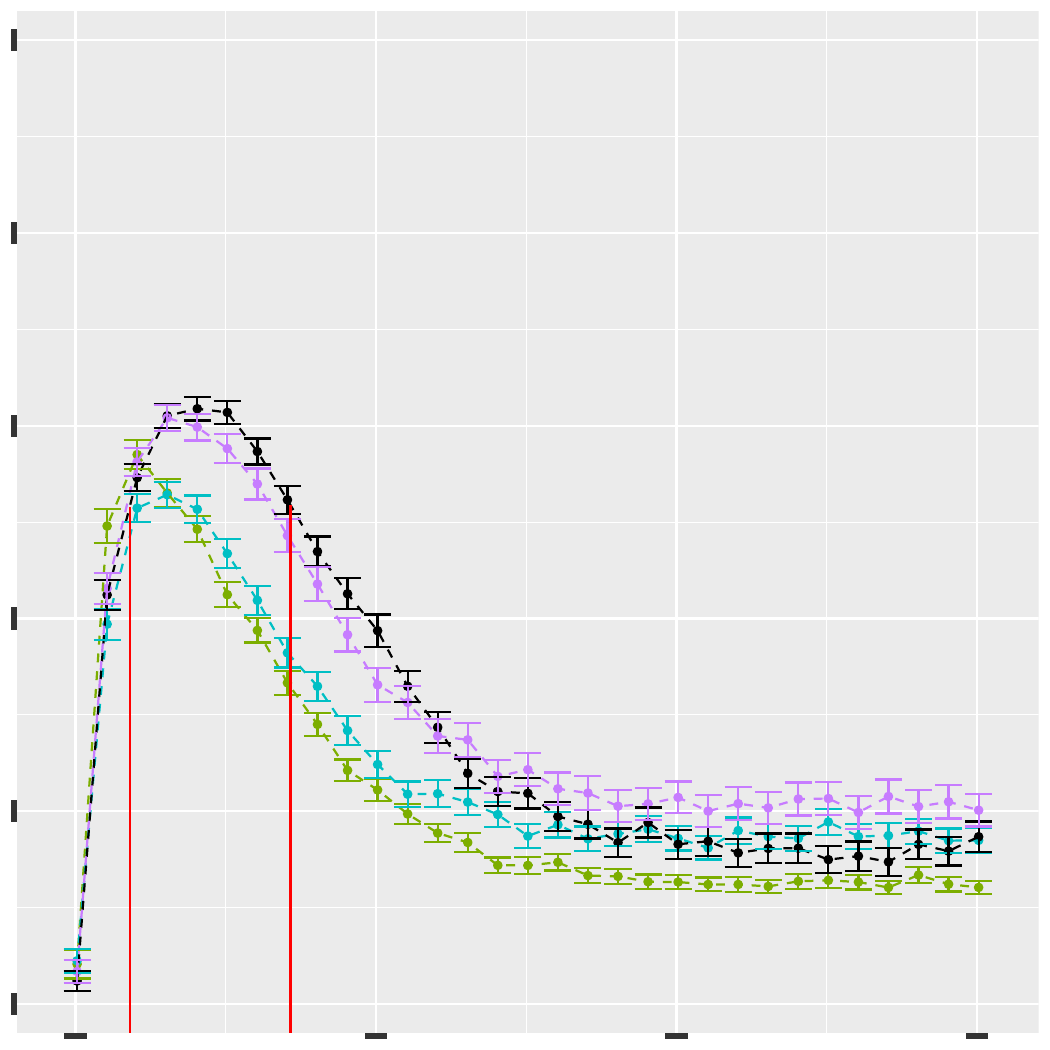}}
    \subfloat{\includegraphics[width=0.145\linewidth]{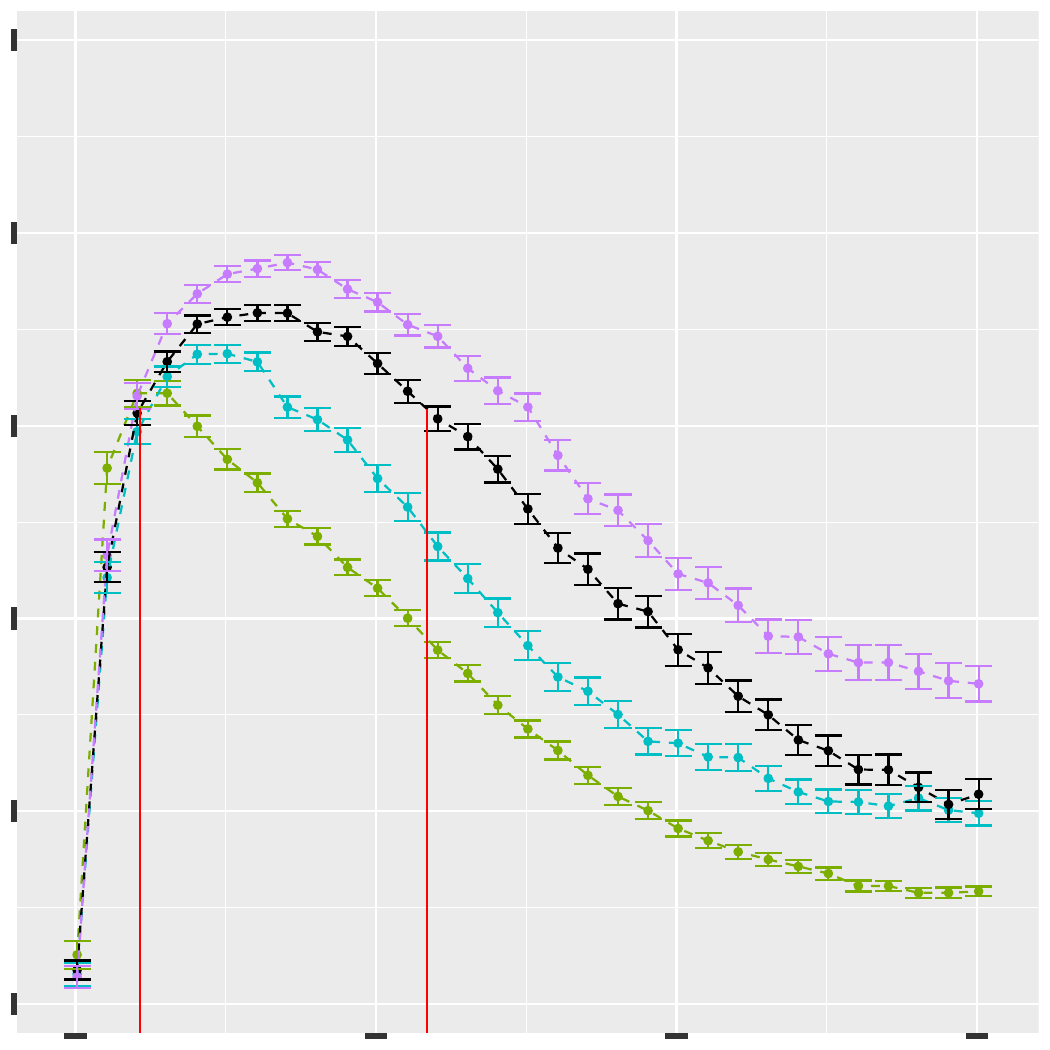}}
    \subfloat{\includegraphics[width=0.145\linewidth]{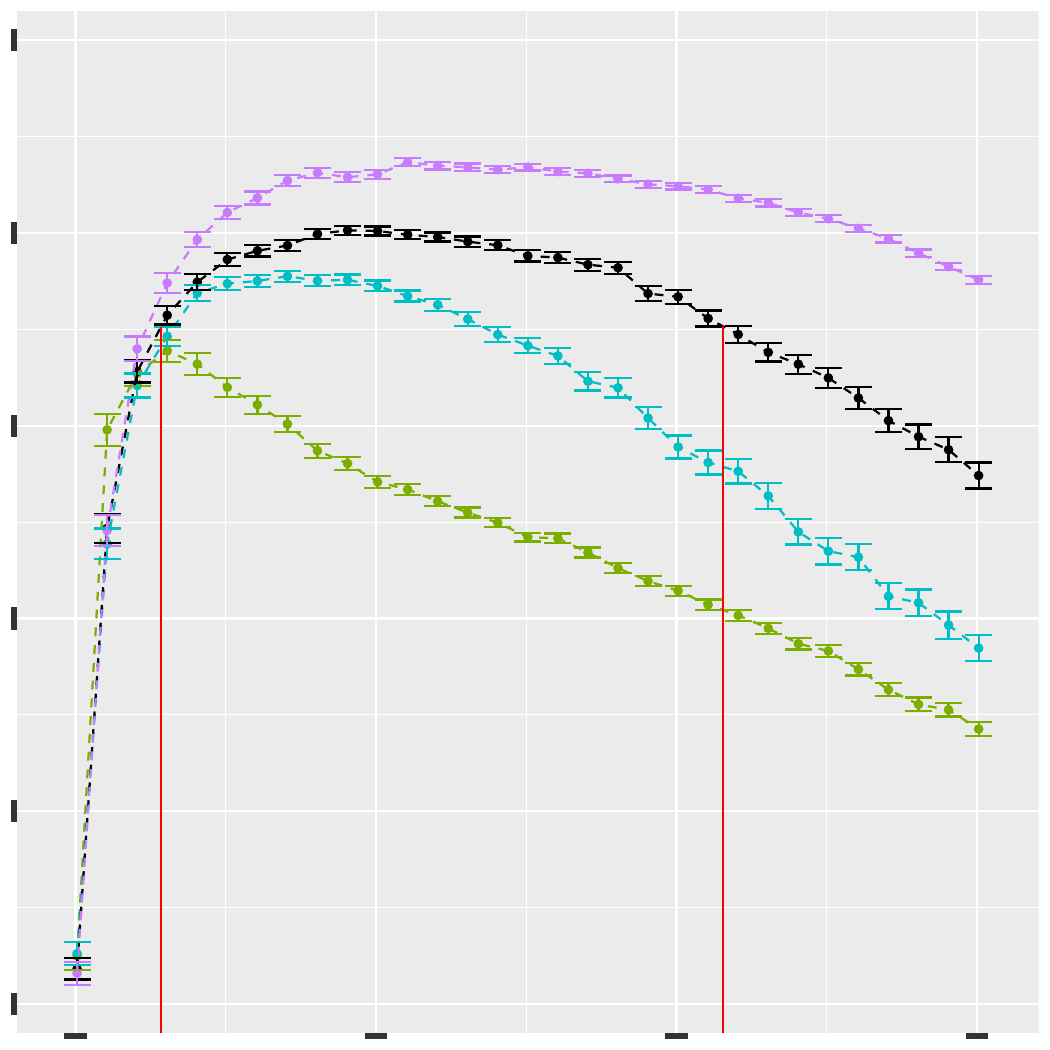}}\\\vspace{-6.5pt}
    % epsilon=150
    \subfloat{\includegraphics[width=0.18\linewidth]{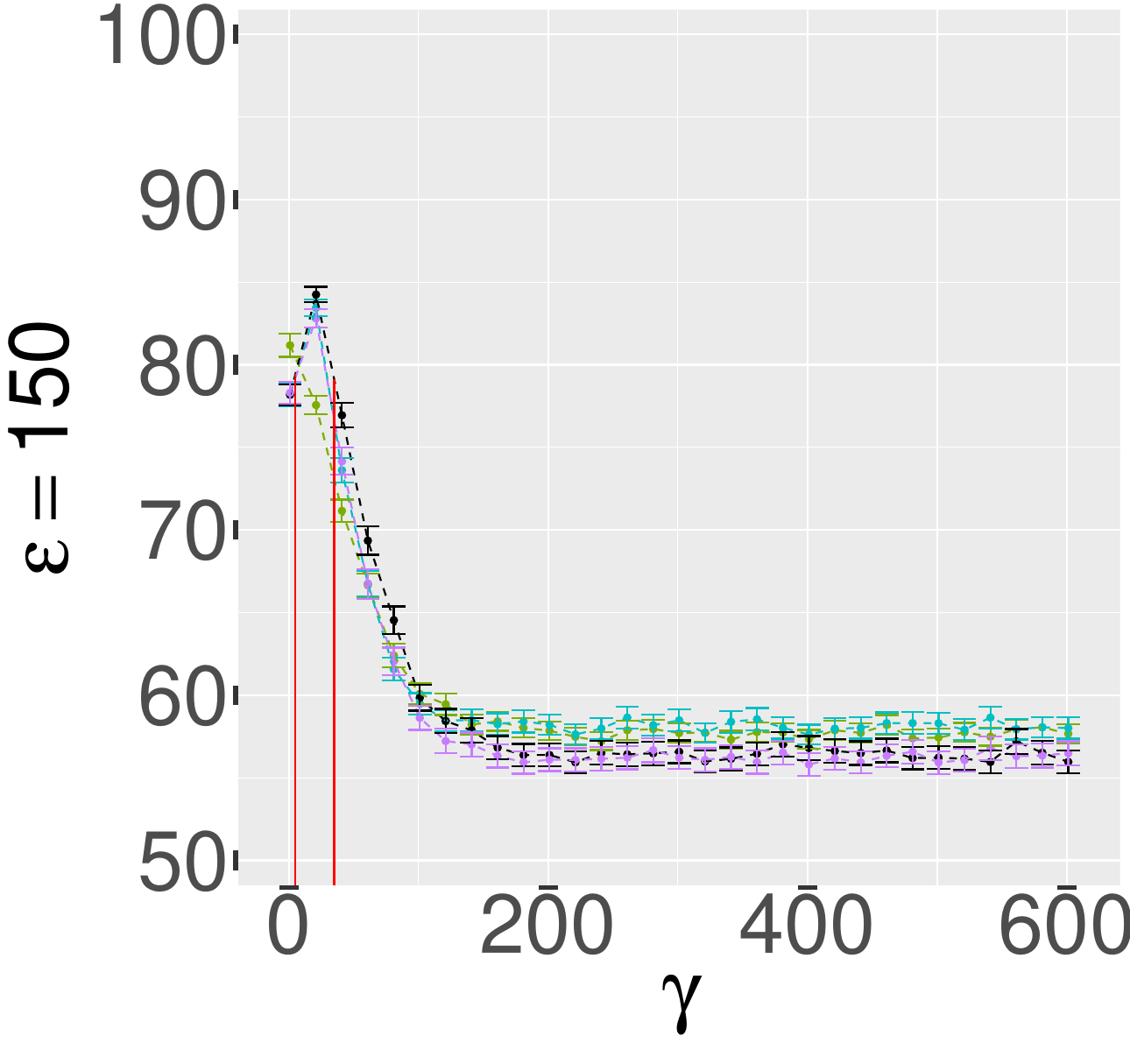}}
    \subfloat{\includegraphics[width=0.145\linewidth]{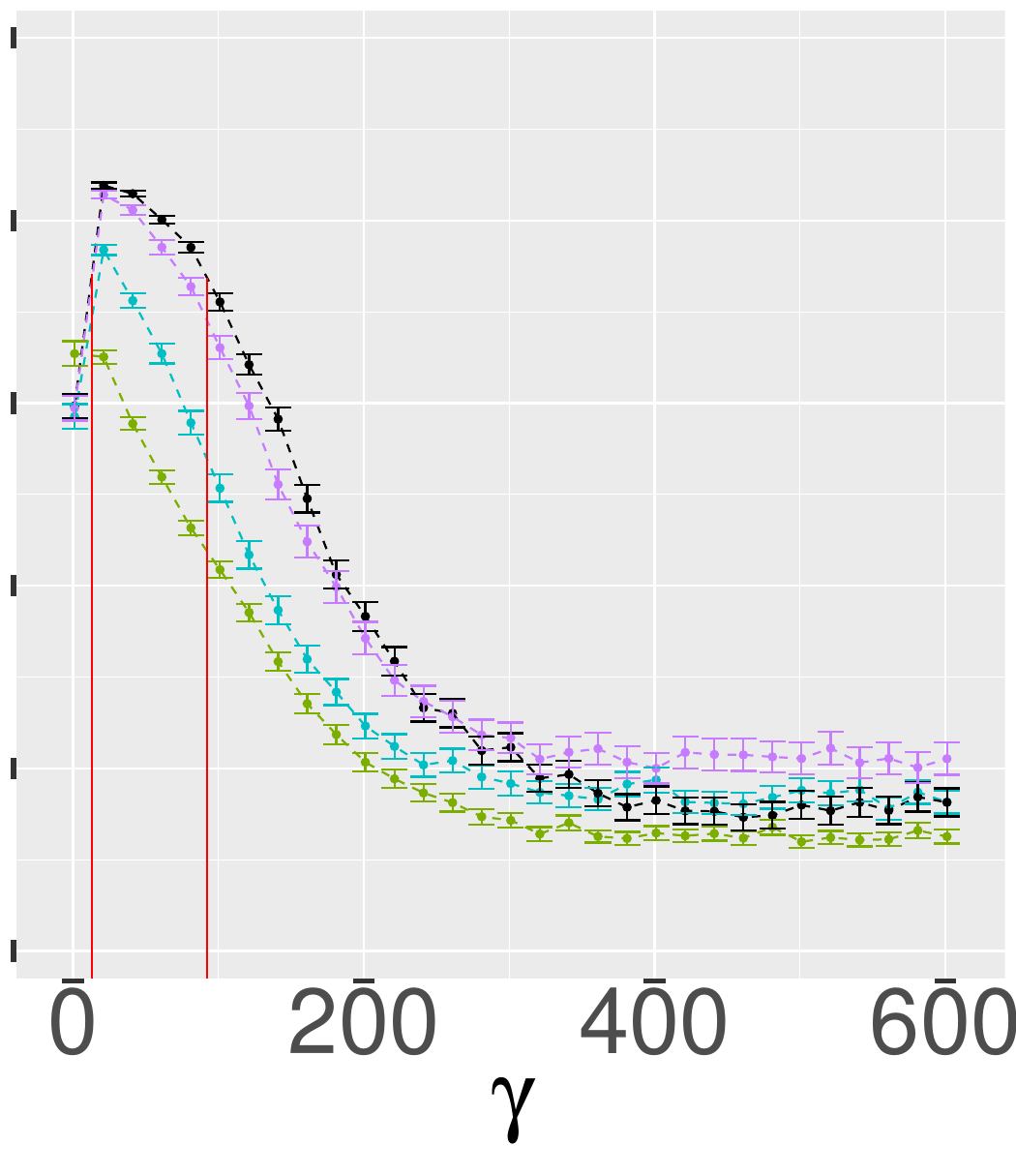}}
    \subfloat{\includegraphics[width=0.145\linewidth]{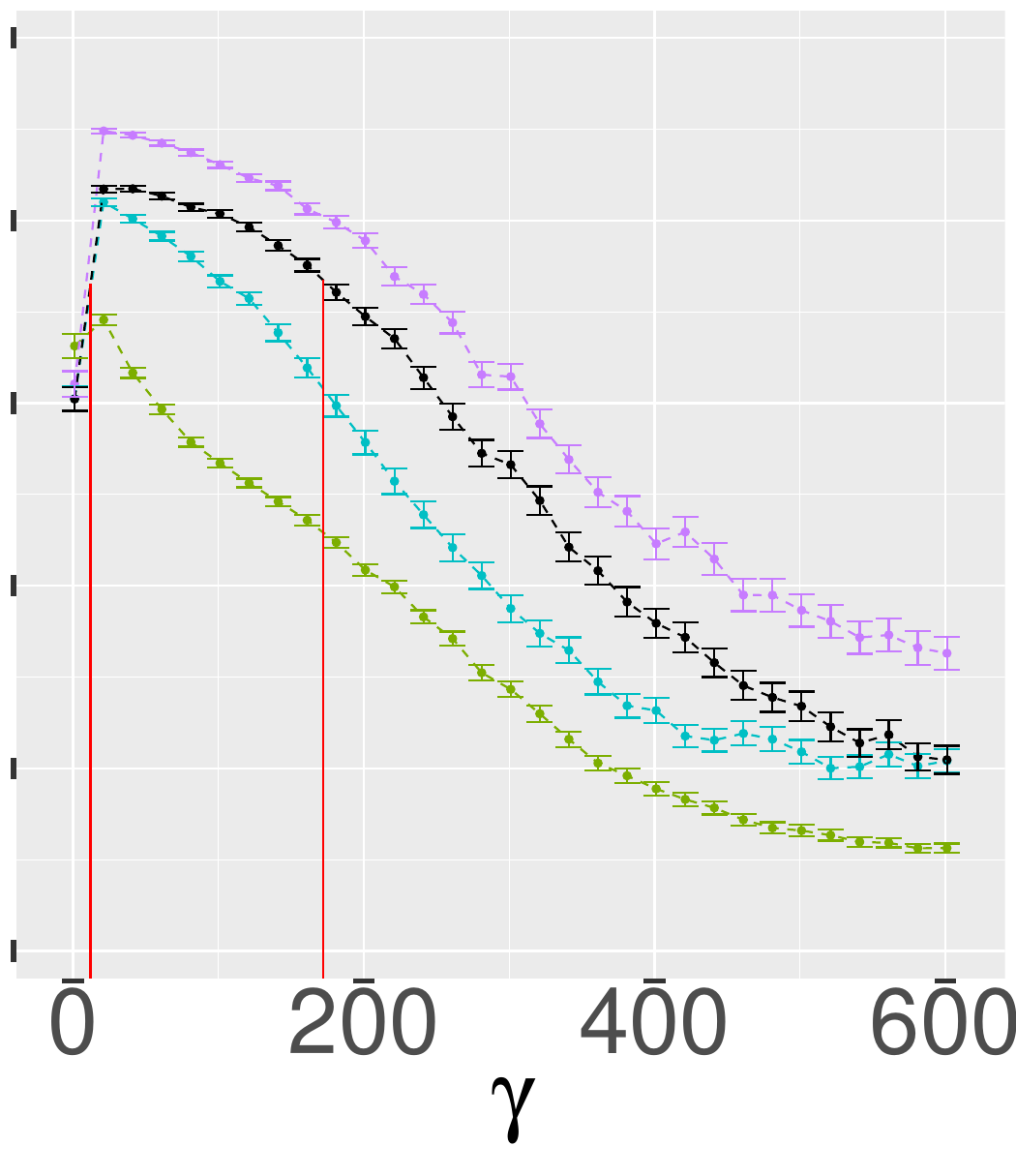}}
    \subfloat{\includegraphics[width=0.145\linewidth]{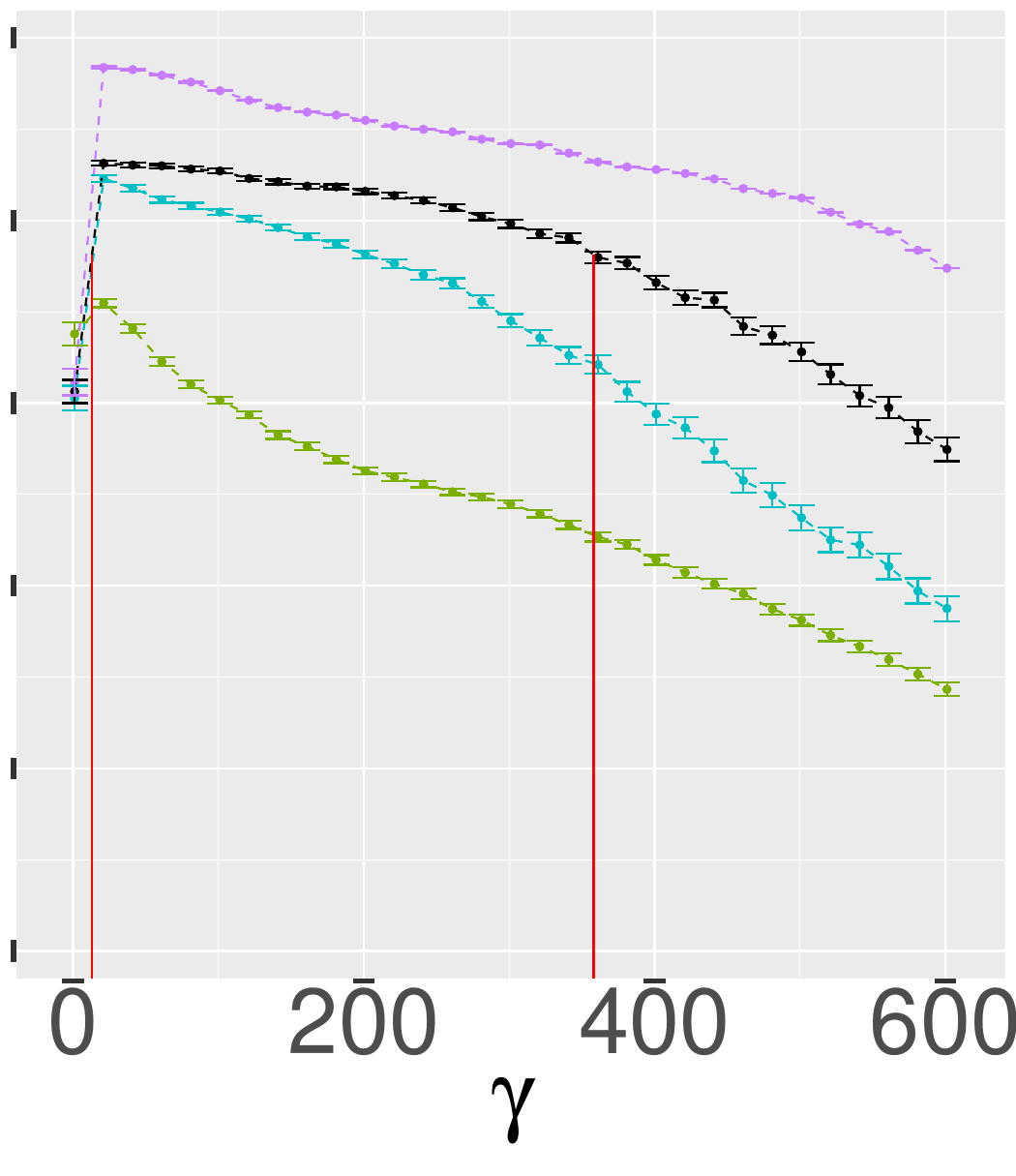}}\\
    \footnotesize{(a) 
    Different lines represent different independent dataset sizes $n_0$ from the same distribution as the sensitive dataset.}\\\vspace{-6pt}
    % epsilon=0.5
    \subfloat{\includegraphics[width=0.18\linewidth]{Figures/SupplMaterialsTheta/eps_0-5_N_200.pdf}}
    \subfloat{\includegraphics[width=0.145\linewidth]{Figures/SupplMaterialsTheta/eps_0-5_N_500.pdf}}
    \subfloat{\includegraphics[width=0.145\linewidth]{Figures/SupplMaterialsTheta/eps_0-5_N_1000.pdf}}
    \subfloat{\includegraphics[width=0.145\linewidth]{Figures/SupplMaterialsTheta/eps_0-5_N_2000.pdf}}\\\vspace{-6.5pt}
    % epsilon=1
    \subfloat{\includegraphics[width=0.18\linewidth]{Figures/SupplMaterialsTheta/eps_1_N_200.pdf}}
    \subfloat{\includegraphics[width=0.145\linewidth]{Figures/SupplMaterialsTheta/eps_1_N_500.pdf}}
    \subfloat{\includegraphics[width=0.145\linewidth]{Figures/SupplMaterialsTheta/eps_1_N_1000.pdf}}
    \subfloat{\includegraphics[width=0.145\linewidth]{Figures/SupplMaterialsTheta/eps_1_N_2000.pdf}}\\\vspace{-6.5pt}
    % epsilon=5
    \subfloat{\includegraphics[width=0.18\linewidth]{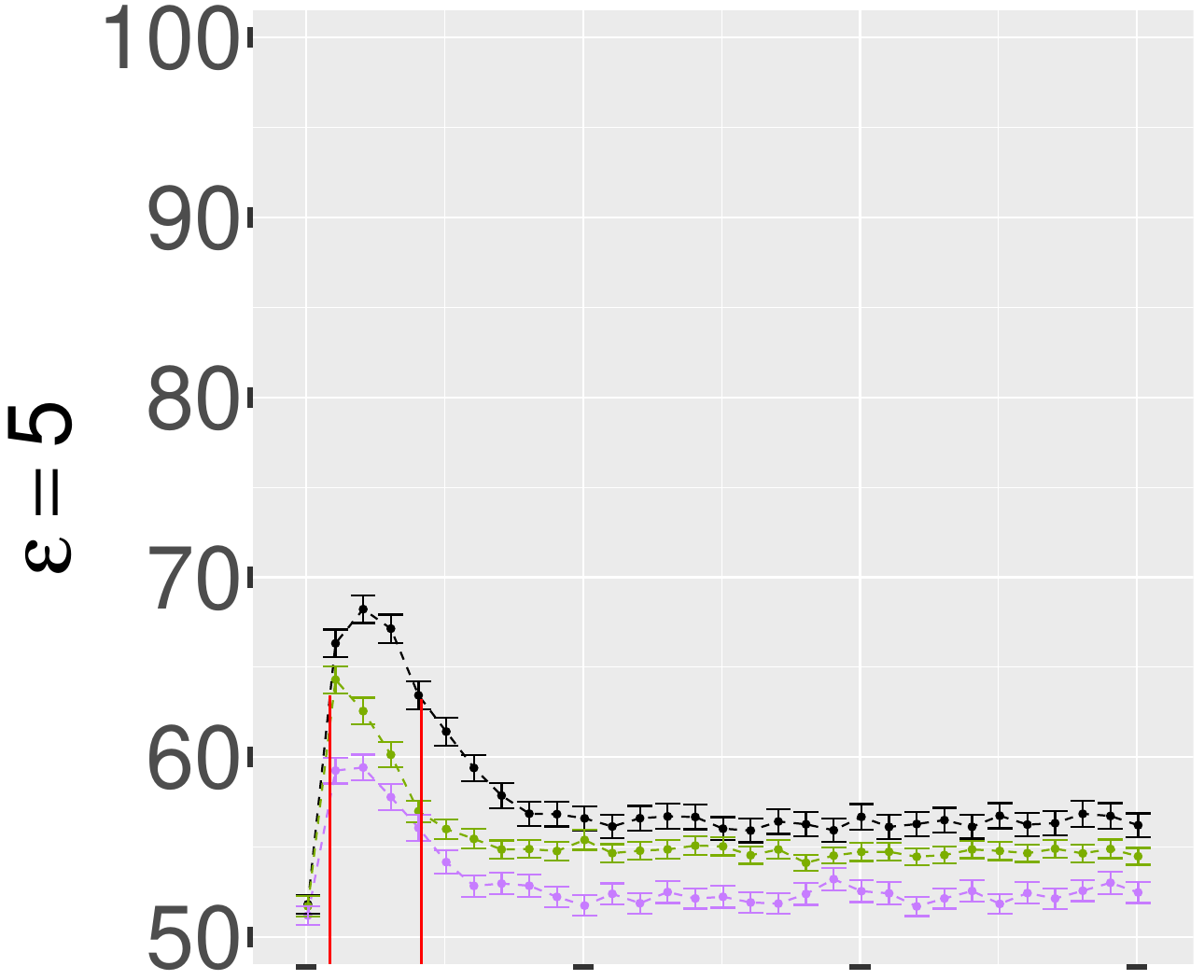}}
    \subfloat{\includegraphics[width=0.145\linewidth]{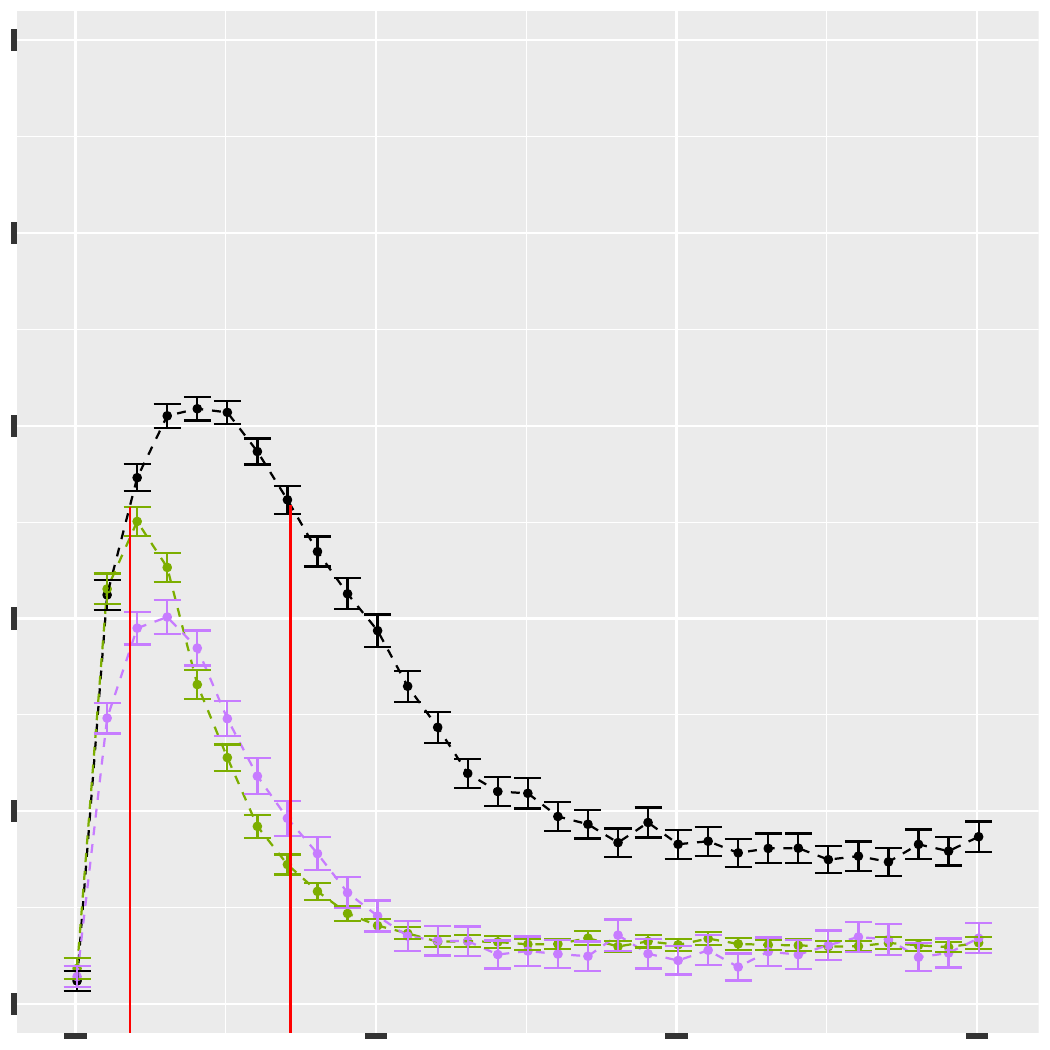}}
    \subfloat{\includegraphics[width=0.145\linewidth]{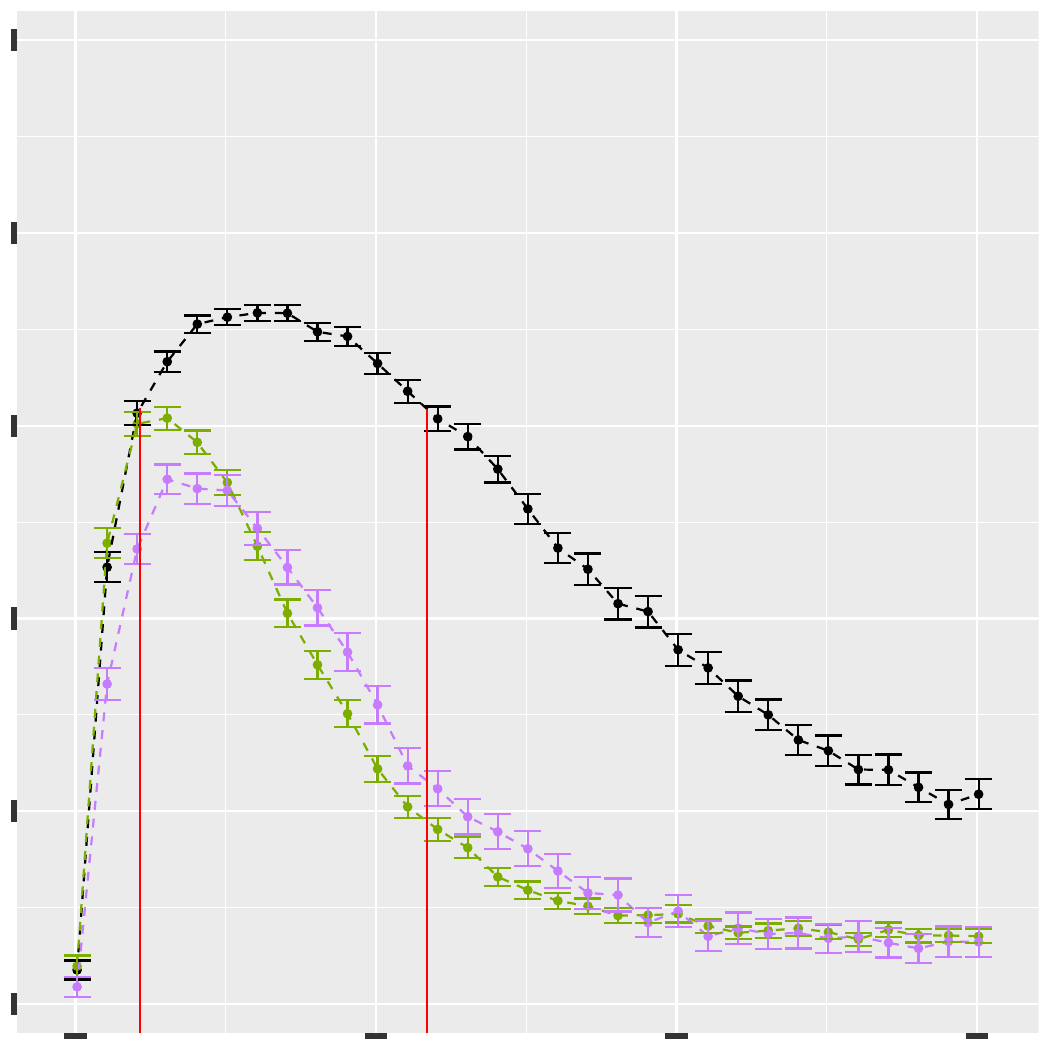}}
    \subfloat{\includegraphics[width=0.145\linewidth]{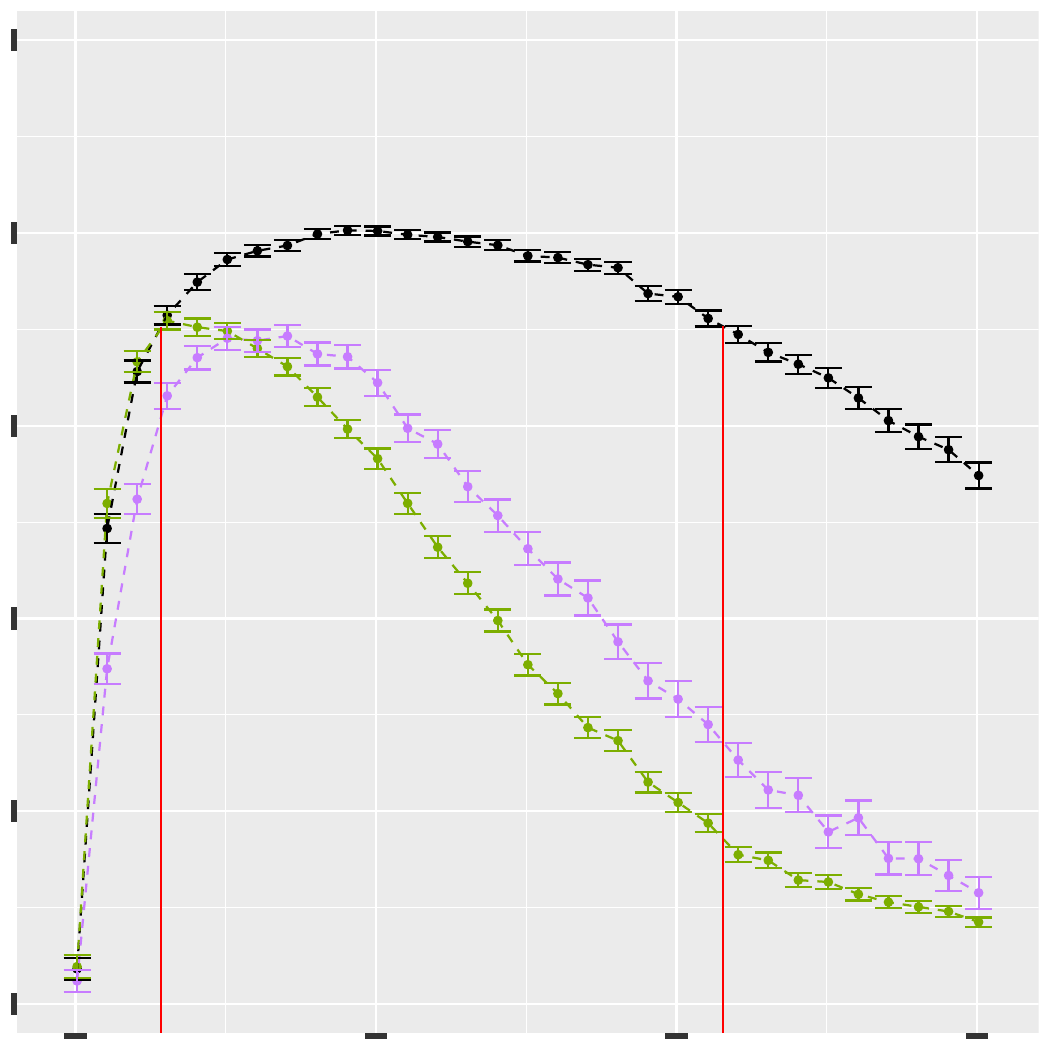}}\\\vspace{-6pt}
    % epsilon=150
    \subfloat{\includegraphics[width=0.18\linewidth]{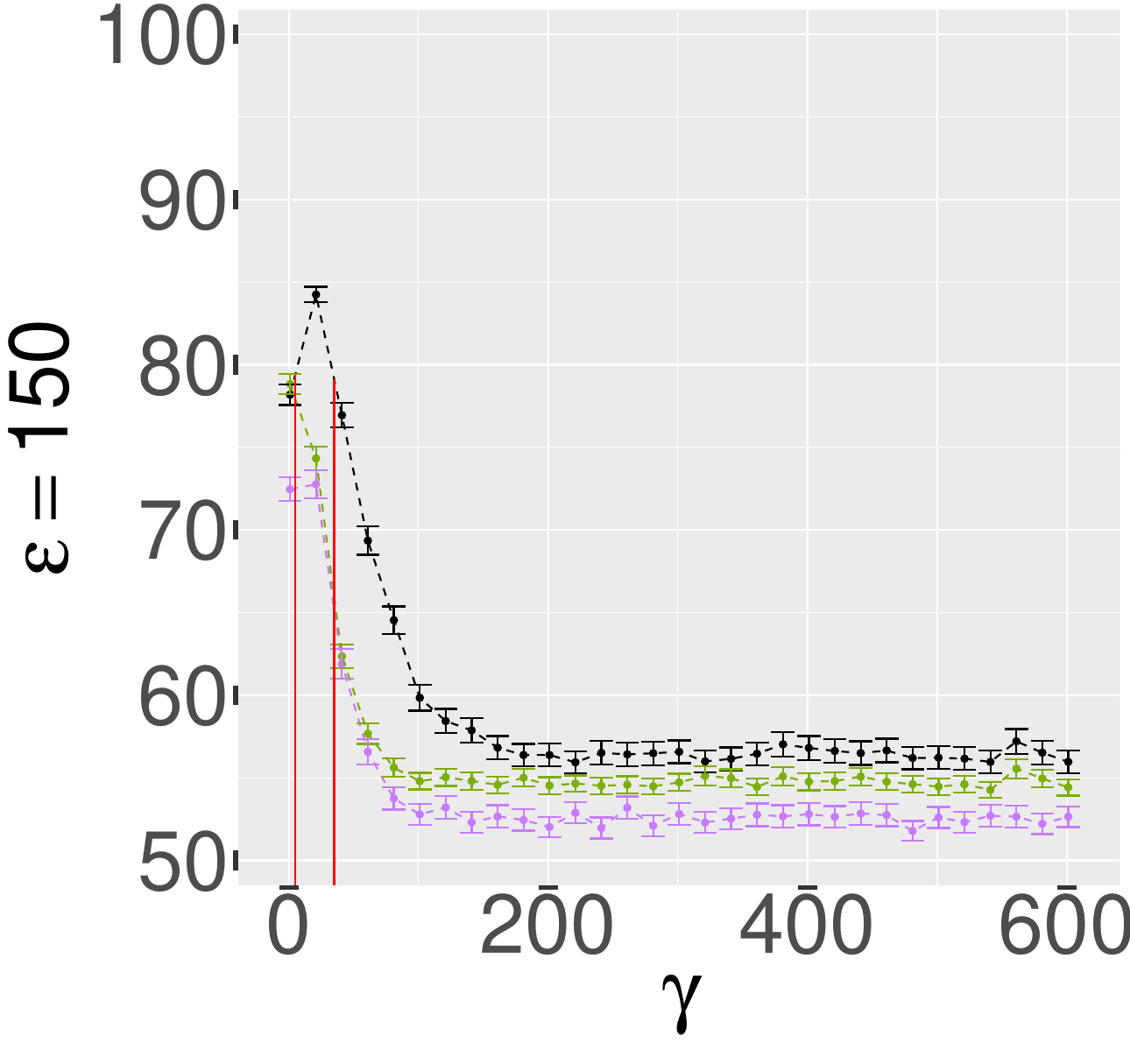}}
    \subfloat{\includegraphics[width=0.145\linewidth]{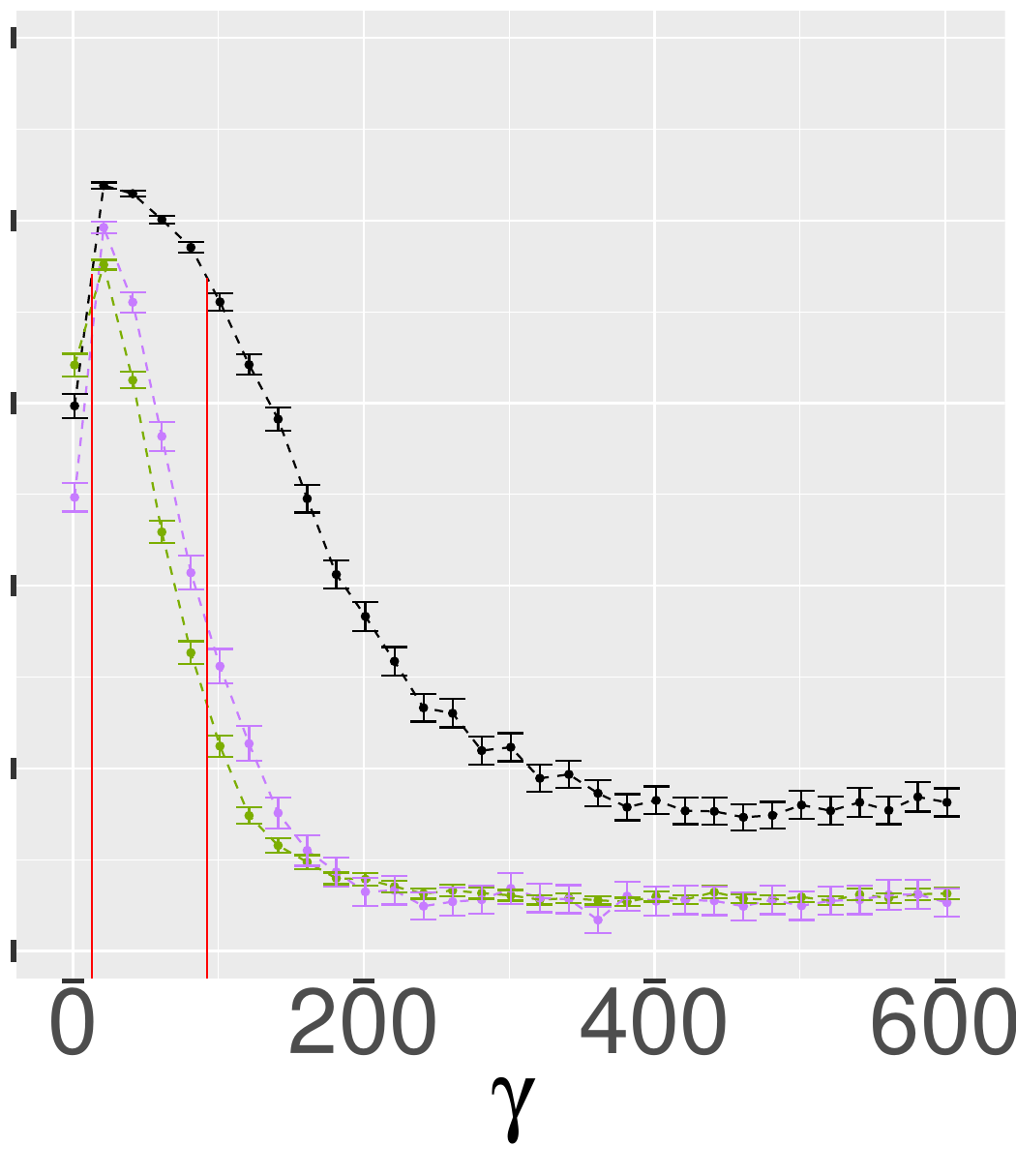}}
    \subfloat{\includegraphics[width=0.145\linewidth]{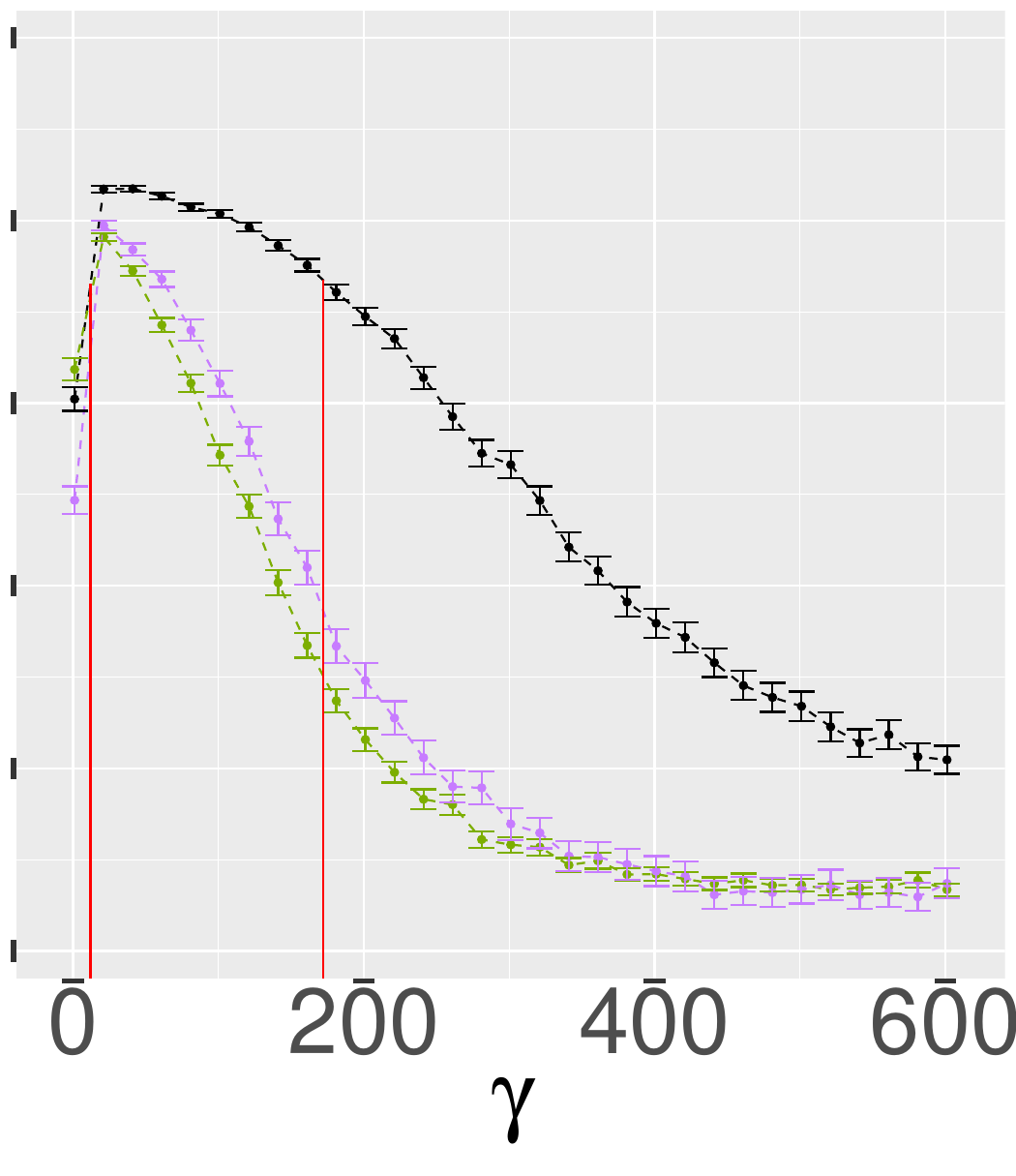}}
    \subfloat{\includegraphics[width=0.145\linewidth]{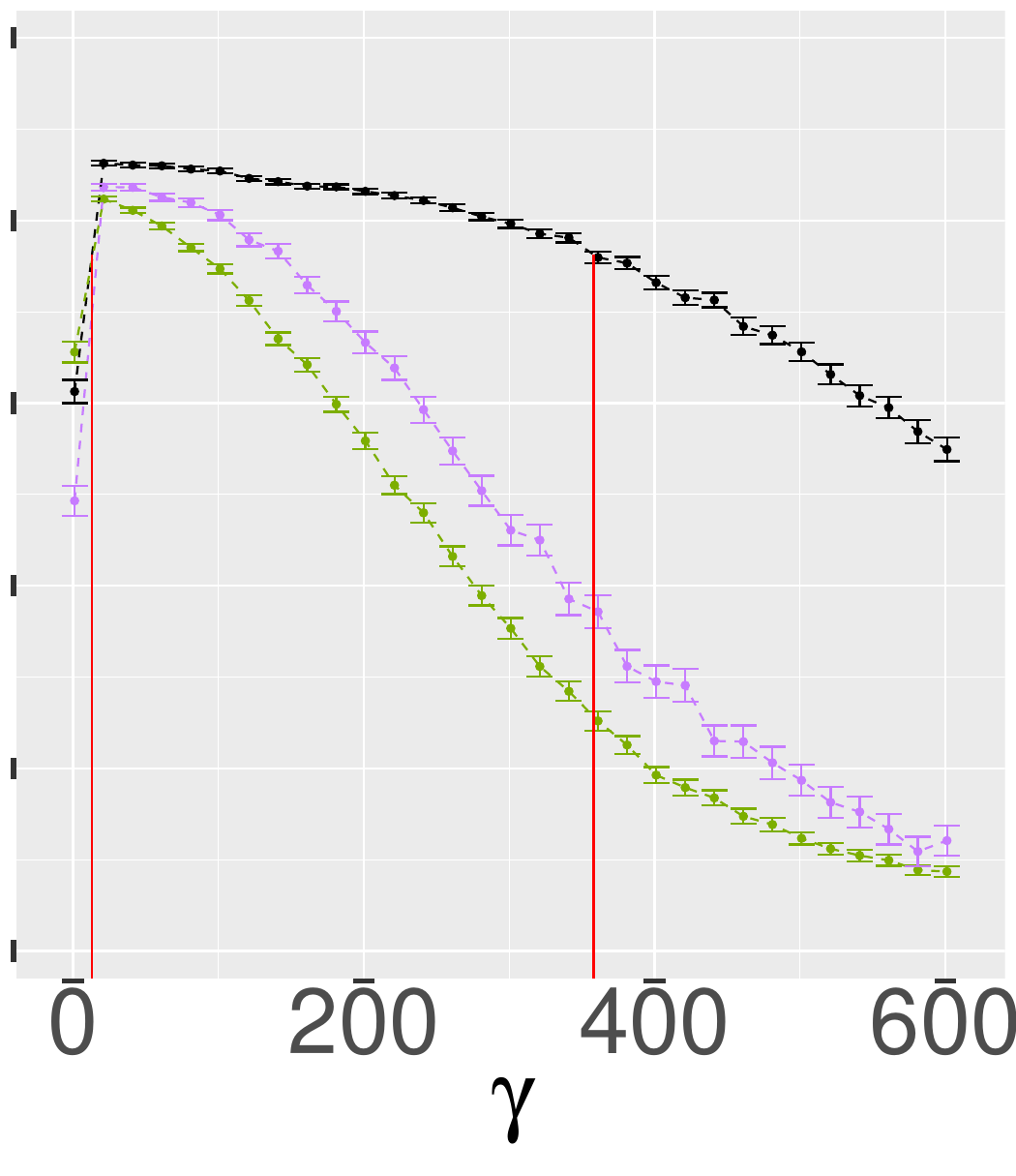}}\\
    \footnotesize{(b)   Different lines represent different coefficients $\boldsymbol{\theta}$.
    $\boldsymbol{\theta}_0 = [1, 1, 1, -1.8, -2.2]^T$ matches the sensitive dataset, while $\boldsymbol{\theta}_1 = [1, -1, -1, 1.5, -1.5]^T$ and $\boldsymbol{\theta}_2 = [1, -0.5, 0.5, 1, 1.5, -2.5, -2]^T$ represent alternatives.}\vspace{-6pt}
    \caption{Average (95\% CI) validation set accuracy as a function of the regularization constant $\gamma$  for various combinations of $\varepsilon$ and sensitive dataset size $n$. 
    The red lines mark the region of $\gamma$ where the corresponding accuracy in the baseline case ($\boldsymbol{\theta}_0$) is within $5\%$ of the highest.}
    \label{fig:suppl_materials_theta}
\end{figure*}

\begin{figure*}[!htb]
    \centering
    % epsilon=0.5
    \subfloat{\includegraphics[width=0.18\linewidth]{Figures/SupplMaterialsDist/eps_0-5_N_200.pdf}}
    \subfloat{\includegraphics[width=0.145\linewidth]{Figures/SupplMaterialsDist/eps_0-5_N_500.pdf}}
    \subfloat{\includegraphics[width=0.145\linewidth]{Figures/SupplMaterialsDist/eps_0-5_N_1000.pdf}}
    \subfloat{\includegraphics[width=0.145\linewidth]{Figures/SupplMaterialsDist/eps_0-5_N_2000.pdf}}\\\vspace{-6pt}
    % epsilon=1
    \subfloat{\includegraphics[width=0.18\linewidth]{Figures/SupplMaterialsDist/eps_1_N_200.pdf}}
    \subfloat{\includegraphics[width=0.145\linewidth]{Figures/SupplMaterialsDist/eps_1_N_500.pdf}}
    \subfloat{\includegraphics[width=0.145\linewidth]{Figures/SupplMaterialsDist/eps_1_N_1000.pdf}}
    \subfloat{\includegraphics[width=0.145\linewidth]{Figures/SupplMaterialsDist/eps_1_N_2000.pdf}}\\\vspace{-6pt}
    % epsilon=5
    \subfloat{\includegraphics[width=0.18\linewidth]{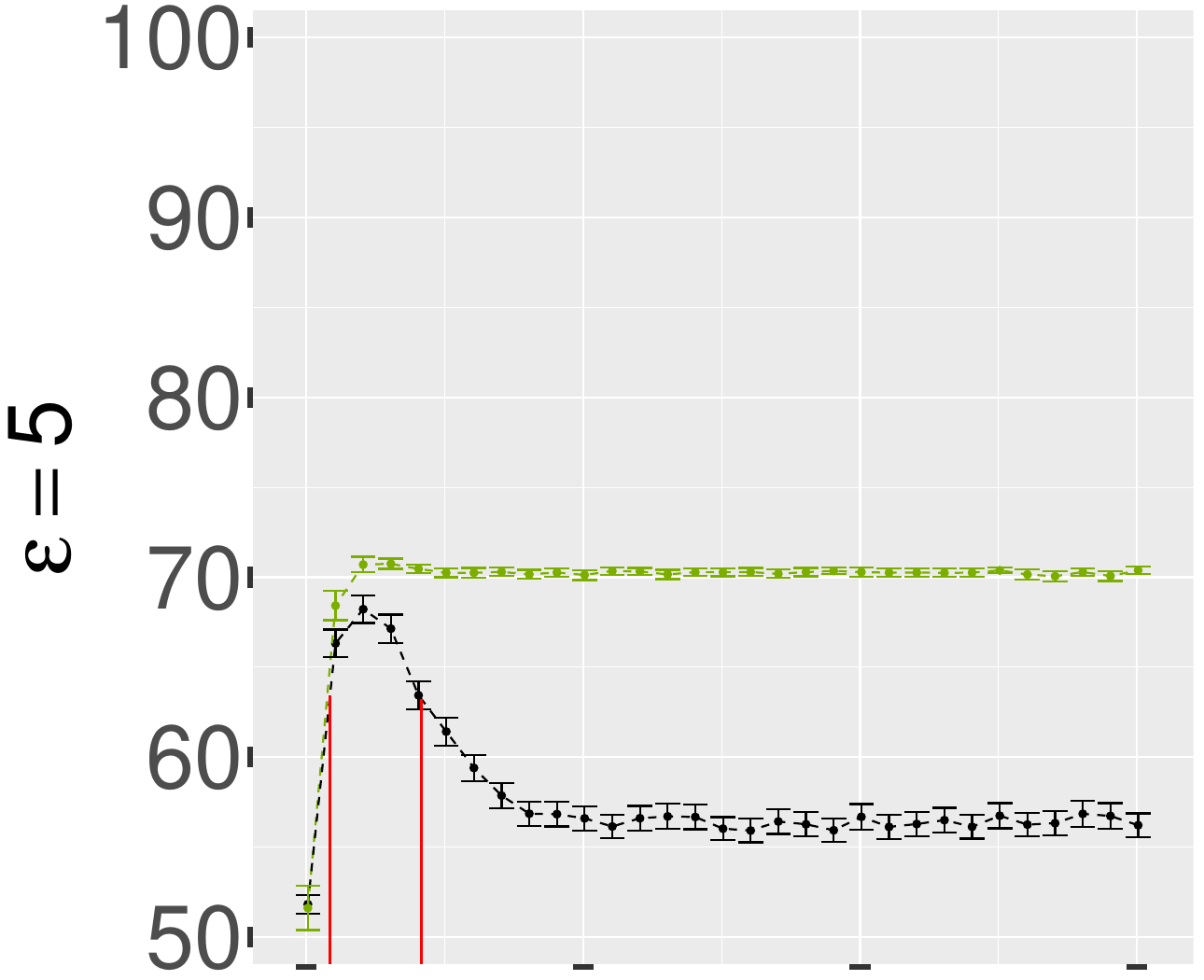}}
    \subfloat{\includegraphics[width=0.145\linewidth]{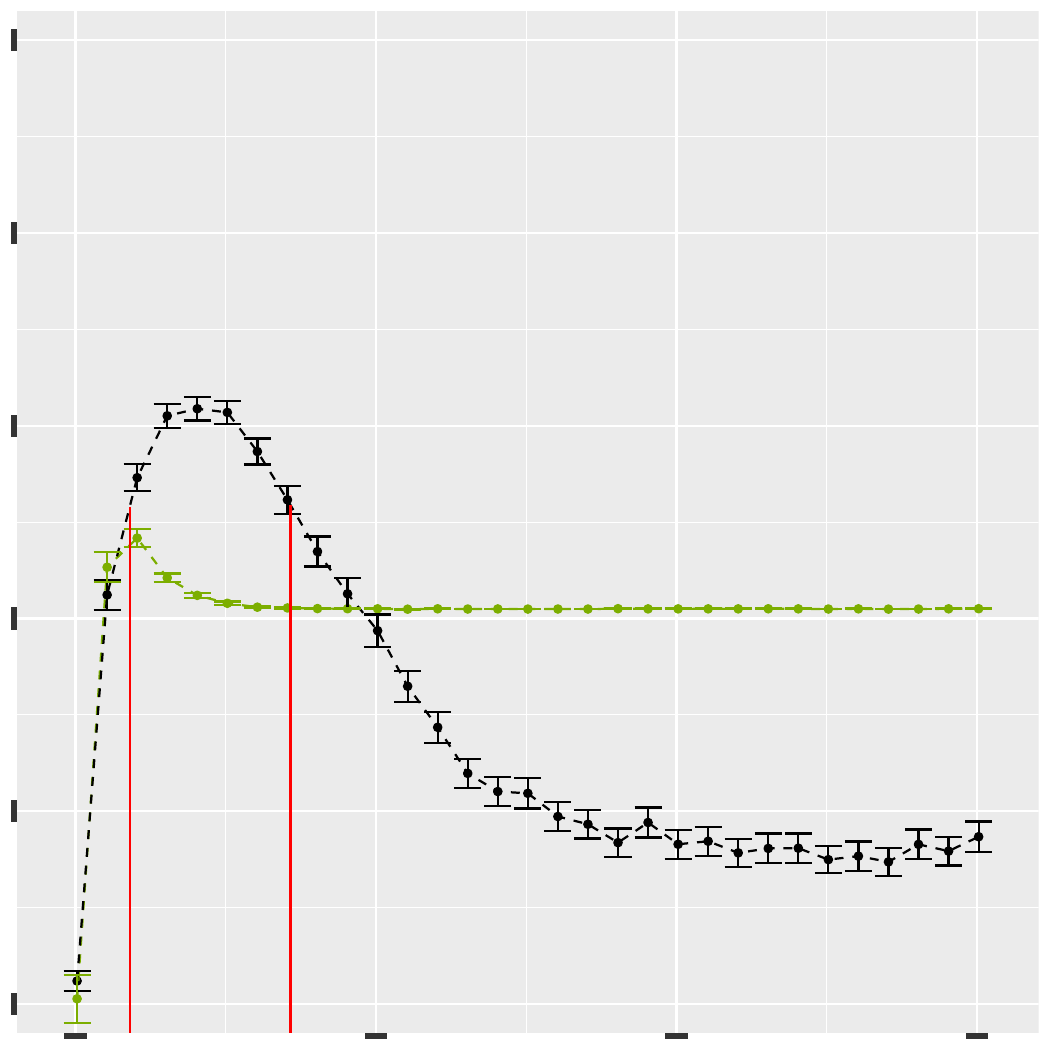}}
    \subfloat{\includegraphics[width=0.145\linewidth]{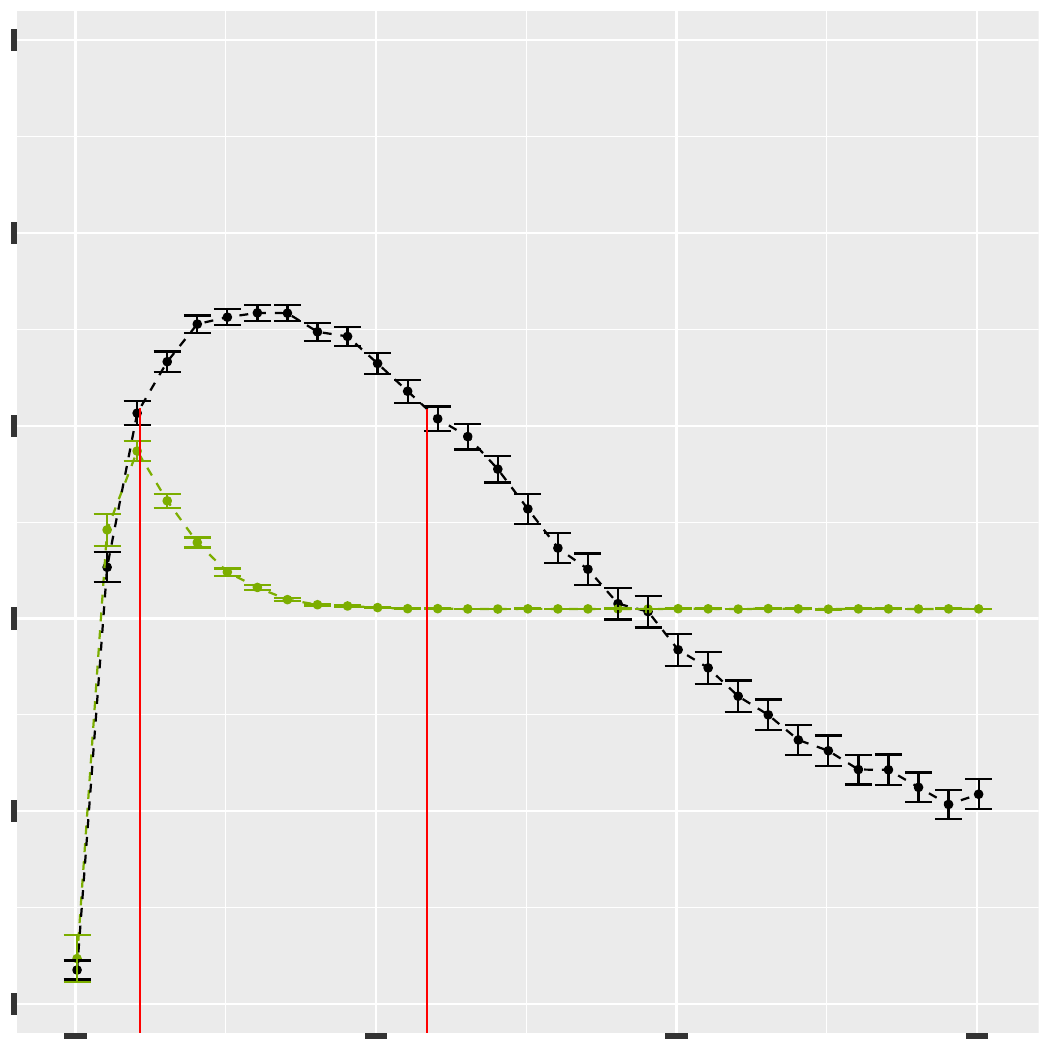}}
    \subfloat{\includegraphics[width=0.145\linewidth]{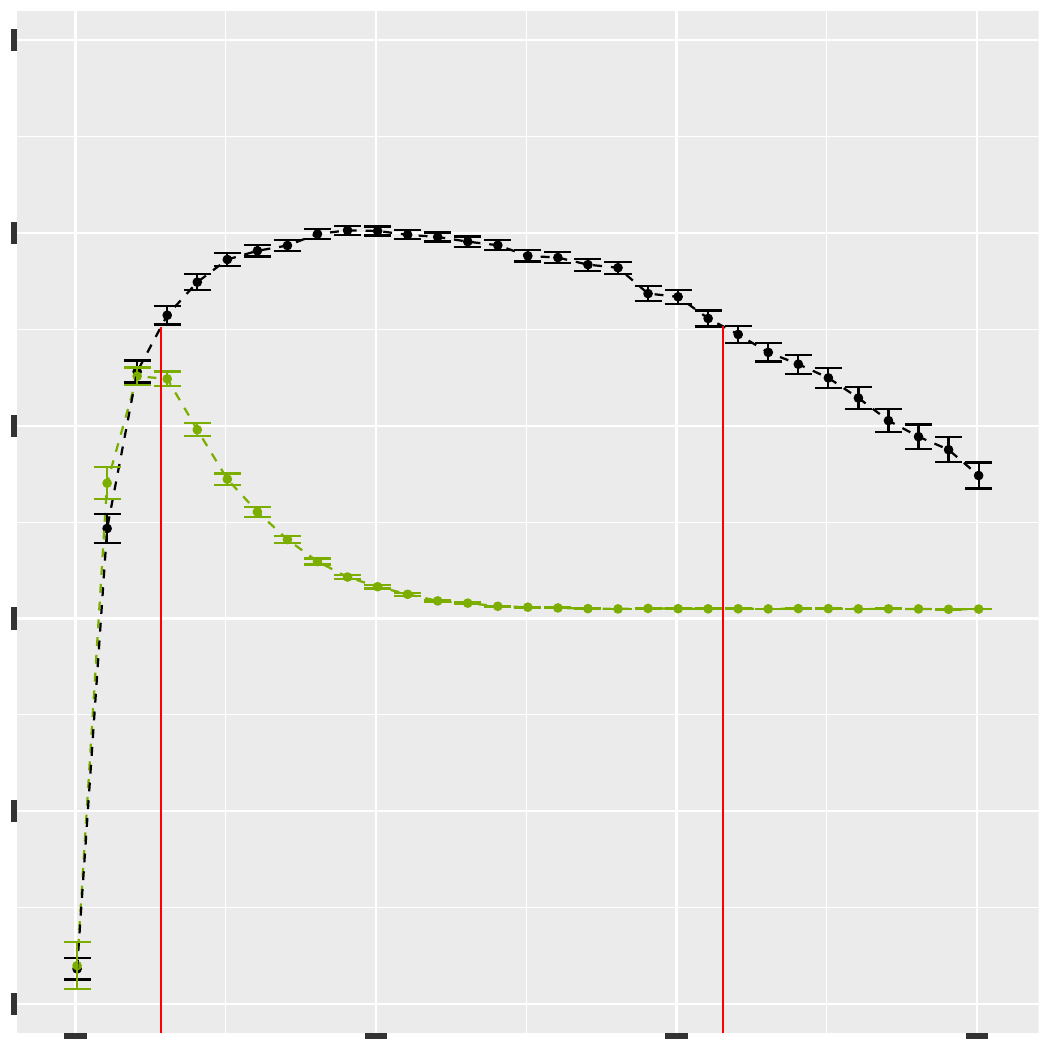}}\\\vspace{-6pt}
    % epsilon=150
    \subfloat{\includegraphics[width=0.18\linewidth]{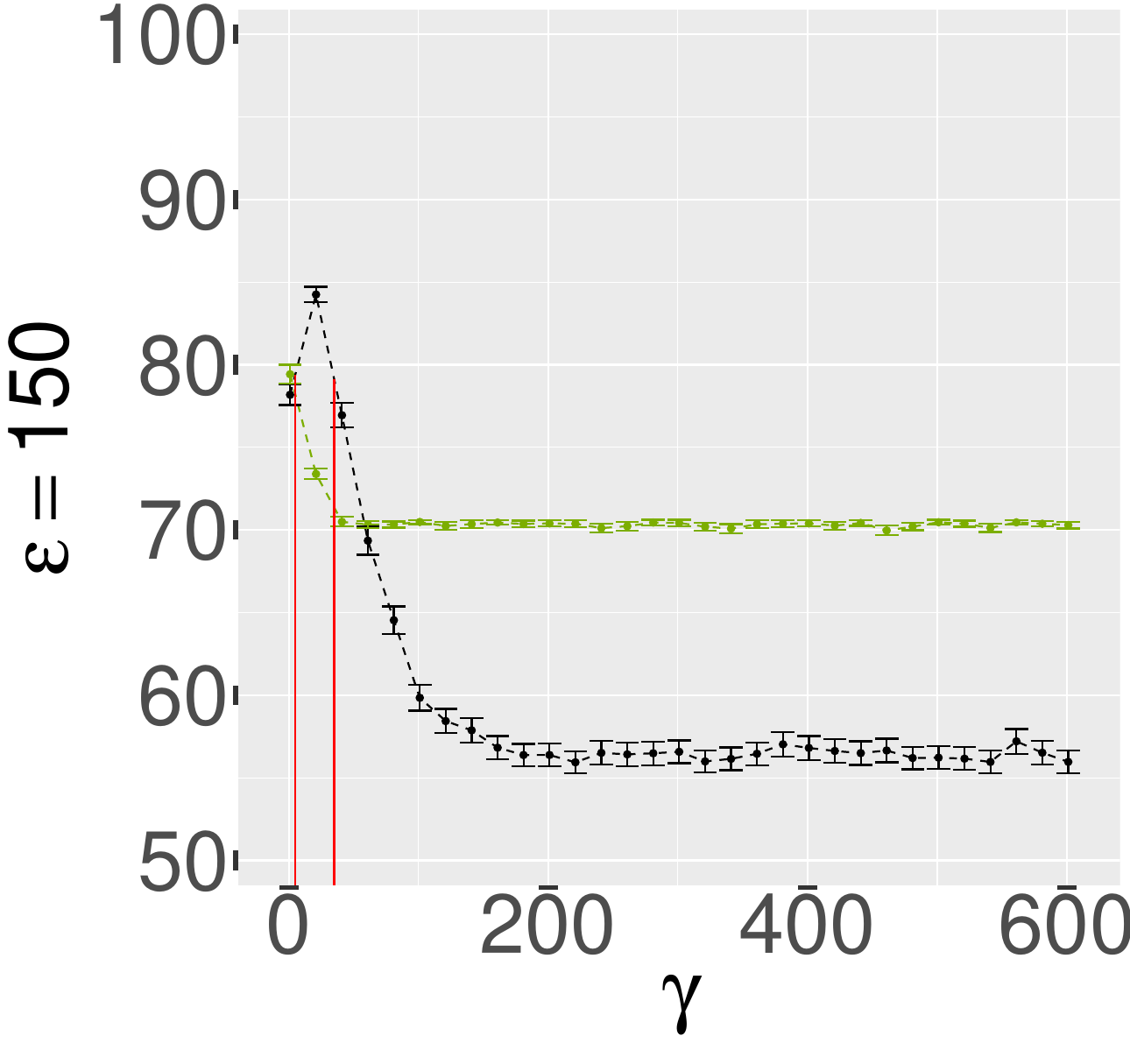}}
    \subfloat{\includegraphics[width=0.145\linewidth]{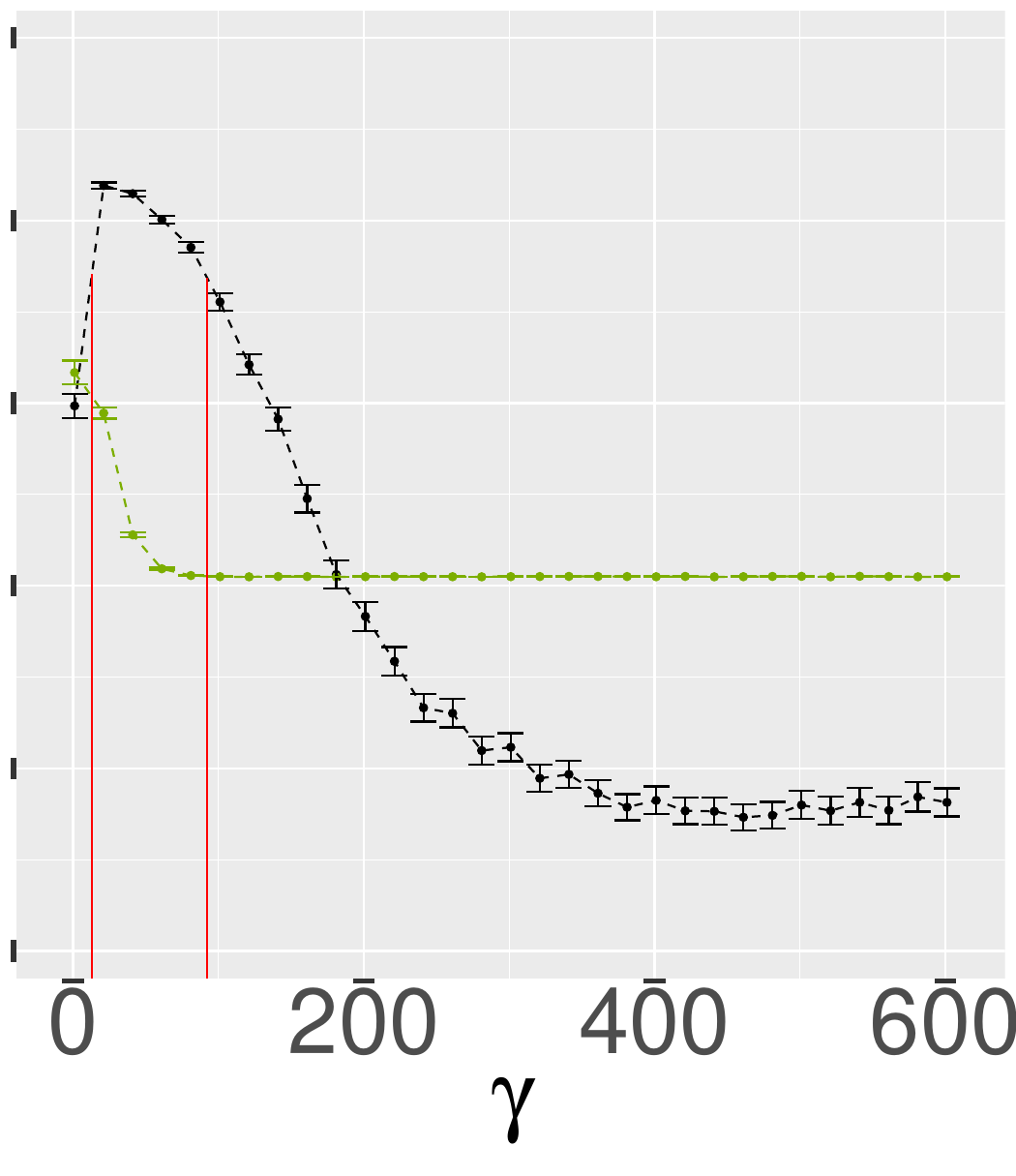}}
    \subfloat{\includegraphics[width=0.145\linewidth]{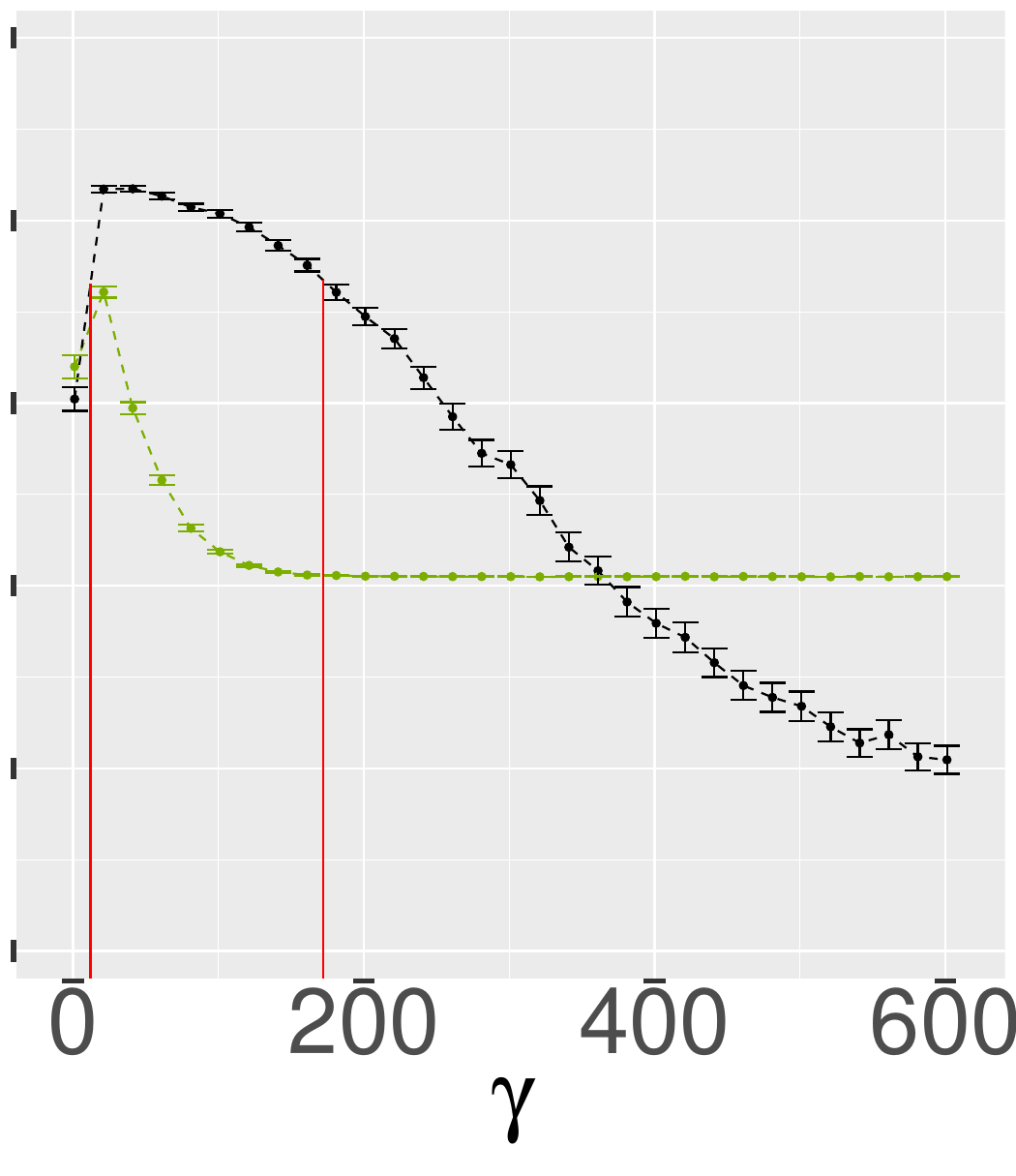}}
    \subfloat{\includegraphics[width=0.145\linewidth]{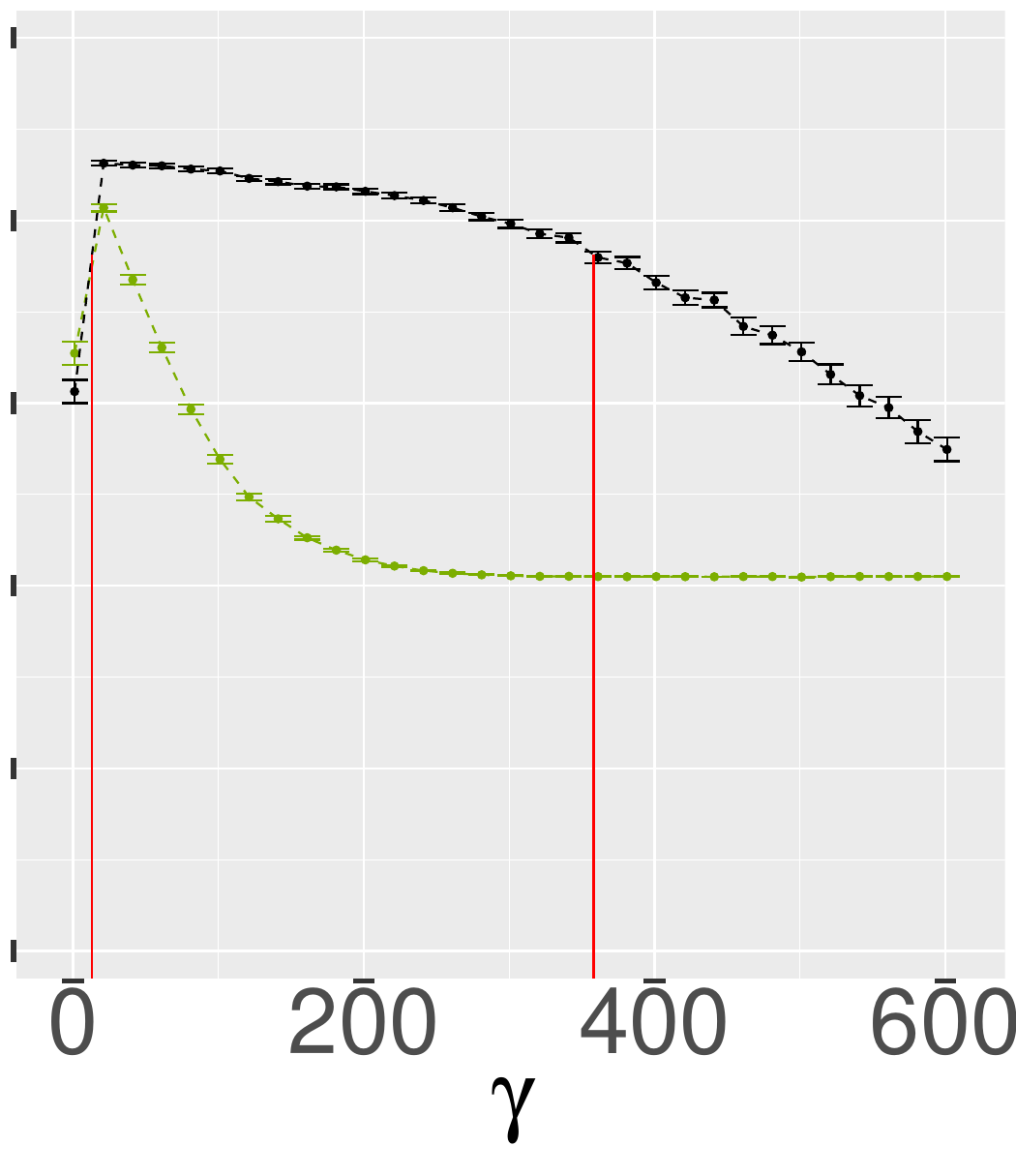}}\\\vspace{-4pt}
    \footnotesize{(a)} Different lines represent different distributions for $\mathbf{X}$ in the independent dataset\\\vspace{-4pt}
    % epsilon=0.5
    \subfloat{\includegraphics[width=0.18\linewidth]{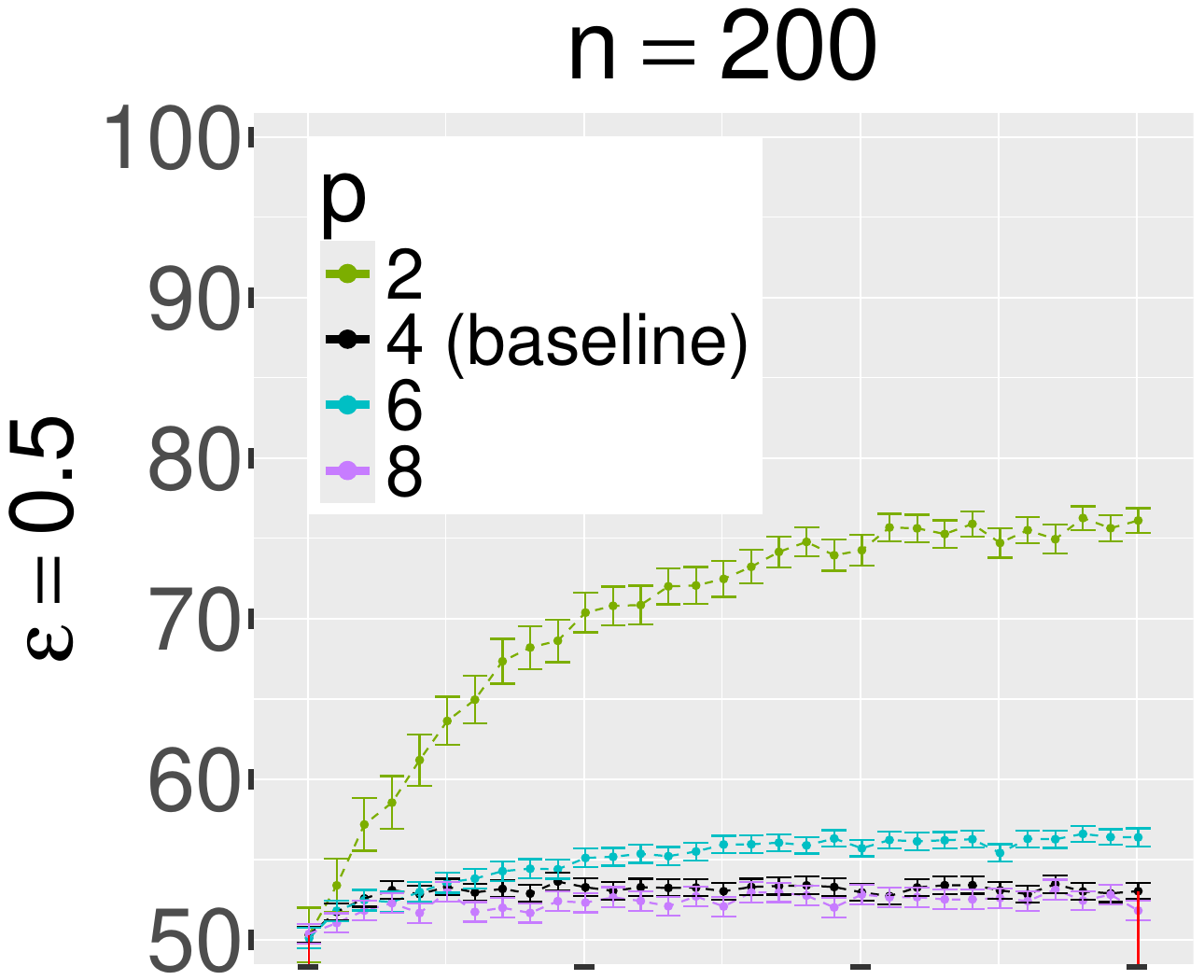}}
    \subfloat{\includegraphics[width=0.145\linewidth]{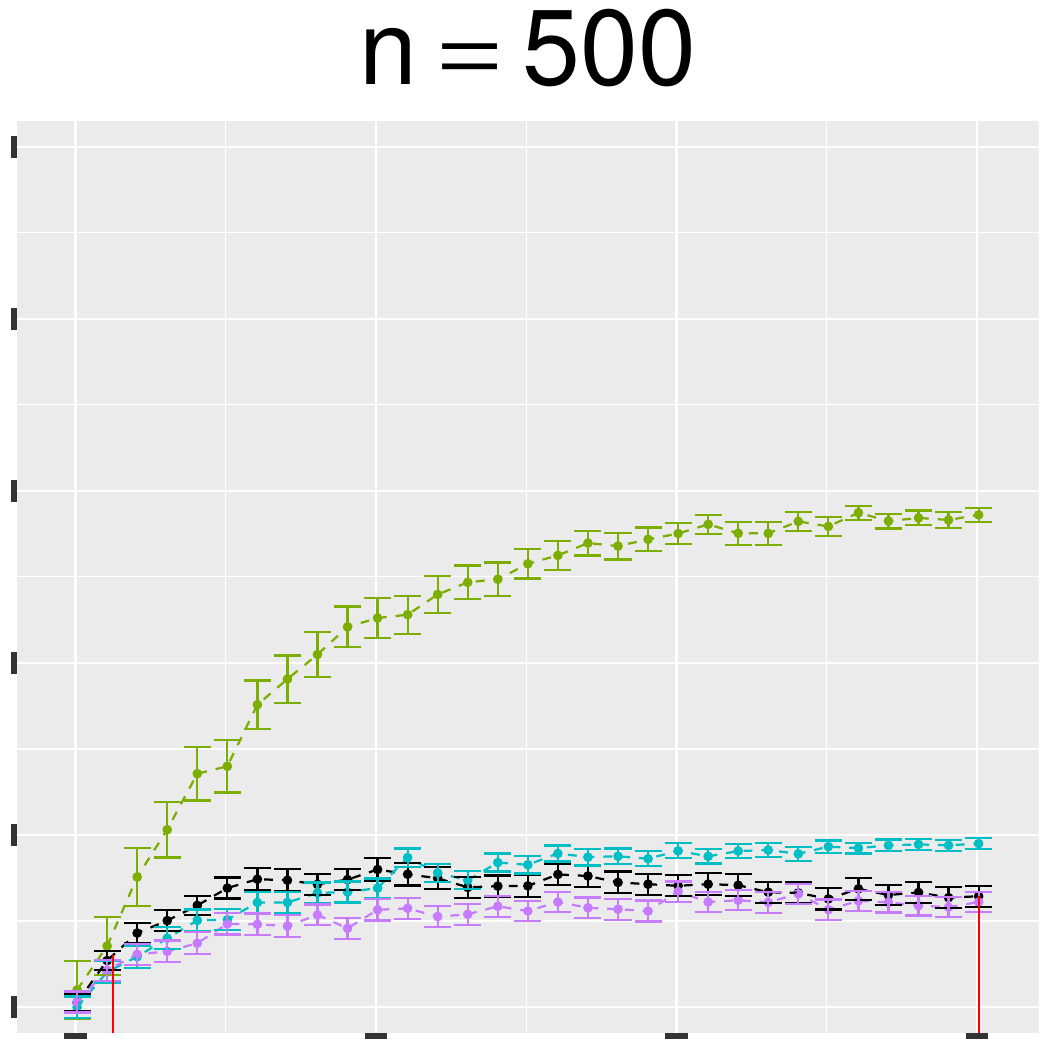}}
    \subfloat{\includegraphics[width=0.145\linewidth]{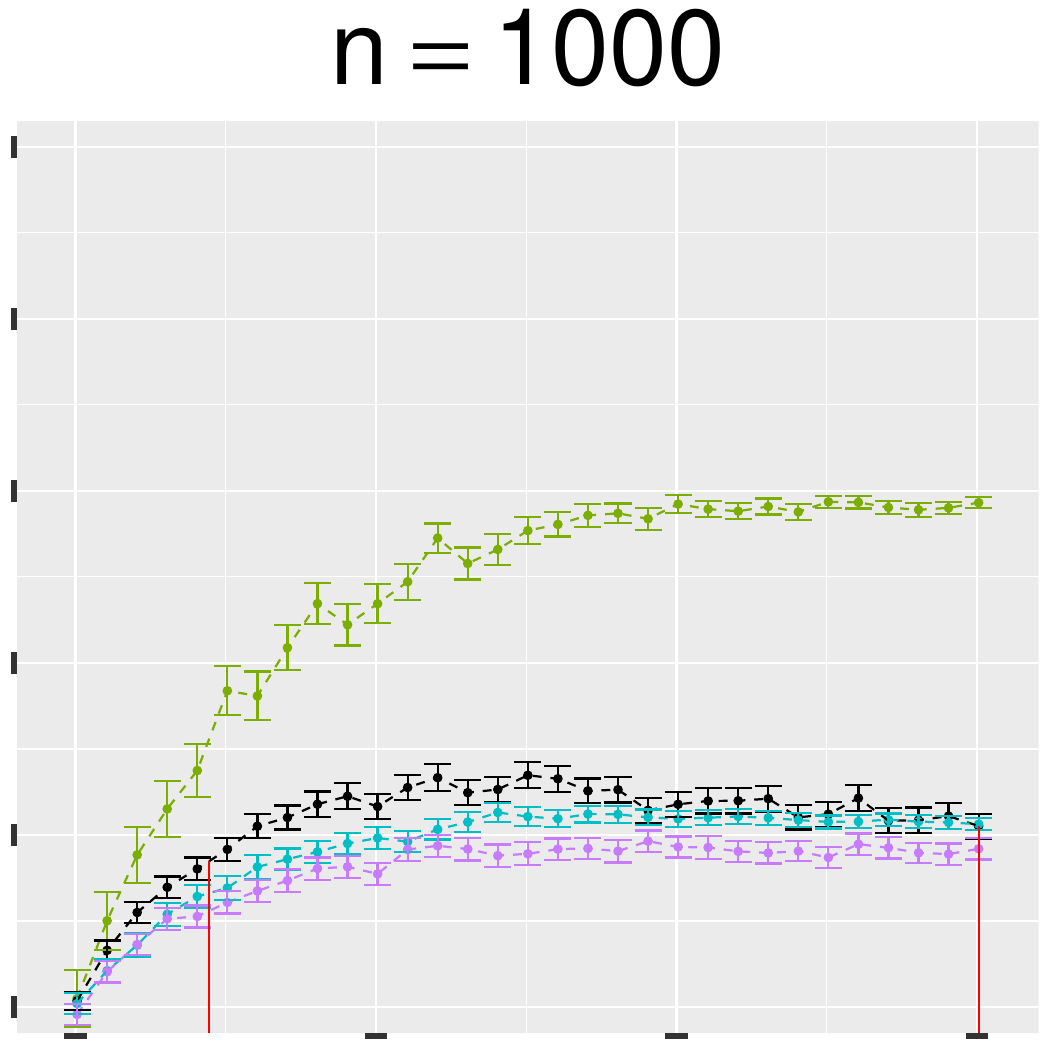}}
    \subfloat{\includegraphics[width=0.145\linewidth]{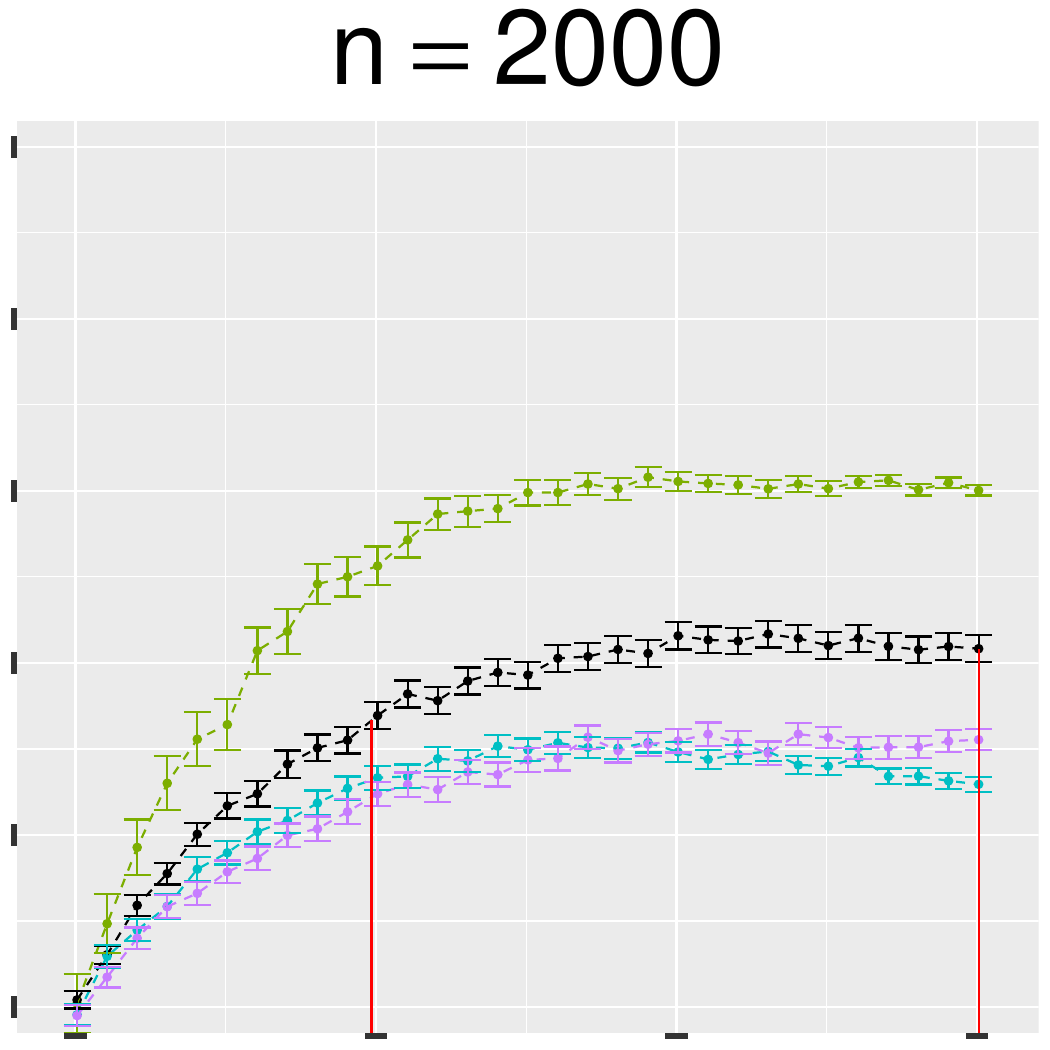}}\\\vspace{-6pt}
    % epsilon=1
    \subfloat{\includegraphics[width=0.18\linewidth]{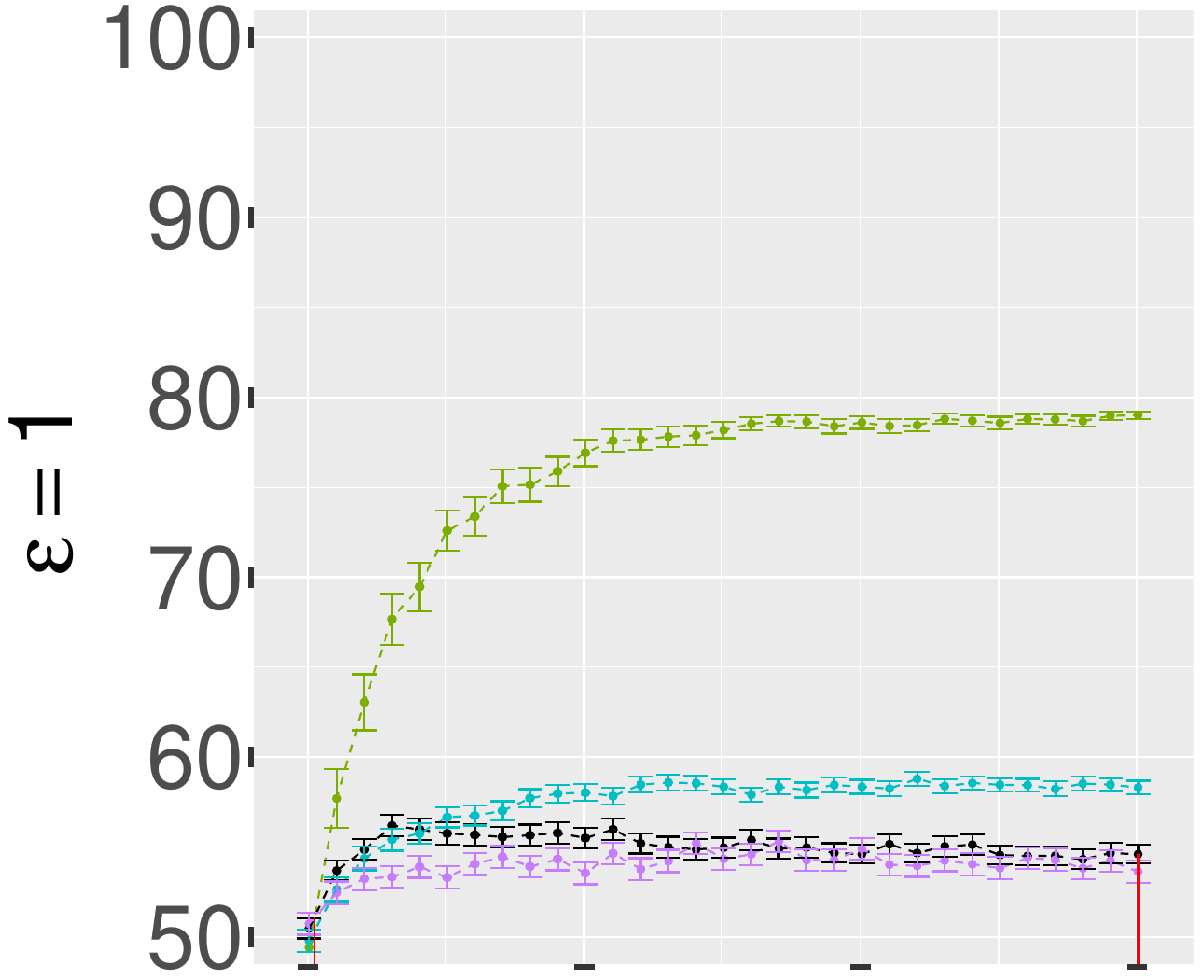}}
    \subfloat{\includegraphics[width=0.145\linewidth]{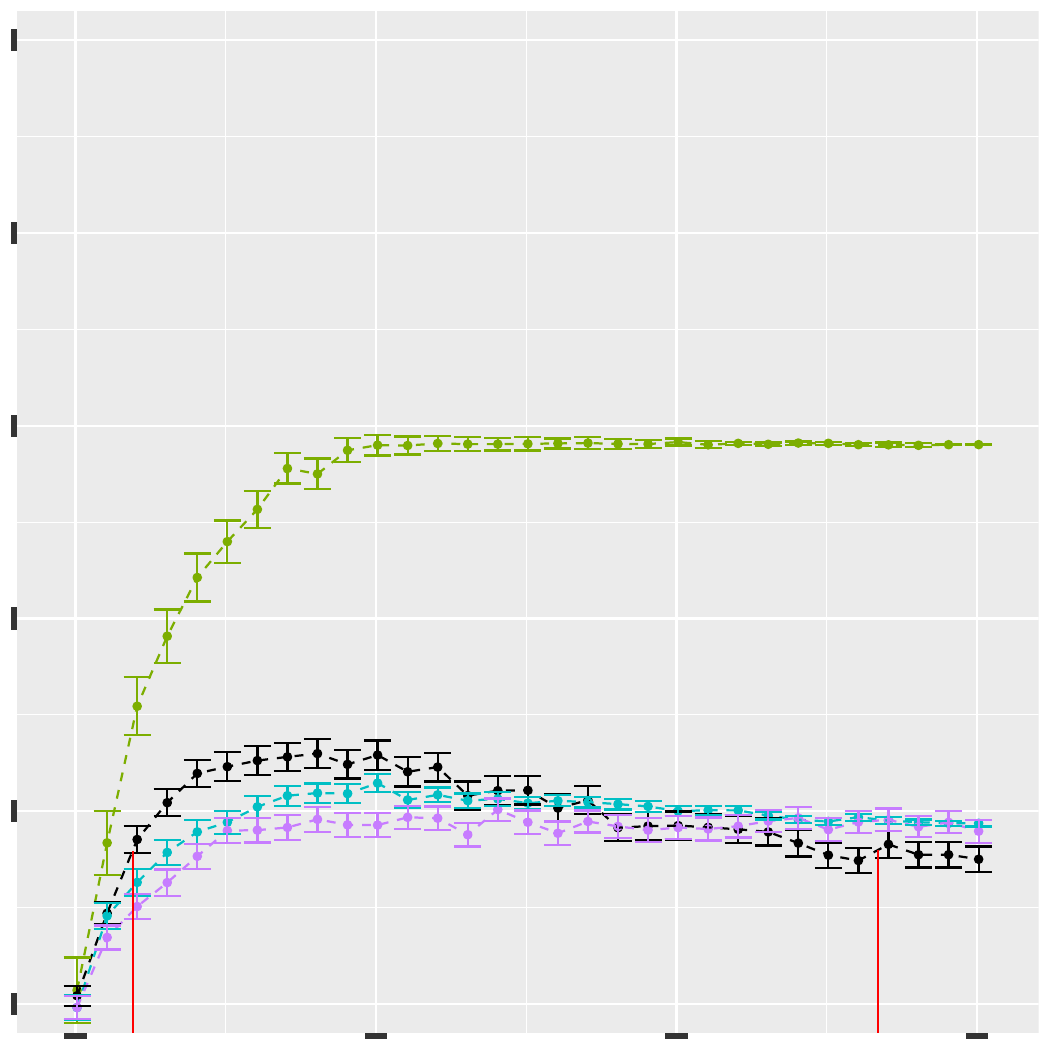}}
    \subfloat{\includegraphics[width=0.145\linewidth]{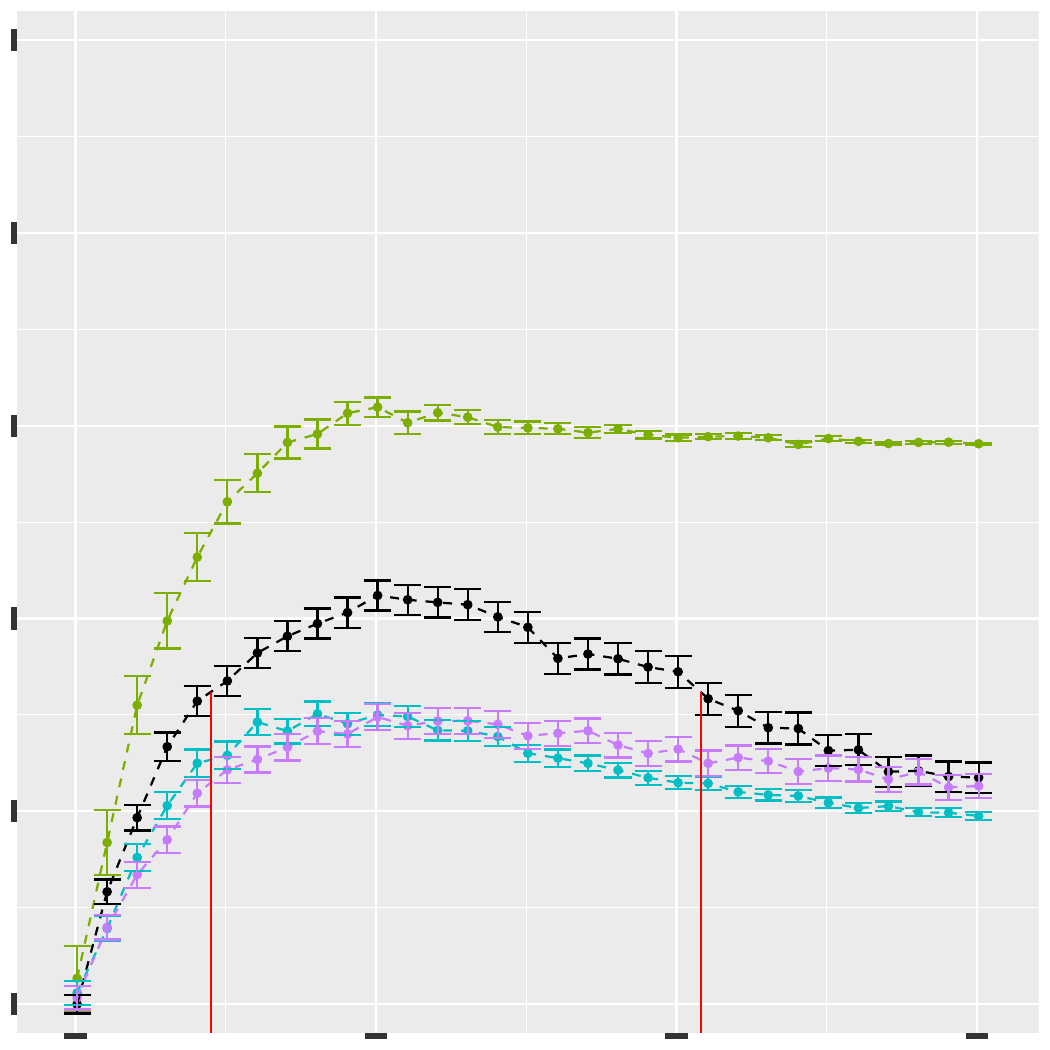}}
    \subfloat{\includegraphics[width=0.145\linewidth]{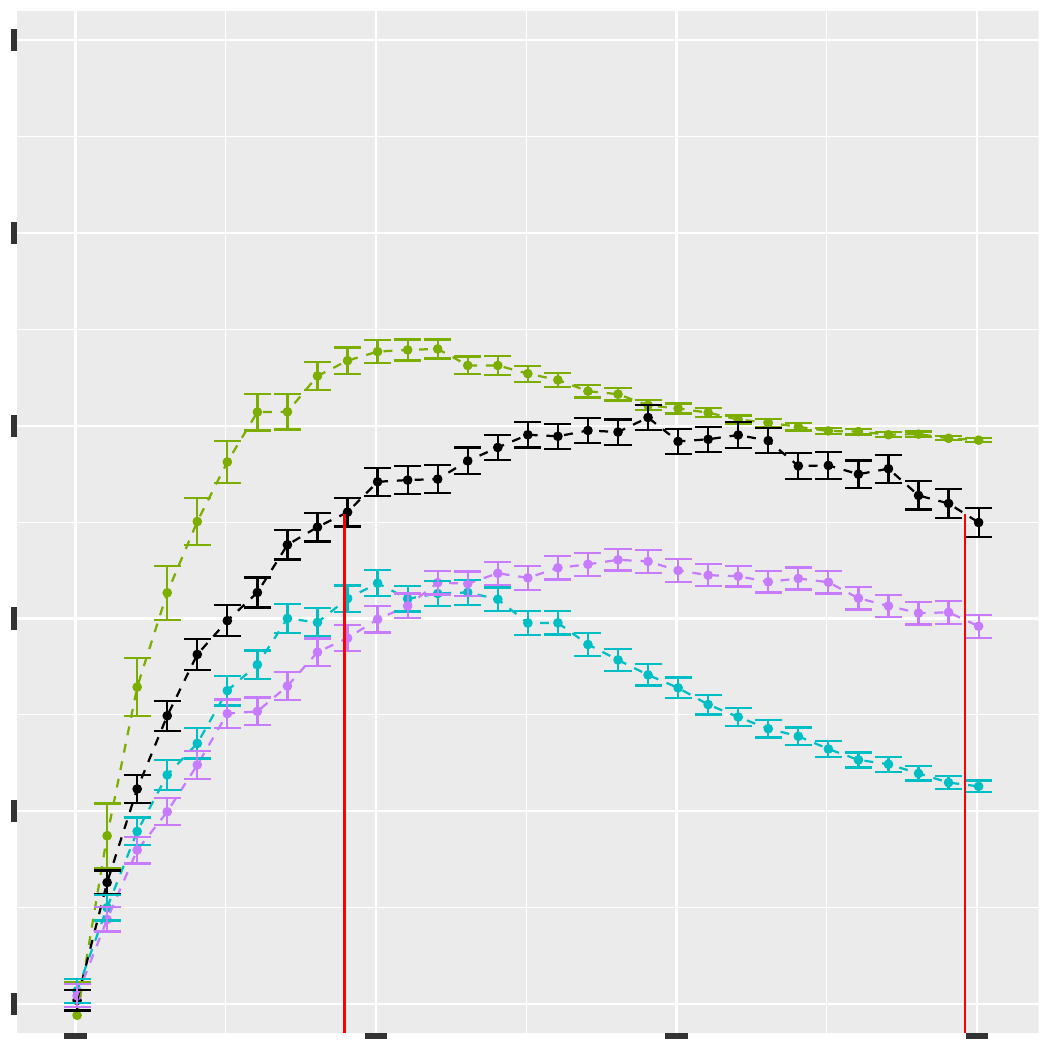}}\\\vspace{-6pt}
    % epsilon=5
    \subfloat{\includegraphics[width=0.18\linewidth]{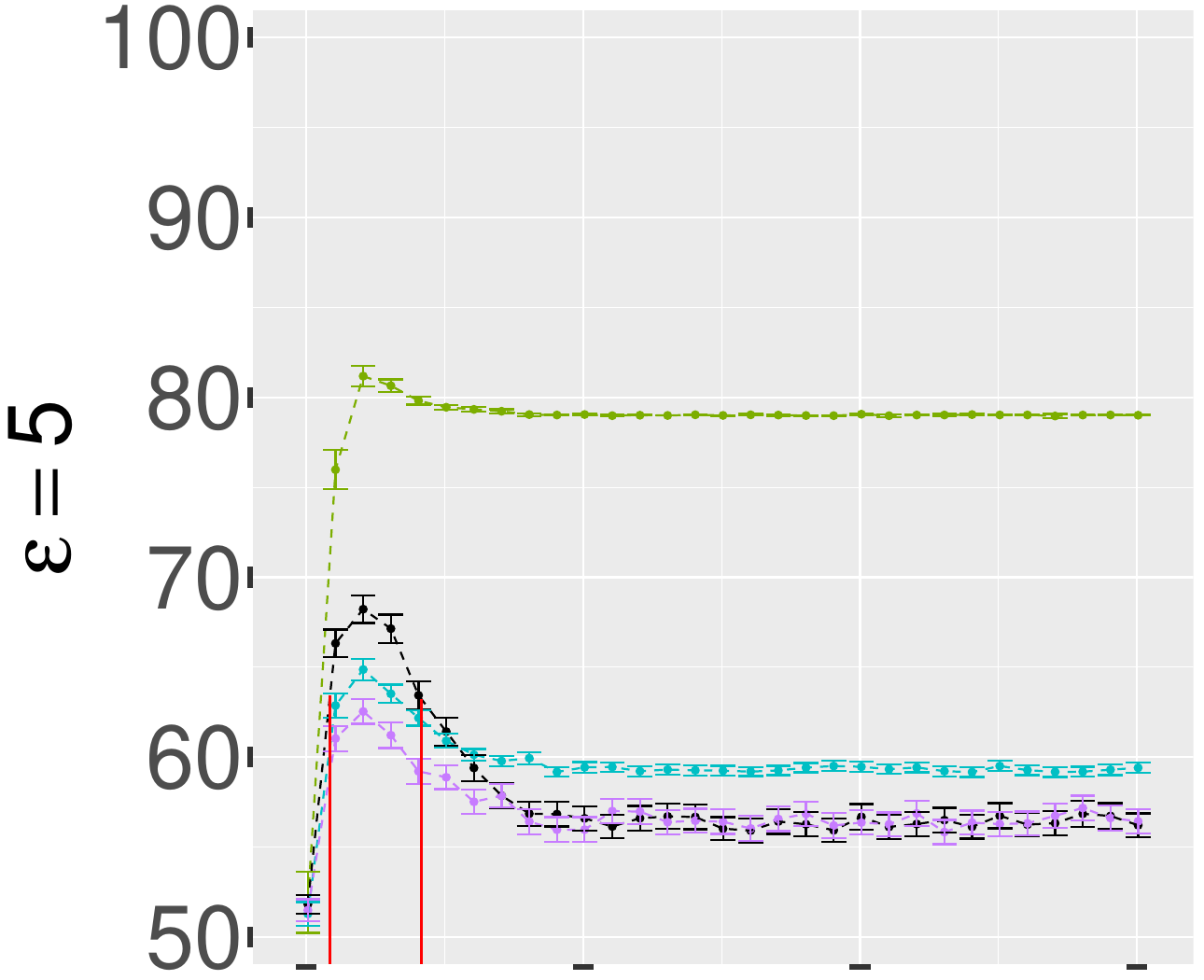}}
    \subfloat{\includegraphics[width=0.145\linewidth]{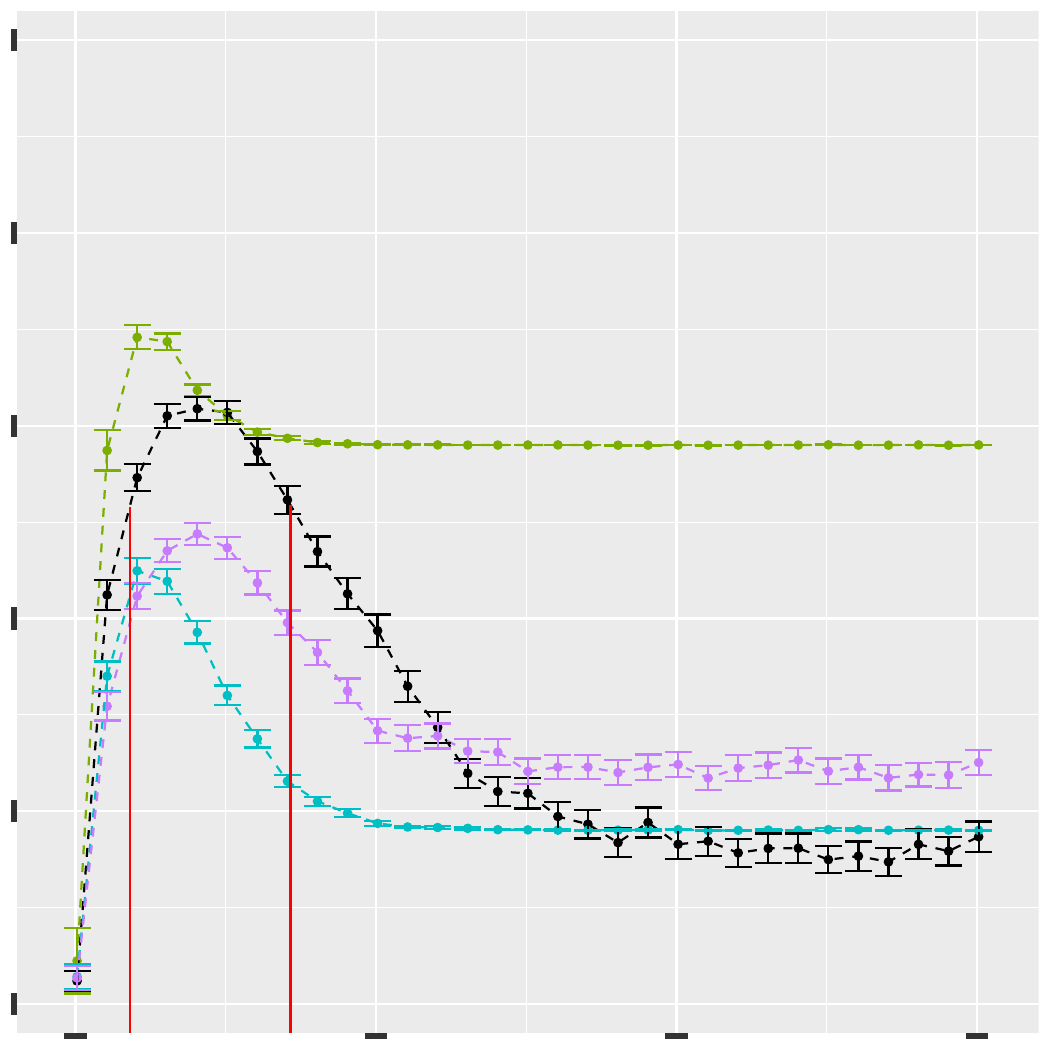}}
    \subfloat{\includegraphics[width=0.145\linewidth]{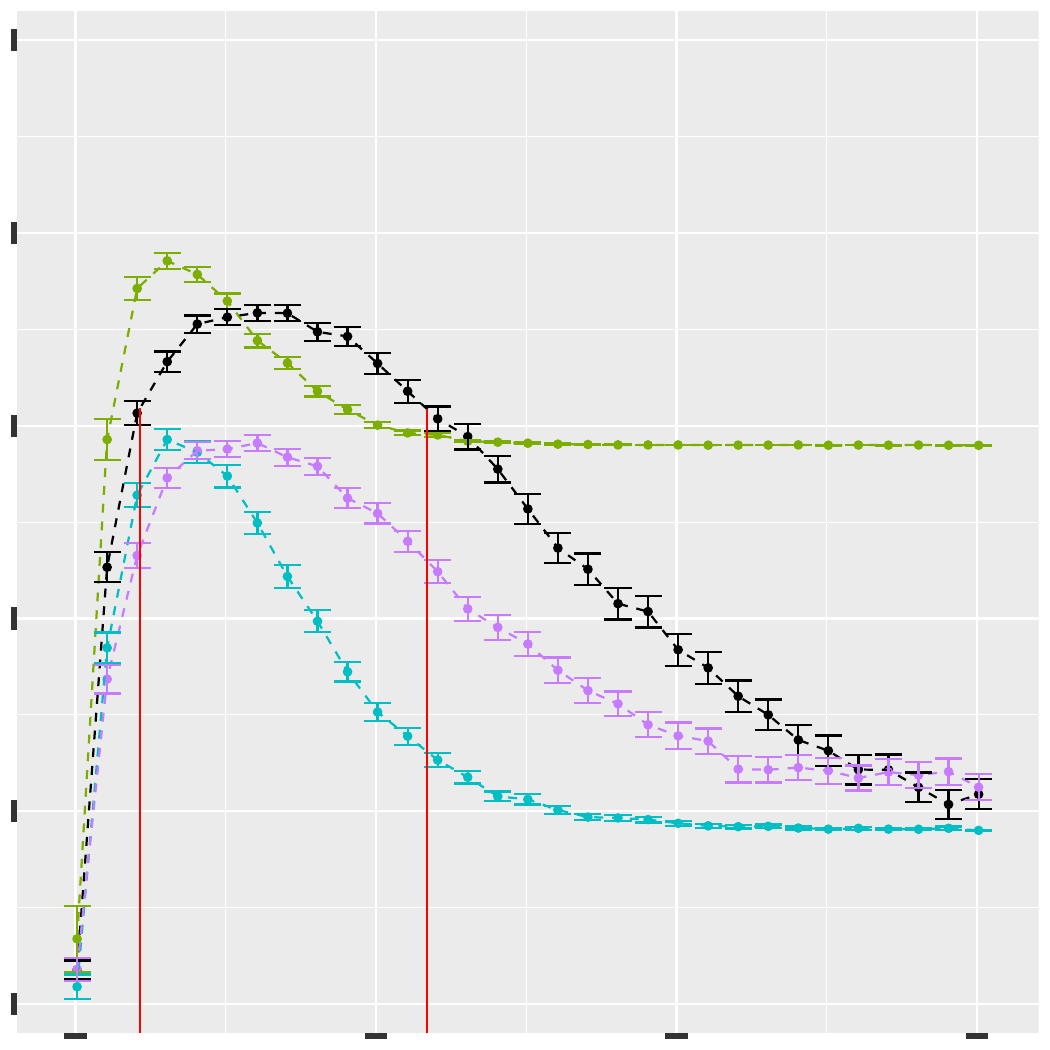}}
    \subfloat{\includegraphics[width=0.145\linewidth]{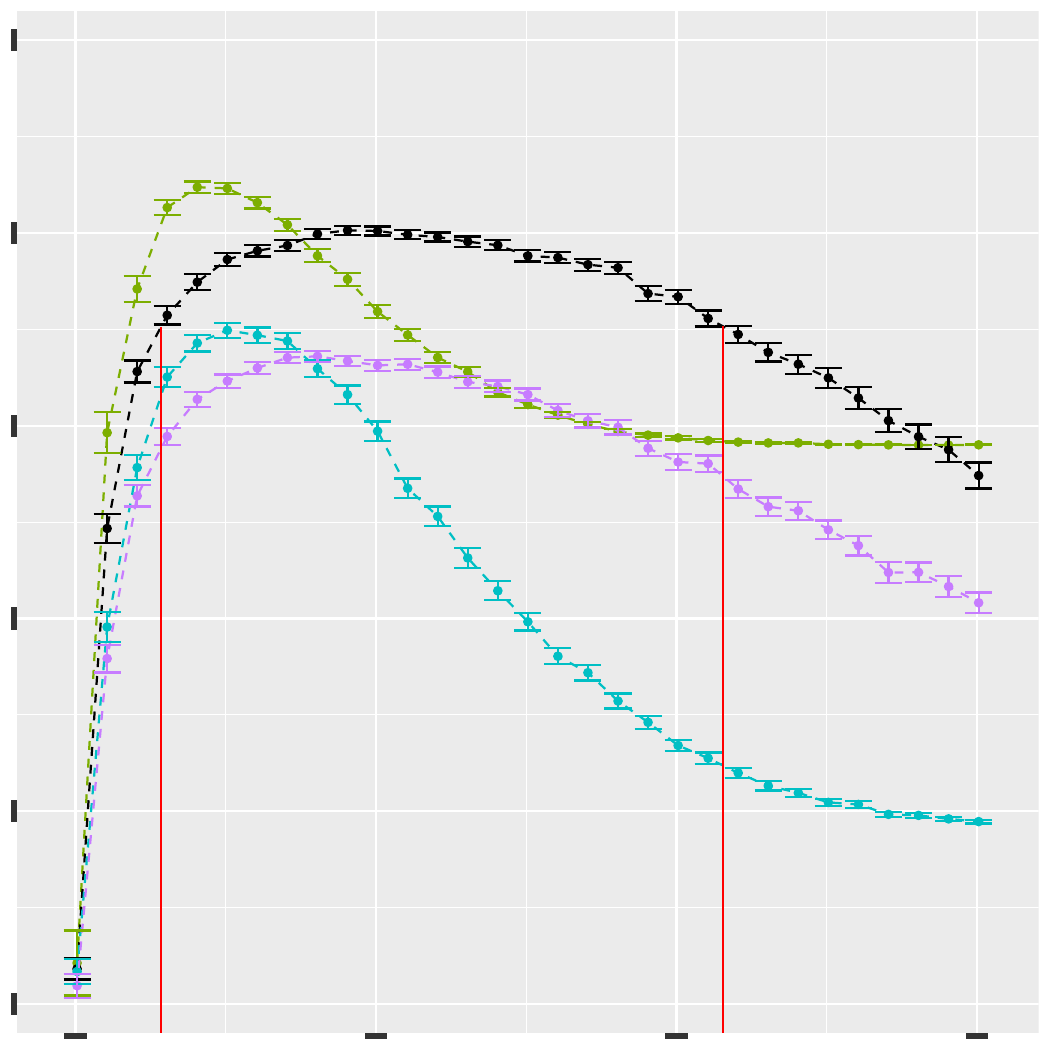}}\\\vspace{-6pt}
    % epsilon=150
    \subfloat{\includegraphics[width=0.18\linewidth]{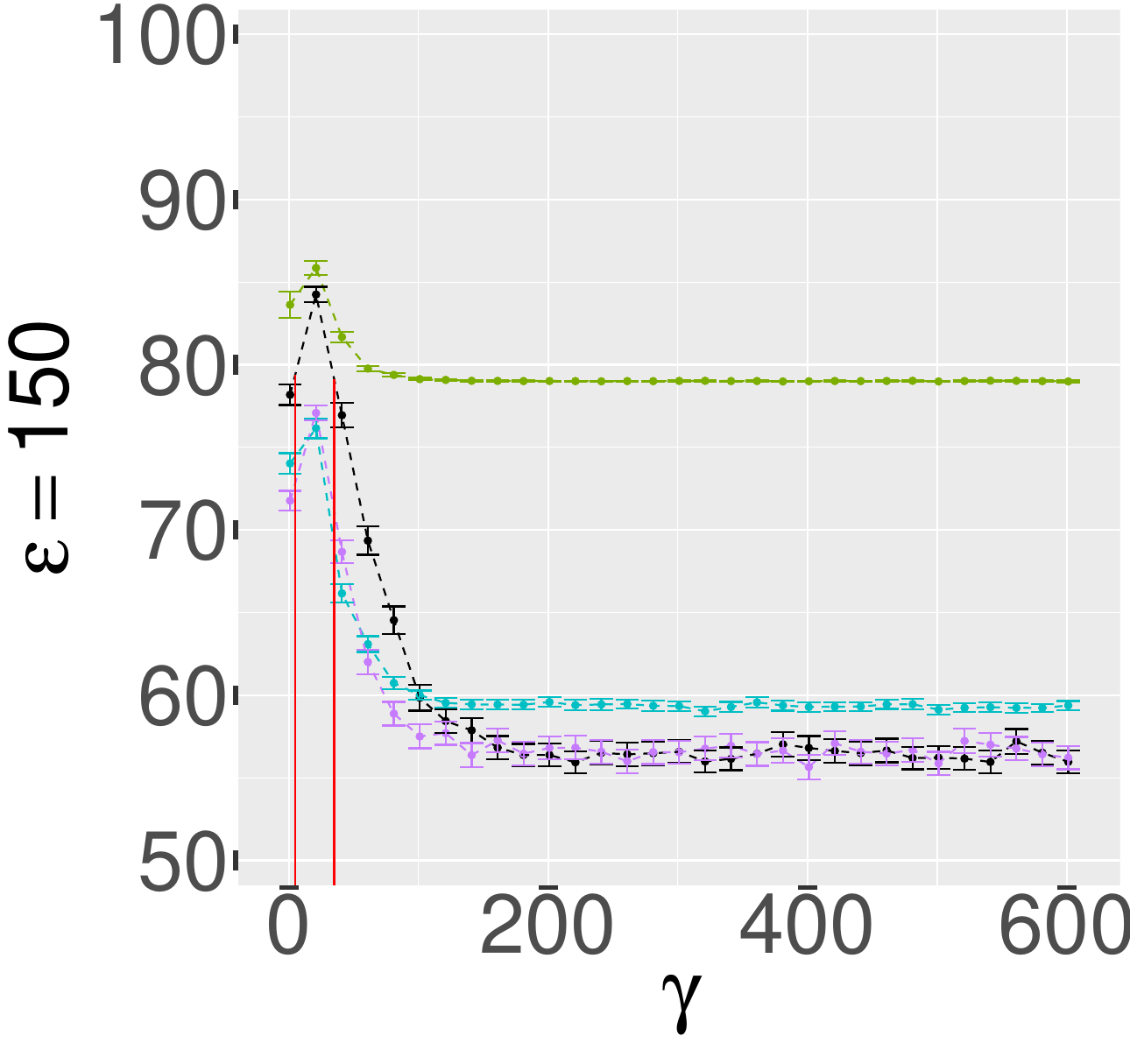}}
    \subfloat{\includegraphics[width=0.145\linewidth]{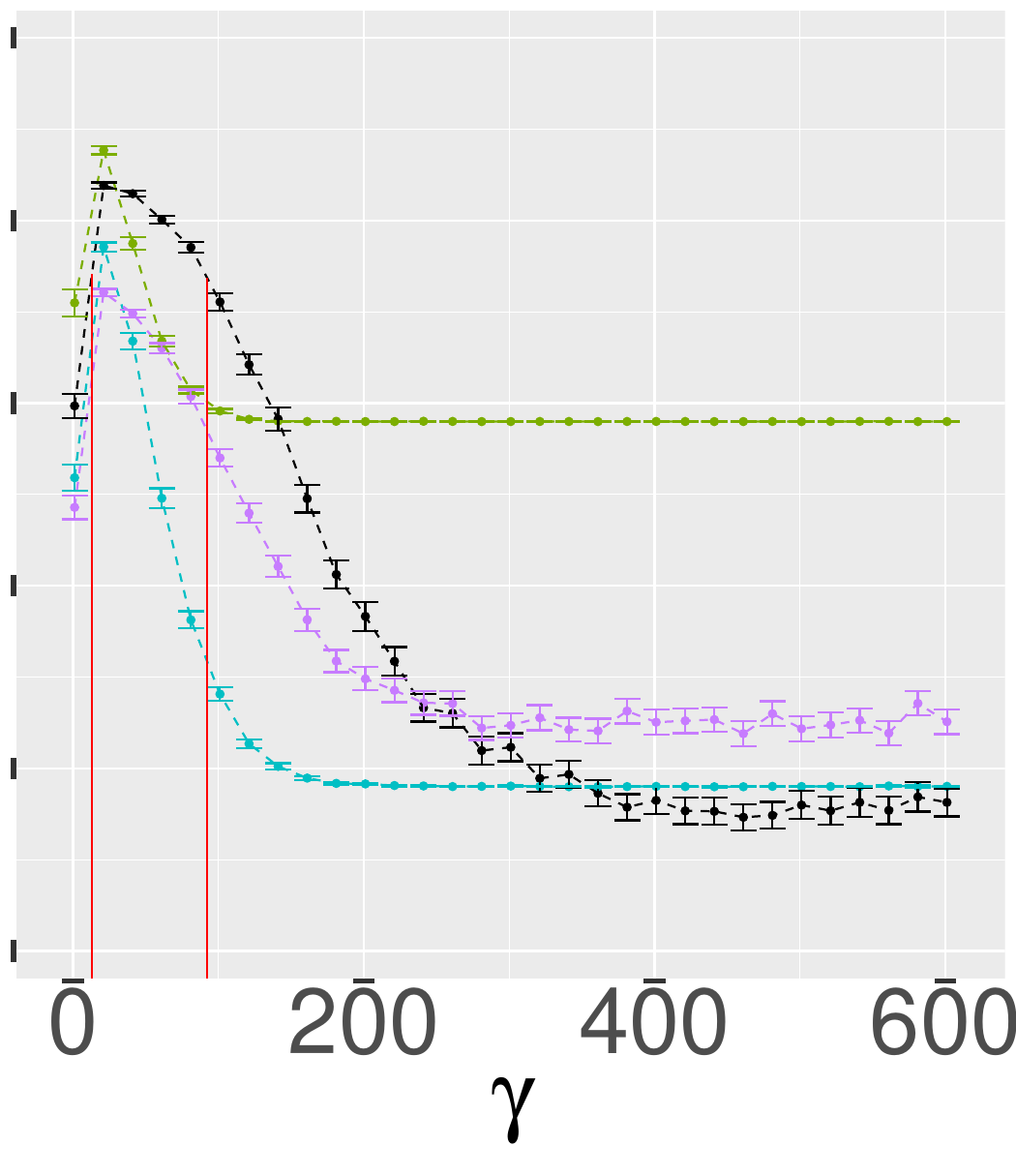}}
    \subfloat{\includegraphics[width=0.145\linewidth]{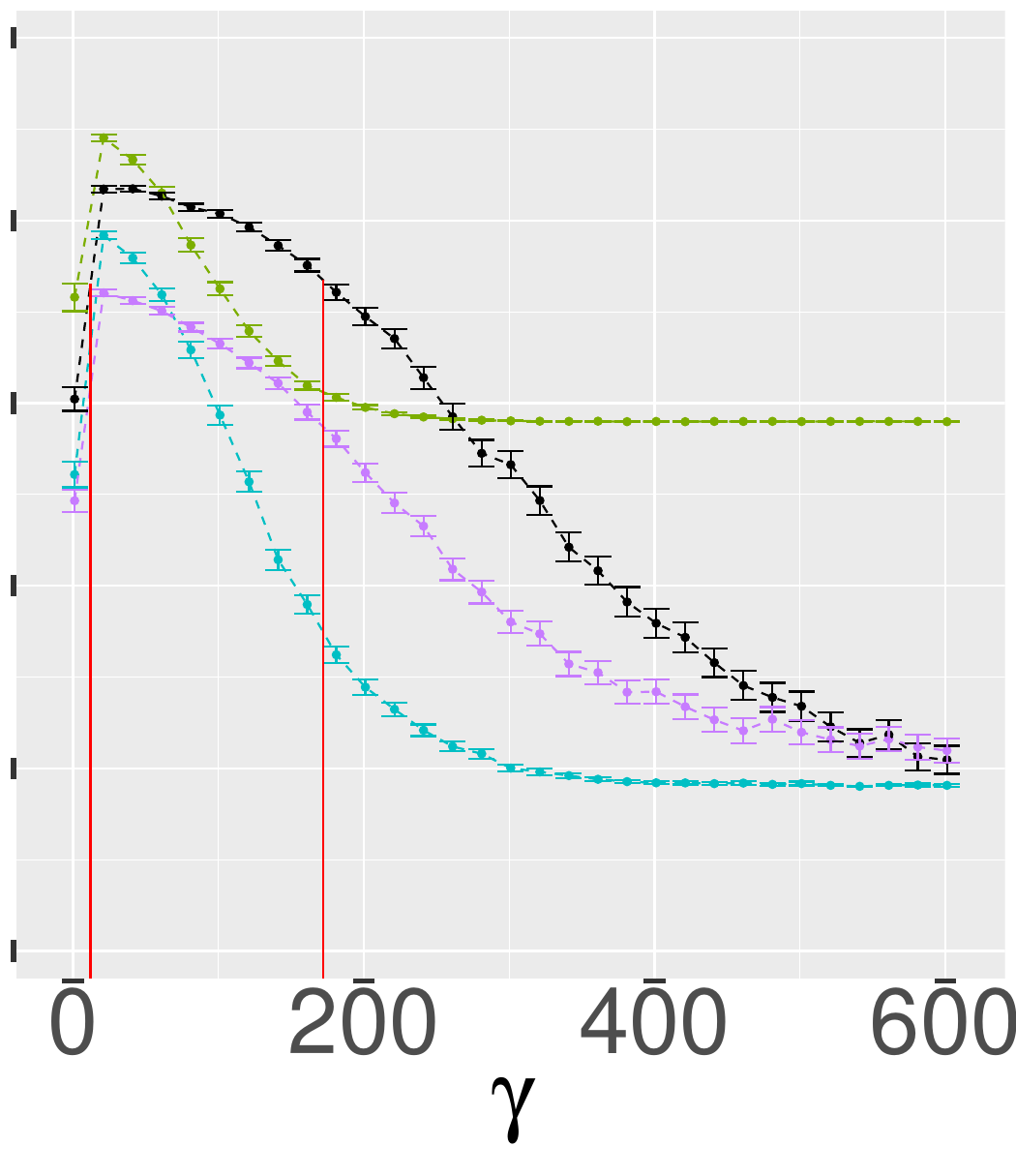}}
    \subfloat{\includegraphics[width=0.145\linewidth]{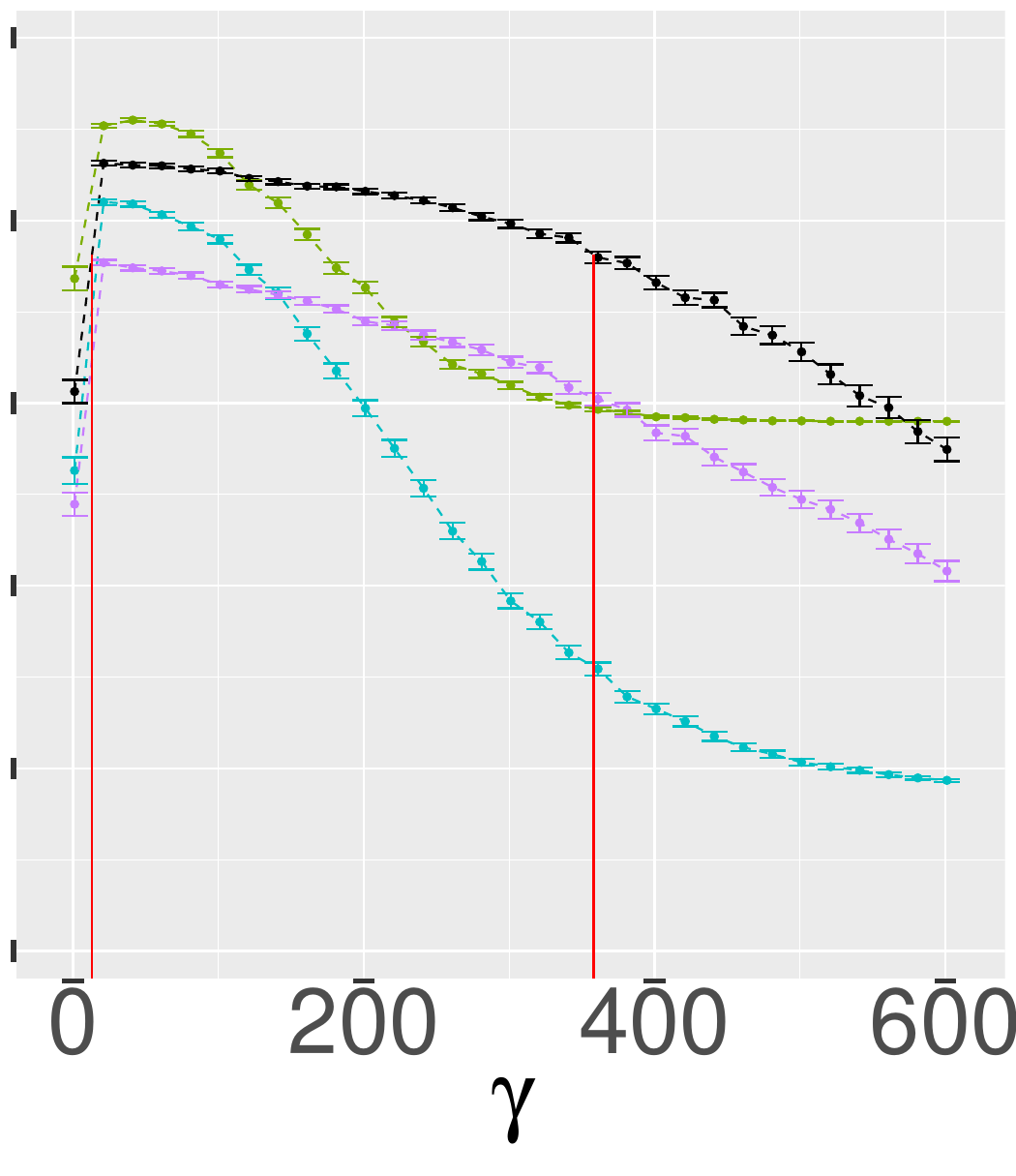}}\\\vspace{-4pt}
    \footnotesize{(b)   Different lines represent different numbers of non-zero entries in $\boldsymbol{\theta}$ in the optimal treatment function in the independent dataset.}\vspace{-6pt}   
    \caption{Average (95\% CI) validation set accuracy as a function of the regularization constant $\gamma$  for various combinations of $\varepsilon$ and sensitive dataset size $n$.
    The red lines mark the region of $\gamma$ where the corresponding accuracy in the baseline case ($p=4$) is within $5\%$ of the highest.}
    \label{fig:suppl_materials_P}
\end{figure*}

%%%%%%%%%%%%%%%%%%%%%%%%%%%%%%%%%%%%%%%%%%%%%%%%%%%%%%%%%%%%
 \newpage
\section{Numerical results in the synthetic data experiment}
\vspace{-12pt}
\begin{table}[!htp]
   \caption{Results for a subset of combinations of $\epsilon$ and $n$. The case without DP (corresponding to $\epsilon=\infty$) is also shown for reference.}
    \label{tab:sim_results}\vspace{-12pt}
    \begin{center}
\resizebox{0.95\linewidth}{!}{
    \begin{tabular}{c|c|c|c|c|cc}
    \hline
    \multicolumn{2}{c}{\textbf{}} & \multicolumn{2}{|c|}{\textbf{Accuracy (\%)}} & \multicolumn{2}{c}{\textbf{Treatment Value}} \\
    \hline
    $\epsilon$ & $n$ & Mean & 95\% CI & Mean & 95\% CI \\
    \hline
    \multirow{4}{*}{0.1} & 200 & 50.24 & (49.06, 51.42) & 8.36 & (8.26, 8.46)\\
    & 500 & 50.20 & (48.89, 51.52) & 8.42 & (8.32, 8.52)\\
    & 800 & 50.70 & (49.45, 51.94) & 8.40 & (8.30, 8.50)\\
    & 1000 & 51.08 & (49.72, 52.44) & 8.41 & (8.32, 8.51)\\
    \hline
    \multirow{4}{*}{0.5} & 200 & 52.87 & (51.49, 54.25) & 8.56 & (8.46, 8.66)\\

    & 500 & 58.02 & (56.58, 59.47) & 8.84 & (8.73, 8.95)\\

    & 800 & 60.20 & (58.61, 61.79) & 9.04 & (8.92, 9.16)\\

    & 1000 & 60.95 & (59.22, 62.68) & 9.05 & (8.93, 9.16)\\
    \hline
    \multirow{4}{*}{1} & 200 & 55.53 & (54.23, 56.83) & 8.77 & (8.67, 8.86)\\

    & 500 & 61.03 & (59.36, 62.70) & 9.07 & (8.96, 9.18)\\

    & 800 & 64.08 & (62.22, 65.93) & 9.25 & (9.13, 9.37)\\

    & 1000 & 64.70 & (62.93, 66.48) & 9.28 & (9.17, 9.40)\\
    \hline
    \multirow{4}{*}{2} & 200 & 58.99 & (57.30, 60.68) & 8.90 & (8.79, 9.02)\\

    & 500 & 68.17 & (66.27, 70.08) & 9.49 & (9.38, 9.61)\\

    & 800 & 74.32 & (72.50, 76.14) & 9.91 & (9.80, 10.02)\\

    & 1000 & 79.03 & (77.42, 80.63) & 10.12 & (10.03, 10.21)\\
    \hline
    \multirow{4}{*}{5} & 200 & 67.66 & (65.88, 69.43) & 9.51 & (9.40, 9.63)\\

    & 500 & 78.79 & (77.29, 80.29) & 10.13 & (10.04, 10.21)\\

    & 800 & 81.97 & (80.57, 83.37) & 10.23 & (10.15, 10.31)\\

    & 1000 & 84.80 & (83.62, 85.99) & 10.31 & (10.24, 10.38)\\
    \hline
    \multirow{4}{*}{$\infty$} & 200 & 88.65 & (88.02, 89.27) & 10.41 & (10.36, 10.47)\\

    & 500 & 92.18 & (91.75, 92.60) & 10.52 & (10.47, 10.57)\\

    & 800 & 93.39 & (93.02, 93.75) & 10.54 & (10.49, 10.60)\\

    & 1000 & 94.29 & (93.97, 94.61) & 10.62 & (10.57, 10.68)\\
    \hline
    \end{tabular}}
    \end{center}
\end{table}

\clearpage
\bibliographystyle{IEEEtran}
\bibliography{references}